\documentclass[accepted]{uai2026} 
                        
\usepackage[american]{babel}

\usepackage{natbib} 

\usepackage{mathtools} 
\usepackage{amsthm}
\usepackage{amssymb}
\usepackage{algorithm}
\usepackage{algorithmic}
\usepackage[dvipsnames]{xcolor}
\usepackage{pdflscape}
\usepackage{booktabs} 
\usepackage{subcaption} 
\usepackage{graphicx}
\usepackage{hyperref}
\usepackage[capitalize,noabbrev]{cleveref}
\usepackage{microtype}

\usepackage{xcolor}
\usepackage{soul}

\theoremstyle{plain}
\newtheorem{theorem}{Theorem}[section]
\newtheorem{proposition}[theorem]{Proposition}

\theoremstyle{definition}

\theoremstyle{remark}

\author[1]{Simon D. Nguyen}
\author[1]{Troy Russo}
\author[1]{Kentaro Hoffman}
\author[1,2]{Tyler H. McCormick}
\affil[1]{Department of Statistics, University of Washington}
\affil[2]{Department of Sociology, University of Washington}
  
\begin{document}
\title{Adaptive Active Learning for Regression via Reinforcement Learning}
\maketitle
\begin{abstract}
Active learning for regression reduces labeling costs by selecting the most informative samples. Improved Greedy Sampling is a prominent method that balances feature-space diversity and output-space uncertainty using a static, multiplicative rule. We propose Weighted improved Greedy Sampling (WiGS), which replaces this framework with a dynamic, additive criterion. We formulate weight selection as a reinforcement learning problem, enabling an agent to adapt the exploration-investigation balance throughout learning. Experiments on 18 benchmark datasets and a synthetic environment show WiGS outperforms iGS and other baseline methods in both accuracy and labeling efficiency, particularly in domains with irregular data density where the baseline's multiplicative rule ignores high-error samples in dense regions. 
\end{abstract}

\section{Introduction}
\label{sec:Intro}

Acquiring labeled data is a persistent bottleneck across supervised machine learning applications. In regression tasks specifically, which are prevalent in domains like robotics \citep{ALR_Robotics}, drug discovery \citep{ALR_Drugs}, environmental science \cite{ALR_Environment}, and materials science \citep{AL_Regression_Benchmark}, active learning (AL) mitigates this cost by strategically querying samples that maximize information gain, thereby achieving high predictive performance while minimizing the total labeling cost. In active learning for regression, effective strategies must navigate the trade-off between representativeness and informativeness \citep{Huang2010}. In this work, we frame this as a balance between \textit{exploration} (querying sparse regions to comprehensively span the input feature space) and \textit{investigation} (querying high-uncertainty regions to refine the decision boundary).

Current state-of-the-art heuristics, such as the Greedy Sampling (GS) family proposed by \citet{iGS}, operationalize this trade-off using metrics for feature-space diversity and output-space uncertainty. However, the standard improved Greedy Sampling (iGS) baseline combines these signals using a static, multiplicative rule. We identify that this fixed balance is sub-optimal in heterogeneous domains: the multiplicative criterion suppresses necessary investigation in dense feature regions, effectively ``vetoing'' high-error samples simply because they lack feature novelty. We argue that the optimal balance is not fixed, but dynamic, requiring adaptation to the evolving state of the learner. Indeed, our empirical analysis of the learned policies reveals that the agent never converges to a single static weight, but instead maintains a high-variance, reactive strategy throughout the learning process (see Appendix \ref{subsec:AnalysisOfLearnedAgentPolicy}). Furthermore, even in scenarios where a specific static balance suffices, identifying this optimal scalar \textit{a priori} is impossible without expensive, hindsight-based validation so an autonomous method to find this optimal scalar is superior.

To achieve this, we introduce \textit{Weighted improved Greedy Sampling (WiGS)}, a flexible framework that recasts the selection criterion as an additive combination of diversity and uncertainty. We propose two distinct mechanisms for varying this balance: (1) \textit{dynamic} strategies, which follow a fixed temporal schedule independent of the data; and (2) \textit{adaptive} strategies, which actively learn to adjust the weight based on model feedback. Focusing on the latter, we build on the “learning to active learn” paradigm \citep{LearningToActiveLearning} to formulate the weighting strategy as a reinforcement learning (RL) problem. Unlike predecessors that learn complex policies from raw data for classification tasks \citep{Woodward2017, Liu2019}, our agent manages a control signal defined as the scalar weight balancing exploration and investigation. This reduces the agent's role to a strategic decision-maker that determines the optimal weight to balance these two scores. By modulating the exploration-investigation trade-off based on the evolving data state, the agent effectively corrects the failure modes of static heuristics or autonomously converges to an optimal static configuration.

Our contributions are: (1) we propose the flexible WiGS framework; (2) we formulate the dynamic balancing of AL objectives as a continuous-control RL problem; and (3) we demonstrate through extensive simulations across 20 benchmark datasets that this adaptive approach outperforms static baselines in both accuracy and label efficiency.\footnote{To ensure reproducibility, our codebase is available at\\ \href{https://github.com/thatswhatsimonsaid/WeightedGreedySampling}{https://github.com/thatswhatsimonsaid/WeightedGreedySampling}.}

\section{Greedy Sampling}
\label{sec:Background}


We focus on the pool-based active learning setting, where the algorithm iteratively queries a single sample from a large, fixed collection of unlabeled data \citep{Settles2012}. Let the input space be $\mathcal{X}$ and the output space be $\mathcal{Y}$. At any iteration, we have a labeled training set $D_{tr} = \{(\mathbf{x}_i, y_i)\}_{i=1}^{k}$ and a large unlabeled candidate pool $D_{cdd} = \{\mathbf{x}_j\}_{j=k+1}^{N}$. A regression model $f: \mathcal{X} \to \mathcal{Y}$ is trained on $D_{tr}$. The goal is to select the sample $\mathbf{x}^* \in D_{cdd}$ that, when labeled, maximally improves the model's predictive performance, defined practically as the reduction in generalization error across the entire domain. \citet{iGS} provide a set of powerful, intuitive heuristics for this task which are henceforth referred to as the Greedy Sampling (GS) family of selectors.

\subsection{Exploration (GS\MakeLowercase{x}) and Investigation (GS\MakeLowercase{y})}
\label{subsec:GSxGSY}

Greedy Sampling frames the exploration-investigation trade-off using a ``furthest nearest neighbor'' logic: selecting candidates that are maximally distant from the nearest labeled point in either feature or output space. For any candidate $\mathbf{x}_n \in D_{cdd}$ and labeled sample $(\mathbf{x}_m, y_m) \in D_{tr}$, we define the pairwise distances: 
\begin{equation}
    d_{nm}^x \equiv ||\mathbf{x}_n - \mathbf{x}_m|| \text{ and } d_{nm}^y \equiv |f(\mathbf{x}_n) - y_m|
\end{equation}

\textbf{Greedy Sampling on Features (GSx)} targets diversity in $\mathcal{X}$ (exploration). GSx is model-agnostic, selecting the candidate with the maximum distance to its nearest labeled neighbor in the input space:
\begin{equation}
    \mathbf{x}^*_{GSx} = \underset{\mathbf{x}_n \in D_{cdd}}{\arg\max} ~d_n^x, \quad \text{where} \quad d_n^x \equiv \min_{m} d_{nm}^x
\end{equation}

\textbf{Greedy Sampling on the Output (GSy)} targets diversity in $\mathcal{Y}$ (investigation). GSy utilizes the current model $f(\cdot)$ to select the candidate with the maximum prediction distance to the nearest known label:
\begin{equation}
    \mathbf{x}^*_{GSy} = \underset{\mathbf{x}_n \in D_{cdd}}{\arg\max} ~d_n^y, \quad \text{where} \quad d_n^y \equiv \min_{m} d_{nm}^y
\end{equation}

\subsection{Improved Greedy Sampling (\MakeLowercase{i}GS)}
\label{subsec:iGS}
As GSx focuses purely on feature space diversity and GSy on predicted output space diversity, both have limitations. GSx ignores the predictive regression model, while GSy can be unreliable if the initial model is poor. \citep{iGS} propose the improved Greedy Sampling
(iGS) approach to create a balanced strategy by combining both. The final score for a candidate $\mathbf{x}_n$ is the minimum of the products of its pairwise distances to each labeled point,
and the selection criterion is to maximize this value:
\begin{equation} \label{eq:igs}
\mathbf{x}^*_{iGS} = \underset{\mathbf{x}_n \in D_{cdd}}{\arg\max} ~s_n^{iGS}, \ \text{s.t.} \ s_n^{iGS} = \min_{m} (d_{nm}^x \cdot d_{nm}^y)
\end{equation}
This multiplicative rule ensures that a sample must score high in both exploration and investigation to be selected, preventing the selection of outliers (high diversity, low uncertainty) or redundant points (low diversity). However, as we argue in Section \ref{sec:WiGS}, this strict multiplicative requirement can become overly conservative in feature spaces with heterogeneous data densities.

\section{Weighted Improved Greedy Sampling (W\MakeLowercase{i}GS)}
\label{sec:WiGS}

The static, multiplicative selection criterion of the iGS baseline (Equation \ref{eq:igs}) assumes a fixed relationship between the importance of feature space exploration and predicted output space investigation. We contend that this balance is not fixed but should instead be dynamic, adapting to the specific characteristics of the dataset and the current stage of the active learning procedure. To address this, the \textbf{Weighted improved Greedy Sampling (WiGS)} framework recasts the selection criterion as a flexible, additive combination of scores. This formulation allows for explicit and dynamic control over the exploration-investigation trade-off.

\subsection{The W\MakeLowercase{i}GS Framework and Score}
\label{subsec:WiGS_FrameworkScore}

To enable dynamic balancing (which here refers to temporal changes independent of data), WiGS recasts the selection criterion as a flexible, weighted additive function. First, to ensure the input-space exploration $d_{nm}^x$ and output-space investigation $d_{nm}^y$ metrics are comparable, we apply a normalization function $\phi(\cdot)$ to ensure the raw distances have comparable magnitudes.

WiGS computes a weighted additive distance between each candidate and every point in the labeled set using a dynamic weight, $w_x^{(t)} \in [0, 1]$. The final score for a candidate $\mathbf{x}_n$ is the minimum of these combined scores taken over all labeled points:
\begin{equation} \label{eq:wigs_score}
    s_n^{WiGS} = \min_{m} \left( w_x^{(t)} \phi(d_{nm}^x) + (1-w_x^{(t)}) \phi(d_{nm}^y) \right)
\end{equation}
The candidate observation with the highest score is selected for labeling: $\mathbf{x}^*_{WiGS} = \underset{\mathbf{x}_n \in D_{cdd}}{\arg\max} ~ s_n^{WiGS}$. The primary challenge lies in devising an optimal strategy for dynamically updating $w_x^{(t)}$ throughout the active learning process.

\subsection{Theoretical Analysis: Density Veto}
\label{sec:TheoreticalProof}

To motivate our additive framework, we identify a critical failure mode of the multiplicative iGS criterion, which we term the \textit{density veto}. For this analysis, let $d \in [0,1]$ and $u \in [0,1]$ denote the normalized feature exploration (diversity) and predictive investigation (uncertainty) scores, respectively.

In regions of high feature density ($d \to 0$), the multiplicative score $S_{iGS} = d \cdot u$ is suppressed to near-zero regardless of the magnitude of $u$. This effectively ``vetoes'' the selection of high-uncertainty points simply because they reside in dense regions. We prove that in these regimes, the multiplicative selector is mathematically incapable of prioritizing uncertainty, whereas the additive selector retains this capability.

\begin{proposition}[The Density Veto]
Consider a candidate pool containing a ``target'' $x^*$ with high uncertainty $u^*$ and low diversity $d^*$ (high density), and a ``distractor'' $x'$ with lower uncertainty $u' < u^*$ but moderate diversity $d' > d^*$.

As the feature density around $x^*$ increases ($d^* \to 0$), there exists a threshold $\delta > 0$ such that if $d^* < \delta$, the multiplicative selector strictly prefers the sub-optimal distractor $x'$ ($S_{iGS}(x') > S_{iGS}(x^*)$), regardless of the uncertainty gap. In contrast, there always exists a weight $w \in [0, 1]$ such that the additive selector prefers the high-uncertainty target $x^*$.
\end{proposition}

The full proof is provided in Appendix \ref{sec:proofs}.

\subsection{Weighting Strategies: Static and Time-Decay}
\label{subsec:TimeDecayingHeuristics}

As a starting point, we explore simple, non-adaptive strategies for setting the weight $w_x^{(t)}$.

\textbf{Static Weights:} The most straightforward approach is to set $w_x^{(t)}$ to a fixed constant throughout the learning process (i.e., $w_x^{(t)} = w_x$). This reduces WiGS to a simple hyperparameter and is useful when we know the true relationship between informativity from the feature space and the output space. While the optimal balance is unknown in practice, this heuristic allows us to test whether a different fixed additive balance (e.g., $w_x = 0.25$ or $w_x = 0.75$) can outperform the implicit, multiplicative balance of the iGS method.

\textbf{Time-Decay Weights:} We can incorporate a simple dynamic heuristic by making the weight a function of the iteration number. We hypothesize that exploration is more critical at the beginning of the process when the model is poorly specified. As more data is collected, the model becomes more reliable, and investigation becomes more fruitful. We implement two common decay functions:
\begin{itemize}[leftmargin=*, noitemsep, topsep=0pt]
    \item \textbf{Linear Decay:} $w_x^{(t)} = 1 - (c\cdot t/T)$, where $T$ is the total number of AL iterations.
    \item \textbf{Exponential Decay:} $w_x^{(t)} = \exp(-c \cdot t/T)$, where $c$ is a decay-rate constant.
\end{itemize}

\subsection{Adaptive Weighting via Reinforcement Learning}
\label{subsubsec:WiGS_RL}

To move beyond data-agnostic heuristics onto adaptive methods, we formulate weight selection as a reinforcement learning (RL) problem. We adopt the ``learning to active learn'' paradigm \citep{LearningToActiveLearning}. However, whereas \citet{LearningToActiveLearning} trained an agent for a stream-based classification task with a binary ``label/discard'' action, we train an agent for a pool-based regression task where its action is the continuous weight $w_x^{(t)}$ that balances exploration and investigation.

A critical methodological constraint is to avoid ``double-dipping'' by optimizing directly on the evaluation metric (which requires oracle labels). We resolve this by decoupling the reward signal: instead of using the true test error, we derive the reward strictly from the currently labeled set $D_{tr}^{(t)}$ using \textit{K-fold Cross-Validation}. This ensures the agent learns from a clean, unbiased signal without data leakage from the unlabeled pool.

\subsubsection{Discretized Adaptation (WiGS-MAB)}
We first simplify the problem using Multi-Armed Bandits (MAB), which learn the best action on average without considering state context.
\begin{itemize}[leftmargin=*, noitemsep, topsep=0pt]
    \item \textbf{Arms:} A discrete, coarse set of weights (e.g., $\{0.25, 0.50, 0.75\}$).
    \item \textbf{Action:} Selection of one arm at iteration $t$.
    \item \textbf{Reward:} The improvement in generalization: $r_t = CV_{RMSE}^{(t-1)} - CV_{RMSE}^{(t)}$.
\end{itemize}
We employ the \textit{UCB1} algorithm \citep{MAB_FiniteTimeAnalysis} to balance exploring under-tested weights and exploiting those with the highest historical reward. To navigate the fundamental trade-off between weight granularity and exploration limits, we intentionally restrict the action space to a coarse grid. Because MAB algorithms treat arms independently and cannot generalize across similar weights, a finer grid would exhaust the strictly bounded labeling budget on exploration alone. This approach allows the selector to autonomously identify the most effective constant trade-off for a given dataset without requiring manual tuning or risking exploitation failure.

\subsubsection{Continuous Adaptation (WiGS-SAC)}
\label{subsec:WiGS_SAC}
While MAB is limited to a discrete set of static actions, the optimal strategy often requires (1) \textit{continuous control} over the weight $w_x$, and (2) \textit{dynamic adaptation} to the evolving learning state. To achieve this, we formulate the task as a full, continuous-state Markov Decision Process (MDP).

\paragraph{The MDP Formulation}
\begin{itemize}[leftmargin=*, noitemsep, topsep=0pt]
    \item \textbf{State ($s_t$):} A feature vector capturing the learning context, computed \textit{solely} from $D_{tr}^{(t)}$. It includes: (1) current generalization performance via K-fold cross-validation ($CV_{RMSE}$); (2) learning progress ($t/T$); and (3) distributional statistics of $D_{tr}^{(t)}$.
    \item \textbf{Action ($a_t$):} The continuous exploration weight $w_x^{(t)} \in [0, 1]$.
    \item \textbf{Reward ($r_t$):} The reduction in generalization error: $r_t = CV_{RMSE}^{(t-1)} - CV_{RMSE}^{(t)}$. A positive reward indicates the selected weight improved the model.
\end{itemize}

\paragraph{Optimization} We optimize this policy using \textit{Soft Actor-Critic (SAC) }\citep{SAC_Paper}. SAC maximizes a trade-off between expected reward and policy entropy. This entropy term is critical for active learning: it encourages the agent to maintain stochasticity in its weight selection when the reward signal is ambiguous, preventing premature convergence to a deterministic heuristic.

The complete training procedure and active learning loop are summarized in Algorithm \ref{alg:WiGSAlgorithm} in Appendix \ref{sec:algorithm}

\begin{table}[htbp]
    \centering
    \resizebox{\columnwidth}{!}{%
    \begin{tabular}{l c c c}
        \toprule
        \textbf{Method} & \textbf{Weight $w_x^{(t)}$} & \textbf{Dynamic} & \textbf{Adaptive} \\ \midrule
        Random Sampling & NA & \textcolor{red}{$\times$} & \textcolor{red}{$\times$} \\
        GSx & 1 & \textcolor{red}{$\times$} & \textcolor{red}{$\times$} \\
        GSy & 0 & \textcolor{red}{$\times$} & \textcolor{red}{$\times$} \\
        iGS & NA & \textcolor{red}{$\times$} & \textcolor{red}{$\times$} \\
        WiGS - Static Weight (investigation focus) & $<0.5$ & \textcolor{red}{$\times$} & \textcolor{red}{$\times$} \\
        WiGS - Static Weight (exploration focus) & $>0.5$ & \textcolor{red}{$\times$} & \textcolor{red}{$\times$} \\
        WiGS - Linearly Decaying Weight & $1 - c \cdot t/T$ & \textcolor{green}{$\checkmark$} & \textcolor{red}{$\times$} \\
        WiGS - Exponentially Decaying Weight & $\exp(-c \cdot t/T)$ & \textcolor{green}{$\checkmark$} & \textcolor{red}{$\times$} \\
        WiGS (with Multi-Armed Bandits) & Adaptive (Discrete) & \textcolor{green}{$\checkmark$} & \textcolor{green}{$\checkmark$} \\
        WiGS (with Soft Actor-Critic) & Adaptive (Continuous) & \textcolor{green}{$\checkmark$} & \textcolor{green}{$\checkmark$} \\ \bottomrule
    \end{tabular}
    }
    \caption{Summary of baseline and proposed methods. $w_x^{(t)}$ denotes the weight assigned to the feature-space distance at iteration $t$ in Equation \ref{eq:wigs_score}. We distinguish between \textbf{Dynamic} strategies (time-varying based on a fixed schedule) and \textbf{Adaptive} strategies (state-dependent based on learned feedback from the data).}
    \label{tab:MethodsTable}
\end{table}

\section{Experimental Setup and Synthetic Data}
\label{sec:SyntheticExperiments}

To evaluate the performance of our proposed framework, we compare the 10 selector strategies of Table \ref{tab:MethodsTable} in our experiments. These contain four baseline strategies (passive learning, GSx, GSy, and iGS), two static WiGS (one with an investigation-focused strategy with a fixed $w_x$ at $0.25$ and another exploration-focused strategy with $w_x$ fixed at $0.75$, two time-decaying WiGS (linearly decaying at $c=1$ and exponentially decaying at $c=5$), one adaptive WiGS strategy using Multi-Armed-Bandits with discrete weight choices of $\{0.25, 0.5, 0.75\}$ and one adaptive WiGS strategy using Soft Actor-Critic with a continuous action space for the weights. 

To provide a comprehensive comparison with baseline methods, we incorporate four additional strategies outside the GS family representing distinct active learning paradigms. First, we include Uncertainty Sampling \citep{Cohn1996}, which selects the instance maximizing the analytical predictive variance of the ridge estimator. To benchmark against ensemble-based uncertainty sampling, we implement Query-By-Committee (QBC) \citep{Seung1992}. We utilize a ``Query-by-Bagging''  \citep{Abe1998} approach where the selection criterion is the prediction variance across a committee of bootstrap-trained models. We include Expected Model Change Maximization (EMCM) \citep{Cai2013} as a representative of gradient-based optimization. This method selects the candidate instance expected to induce the largest norm change in the model parameters. Finally, we evaluate Exploration Guided Active Learning (EGAL) \citep{Hu2010}, a model-free strategy that selects points maximizing a combined score of density (similarity to the unlabeled pool) and diversity (dissimilarity to the labeled set), thereby prioritizing representative sampling without reliance on the predictive model.

\subsection{Data Generating Process (DGP)}
\label{subsec:SyntheticDGP}
To rigorously evaluate the exploration-investigation trade-off, we construct a synthetic environment specifically designed to trigger the \textit{density veto} failure mode identified in Section \ref{sec:TheoreticalProof}. 

Features $x \in [0, 1]$ are drawn from a non-uniform three-component Gaussian Mixture Model (GMM) to create distinct regions of high and low density:
\begin{align*}
    p(x) &= 0.4 \mathcal{N}(x \mid 0.2, 0.07^2) + 0.3\mathcal{N}(x \mid 0.5, 0.1^2)\\
    &+ 0.3 \mathcal{N}(x \mid 0.85, 0.05^2).
\end{align*}

The target response is generated via a piecewise function $y = f(x) + \epsilon$ with heteroscedastic noise:
\begin{equation}
\begin{split}
    f(x) &= \begin{cases}
        \sin(10\pi x) & \text{if } x < 0.5 \\
        2x - 1 & \text{if } x \geq 0.5
    \end{cases} \\
    \text{with } \quad \epsilon &\sim \begin{cases}
        \mathcal{N}(0,1) & \text{if } 0.8 < x < 0.9 \\
        \mathcal{N}(0, 0.1^2) & \text{otherwise.}
    \end{cases}
\end{split}
\end{equation}

Crucially, the high-noise band ($0.8 < x < 0.9$) is strategically placed to overlap with the densest GMM component centered at $x=0.85$. This creates a conflict: the region contains the highest aleatoric uncertainty (requiring investigation) but also the highest feature density (implying low exploration value).

We hypothesize that the static, multiplicative iGS baseline will fail in this area. The high density ($d_n^x \to 0$) will suppress the total score, effectively ``vetoing'' the selection of these high-error points. In contrast, we predict the adaptive WiGS agent will learn to decouple these objectives, setting $w_x^{(t)} \approx 0$ to ignore the misleading density signal and prioritize the critical investigation of the noise trap.

A more complex \text{Three-Regime} variant of this DGP with additional noise traps and its corresponding results are detailed in Appendix \ref{sec:ThreeRegime}.

\begin{figure}[htbp]
    \centering
    \includegraphics[width=0.48\textwidth]{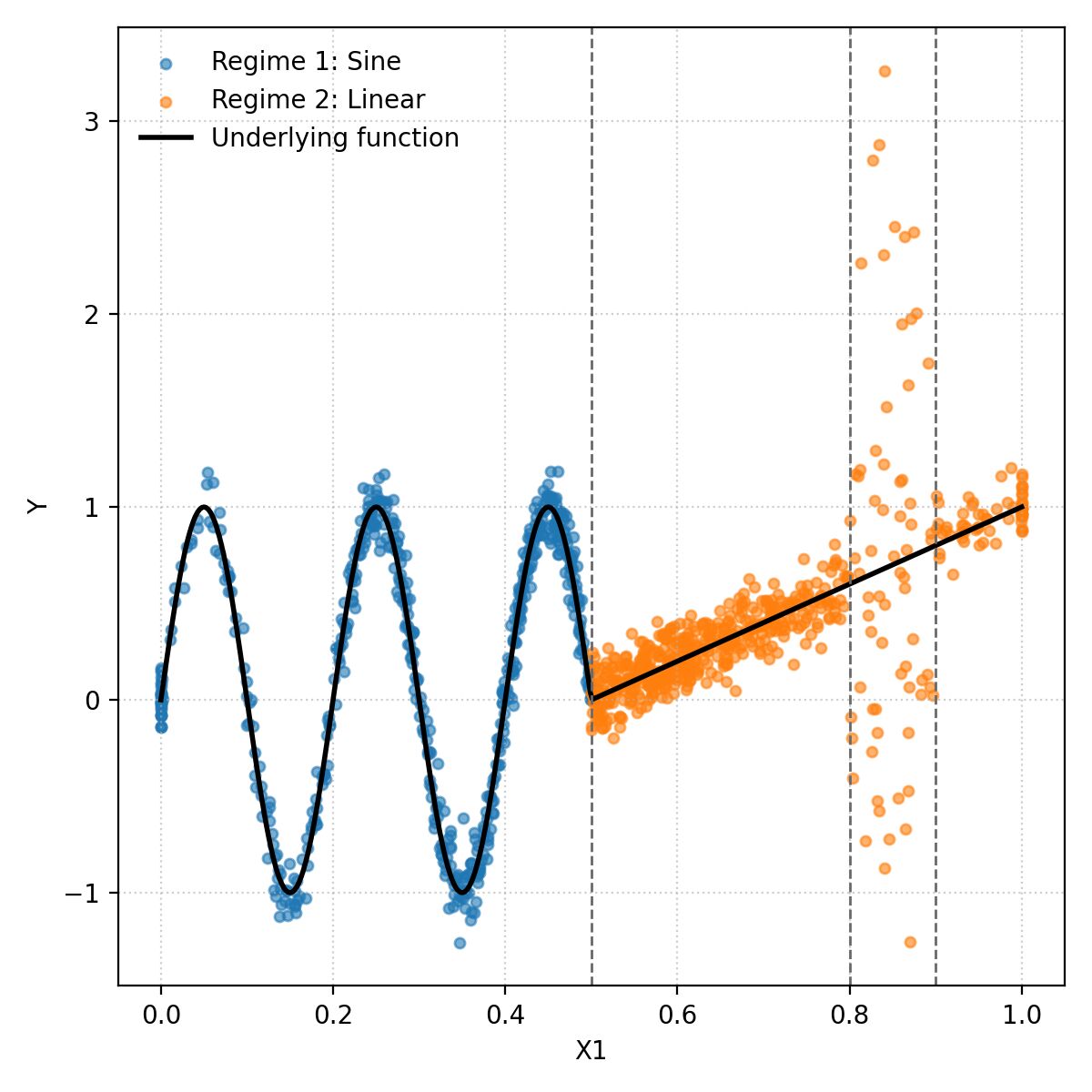}
    \caption{Visualization of the synthetic dataset ($N=1000$). The high-noise region (trap) at $x \approx 0.85$ coincides with high data density.}
    \label{fig:synthetic_dgp}
\end{figure}

\subsection{Experimental Protocol}
\label{subsec:experimental_protocol}

We evaluate performance on synthetic datasets ($N=1000$) and 18 real-world benchmarks (detailed in Appendix \ref{tab:DatasetsTable}). For each dataset, we perform 100 independent replications with different random seeds. In each trial, the data is randomly partitioned into an initial training set $D_{tr}$ ($5\%$ of samples) and a candidate pool $D_{cdd}$ ($95\%$). The active learning process iterates until the pool is exhausted.

Following \citet{iGS}, all strategies employ \textit{Ridge Regression} ($\alpha=0.01$) as the predictor model to isolate the impact of the selection strategy from model complexity. To further evaluate the generalizability of our framework to non-linear methods, we provide additional results using a Random Forest Regressor in Appendix \ref{subsec:RF_Results}. For the adaptive WiGS agents, the reward signal is computed via \textit{5-fold Cross-Validation} on the currently labeled set $D_{tr}$. This design strictly avoids data leakage from the unlabeled pool, ensuring the agent learns a policy deployable in real-world settings where oracle labels are unavailable.

\subsection{Evaluation Metrics}
\label{subsec:evaluation_metric}

A critical component of our methodology is the decoupling of the evaluation metric from the reinforcement learning reward. To avoid the ``active learning validation paradox'' \citep{ALPitfalls}, where optimizing on a test set implies access to oracle labels, our agents derive their reward signal solely from K-fold Cross-Validation RMSE on the currently labeled set $D_{tr}$. This ensures the policy is learned without data leakage, consistent with best practices for real-world deployment \citep{Konyushkova2017}.

For empirical evaluation, we follow the precedent of \citet{iGS} and use the \textit{Full-Pool Root Mean Squared Error (RMSE)}. This metric assesses performance across the entire domain ($D_{tr} \cup D_{cdd}$) by comparing a hybrid prediction vector against the ground truth labels. Specifically, we define the prediction for labeled samples ($D_{tr}$) as their known true values, and for unlabeled samples ($D_{cdd}$) as the model's estimates. This yields a performance trace over time, which we report as the deviation from the iGS baseline (averaged across 100 seeds).

To summarize global performance, we calculate the \textit{Area Under the Learning Curve (AUC)} via trapezoidal integration of the RMSE trace. We report the \textit{Relative AUC}, normalized by the iGS baseline:
\begin{equation}
    \text{Rel. AUC} = \frac{\int \text{RMSE}_{\text{method}}(t) dt}{\int \text{RMSE}_{\text{iGS}}(t) dt}
\end{equation}
Because a single scalar AUC can obscure how performance evolves over time in active learning, we supplement this global metric with milestone label efficiency ($N_{rel}$) and full RMSE trace plots (Appendix \ref{sec:MainResultsTracePlots}) to fully characterize the learning trajectories.

\subsection{Results and Analysis of Agent Policy}
\label{subsec:AgentPolicyResults}

We evaluate the performance of the proposed WiGS strategies against the baseline on the synthetic Two-Regime dataset. Figure \ref{fig:SyntheticTracePlots} presents the progression of the Full-Pool RMSE as the active learning process acquires more labels. To highlight performance gains relative to the state-of-the-art, we visualize the RMSE deviation from the iGS baseline: $\Delta \text{RMSE} \equiv \text{RMSE}_{\text{method}} - \text{RMSE}_{\text{iGS}}.$
A value below $0$ (the red line) indicates lower error compared to the iGS baseline. The relative AUC can be seen in Figure \ref{fig:AUC_Heatmap}.

Statistical analysis confirms that the observed performance gap is significant. Specifically, the Wilcoxon signed-rank tests (Appendix \ref{subsec:WRSTResults}) show that WiGS-SAC outperforms the baselines with statistical significance ($p < 0.05$).

Consistent with literature, the single-objective strategies GSx (pure exploration) and GSy (pure investigation) perform poorly, frequently resulting in errors higher than the iGS baseline \citep{iGS, AL_Regression_Benchmark}. This confirms that neither feature-space coverage nor uncertainty sampling alone is sufficient for complex regression landscapes.

These results directly validate the ``density veto'' proposition from Section \ref{sec:TheoreticalProof}. The high-noise trap ($x \in [0.8, 0.9]$) deliberately overlaps with a Gaussian density peak. For a candidate in this region, the feature diversity score is locally minimized (due to high density), while the uncertainty is maximized. As predicted, the multiplicative baseline suppresses this signal and fails to reduce error in the trap region, whereas the adaptive agent successfully targets it.

\begin{figure}[t!]
    \centering
    \includegraphics[width=0.48\textwidth, height=0.48\textwidth]{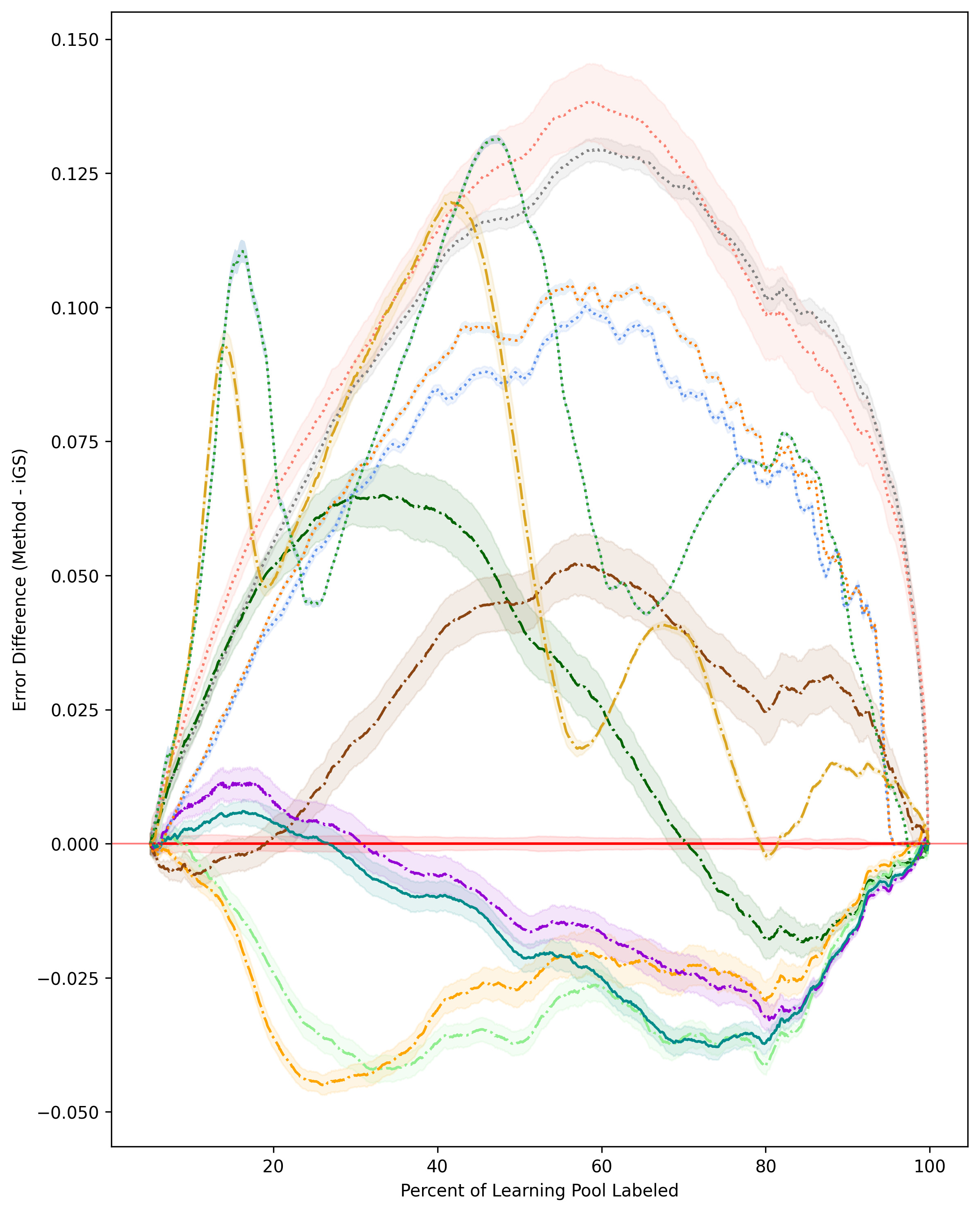}    
    \vspace{0.5em}    
    \centering
    \includegraphics[width=\linewidth]{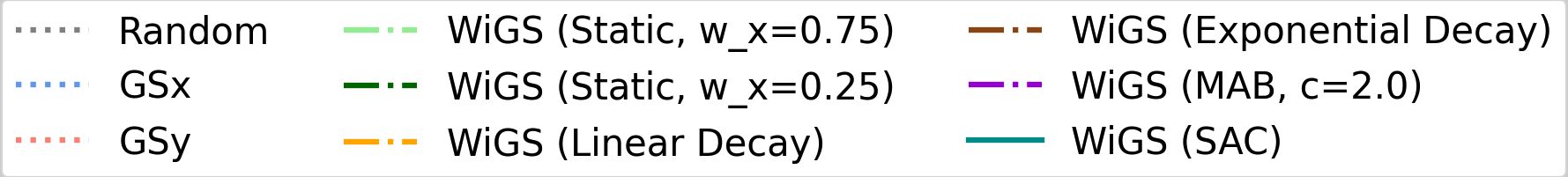}
    \caption{Performance difference relative to iGS (red line). Values below zero indicate superior performance. The adaptive WiGS-SAC agent and exploration-focused static variants (bottom cluster) consistently outperform the baseline, reducing absolute error by up to 0.05.}
    \label{fig:SyntheticTracePlots}
\end{figure}

\begin{figure*}[tbp]
    \centering
    \includegraphics[width=\textwidth]{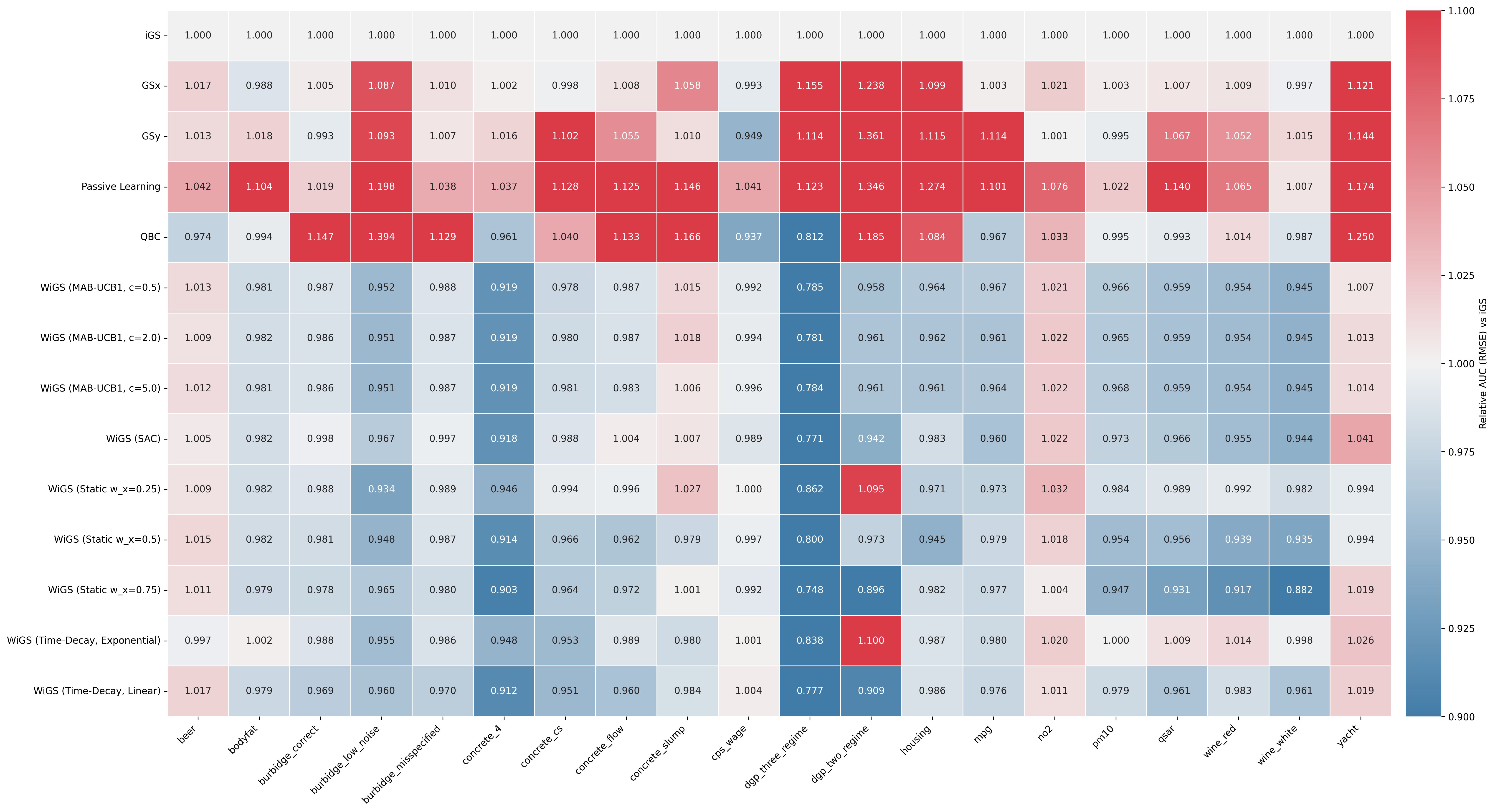}
    \caption{Global Performance Heatmap. Values represent the ratio of the Area Under the RMSE Curve (AUC) for each method relative to the iGS baseline. Blue cells ($<1.0$) indicate superior performance (lower cumulative error), while red cells ($>1.0$) indicate inferior performance. Notably, the WiGS methods (bottom rows) demonstrate consistent robustness across diverse domains.}
    \label{fig:AUC_Heatmap}
\end{figure*}

Interestingly, the static WiGS strategy with $w_x = 0.75$ (exploration-focused) performed exceptionally well, rivaling the adaptive SAC agent. This suggests that for this specific non-uniform DGP, a strong exploration bias happened to be the optimal global strategy. We emphasize that this static baseline represents a best-case performance selected \textit{ex post facto} from our grid search ($w_x \in \{0, 0.25, 0.5, 0.75, 1\}$). However, the critical advantage of WiGS-SAC is its \textit{autonomy}. Without knowing \textit{a priori} that $0.75$ was the optimal parameter, the RL agent learned a policy that matched this performance. In a real-world setting, identifying this optimal scalar would require an exhaustive search with access to ground-truth labels. In contrast, the investigation-heavy static weight ($w_x = 0.25$) performed significantly worse. This underscores the risk of static hyperparameters and the value of the SAC agent's ability to adaptively discover the optimal weighting strategy.

\section{Benchmark Experiments}
\label{sec:BenchmarkExperiments}

To ensure broad applicability, we utilize a suite of 18 publicly available datasets popular in the active learning literature \citep{iGS, iRDM, IDEAL, Burbridge, Cai2013}. As summarized in Table \ref{tab:DatasetsTable} in the appendix, these datasets cover a wide range of properties, varying significantly in sample size and feature dimensionality. Prior to training, each dataset is processed through a standardized pipeline: continuous features are scaled to zero mean and unit variance, while categorical features are one-hot encoded. To control for the potential impact of outliers on this preprocessing normalization choice, we also replicate our full experimental suite using robust scaling (based on the interquartile range), with results detailed in Appendix \ref{subsec:robust_norm}.

We apply the same experimental protocol described in Section \ref{subsec:experimental_protocol} to all 18 benchmark datasets, comparing all 14 selection strategies with the full-pool RMSE metric averaged across 100 replications.

\subsection{Experimental Results}
\label{subsec:ExperimentalBenchmarkResutlts}


Figure \ref{fig:AUC_Heatmap} summarizes the global performance across all benchmarks. A value below $1.0$ (blue) indicates that a method achieved a lower cumulative error than the iGS baseline. The widespread red among the alternative baselines underscores the strength of the iGS standard, which in turn highlights the significance of the WiGS framework consistently achieving superior (blue) performance. Temporal trace plots are provided in Figures \ref{fig:MainResults1} and \ref{fig:MainResults2} in Appendix \ref{sec:MainResultsTracePlots}.

The adaptive WiGS-SAC agent demonstrates remarkable robustness, matching or outperforming the iGS baseline on 15 of the 20 evaluated datasets. Although static baselines can perform competitively if optimally tuned \textit{a priori}, the adaptive agent strictly outperforms the best static baseline ($w_x=0.25$) on the \texttt{mpg} dataset. This confirms that for sufficiently complex tasks, the optimal exploration-investigation trade-off is dynamic, a finding further reinforced by the superiority of the adaptive agent in our Random Forest experiments (Appendix \ref{subsec:RF_Results}). On datasets where a fixed weight does suffice, WiGS-SAC still provides the practical advantage of \textit{autonomy} by successfully converging to the optimal policy purely from the evolving data state, effectively automating the otherwise impossible hyperparameter search.

For example, while the investigation-heavy static strategy ($w_x=0.25$) fails on the \textit{Two-Regime} dataset (relative error $1.095$), the adaptive SAC agent successfully navigates the trade-off ($0.943$), validating its ability to dynamically optimize the weighting policy without human intervention. The advanced baselines exhibit significant volatility. Consistent with \citet{AL_Regression_Benchmark}, QBC and Uncertainty Sampling suffer catastrophic failures on noisy domains (e.g., \textit{Burbidge} errors $>1.40$). Conversely, EGAL consistently underperforms (red columns), confirming that purely density-based sampling prioritizes redundant regions. In contrast, WiGS-SAC maintains consistent stability across the entire benchmark suite, avoiding the high-variance failure modes of pure uncertainty or pure density strategies.

\subsection{Label Efficiency and Robustness}
\label{subsec:LabelEfficiency}

Practical deployment of active learning strategies often hinges on cost reduction. We quantify this using \textit{Relative Label Efficiency} ($N_{rel}$): the ratio of labels required by a method to reach a specific performance milestone relative to the iGS baseline. An $N_{rel} < 1.0$ indicates the method requires fewer labels to achieve the same accuracy.

Figure \ref{fig:label_efficiency} aggregates these efficiency scores across our entire experimental suite. Each boxplot summarizes the distribution of $N_{rel}$ values calculated over the 20 distinct datasets, where each data point represents the mean efficiency over 100 independent simulation trials. 

\begin{figure*}[t]
    \centering
    \includegraphics[width=0.9\textwidth]{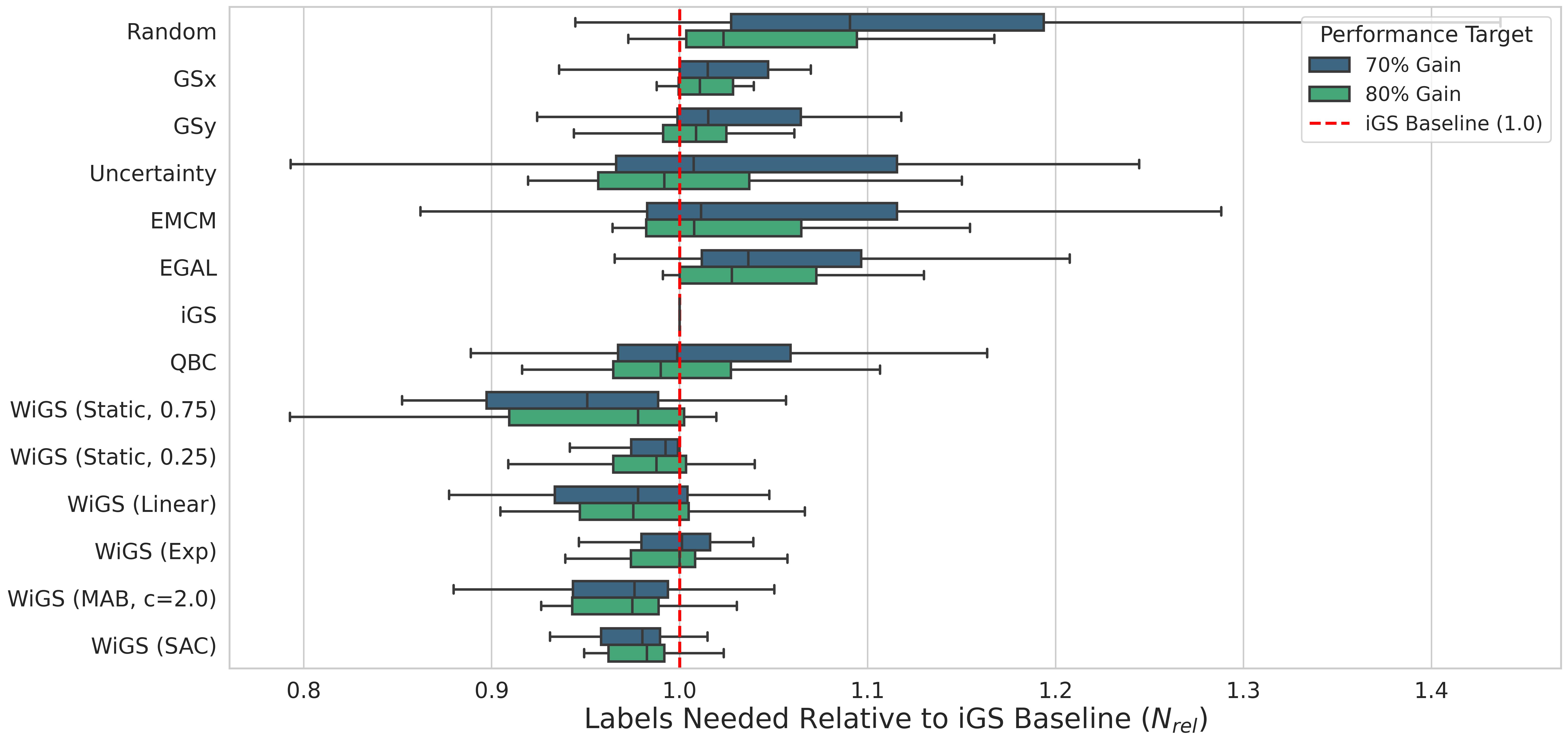}
    \caption{Relative Label Efficiency ($N_{rel}$) aggregated across 20 datasets. The distribution represents the labeling budget required to achieve $70\%$ (blue) and $80\%$ (green) of the total possible performance gain. Values to the left ($<1.0$) indicate superior efficiency. The adaptive WiGS agents (SAC, MAB) not only require fewer labels to reach these milestones than the iGS baseline, but also exhibit narrower variance, highlighting their robust generalization across different datasets.}
    \label{fig:label_efficiency}
\end{figure*}
The results highlight two critical advantages of the proposed framework. Firstly, the adaptive WiGS strategies consistently shift the efficiency distribution to the favorable region indicating a reduction of labeling costs. For WiGS-MAB ($c=2.0$), the median efficiency is approximately $0.96$, implying a $4\%$ reduction in labeling costs. More importantly, the upper quartile remains at or below $1.0$, indicating that the adaptive strategy yields benefits in the vast majority of scenarios and rarely underperforms the baseline.

Secondly, a major finding is the volatility of the advanced baselines. While QBC and EMCM achieve high efficiency on select datasets (left whiskers extending to $\approx 0.85$), they exhibit heavy tails of inefficiency (right whiskers extending beyond $1.15$). In contrast, WiGS-SAC demonstrates superior worst-case performance, maintaining a tight distribution with minimal right skew. This reliability makes WiGS a safer choice for real-world applications where the dataset characteristics are initially unknown.

\subsection{Algorithmic Warm-up}
\label{subsec:SampleComplexity}

A critical insight is the characteristic \textit{warm-up period} of the RL-based strategies. Unlike static heuristics that operate on a fixed criterion from the first iteration, the WiGS-SAC agent requires initial data to estimate the value function and refine its policy \citep{Kakade2003}. Consequently, performance often lags behind the iGS baseline during the early exploration phase before surpassing it as the policy converges.

This sample complexity explains the performance stratification observed across dataset scales in the trace plots of Figures \ref{fig:MainResults1} and \ref{fig:MainResults2}. During early iterations, the reinforcement learning agent must actively explore the continuous action space to model the reward landscape. In data-scarce regimes with severely constrained labeling budgets, such as \textit{Yacht} ($N=308$) or \textit{NO2} ($N=500$), the active learning loop terminates while the agent is still in this algorithmic warm-up phase, resulting in performance parity with static heuristics. Conversely, in large-scale tasks like \textit{Wine-White} ($N=4898$) and \textit{Wine-Red} ($N=1599$), the extended labeling horizon allows the downstream reduction in generalization error to heavily amortize this initial exploration cost. These findings yield a practical operational guideline: WiGS-SAC is optimal when the labeling budget is sufficient to overcome the initial RL warm-up. 

\section{Conclusion}
\label{sec:Conclusion}

As data labeling remains a critical bottleneck in domains ranging from materials science to generative AI \citep{InstructGPT}, active learning is essential for efficient model development. This work challenged the prevailing reliance on static, multiplicative heuristics, demonstrating that they enforce a rigid trade-off between exploration and investigation that fails in heterogeneous feature spaces. To overcome this, we introduced the \textit{Weighted improved Greedy Sampling (WiGS)} framework, which captures this balance as a dynamic, additive control problem solvable via reinforcement learning.

Through extensive experiments on 18 benchmarks, our study establishes autonomy as the definitive advantage of the WiGS framework. Unlike static heuristics that require expensive, \textit{ex post facto} grid searches, WiGS-SAC effectively automates hyperparameter tuning by autonomously converging to the optimal static weight when a fixed trade-off suffices. Moreover, on complex benchmarks, the agent validates the time-dependent nature of active learning by strictly outperforming even the best-tuned static baselines, proving that a self-regulating agent can effectively replace manual heuristics without prior knowledge.

Our theoretical and empirical analyses confirm the mechanism behind this success. We proved that the multiplicative iGS criterion suffers from a ``density veto,'' blinding it to high-error samples in dense regions. In contrast, the WiGS-SAC agent learns to decouple these objectives, dynamically shifting priority from feature diversity in complex regions to uncertainty reduction in noisy regimes. By enabling agents to autonomously tune their own exploration strategies, this work moves us closer to general-purpose active learning systems capable of adapting to the unique complexities of diverse scientific and industrial domains.

\bibliographystyle{plainnat}
\bibliography{bibliography}

\newpage
\onecolumn

\title{Adaptive Active Learning for Regression via Reinforcement Learning\\(Supplementary Material)}

\maketitle
\appendix

\section{Proof of Proposition 3.1 (The Density Veto)}
\label{sec:proofs}
\begin{proof}
    Recall the scoring functions for the multiplicative (iGS) and additive (WiGS) strategies:
    \begin{equation}
        S_{iGS}(x) = d(x) \cdot u(x) \quad \text{and} \quad S_{WiGS}(x; w) = w \cdot d(x) + (1-w) \cdot u(x)
    \end{equation}
    
    \textbf{1. Multiplicative Failure Mode:}
    The multiplicative selector fails to select the high-uncertainty target $x^*$ if $S_{iGS}(x^*) < S_{iGS}(x')$. Substituting the scores:
    \begin{equation}
        d^* u^* < d' u' \implies d^* < \frac{d' u'}{u^*}
    \end{equation}
    Let $\delta = \frac{d' u'}{u^*}$. Since $d', u', u^*$ are positive constants, $\delta > 0$. Thus, if the local density around $x^*$ is high enough such that $d^* < \delta$, the algorithm is forced to rank the sub-optimal $x'$ higher than $x^*$. The low diversity score acts as a veto, blinding the model to the uncertainty signal $u^*$.

    \textbf{2. Additive Robustness:}
    For the additive selector, we seek a weight $w$ such that $S_{WiGS}(x^*) > S_{WiGS}(x')$:
    \begin{equation}
        w d^* + (1-w) u^* > w d' + (1-w) u'
    \end{equation}
    Rearranging to isolate $w$:
    \begin{equation}
        (1-w)(u^* - u') > w (d' - d^*) \implies \frac{1-w}{w} > \frac{d' - d^*}{u^* - u'}
    \end{equation}
    Since $\frac{1-w}{w} = \frac{1}{w} - 1$, we obtain the condition:
    \begin{equation}
        \frac{1}{w} > \frac{d' - d^*}{u^* - u'} + 1
    \end{equation}
    Since $u^* > u'$ (target has higher uncertainty) and $d' > d^*$ (distractor is in a sparser region), the term on the right-hand side is a positive constant $K$. The inequality becomes $\frac{1}{w} > K$. As $w \to 0$ (indicating a shift towards pure investigation), the term $\frac{1}{w} \to \infty$. Therefore, the inequality can always be satisfied by choosing any weight $w \in (0, 1/K)$. This demonstrates that the additive framework retains the capacity to select $x^*$ by dynamically adjusting $w$.
\end{proof}

\clearpage
\section{Implementation Details}
\subsection{Algorithm}
\label{sec:algorithm}
\begin{algorithm}[!htb]
   \caption{Weighted Improved Greedy Sampling (WiGS)}
   \label{alg:WiGSAlgorithm}
\begin{algorithmic}[1]
    \STATE {\bfseries Input:} Initial training set $D_{tr}^{(0)}$; Candidate pool $D_{cdd}^{(0)}$; Model $f$; Total iterations $T$;
    \STATE \quad \quad Weighting Strategy $S \in \{\text{Static, Decay, MAB, SAC}\}$
    \STATE {\bfseries Initialize:} Agent (if $S \in \{\text{MAB, SAC}\}$); Replay Buffer $\mathcal{B}$ (if $S = \text{SAC}$)
    \STATE {\bfseries Initialize:} $CV_{RMSE}^{prev} \leftarrow \infty$; $s_{prev} \leftarrow \text{null}$; $w_{prev} \leftarrow \text{null}$
    
    \FOR{$t = 0$ {\bfseries to} $T-1$}
        \STATE Train model $f^{(t)}$ on $D_{tr}^{(t)}$
        
        \STATE $CV_{RMSE}^{(t)} \leftarrow \mathrm{CalculateCV\_RMSE}(f^{(t)}, D_{tr}^{(t)})$
        \STATE $r_t \leftarrow CV_{RMSE}^{prev} - CV_{RMSE}^{(t)}$
        
        \IF{$S = \text{Static}$}
            \STATE $w_x^{(t)} \leftarrow w_{fixed}$
        \ELSIF{$S = \text{Decay}$}
            \STATE $w_x^{(t)} \leftarrow \mathrm{CalculateDecay}(t, T)$ 
        \ELSIF{$S = \text{MAB}$}
            \IF{$t > 0$}
                \STATE $\mathrm{UpdateArmValue}(w_{prev}, r_t)$
            \ENDIF
            \STATE $w_x^{(t)} \leftarrow \mathrm{SelectArmUCB1}()$
        \ELSIF{$S = \text{SAC}$}
            \STATE $s_t \leftarrow \mathrm{GetState}(CV_{RMSE}^{(t)}, D_{tr}^{(t)}, t)$
            \IF{$t > 0$}
                \STATE Store $(s_{prev}, w_{prev}, r_t, s_t)$ in $\mathcal{B}$ 
                \STATE $\mathrm{UpdateSACAgent}(\mathcal{B})$
            \ENDIF
            \STATE $w_x^{(t)} \leftarrow \pi(s_t) + \text{noise}$
        \ENDIF 
        
        \STATE Compute pairwise distances $d_{nm}^x, d_{nm}^y$ and normalize via $\phi(\cdot)$
        \FOR{{\bfseries each} candidate $\mathbf{x}_n \in D_{cdd}^{(t)}$} 
            \STATE $s_n^{WiGS} \leftarrow \min_{m \in D_{tr}^{(t)}} \left( w_x^{(t)} \phi(d_{nm}^x) + (1-w_x^{(t)}) \phi(d_{nm}^y) \right)$
        \ENDFOR
        
        \STATE $\mathbf{x}^* \leftarrow \underset{\mathbf{x}_n \in D_{cdd}^{(t)}}{\arg\max} ~ s_n^{WiGS}$
        \STATE $y^* \leftarrow \mathrm{QueryOracle}(\mathbf{x}^*)$
        \STATE $D_{tr}^{(t+1)} \leftarrow D_{tr}^{(t)} \cup \{(\mathbf{x}^*, y^*)\}$
        \STATE $D_{cdd}^{(t+1)} \leftarrow D_{cdd}^{(t)} \setminus \{\mathbf{x}^*\}$
        
        \STATE $CV_{RMSE}^{prev} \leftarrow CV_{RMSE}^{(t)}$
        \STATE $w_{prev} \leftarrow w_x^{(t)}$
        \IF{$S = \text{SAC}$}
            \STATE $s_{prev} \leftarrow s_t$ 
        \ENDIF
    \ENDFOR
    \STATE {\bfseries Output:} Final Model $f^{(T)}$
\end{algorithmic}
\end{algorithm}

\subsection{Reinforcement Architecture and Hyperparameters}
\label{sec:Hyperparams}

To ensure reproducibility, we provide the specific hyperparameters used for both the Multi-Armed Bandits (MAB) and the Soft Actor-Critic (SAC) agents and their training processes. The implementation uses the PyTorch framework.

\textbf{Soft Actor-Critic (WiGS-SAC) Configuration:}
\textbf{Network Architecture:}
The Actor and Critic networks share a similar Multi-Layer Perceptron (MLP) architecture:
\begin{itemize}[leftmargin=*, noitemsep, topsep=0pt]
    \item \textbf{Hidden Layers:} 2 fully connected layers with 64 units each.
    \item \textbf{Activation:} ReLU activation for hidden layers.
    \item \textbf{Output (Actor):} A Tanh activation to bound the action space to $[-1, 1]$, which is then linearly scaled to $[0, 1]$ for the weight $w_x$.
    \item \textbf{Output (Critic):} Linear output for Q-value estimation.
\end{itemize}

\textbf{Hyperparameters:}
Table \ref{tab:Hyperparams} lists the hyperparameters held constant across all 20 datasets (2 synthetic + 18 benchmarks). No dataset-specific hyperparameter tuning was performed, highlighting the robustness of the default configuration.

\begin{table}[htbp]
    \centering
    \begin{tabular}{l c}
        \toprule
        \textbf{Parameter} & \textbf{Value} \\ \midrule
        Optimizer & Adam \\
        Learning Rate ($\alpha_{lr}$) & $3 \cdot 10^{-4}$ \\
        Discount Factor ($\gamma$) & 0.99 \\
        Replay Buffer Size ($\mathcal{B}$) & 10,000 \\
        Batch Size & 64 \\
        Soft Update Rate ($\tau$) & 0.005 \\
        Entropy Coefficient ($\alpha$) & 0.2 \\
        Hidden Units & 64 \\
        \bottomrule
    \end{tabular}
    \caption{Hyperparameters for the WiGS-SAC agent.}
    \label{tab:Hyperparams}
\end{table}

\textbf{Multi-Armed Bandit (WiGS-MAB) Configuration:}
For the discretized adaptive strategy, we employed the Upper Confidence Bound (UCB1) algorithm \citep{MAB_FiniteTimeAnalysis} to select the weight $w_x$ from a finite set of options.
\begin{itemize}[leftmargin=*, noitemsep, topsep=0pt]
    \item \textbf{Arms:} The set of candidate weights was defined as $\mathcal{A} = \{0.25, 0.50, 0.75\}$.
    \item \textbf{Exploration Constant ($c$):} This hyperparameter controls the degree of exploration in the UCB1 formula. We evaluated $c \in \{0.5, 2.0, 5.0\}$ to test sensitivity to exploration aggressiveness.
    \item \textbf{Initialization:} The agent sequentially selects each arm once during the first $|\mathcal{A}|$ iterations to initialize the reward estimates before switching to the UCB selection rule.
\end{itemize}

\textbf{Computational Resources:}
All experiments were conducted on a standard high-performance computing cluster.

\clearpage
\section{Additional Synthetic Simulation Details}
\subsection{Three Regime Synthetic Simulation}
\label{sec:ThreeRegime}

This DGP presents a more complex, three-regime investigation challenge to test the agent's ability to learn a more granular policy. The function $f(x)$ and noise $\epsilon$ are defined as:
\begin{equation}
    f(x) = \begin{cases}
        \sin(8\pi x) & \text{if } x < 0.4 \\
        3x - 1.5 & \text{if } 0.4 \leq x < 0.7 \\
        2\cos(6\pi x) & \text{if } X \geq 0.7
    \end{cases} 
    \text{ with }
    \epsilon \sim \begin{cases}
        \mathcal{N}(0, 1.5^2) & \text{if } 0.6 < x < 0.65 \\
        \mathcal{N}(0, 0.15^2) & \text{if } x \geq 0.7 \\
        \mathcal{N}(0, 0.1^2) & \text{otherwise}
    \end{cases}
\label{eq:DGP3}
\end{equation}

This DGP tests the agent's ability to prioritize different degrees of investigation, with three distinct functional forms and three noise levels. It includes an extreme noise, sparse region ($\sigma=1.5$) and a moderately high-noise, dense region ($\sigma=0.15$ for $X \geq 0.7$). This latter region serves as our primary test, again overlapping with the GMM clump at $x=0.85$ to trigger the same hypothesized iGS failure mechanism. A visualization of this DGP is presented in Figure \ref{fig:synthetic_dgp_3regime}.

\begin{figure}[htbp]
    \centering
    \begin{subfigure}[b]{0.48\textwidth}
        \centering
        \includegraphics[width=\textwidth]{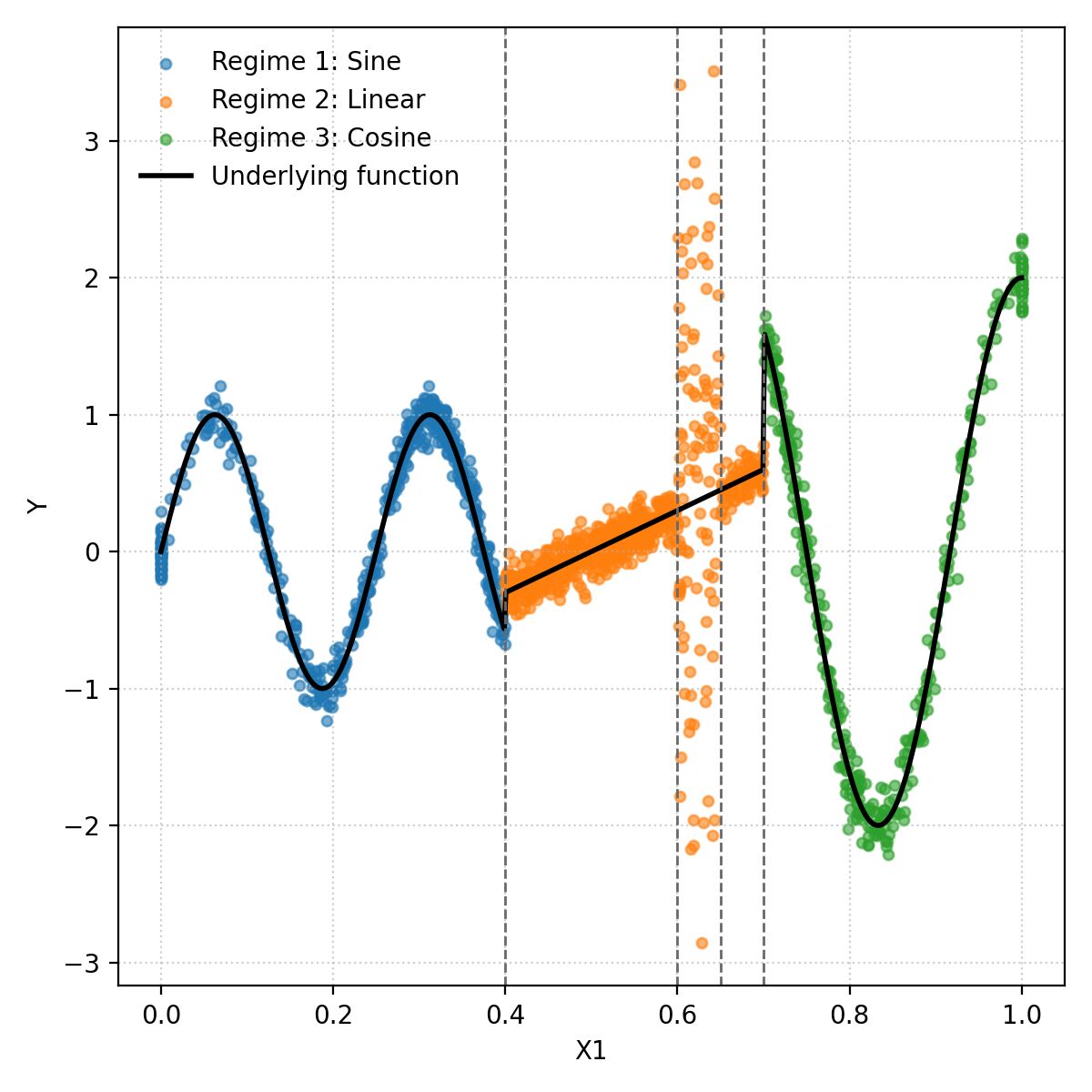}
        \caption{Three-Regime DGP}
        \label{fig:synthetic_dgp_3regime}
    \end{subfigure}
    \caption{Visualization of the synthetic datasets. Shows $N=1000$ sample data points (dots) and the true underlying function $f(x)$ (solid line). Shaded regions indicate areas of high noise that overlap with dense data regions from the GMM, creating strategic conflicts between exploration and investigation.}
    \label{fig:synthetic_dgps}
\end{figure}

\subsection{Analysis of Learned Agent Policy}
\label{subsec:AnalysisOfLearnedAgentPolicy}

To understand why the adaptive strategies succeed where static heuristics fail, we analyze the specific policies learned by the WiGS-SAC agent. Specifically, we examine how the agent adapts the weight $w_x^{(t)}$ in response to local data characteristics, verifying our hypothesis that the optimal balance is non-stationary.

\subsection{Spatial Adaptation of Weights}

Figure \ref{fig:Heatmaps} visualizes the average weight $w_x$ chosen by the SAC agent for each selected data point across the feature space. The color gradient indicates the agent's preference: red denotes a preference for feature-space exploration ($w_x \to 1$), while blue denotes a preference for output-space investigation ($w_x \to 0$).

The results confirm that the agent learns a region-specific policy. In the Two-Regime dataset (Figure \ref{fig:Heatmap_2Regime}), the agent avoids applying a blanket strategy to the sine wave ($x < 0.5$). Instead, we observe a distinct pattern that appears tied to local curvature: the vertices are predominantly red/orange, while the linear segments connecting them shift toward blue/white. This observation is consistent with the intuition that feature-space exploration is most critical at points of high curvature to define the function's extrema, whereas linear regions can be easily interpolated, allowing the agent to focus on minimizing prediction error via investigation.

In the right regime ($x \geq 0.5$), the agent encounters a linear function with a high-noise trap ($0.8 < x < 0.9$). Here, the policy becomes more heterogeneous. Unlike the clear exploration signal at the sine wave vertices, the selection weights in the trap region show a mix of investigation (blue) and exploration (red). This suggests the agent is balancing two competing signals: the high prediction error (driving investigation) and the need to maintain coverage in a difficult region (driving exploration), rather than committing to a single heuristic.

In the Three-Regime dataset (Figure \ref{fig:Heatmap_3Regime}), these behaviors are more pronounced with the sine and cosine regimes strongly replicating the ``red vertices, blue slopes'' pattern. The agent's behavior in the central linear region ($0.4 < x < 0.6$) and the subsequent noise trap ($0.6 < x < 0.7$) reveals a sophisticated transition strategy. Initially, the linear region is dominated by investigation (blue/white) as the function is simple. However, as $x$ approaches $0.6$, we observe an increasing mix of red/orange weights. This suggests the agent is reacting to the decreasing data density (leaving the central GMM clump) by ramping up exploration to prepare for the upcoming sparsity.

Most notably, the high-noise trap ($0.6 < x < 0.7$) displays a distinct split in policy. In the first half ($0.60 < x < 0.65$), where noise is extreme ($\sigma=1.5$) and density is lowest (see Equation \ref{eq:DGP3}), the agent commits to deep red (pure exploration). Here, the agent essentially ignores the chaotic error signal to force a ``bridge'' across the data desert. As it crosses into the second half ($0.65 < x < 0.70$) and approaches the dense cosine regime, the weights shift back to blue (investigation), allowing the model to stabilize its predictions as it lands in the new regime.

\begin{figure}[htbp]
    \centering
    \begin{subfigure}[b]{0.48\textwidth}
        \centering
        \includegraphics[width=\textwidth]{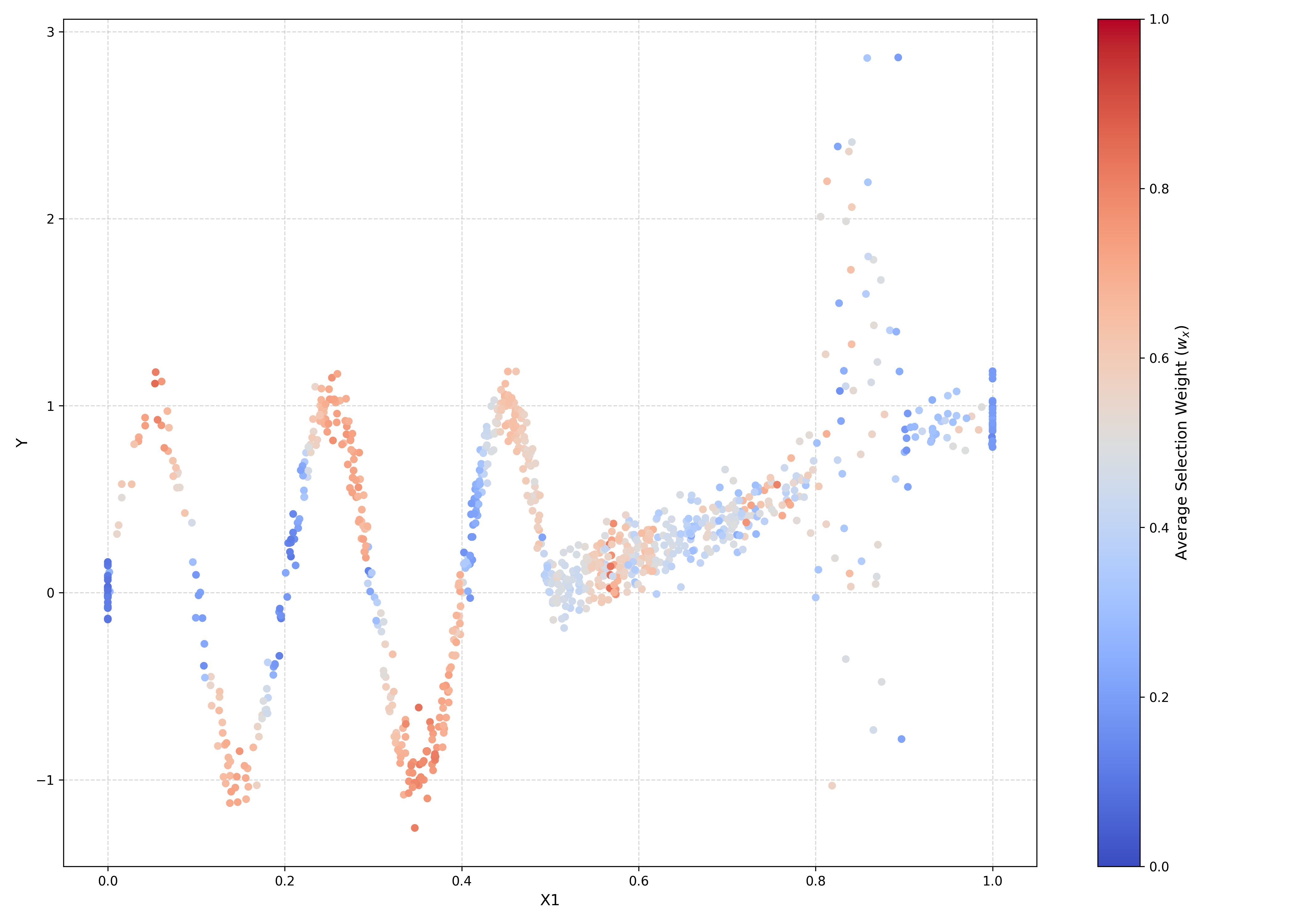}
        \caption{Two-Regime Policy Heatmap}
        \label{fig:Heatmap_2Regime}
    \end{subfigure}
    \hfill
    \begin{subfigure}[b]{0.48\textwidth}
        \centering
        \includegraphics[width=\textwidth]{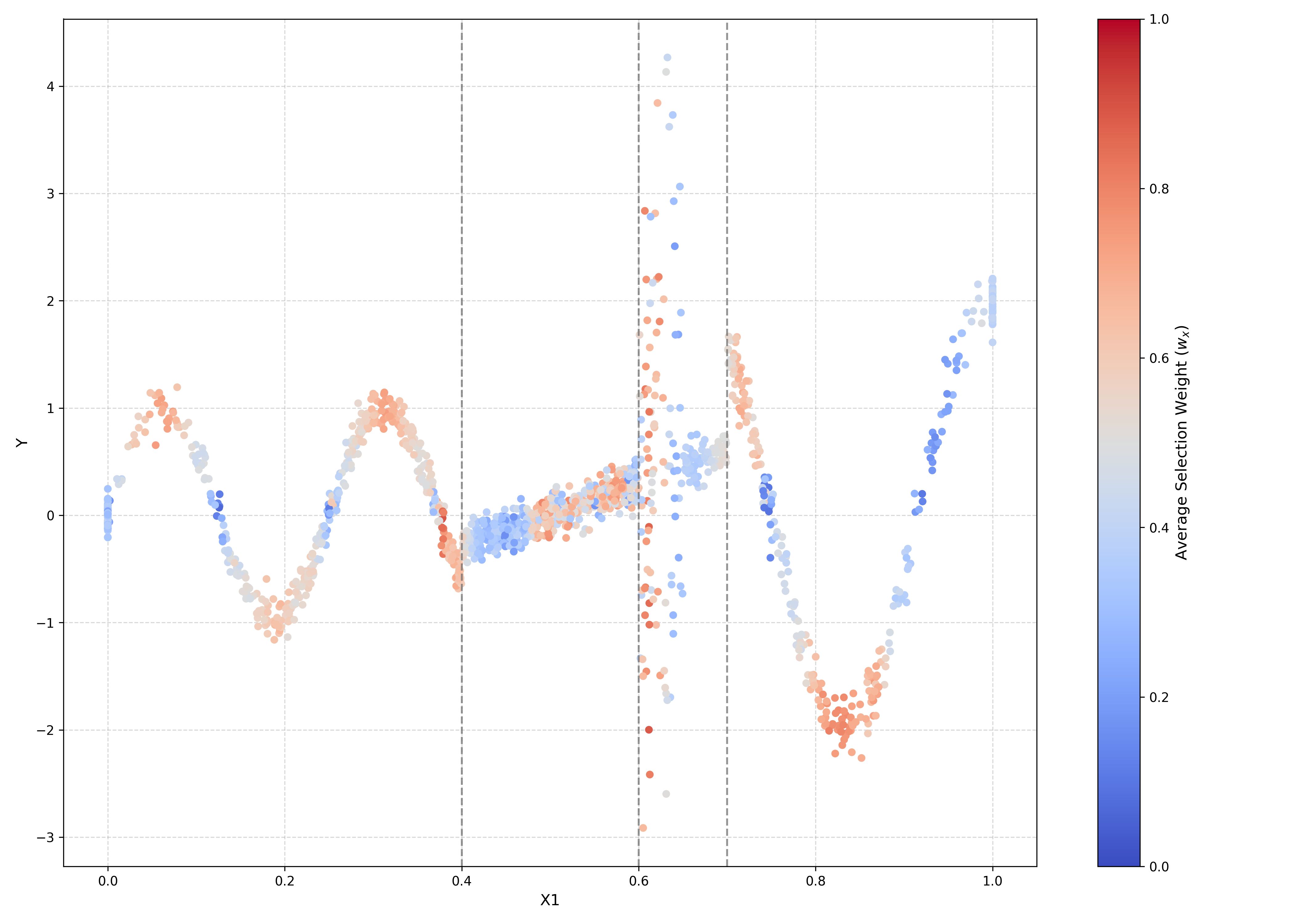}
        \caption{Three-Regime Policy Heatmap}
        \label{fig:Heatmap_3Regime}
    \end{subfigure}
    \caption{Spatial visualization of the learned weighting policy. Each point represents a selected sample, colored by the average weight $w_x$ assigned to it by the WiGS-SAC agent. Red indicates high exploration focus; Blue indicates high investigation focus. The agent learns to decouple the strategy, applying exploration to complex oscillations (vertices) and investigation to linear regions.}
    \label{fig:Heatmaps}
\end{figure}

\subsection{Temporal Adaptation of Weights}
We further investigate the temporal evolution of the weight parameter. Figure \ref{fig:WeightTrends} displays the average $w_x^{(t)}$ across all 100 seeds over the course of the active learning iterations. 

Notably, the agent does not converge to a single static ``optimal'' weight, nor does it follow a monotonic decay curve. While the mean weight remains centered near $0.5$, the large standard deviation (shaded region) indicates that the agent maintains a high degree of variance throughout the entire process. This suggests that the policy is highly reactive: at any given iteration, the agent may swing heavily toward exploration or investigation depending on the specific candidates remaining in the pool. This supports our hypothesis that the ``Exploration-Investigation'' trade-off is not a global hyperparameter to be tuned, but a dynamic decision to be made at every step.

\begin{figure}[htbp]
    \centering
    \begin{subfigure}[b]{0.48\textwidth}
        \centering
        \includegraphics[width=\textwidth]{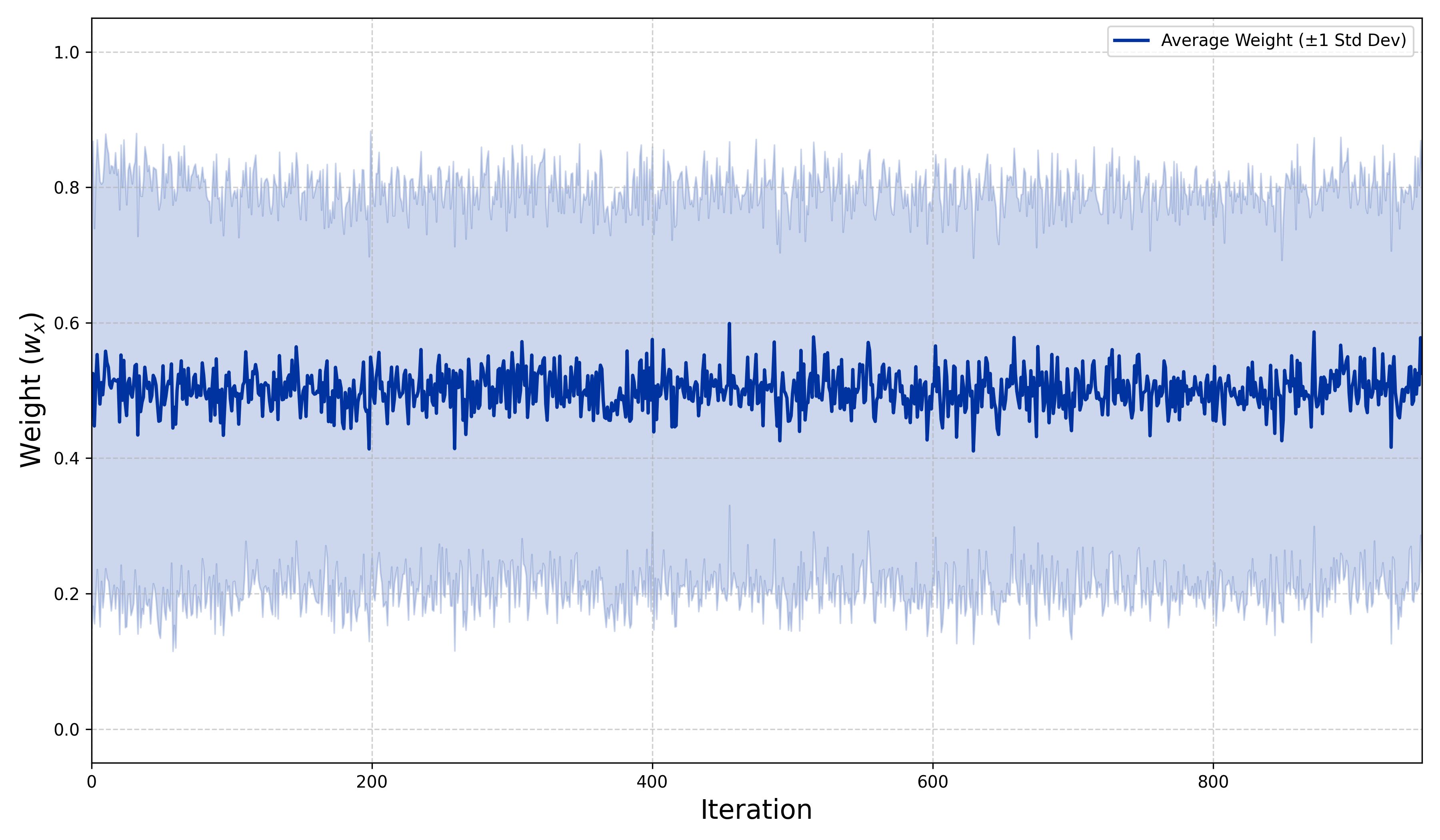}
        \caption{Two-Regime Policy Average Weights}
        \label{fig:Trend_2Regime}
    \end{subfigure}
    \hfill
    \begin{subfigure}[b]{0.48\textwidth}
        \centering
        \includegraphics[width=\textwidth]{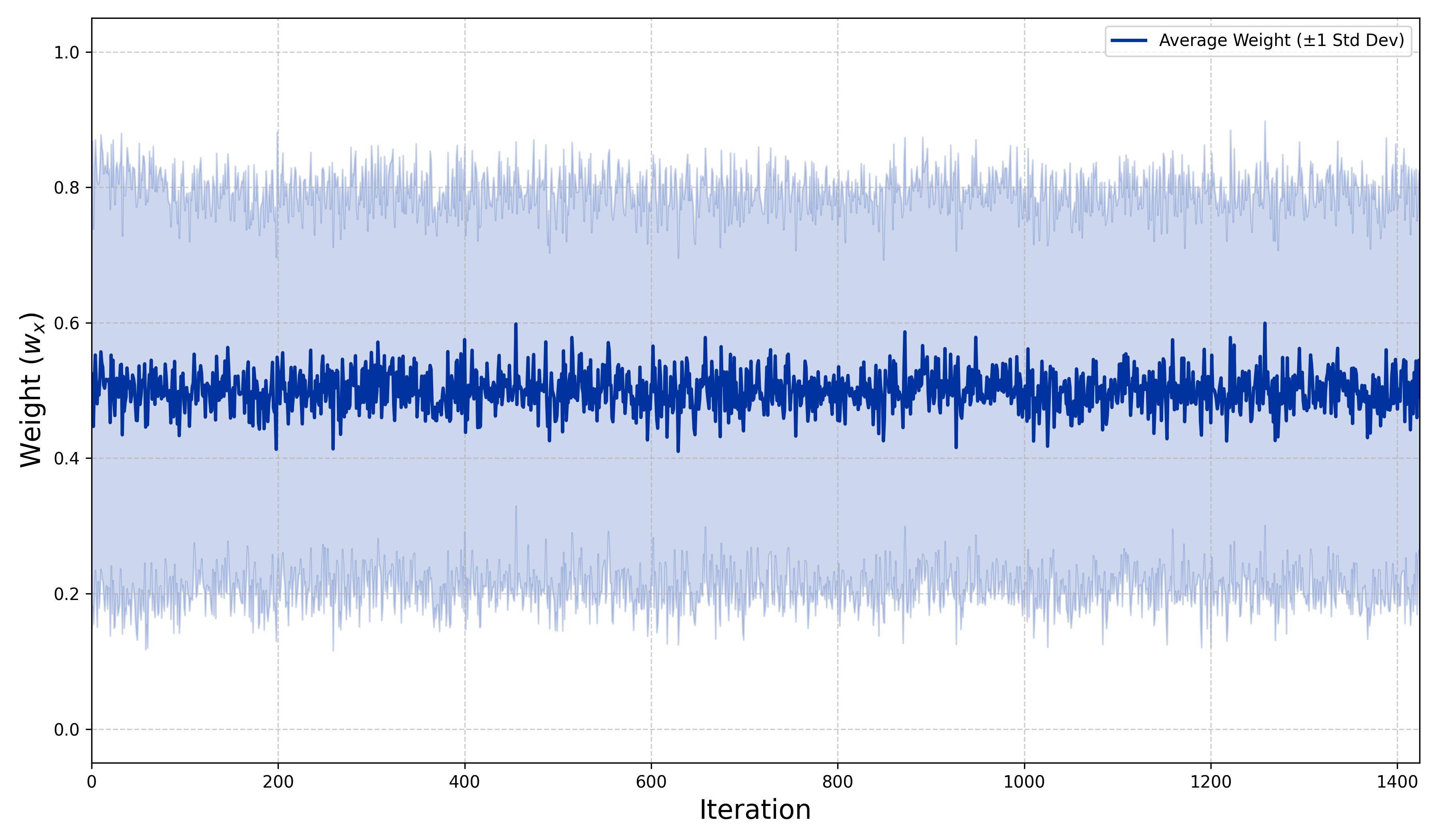}
        \caption{Three-Regime Policy Average Weights}
        \label{fig:Trend_3Regime}
    \end{subfigure}
    \caption{Average weight trend for WiGS-SAC on the Two-Regime and Three-Regime dataset. The solid line shows the mean weight across 100 seeds, while the shaded region represents $\pm 1$ standard deviation. The persistent high variance demonstrates that the agent remains reactive throughout the learning process, rather than converging to a static heuristic.}
    \label{fig:WeightTrends}
\end{figure}

\clearpage
\begin{landscape}
\subsection{Wilcoxon Signed-Ranked Signed Tests}
\label{subsec:WRSTResults}
Table \ref{tab:WRST_2Regime} and \ref{tab:WRST_3Regime} present the 
p-values from the pairwise Wilcoxon signed-rank tests for the two 
synthetic datasets. This test was run on the paired vectors of 
average Full-Pool RMSEs from the 100 simulation seeds. A p-value 
below 0.05 indicates a statistically significant difference in 
performance between the two strategies.

\begin{table}[htbp]
    \centering    
    \begin{subtable}[t]{\linewidth}
        \centering
        \resizebox{\linewidth}{!}{%
            \begin{tabular}{lllllllllllllll}
\toprule
 & Passive & GSx & GSy & iGS & WiGS-S (0.25) & WiGS-S (0.75) & WiGS-Lin & WiGS-Exp & WiGS-MAB (c=2.0) & WiGS-SAC & QBC & Uncertainty Sampling & EGAL & EMCM \\
\midrule
Passive & $1.000$ &  &  &  &  &  &  &  &  &  &  &  &  &  \\
GSx & $<0.001$ & $1.000$ &  &  &  &  &  &  &  &  &  &  &  &  \\
GSy & $0.094$ & $<0.001$ & $1.000$ &  &  &  &  &  &  &  &  &  &  &  \\
iGS & $<0.001$ & $<0.001$ & $<0.001$ & $1.000$ &  &  &  &  &  &  &  &  &  &  \\
WiGS-S (0.25) & $<0.001$ & $<0.001$ & $<0.001$ & $<0.001$ & $1.000$ &  &  &  &  &  &  &  &  &  \\
WiGS-S (0.75) & $<0.001$ & $<0.001$ & $<0.001$ & $<0.001$ & $<0.001$ & $1.000$ &  &  &  &  &  &  &  &  \\
WiGS-Lin & $<0.001$ & $<0.001$ & $<0.001$ & $<0.001$ & $<0.001$ & $<0.001$ & $1.000$ &  &  &  &  &  &  &  \\
WiGS-Exp & $<0.001$ & $<0.001$ & $<0.001$ & $<0.001$ & $0.616$ & $<0.001$ & $<0.001$ & $1.000$ &  &  &  &  &  &  \\
WiGS-MAB (c=2.0) & $<0.001$ & $<0.001$ & $<0.001$ & $<0.001$ & $<0.001$ & $<0.001$ & $<0.001$ & $<0.001$ & $1.000$ &  &  &  &  &  \\
WiGS-SAC & $<0.001$ & $<0.001$ & $<0.001$ & $<0.001$ & $<0.001$ & $<0.001$ & $<0.001$ & $<0.001$ & $<0.001$ & $1.000$ &  &  &  &  \\
QBC & $<0.001$ & $<0.001$ & $<0.001$ & $<0.001$ & $<0.001$ & $<0.001$ & $<0.001$ & $<0.001$ & $<0.001$ & $<0.001$ & $1.000$ &  &  &  \\
Uncertainty Sampling & $<0.001$ & $<0.001$ & $<0.001$ & $<0.001$ & $<0.001$ & $<0.001$ & $<0.001$ & $<0.001$ & $<0.001$ & $<0.001$ & $<0.001$ & $1.000$ &  &  \\
EGAL & $<0.001$ & $<0.001$ & $<0.001$ & $<0.001$ & $<0.001$ & $<0.001$ & $<0.001$ & $<0.001$ & $<0.001$ & $<0.001$ & $<0.001$ & $0.003$ & $1.000$ &  \\
EMCM & $<0.001$ & $<0.001$ & $<0.001$ & $<0.001$ & $<0.001$ & $<0.001$ & $<0.001$ & $<0.001$ & $<0.001$ & $<0.001$ & $<0.001$ & $1.000$ & $0.003$ & $1.000$ \\
\bottomrule
\end{tabular}

        }
        \caption{P-values for the Two-Regime Dataset.}
        \label{tab:WRST_2Regime}
    \end{subtable}
    \begin{subtable}[t]{\linewidth}
        \centering
        \resizebox{\linewidth}{!}{%
            \begin{tabular}{lllllllllllllll}
\toprule
 & Passive & GSx & GSy & iGS & WiGS-S (0.25) & WiGS-S (0.75) & WiGS-Lin & WiGS-Exp & WiGS-MAB (c=2.0) & WiGS-SAC & QBC & Uncertainty Sampling & EGAL & EMCM \\
\midrule
Passive & $1.000$ &  &  &  &  &  &  &  &  &  &  &  &  &  \\
GSx & $<0.001$ & $1.000$ &  &  &  &  &  &  &  &  &  &  &  &  \\
GSy & $0.417$ & $<0.001$ & $1.000$ &  &  &  &  &  &  &  &  &  &  &  \\
iGS & $<0.001$ & $<0.001$ & $<0.001$ & $1.000$ &  &  &  &  &  &  &  &  &  &  \\
WiGS-S (0.25) & $<0.001$ & $<0.001$ & $<0.001$ & $<0.001$ & $1.000$ &  &  &  &  &  &  &  &  &  \\
WiGS-S (0.75) & $<0.001$ & $<0.001$ & $<0.001$ & $<0.001$ & $<0.001$ & $1.000$ &  &  &  &  &  &  &  &  \\
WiGS-Lin & $<0.001$ & $<0.001$ & $<0.001$ & $<0.001$ & $<0.001$ & $<0.001$ & $1.000$ &  &  &  &  &  &  &  \\
WiGS-Exp & $<0.001$ & $<0.001$ & $<0.001$ & $<0.001$ & $0.002$ & $<0.001$ & $<0.001$ & $1.000$ &  &  &  &  &  &  \\
WiGS-MAB (c=2.0) & $<0.001$ & $<0.001$ & $<0.001$ & $<0.001$ & $<0.001$ & $<0.001$ & $0.386$ & $<0.001$ & $1.000$ &  &  &  &  &  \\
WiGS-SAC & $<0.001$ & $<0.001$ & $<0.001$ & $<0.001$ & $<0.001$ & $<0.001$ & $0.038$ & $<0.001$ & $<0.001$ & $1.000$ &  &  &  &  \\
QBC & $<0.001$ & $<0.001$ & $<0.001$ & $<0.001$ & $<0.001$ & $<0.001$ & $<0.001$ & $<0.001$ & $<0.001$ & $<0.001$ & $1.000$ &  &  &  \\
Uncertainty Sampling & $<0.001$ & $<0.001$ & $<0.001$ & $<0.001$ & $0.379$ & $<0.001$ & $<0.001$ & $<0.001$ & $<0.001$ & $<0.001$ & $<0.001$ & $1.000$ &  &  \\
EGAL & $<0.001$ & $<0.001$ & $<0.001$ & $<0.001$ & $<0.001$ & $<0.001$ & $<0.001$ & $<0.001$ & $<0.001$ & $<0.001$ & $<0.001$ & $<0.001$ & $1.000$ &  \\
EMCM & $<0.001$ & $<0.001$ & $<0.001$ & $<0.001$ & $0.379$ & $<0.001$ & $<0.001$ & $<0.001$ & $<0.001$ & $<0.001$ & $<0.001$ & $1.000$ & $<0.001$ & $1.000$ \\
\bottomrule
\end{tabular}

        }
        \caption{P-values for the Three-Regime Dataset}
        \label{tab:WRST_3Regime}
    \end{subtable}
    \caption{Statistical significance of performance differences on synthetic datasets. Tables display p-values from pairwise Wilcoxon Signed-Rank Tests comparing Full-Pool RMSE traces across 100 independent trials. A p-value $< 0.05$ indicates a statistically significant difference in performance distributions.}
\end{table}
\end{landscape}
\clearpage

\clearpage
\section{Additional Benchmark Experiment Details}
\subsection{Dataset Table}
\begin{table}[htbp]
    \centering
    \resizebox{\columnwidth}{!}{%
    \begin{tabular}{lllll}
        \toprule
        \textbf{no.} & \textbf{Dataset} & \textbf{Source} & \textbf{Size} & \textbf{Features} \\ \midrule
        1 & AutoMPG & UCI ML Repository & 392 & 8 \\
        2 & Beer & Kaggle & 365 & 5 \\
        3 & Body Fat & Kaggle & 252 & 13 \\
        4 & Burbidge - Correct & \citet{Burbridge} & 1000 & 1 \\
        5 & Burbidge - Misspecified & \citet{Burbridge} & 1000 & 1 \\
        6 & Burbidge - Low Noise & \citet{Burbridge} & 1000 & 1 \\
        7 & Concrete & UCI ML Repository & 1030 & 8 \\
        8 & Concrete - CS & UCI ML Repository & 103 & 7 \\
        9 & Concrete - Flow & UCI ML Repository & 103 & 7 \\
        10 & Concrete - Slump & UCI ML Repository & 103 & 7 \\
        11 & CPS & CMU Stat Lib & 534 & 16 \\
        12 & Housing & UCI ML Repository & 506 & 13 \\
        13 & NO2 & CMU Stat Lib & 500 & 7 \\
        14 & PM10 & CMU Stat Lib & 500 & 7 \\
        15 & QSAR & UCI ML Repository & 546 & 8 \\
        16 & Wine - Red & UCI ML Repository & 1599 & 11 \\
        17 & Wine - White & UCI ML Repository & 4898 & 11 \\
        18 & Yacht & UCI ML Repository & 308 & 6 \\
        \bottomrule
    \end{tabular}
    }
    \caption{Benchmark datasets used in the experiments.}
    \label{tab:DatasetsTable}
\end{table}

\subsection{Experimental Results Trace Plots}
\label{sec:MainResultsTracePlots}
\begin{figure*}
    \centering
    \vspace*{-1cm} 
    
    \begin{subfigure}[b]{0.31\textwidth}
        \centering
        \includegraphics[width=\linewidth, keepaspectratio]{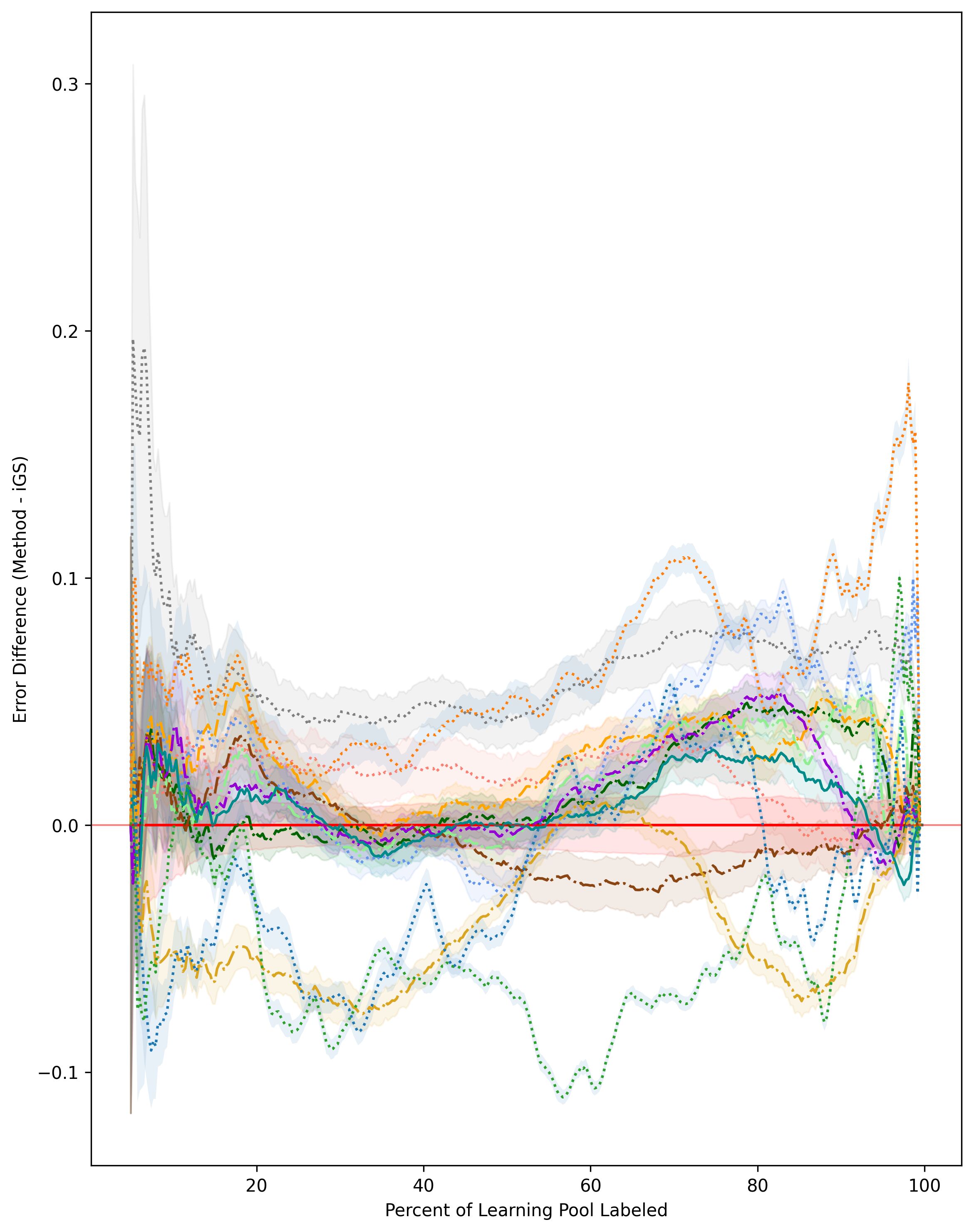}        
        \caption{beer}
    \end{subfigure}
    \hfill
    \begin{subfigure}[b]{0.31\textwidth}
        \centering
        \includegraphics[width=\linewidth, keepaspectratio]{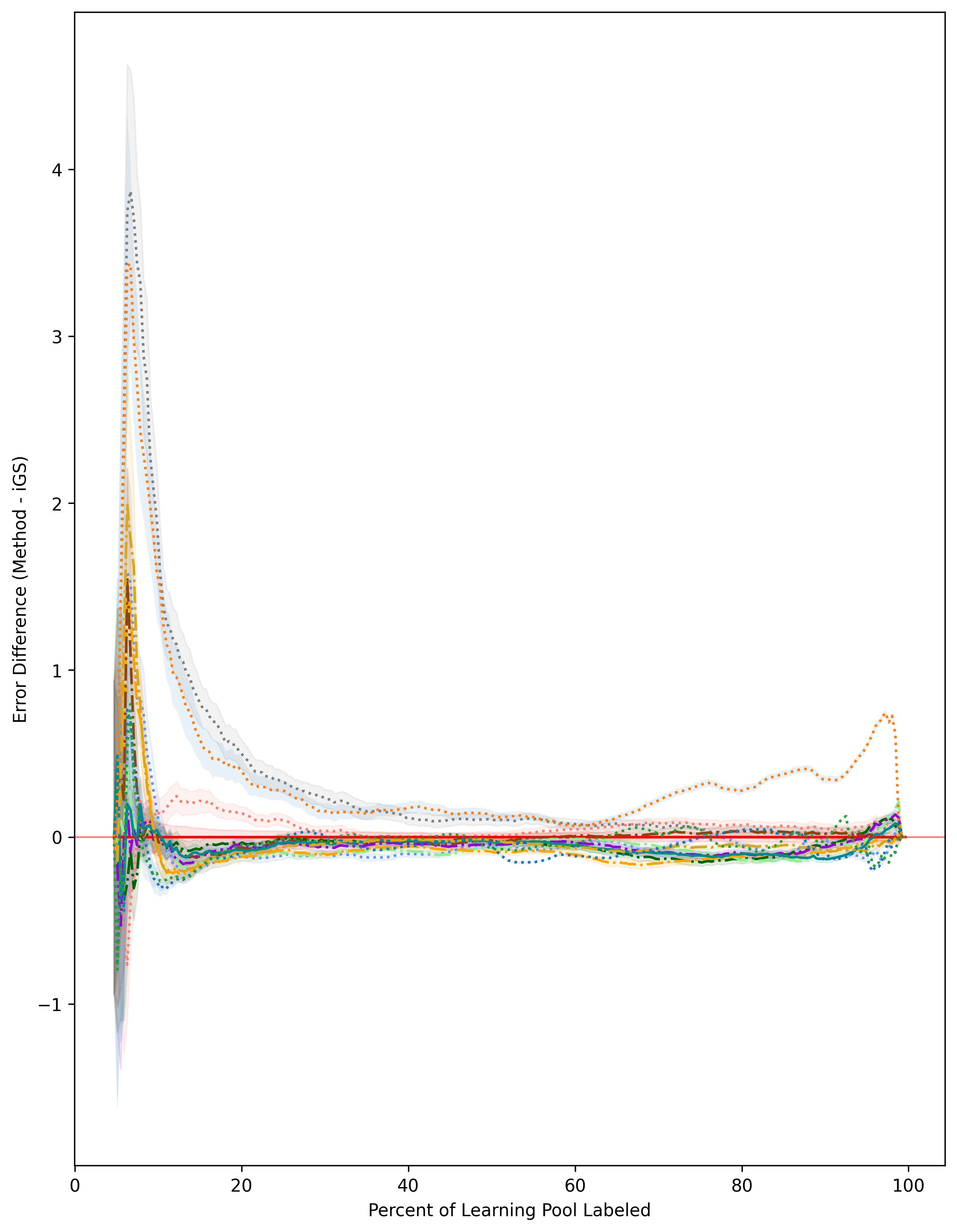}
        \caption{bodyfat}
    \end{subfigure}
    \hfill
    \begin{subfigure}[b]{0.31\textwidth}
        \centering
        \includegraphics[width=\linewidth, keepaspectratio]{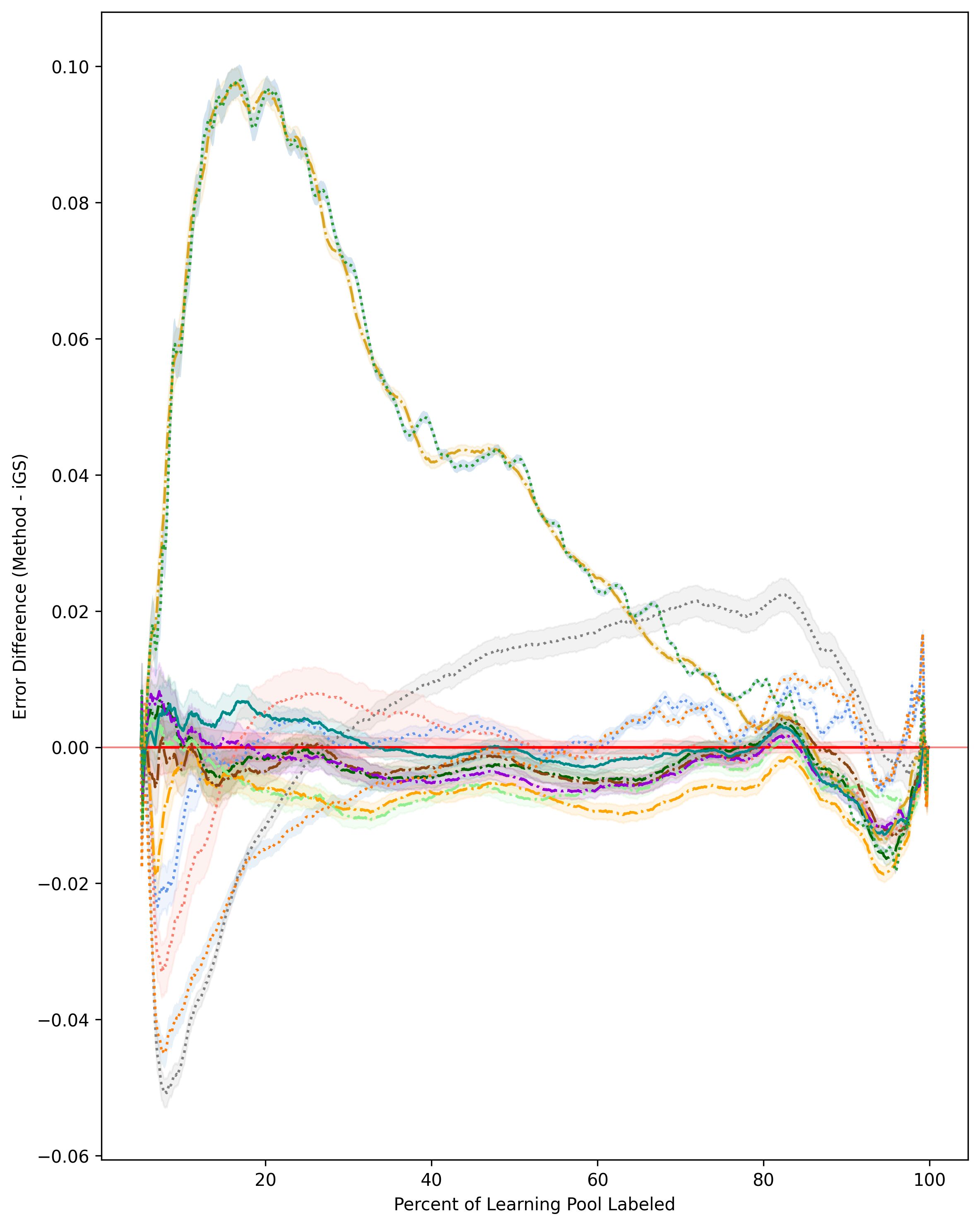}
        \caption{burbidge\_correct}
    \end{subfigure}
    
    \vspace{0.3em} 
    
    \begin{subfigure}[b]{0.31\textwidth}
        \centering
        \includegraphics[width=\linewidth, keepaspectratio]{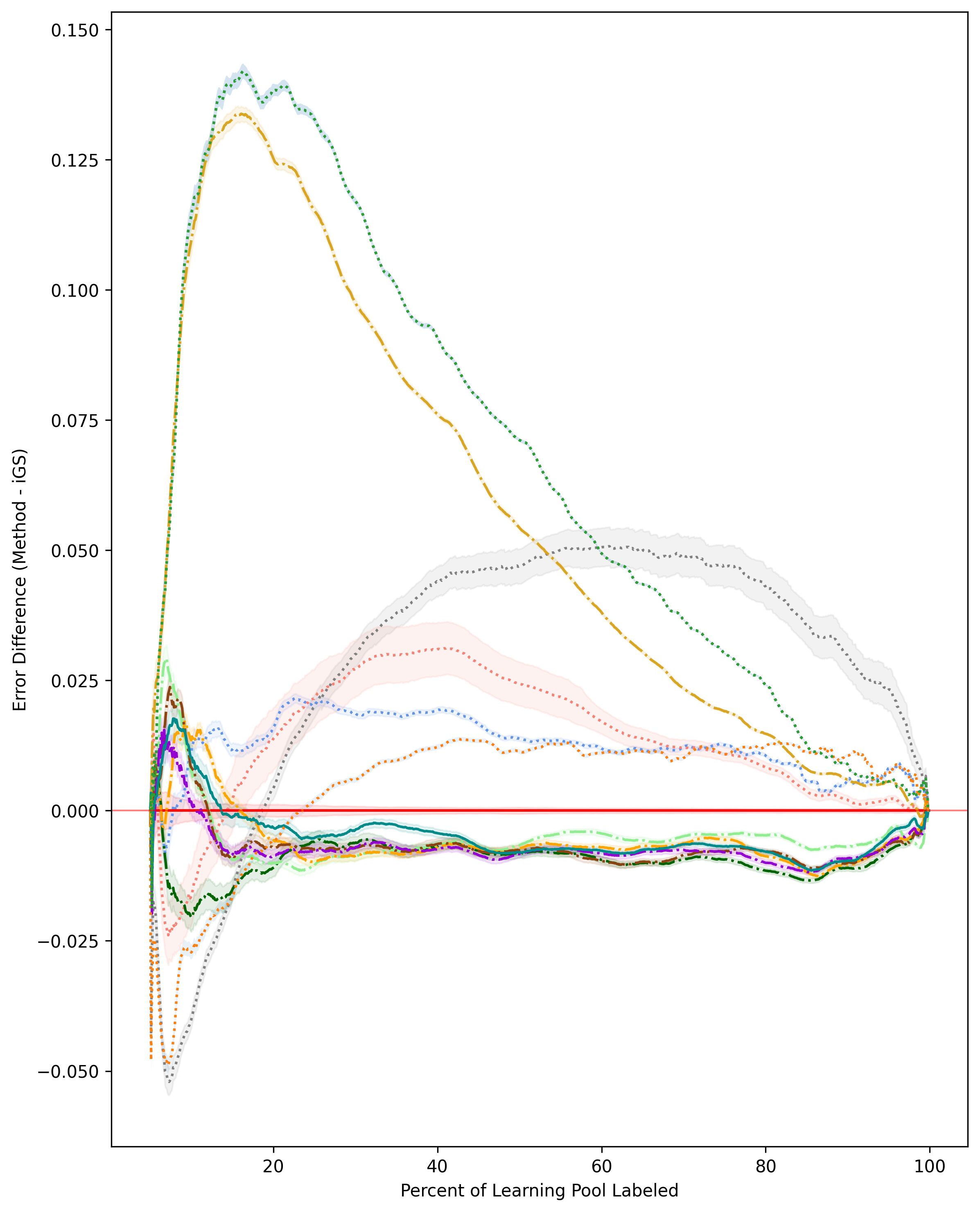}
        \caption{burbidge\_low\_noise}
    \end{subfigure}
    \hfill
    \begin{subfigure}[b]{0.31\textwidth}
        \centering
        \includegraphics[width=\linewidth, keepaspectratio]{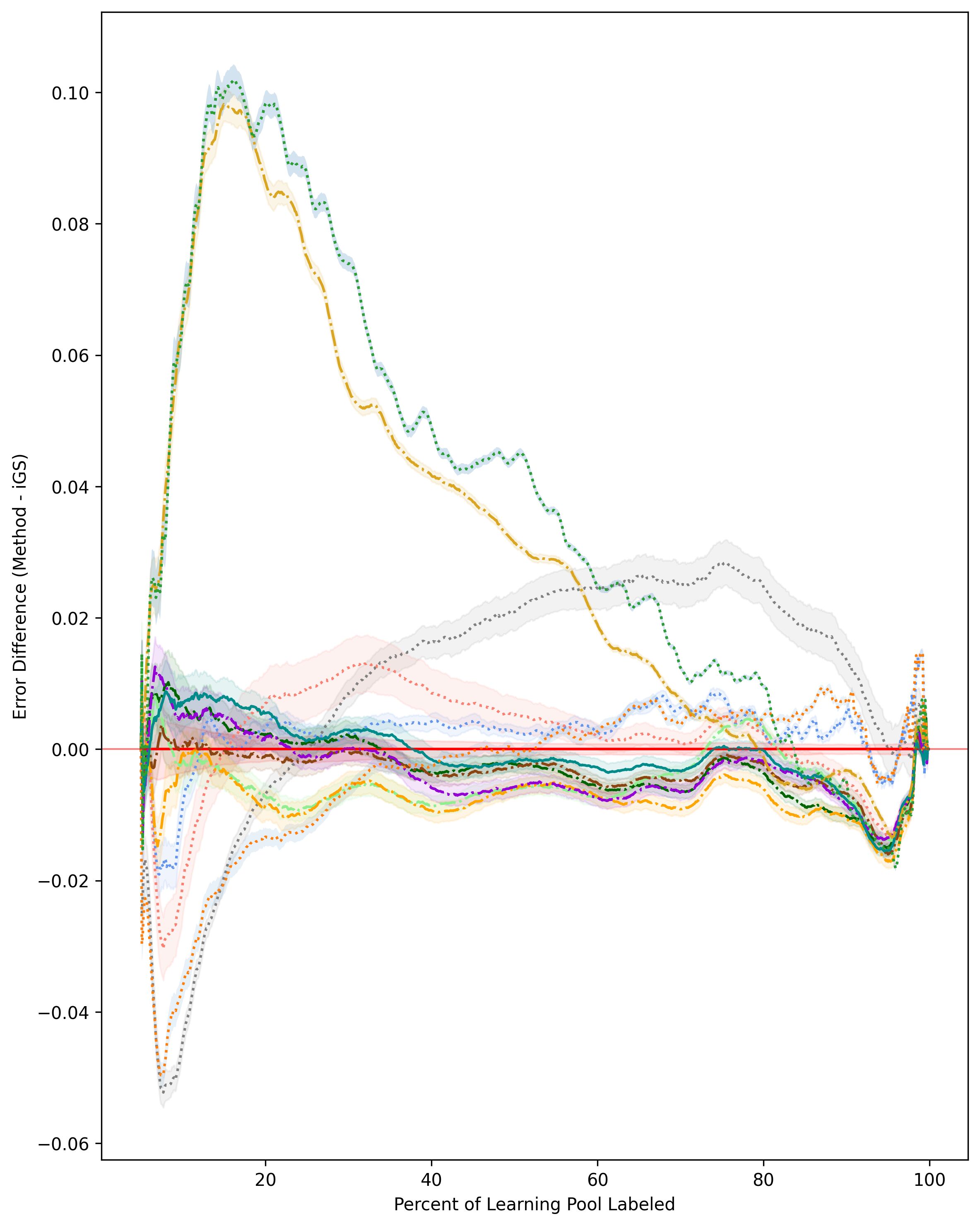}
        \caption{burbidge\_misspecified}
    \end{subfigure}
    \hfill
    \begin{subfigure}[b]{0.31\textwidth}
        \centering
        \includegraphics[width=\linewidth, keepaspectratio]{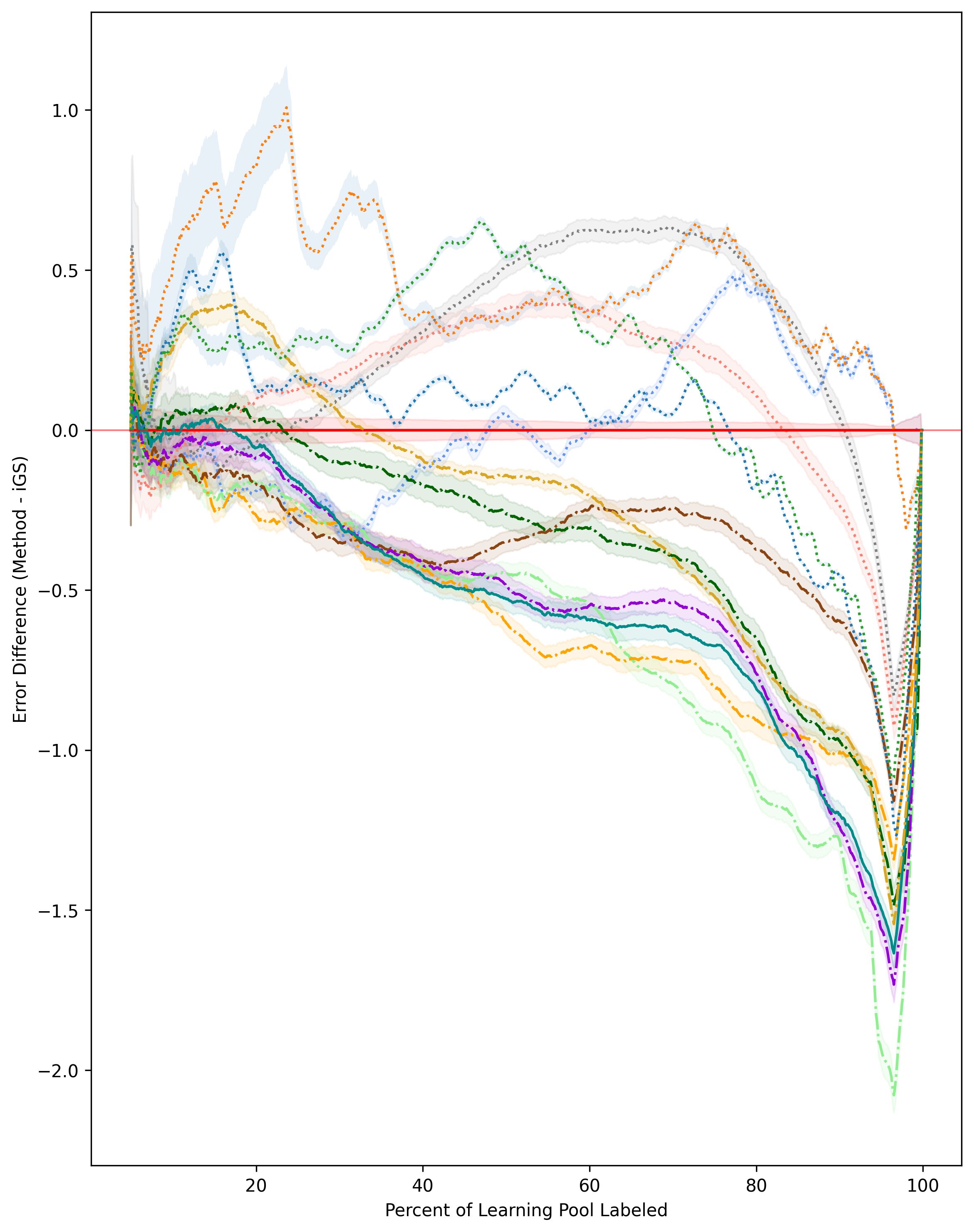}
        \caption{concrete\_4}
    \end{subfigure}
    
    \vspace{0.3em}
    
    \begin{subfigure}[b]{0.31\textwidth}
        \centering
        \includegraphics[width=\linewidth, keepaspectratio]{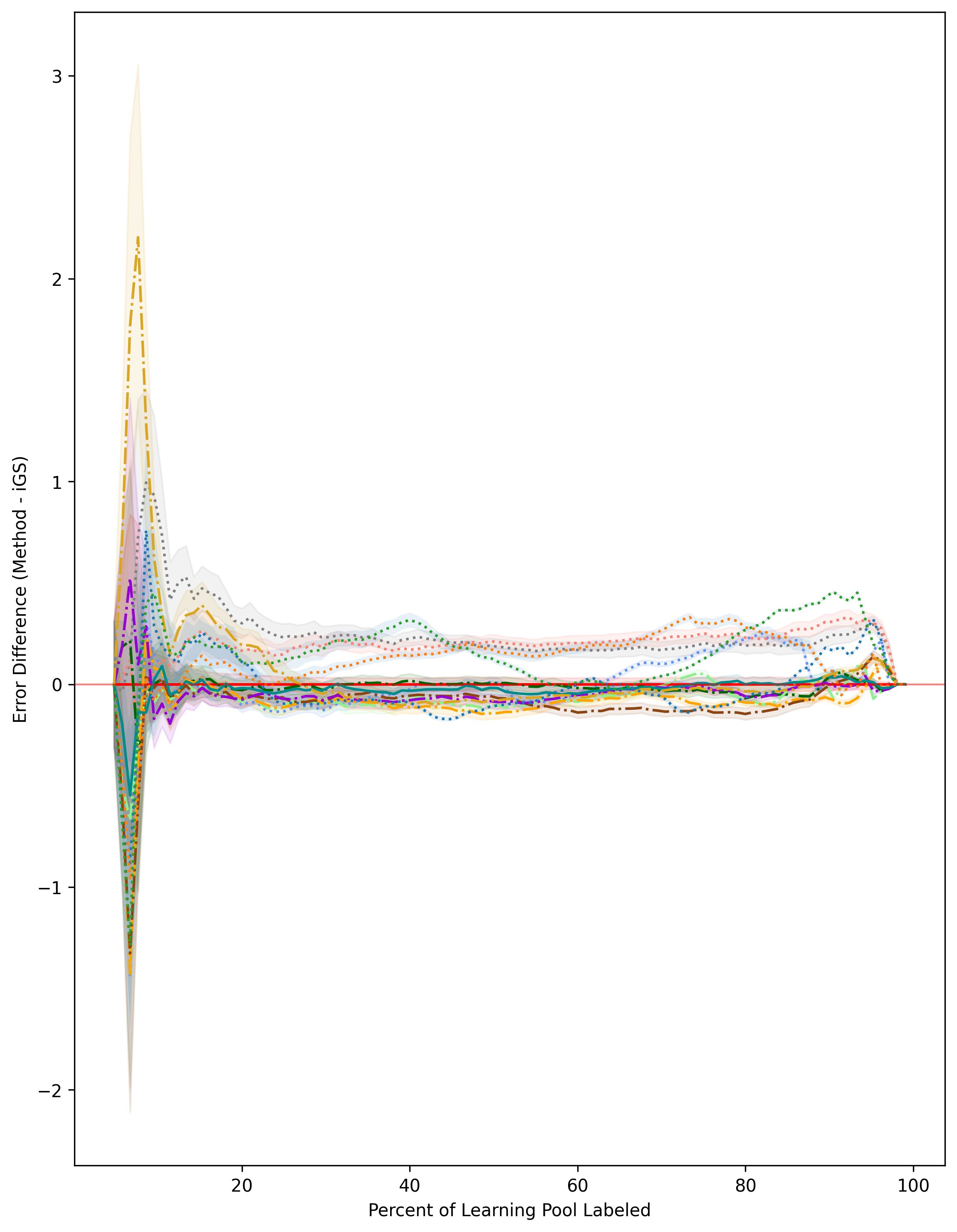}
        \caption{concrete\_cs}
    \end{subfigure}
    \hfill
    \begin{subfigure}[b]{0.31\textwidth}
        \centering
        \includegraphics[width=\linewidth, keepaspectratio]{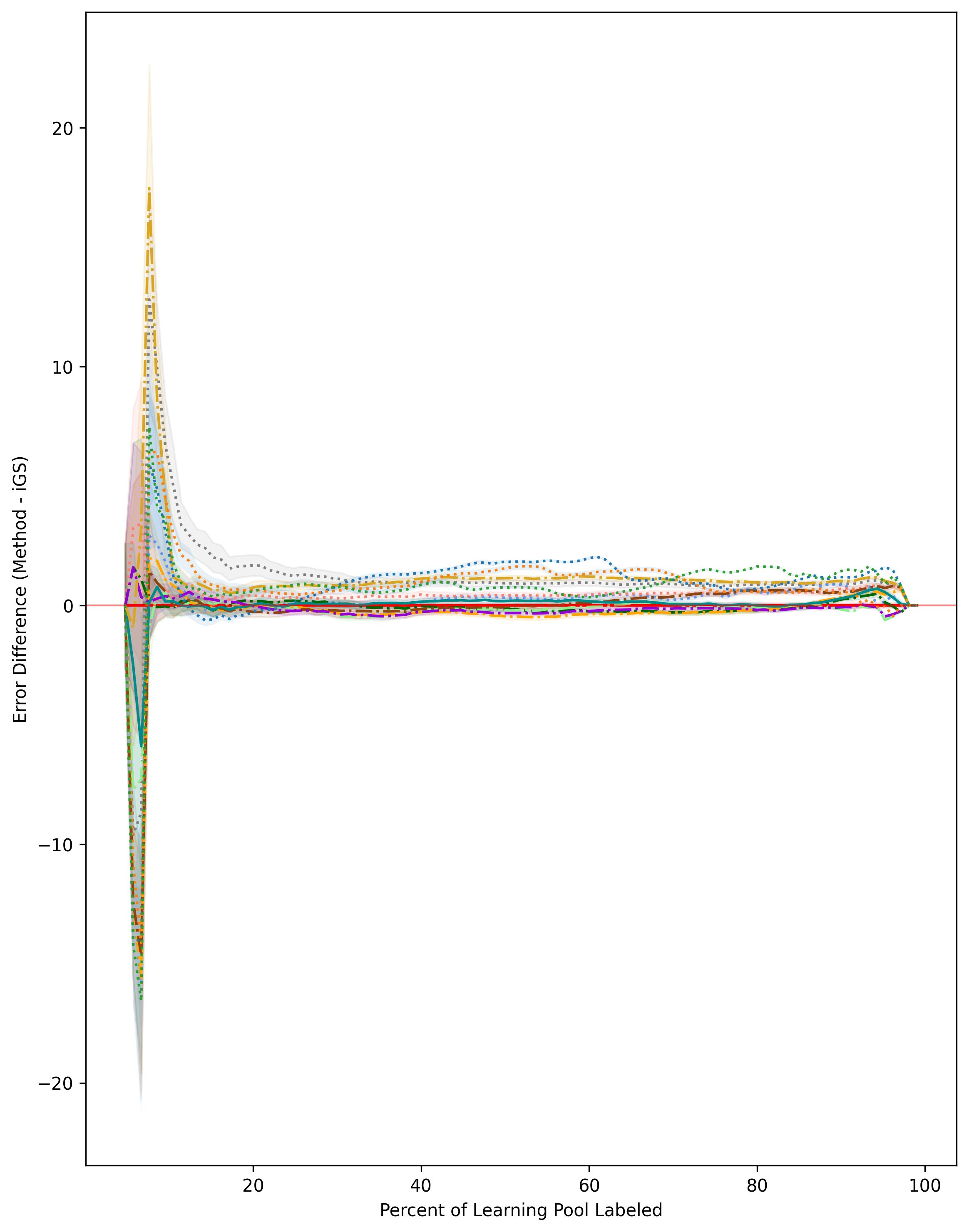}
        \caption{concrete\_flow}
    \end{subfigure}
    \hfill
    \begin{subfigure}[b]{0.31\textwidth}
        \centering
        \includegraphics[width=\linewidth, keepaspectratio]{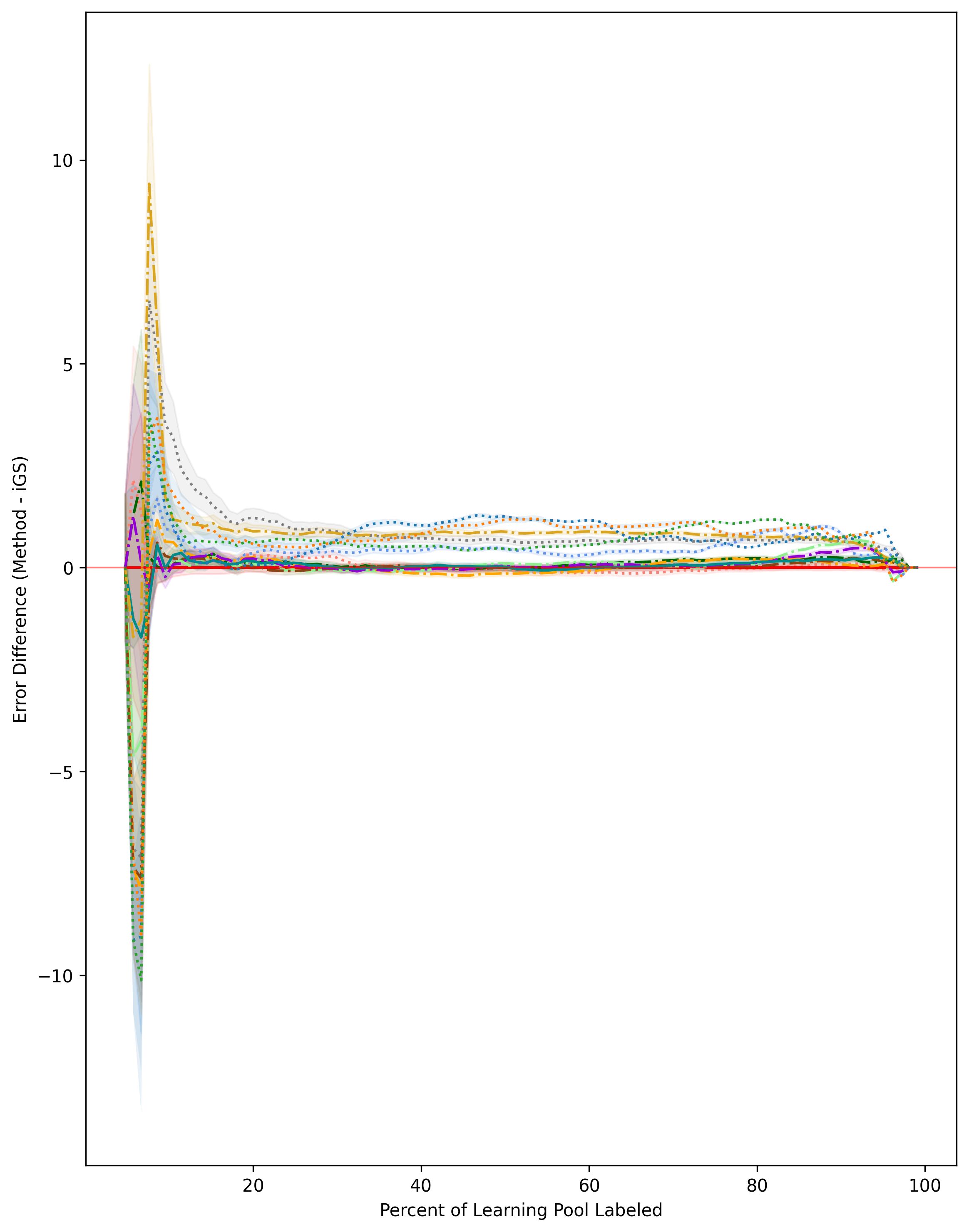}
        \caption{concrete\_slump}
    \end{subfigure}
    
    \vspace{0.5em}
    \centering
    \includegraphics[width=\linewidth]{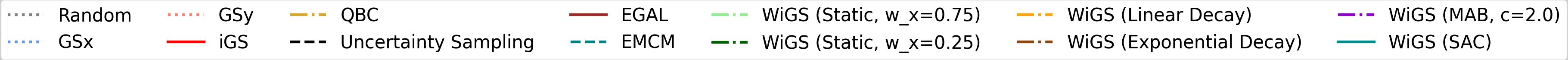}
    
    \caption{Full-pool RMSE trace plots for benchmark datasets (Part 1 of 2).}
    \label{fig:MainResults1}
\end{figure*}

\clearpage
\begin{figure*}
    \centering
    \vspace*{-1cm} 
    
    \begin{subfigure}[b]{0.31\textwidth}
        \centering
        \includegraphics[width=\linewidth, keepaspectratio]{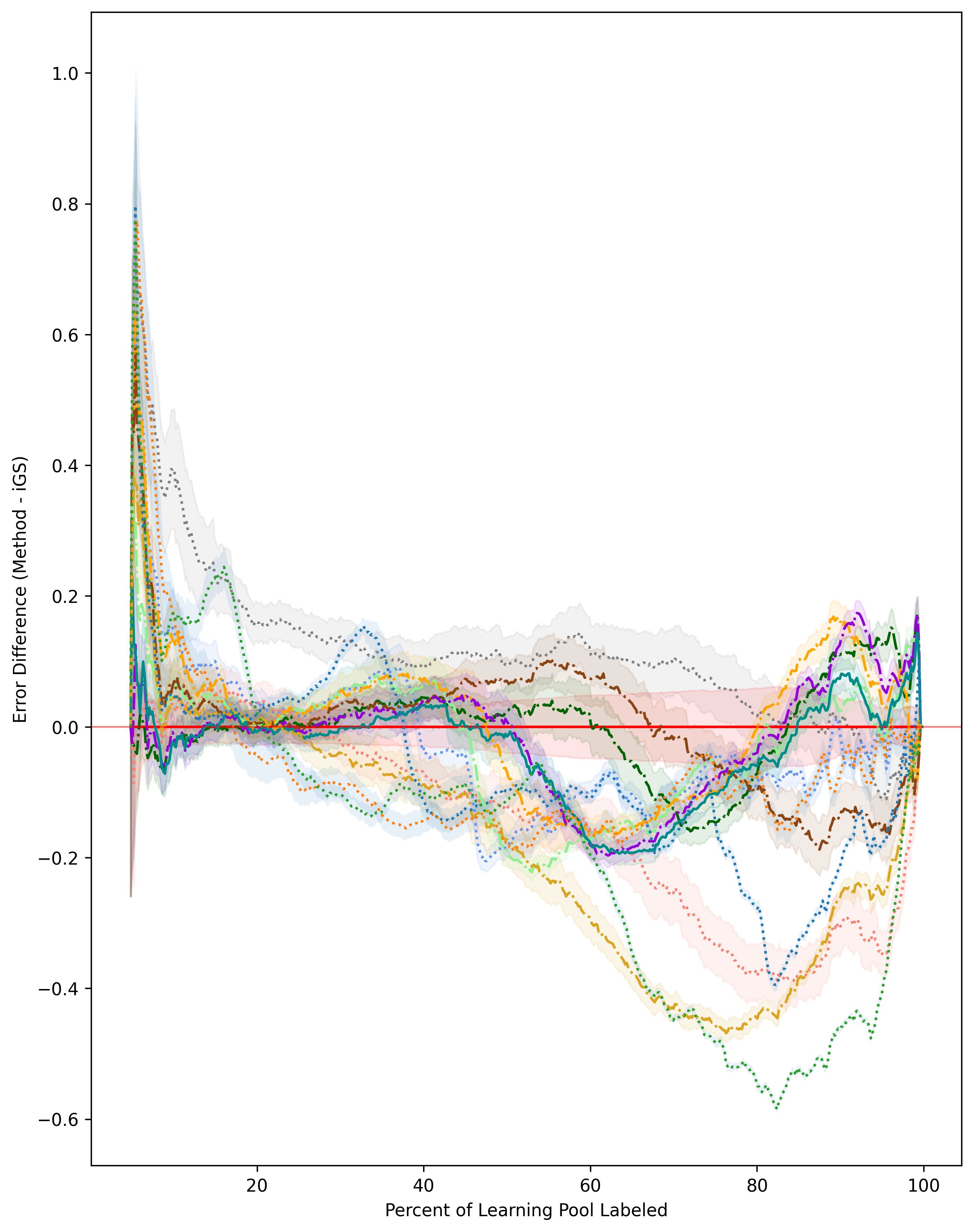}
        \caption{cps\_wage}    \end{subfigure}
    \hfill
    \begin{subfigure}[b]{0.31\textwidth}
        \centering
        \includegraphics[width=\linewidth, keepaspectratio]{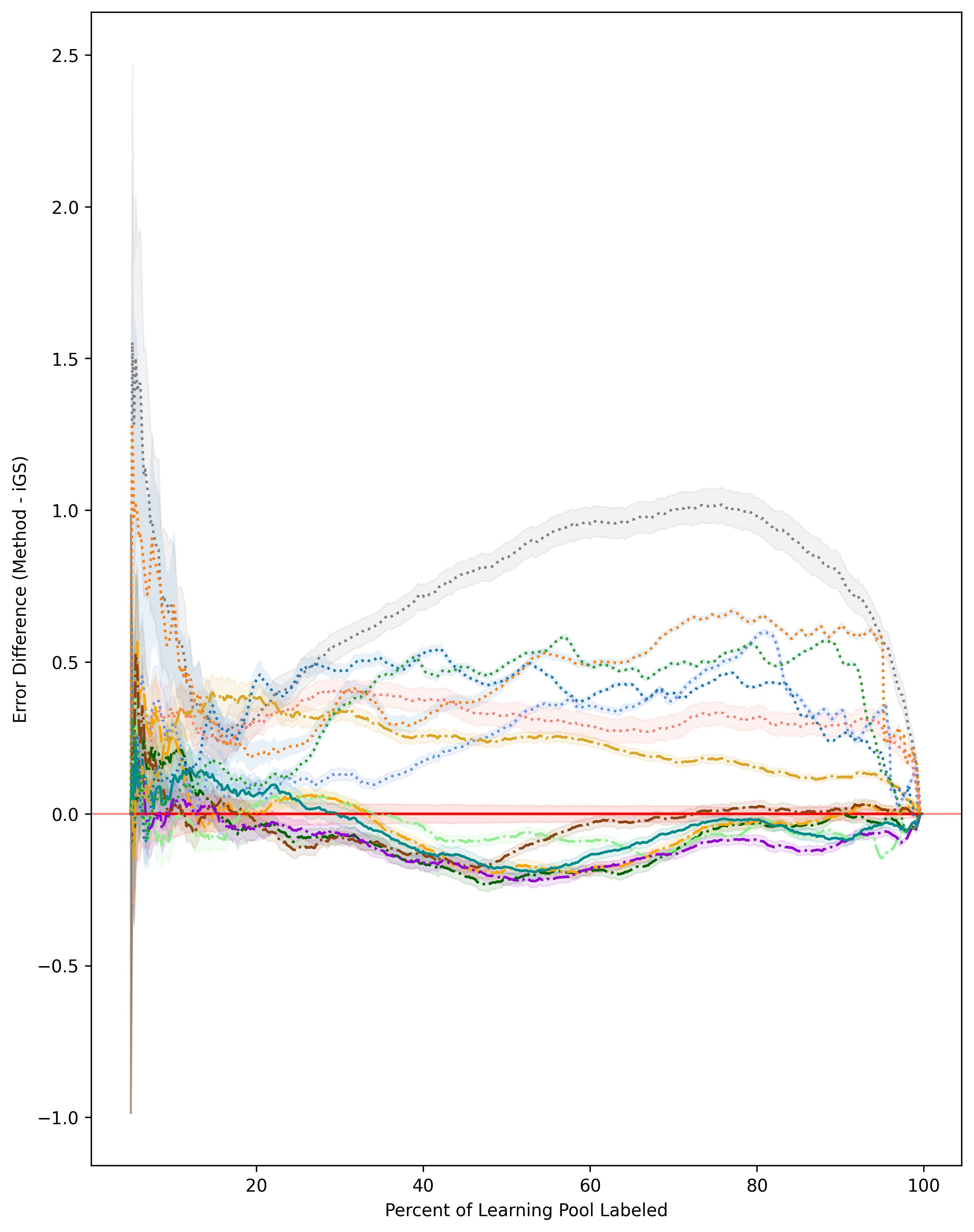}
        \caption{housing}
    \end{subfigure}
    \hfill
    \begin{subfigure}[b]{0.31\textwidth}
        \centering
        \includegraphics[width=\linewidth, keepaspectratio]{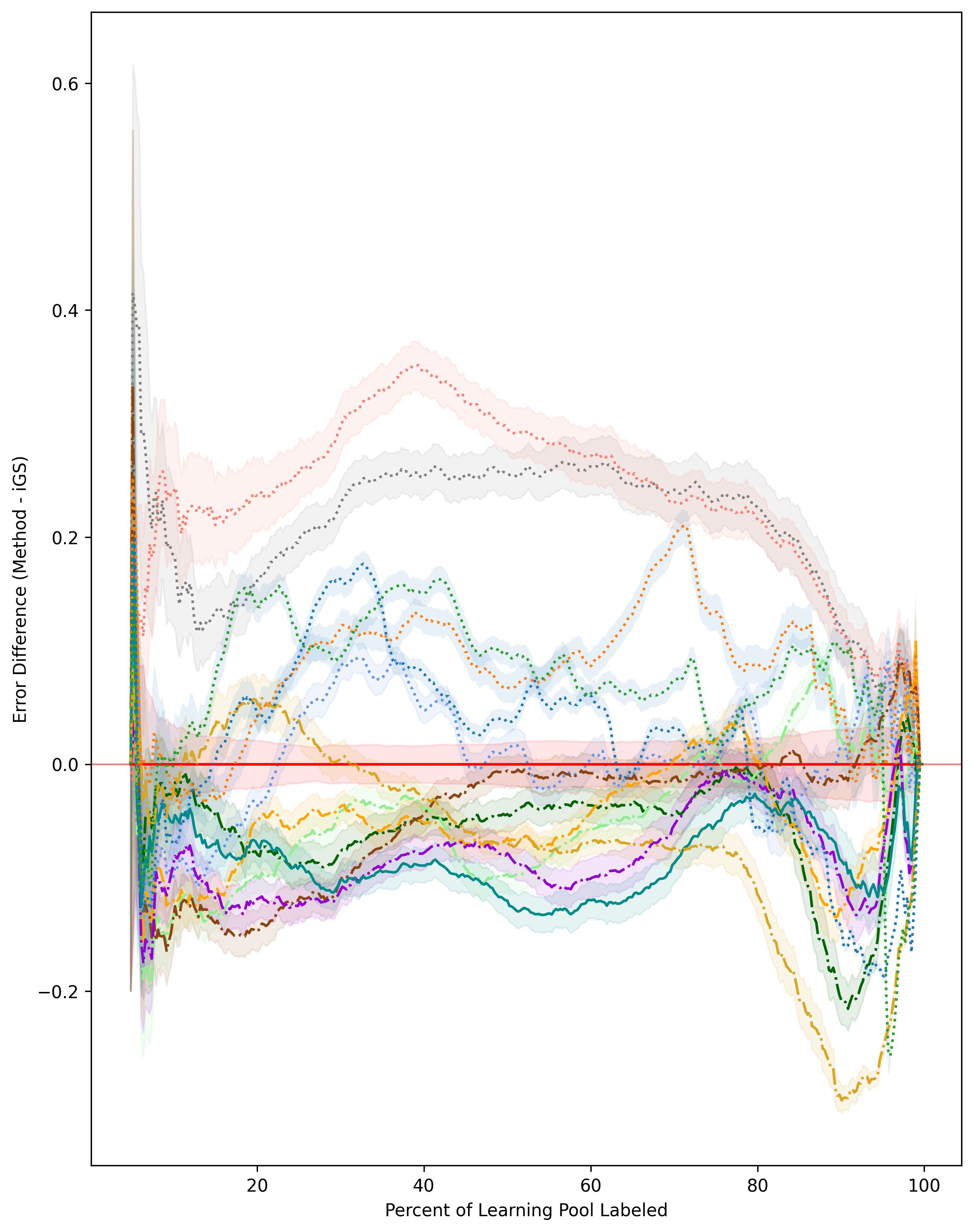}
        \caption{mpg}
    \end{subfigure}
    
    \vspace{0.3em}
    
    \begin{subfigure}[b]{0.31\textwidth}
        \centering
        \includegraphics[width=\linewidth, keepaspectratio]{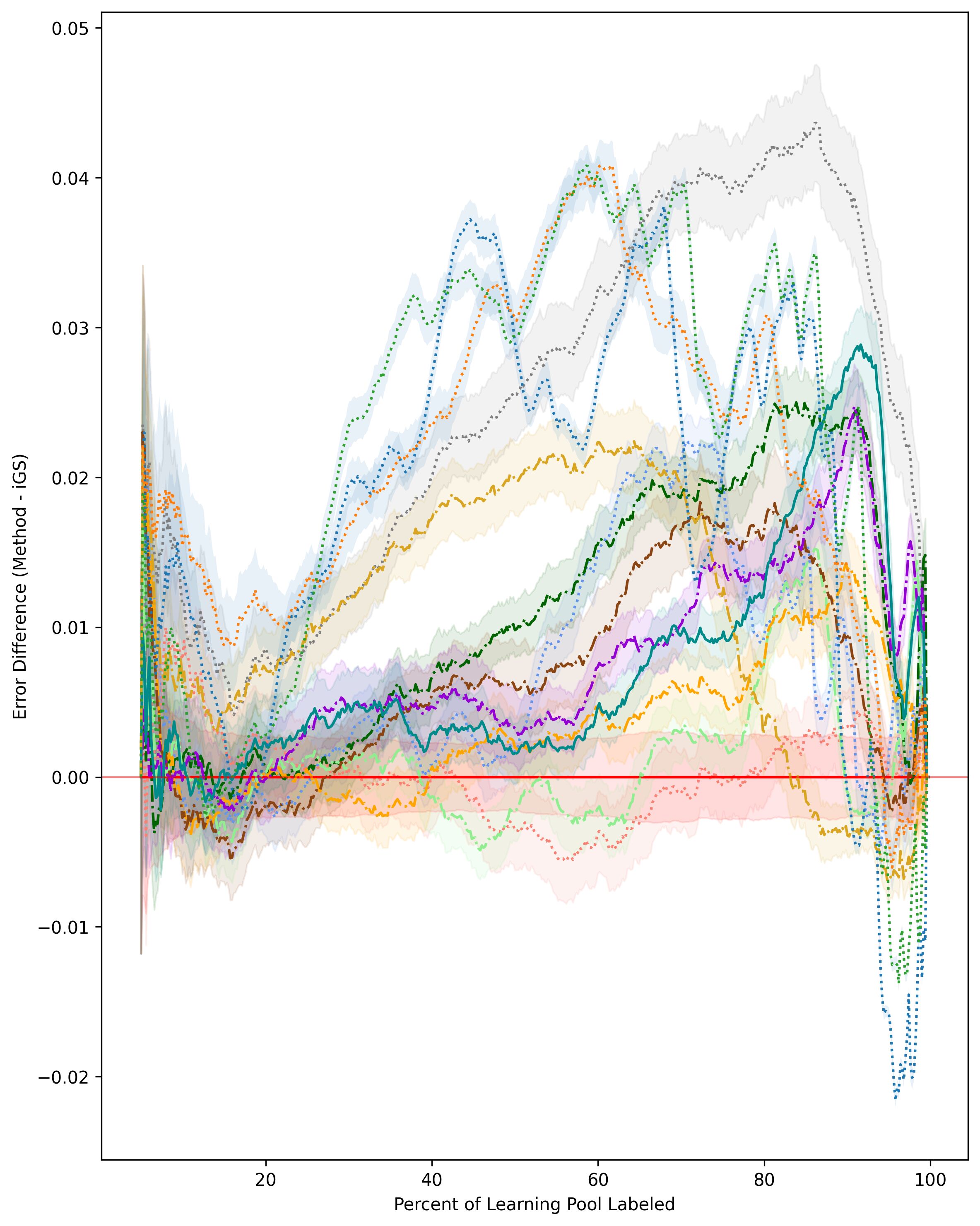}
        \caption{no2}
    \end{subfigure}
    \hfill
    \begin{subfigure}[b]{0.31\textwidth}
        \centering
        \includegraphics[width=\linewidth, keepaspectratio]{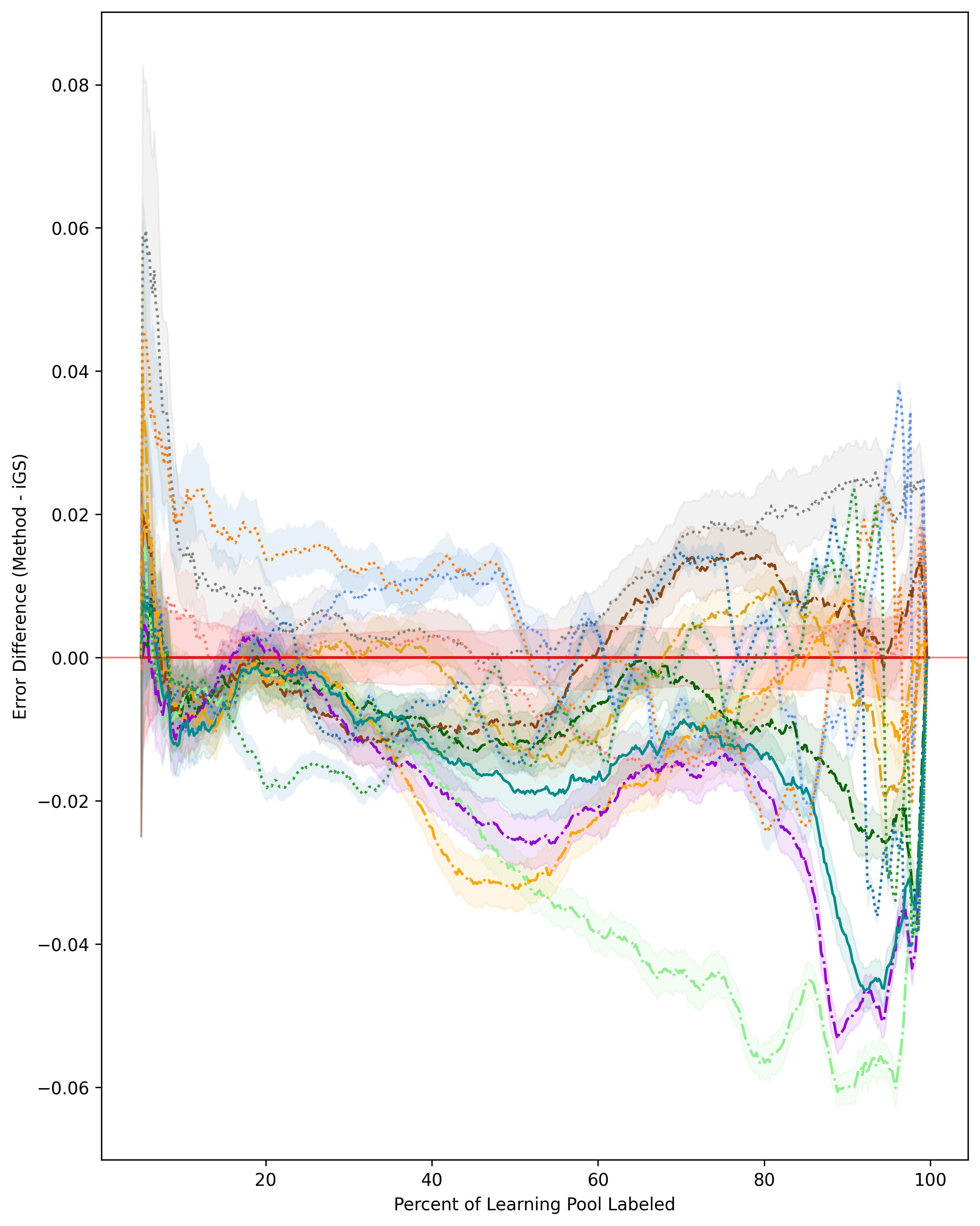}
        \caption{pm10}
    \end{subfigure}
    \hfill
    \begin{subfigure}[b]{0.31\textwidth}
        \centering
        \includegraphics[width=\linewidth, keepaspectratio]{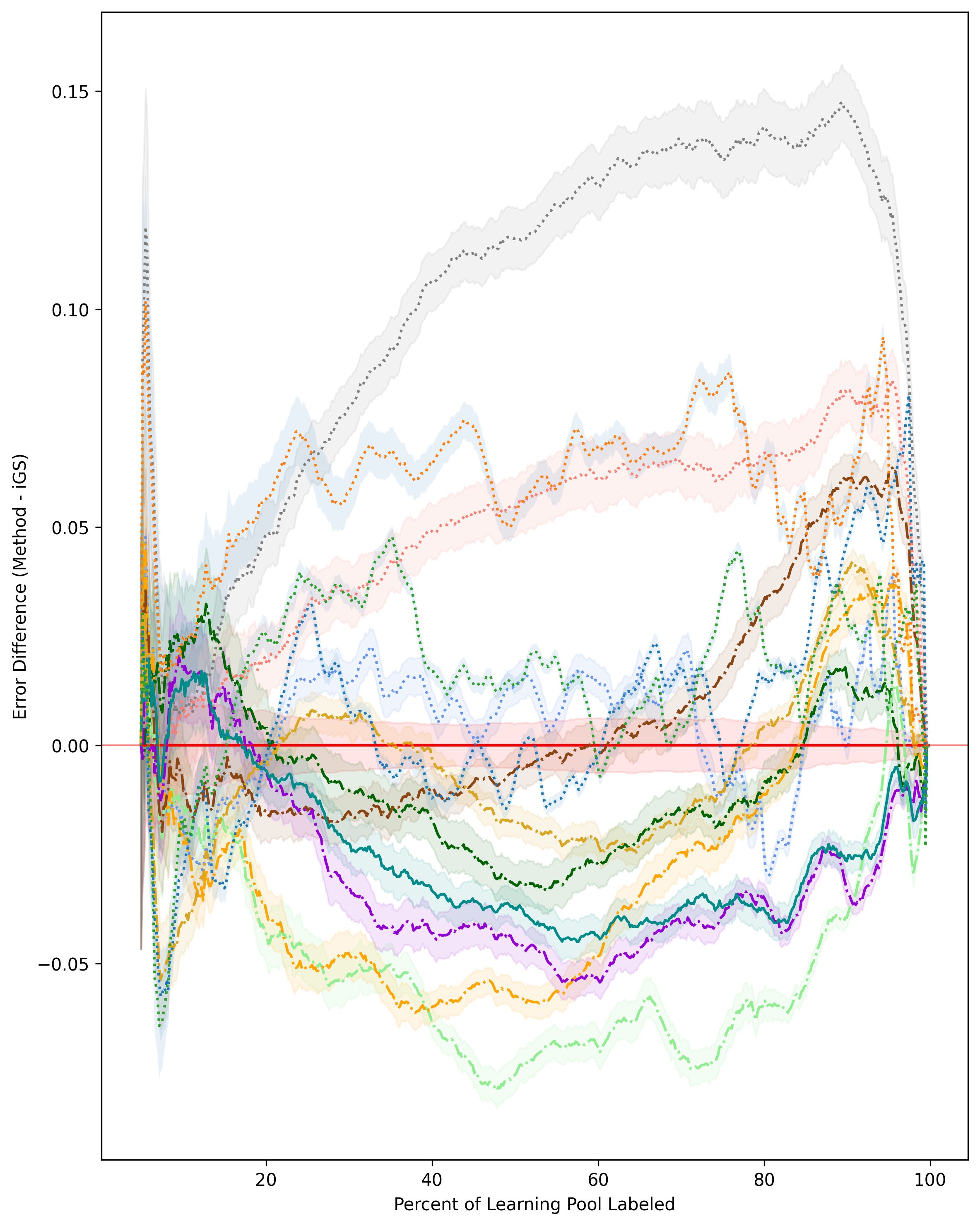}
        \caption{qsar}
    \end{subfigure}
    
    \vspace{0.3em}
    
    \begin{subfigure}[b]{0.31\textwidth}
        \centering
        \includegraphics[width=\linewidth, keepaspectratio]{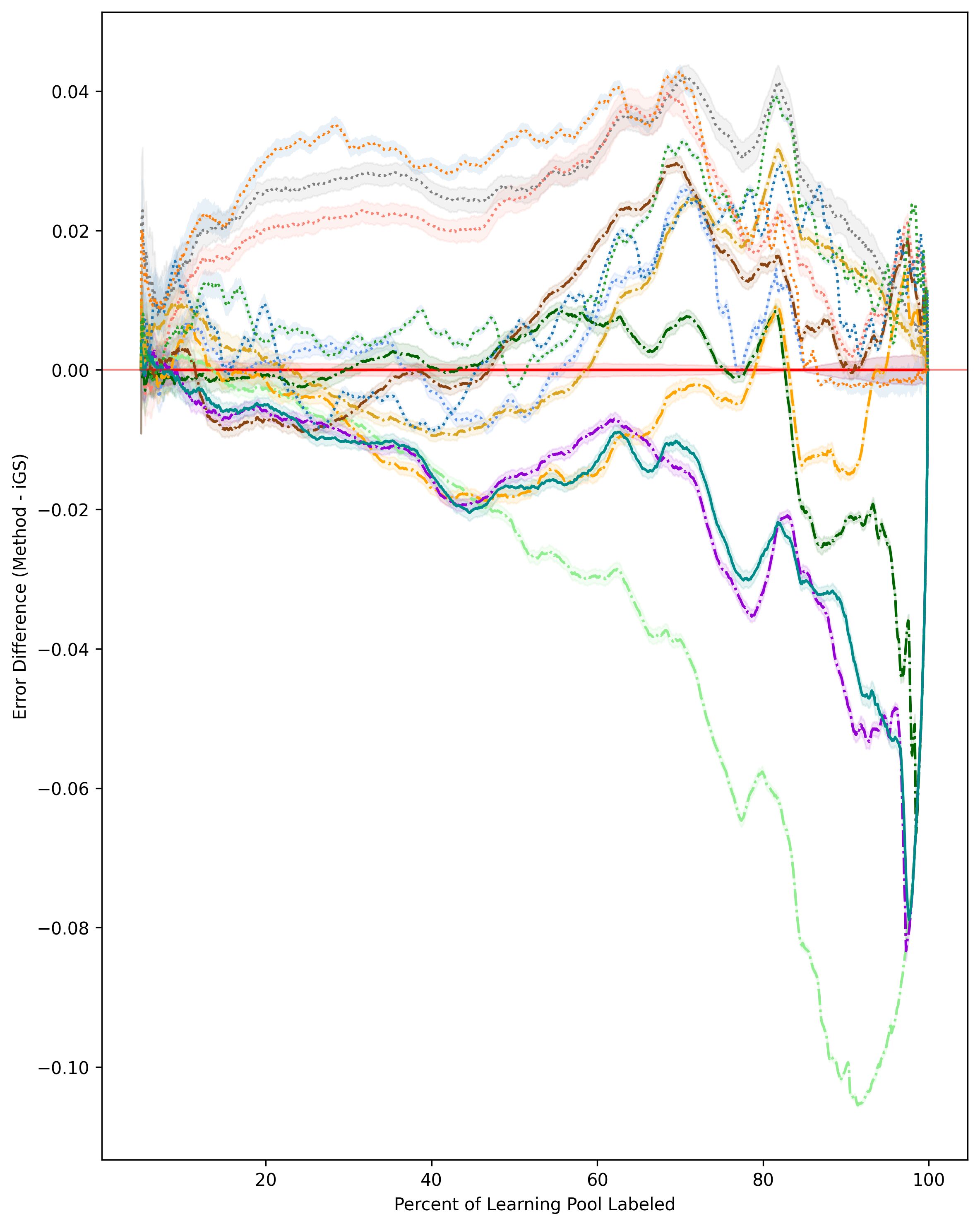}
        \caption{wine\_red}
    \end{subfigure}
    \hfill
    \begin{subfigure}[b]{0.31\textwidth}
        \centering
        \includegraphics[width=\linewidth, keepaspectratio]{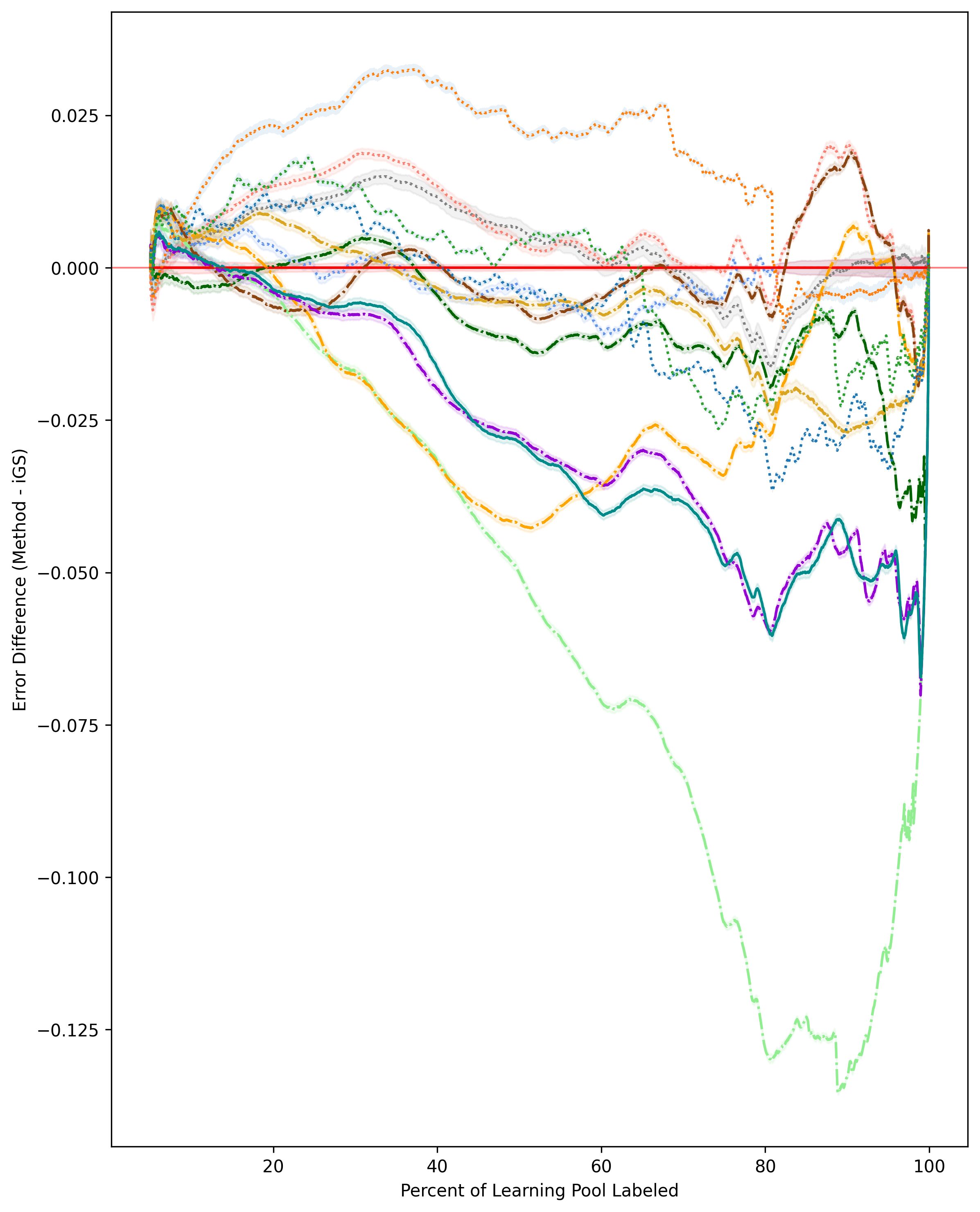}
        \caption{wine\_white}
    \end{subfigure}
    \hfill
    \begin{subfigure}[b]{0.31\textwidth}
        \centering
        \includegraphics[width=\linewidth, keepaspectratio]{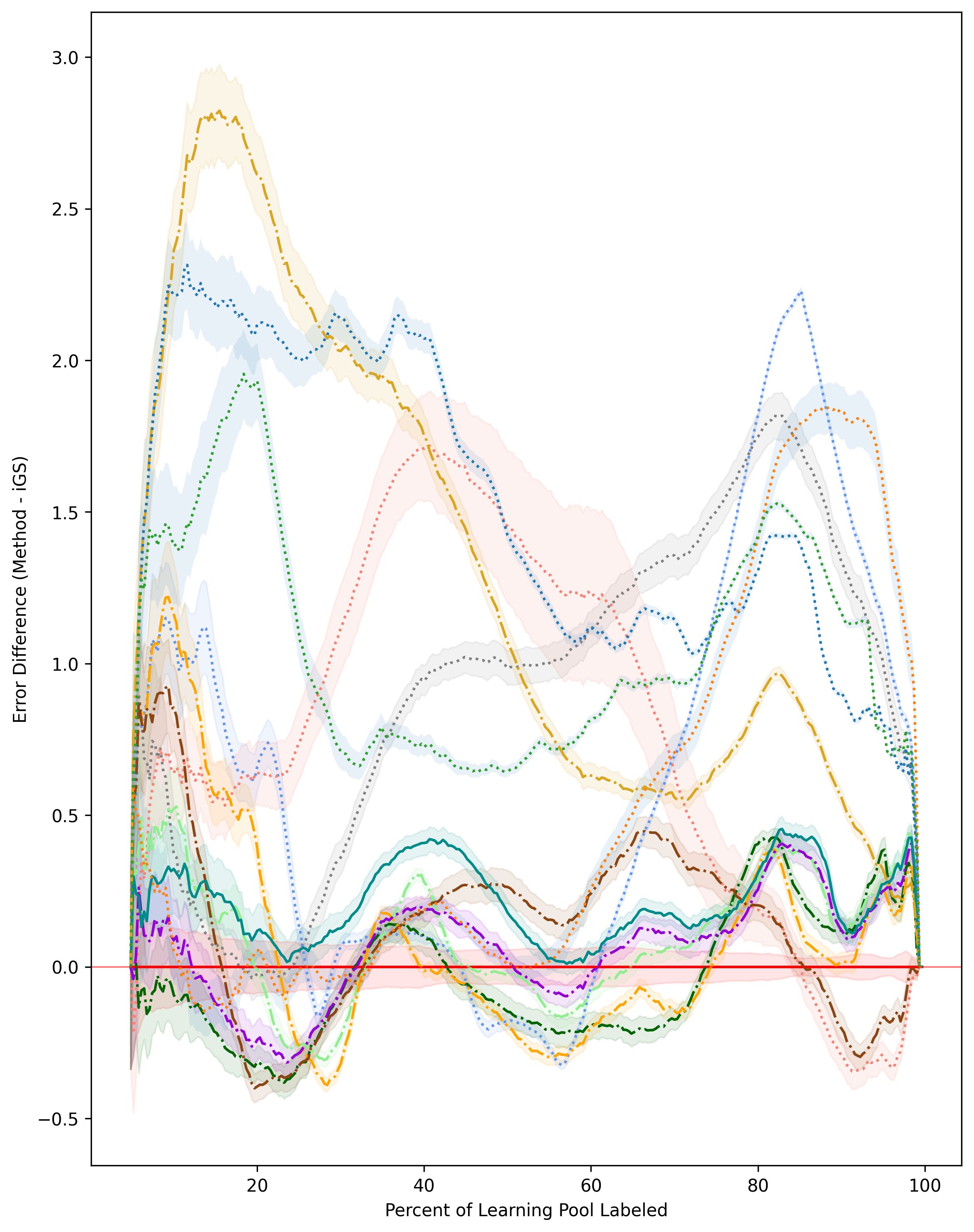}
        \caption{yacht}
    \end{subfigure}
    
    \vspace{0.5em}
    \centering
    \includegraphics[width=\linewidth]{upload_all_files/manuscript/benchmark_legend.jpg}
    
    \caption{Full-pool RMSE trace plots for benchmark datasets (Part 2 of 2).}
    \label{fig:MainResults2}
\end{figure*}
\clearpage

\section{Computational Complexity}
Table \ref{tab:RuntimeComparison} presents the median runtime across all simulations. As expected, the reinforcement learning-based WiGS-SAC incurs the highest computational cost (median $416.4$s), approximately $27\times$ that of the iGS baseline ($15.5$s). This overhead stems from the iterative training of the actor-critic neural networks. However, in practical active learning scenarios, such as material science experimentation or medical diagnosis, the cost of obtaining a single label (oracle) often spans hours or days. In such contexts, a computational delay of minutes is negligible compared to the substantial reduction in labeling budget and sample efficiency achieved by the agent. 

Conversely, the WiGS-MAB and Static variants demonstrate excellent computational efficiency, with runtimes nearly identical to the iGS baseline ($\approx 15.5$s). This highlights the flexibility of the WiGS framework: users can select the lightweight MAB or Static strategies for time-constrained applications, or deploy the deep RL agent when maximizing label efficiency is the main concern. Notably, EGAL exhibits poor scaling on larger datasets (e.g., \texttt{Wine White}), likely due to the complexity of its density calculations, whereas WiGS-SAC's complexity is driven primarily by the fixed-size neural network updates.

\begin{table*}[htbp]
    \centering
    \setlength{\tabcolsep}{2pt}
    \begin{tabular}{l rrrrrrrrrrrrrr}
        \toprule
         & \multicolumn{4}{c}{\textbf{Baselines}} & \multicolumn{4}{c}{\textbf{Advanced Baselines}} & \multicolumn{4}{c}{\textbf{WiGS (Static/Decay)}} & \multicolumn{2}{c}{\textbf{WiGS (Adaptive)}} \\
         \cmidrule(lr){2-5} \cmidrule(lr){6-9} \cmidrule(lr){10-13} \cmidrule(lr){14-15}
        \textbf{Dataset} & Rand. & GSx & GSy & iGS & QBC & Uncert. & EMCM & EGAL & 0.25 & 0.75 & Lin. & Exp. & MAB(2) & SAC \\ \midrule
        Beer & 11.1 & 11.5 & 11.8 & 12.1 & 17.1 & 11.3 & 10.9 & 11.9 & 12.5 & 12.5 & 12.4 & 12.2 & 12.0 & 225.0 \\
        Bodyfat & 6.8 & 7.0 & 7.2 & 7.3 & 10.4 & 6.9 & 6.9 & 7.1 & 7.3 & 7.3 & 7.3 & 7.3 & 7.2 & 205.8 \\
        Burb. Correct & 25.9 & 27.9 & 28.9 & 31.7 & 39.7 & 27.0 & 26.4 & 37.2 & 36.6 & 36.6 & 36.7 & 36.7 & 36.5 & 234.6 \\
        Burb. Low Noise & 25.3 & 27.1 & 28.1 & 31.3 & 39.0 & 26.2 & 25.3 & 34.2 & 36.2 & 36.4 & 36.4 & 36.3 & 35.7 & 448.5 \\
        Burb. Misspecified & 25.0 & 26.8 & 27.7 & 30.9 & 39.1 & 26.2 & 25.5 & 34.1 & 35.8 & 35.9 & 35.9 & 35.7 & 35.4 & 461.2 \\
        Conc. 4 & 28.2 & 30.8 & 31.1 & 35.3 & 44.4 & 29.6 & 29.0 & 39.0 & 40.4 & 40.6 & 40.6 & 40.6 & 40.2 & 428.7 \\
        Conc. Cs & 2.6 & 2.6 & 2.8 & 2.8 & 4.0 & 2.6 & 2.6 & 2.7 & 2.8 & 2.7 & 2.7 & 2.7 & 2.7 & 531.1 \\
        Conc. Flow & 2.6 & 2.6 & 2.7 & 2.7 & 4.0 & 2.6 & 2.6 & 2.7 & 2.7 & 2.7 & 2.8 & 2.8 & 2.7 & 557.4 \\
        Conc. Slump & 2.6 & 2.7 & 2.8 & 2.8 & 4.1 & 2.7 & 2.7 & 2.8 & 2.8 & 2.8 & 2.8 & 2.8 & 2.7 & 604.5 \\
        Cps Wage & 14.7 & 15.5 & 15.7 & 16.2 & 22.9 & 15.2 & 15.1 & 16.7 & 16.5 & 16.5 & 16.5 & 16.5 & 16.4 & 349.5 \\
        Dgp Three Regime & 40.1 & 45.4 & 47.6 & 57.7 & 67.5 & 41.1 & 42.1 & 73.1 & 86.5 & 87.6 & 90.7 & 91.3 & 88.5 & 327.3 \\
        Dgp Two Regime & 26.3 & 28.3 & 29.6 & 32.3 & 43.7 & 28.3 & 27.1 & 36.4 & 37.8 & 38.3 & 39.5 & 39.4 & 39.9 & 239.3 \\
        Housing & 13.6 & 14.3 & 14.7 & 15.1 & 21.0 & 14.0 & 13.7 & 15.1 & 15.4 & 15.4 & 15.3 & 15.2 & 15.1 & 433.1 \\
        Mpg & 10.3 & 10.6 & 11.0 & 11.1 & 15.7 & 10.4 & 10.4 & 11.0 & 11.2 & 11.3 & 11.2 & 11.2 & 11.1 & 393.0 \\
        No2 & 13.1 & 13.6 & 14.0 & 14.4 & 20.3 & 13.4 & 13.3 & 14.5 & 14.6 & 14.6 & 14.6 & 14.6 & 14.5 & 404.1 \\
        Pm10 & 13.2 & 13.7 & 14.1 & 14.4 & 20.3 & 13.4 & 13.4 & 14.6 & 14.7 & 14.7 & 14.7 & 14.7 & 14.6 & 391.4 \\
        Qsar & 14.4 & 15.1 & 15.5 & 15.9 & 22.3 & 14.7 & 14.5 & 16.0 & 16.3 & 16.3 & 16.3 & 16.2 & 16.0 & 387.2 \\
        Wine Red & 44.8 & 54.3 & 51.8 & 67.2 & 70.1 & 47.1 & 45.4 & 93.0 & 87.5 & 87.8 & 87.2 & 87.2 & 85.5 & 541.5 \\
        Wine White & 152.7 & 391.8 & 286.9 & 695.1 & 274.9 & 181.2 & 150.1 & 2307.8 & 1373.5 & 1354.5 & 1340.8 & 1327.4 & 1300.8 & 1380.5 \\
        Yacht & 8.0 & 8.3 & 8.5 & 8.6 & 12.4 & 8.3 & 8.2 & 8.6 & 8.7 & 8.7 & 8.7 & 8.7 & 8.6 & 472.0 \\
        \textbf{MEDIAN} & \textbf{14.0} & \textbf{14.7} & \textbf{15.1} & \textbf{15.5} & \textbf{21.7} & \textbf{14.4} & \textbf{14.1} & \textbf{15.6} & \textbf{15.8} & \textbf{15.9} & \textbf{15.8} & \textbf{15.7} & \textbf{15.5} & \textbf{416.4} \\
        \bottomrule
    \end{tabular}
    \caption{Median runtime (seconds) across simulation seeds. `Uncert.' denotes Uncertainty Sampling.}
    \label{tab:RuntimeComparison}
\end{table*}

\clearpage
\section{Sensitivity Analyses}
\subsection{Additional Evaluation metrics}

\label{sec:OtherMetrics}

In addition to the mean RMSE performance the following additional metrics to provide insight into the predictive quality and stability of the proposed methods.

\textbf{Predictive Quality (Correlation Coefficient):}
While RMSE measures the magnitude of error, the Correlation Coefficient (CC) measures how well the model captures the linear relationship between predictions and true labels. We choose to include this metric as it is noted as a secondary performance metric in the original iGS work by \citet{iGS}. The CC trace plots in Figures \ref{fig:CCResults1} and \ref{fig:CCResults2} confirm that the improvements in RMSE are not artifacts of scaling. WiGS-SAC achieves a higher correlation with the true labels faster than the baselines, particularly in the early stages of learning. This early convergence is critical for active learning scenarios where the labeling budget is strictly limited.


\textbf{Stability (Variance of RMSE):}
The variance plots presented in Figures \ref{fig:VarianceResults1} and \ref{fig:VarianceResults2} illustrate the stability of the active learning process across the 100 random seeds. Despite the use of a stochastic reinforcement learning policy, WiGS-SAC exhibits a variance profile comparable to the deterministic iGS heuristic. This indicates that the adaptive agent does not sacrifice stability for performance; rather, it consistently identifies high-value samples across different initializations of the training set.

\subsection{Model Agnosticism: Results on Random Forest Regression}
\label{subsec:RF_Results}
To demonstrate the versatility of the WiGS framework beyond linear models, we replicated our experimental suite using a non-linear predictor. We initially employed Ridge Regression to align with the experimental protocol of the original iGS baseline \citep{iGS}. However, non-linear models like Random Forests are ubiquitous in practical active learning scenarios where the relationship between features and targets is complex and non-monotonic.

\textbf{Experimental Setup:} The experimental protocol remains identical to that described in Section \ref{subsec:experimental_protocol}, with the substitution of the predictor model and a reduction in the number of independent trials from 100 to 25. We utilized a Random Forest Regressor consisting of $100$ estimators, evaluating a random subset of $\sqrt{p}$ features at each split. A distinct random seed was used for each of the 25 replications to ensure reproducibility. To adapt the GSy and Uncertainty Sampling strategies for this non-linear model, we calculated predictive uncertainty by evaluating the standard deviation of the individual tree predictions within the ensemble. Specifically, for a given candidate, we computed the spread of responses across all 100 trees to serve as the uncertainty proxy, allowing the selector to target regions where the forest's internal consensus was lowest.

\textbf{Performance Analysis:} Figure \ref{fig:RF_Heatmap} presents the Relative AUC (RMSE) for all strategies compared to the iGS baseline. Consistent with the Ridge Regression results, the adaptive WiGS strategies (SAC and MAB) demonstrate robust performance across the benchmark suite. 

Notably, on the synthetic datasets designed to test the "density veto" hypothesis, WiGS-SAC achieves substantial reductions in error with Relative AUCs of $0.837$ (\texttt{dgp\_two\_regime}) and $0.709$ (\texttt{dgp\_three\_regime}). This confirms that the agent's ability to decouple exploration from investigation is particularly valuable in non-linear regimes where static multiplicative heuristics struggle to prioritize high-error samples in dense regions. While baselines like QBC show promise on specific synthetic tasks, they exhibit higher volatility on real-world benchmarks (e.g., performing significantly worse than the baseline on \texttt{wine\_red} with a relative error of $1.093$). In contrast, WiGS-SAC maintains a consistent risk profile, rarely exceeding the baseline error.

\textbf{Label Efficiency:} Figure \ref{fig:RF_Efficiency} illustrates the Relative Label Efficiency ($N_{rel}$). The adaptive WiGS variants consistently shift the efficiency distribution below $1.0$, indicating a reduction in the labeling budget required to achieve $70\%$ and $80\%$ of the total performance gain. Specifically, WiGS-SAC shows a tighter variance compared to methods like Uncertainty Sampling and EGAL, offering a more reliable active learning trajectory for non-linear models.

\begin{figure}[htbp]
    \centering
    \includegraphics[width=\textwidth]{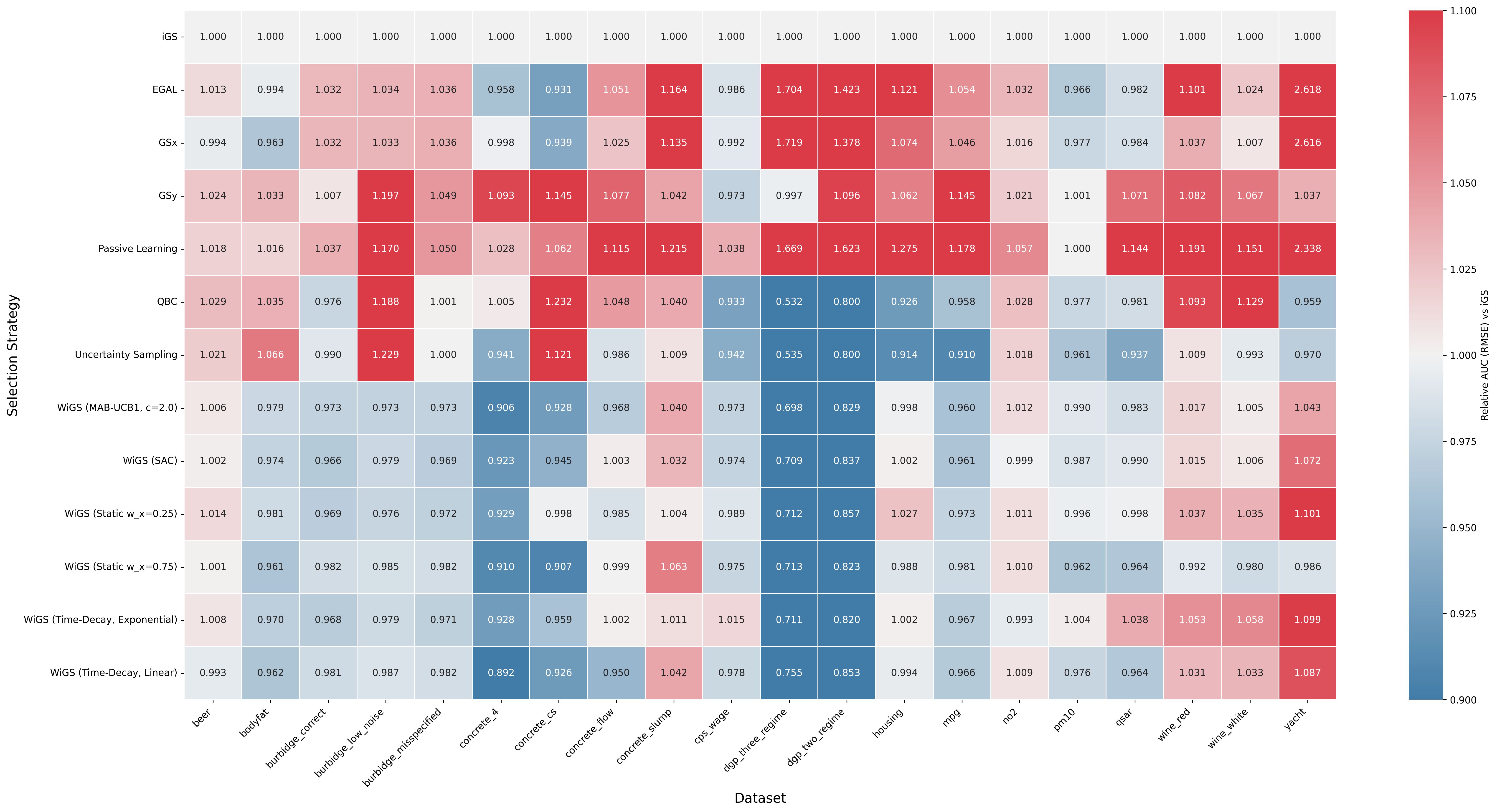}
    \caption{\textbf{Random Forest Results (Heatmap):} Values represent the ratio of the Area Under the RMSE Curve (AUC) relative to the iGS baseline using a Random Forest Regressor. Blue cells ($<1.0$) indicate superior performance.}
    \label{fig:RF_Heatmap}
\end{figure}

\begin{figure}[htbp]
    \centering
    \includegraphics[width=0.8\textwidth]{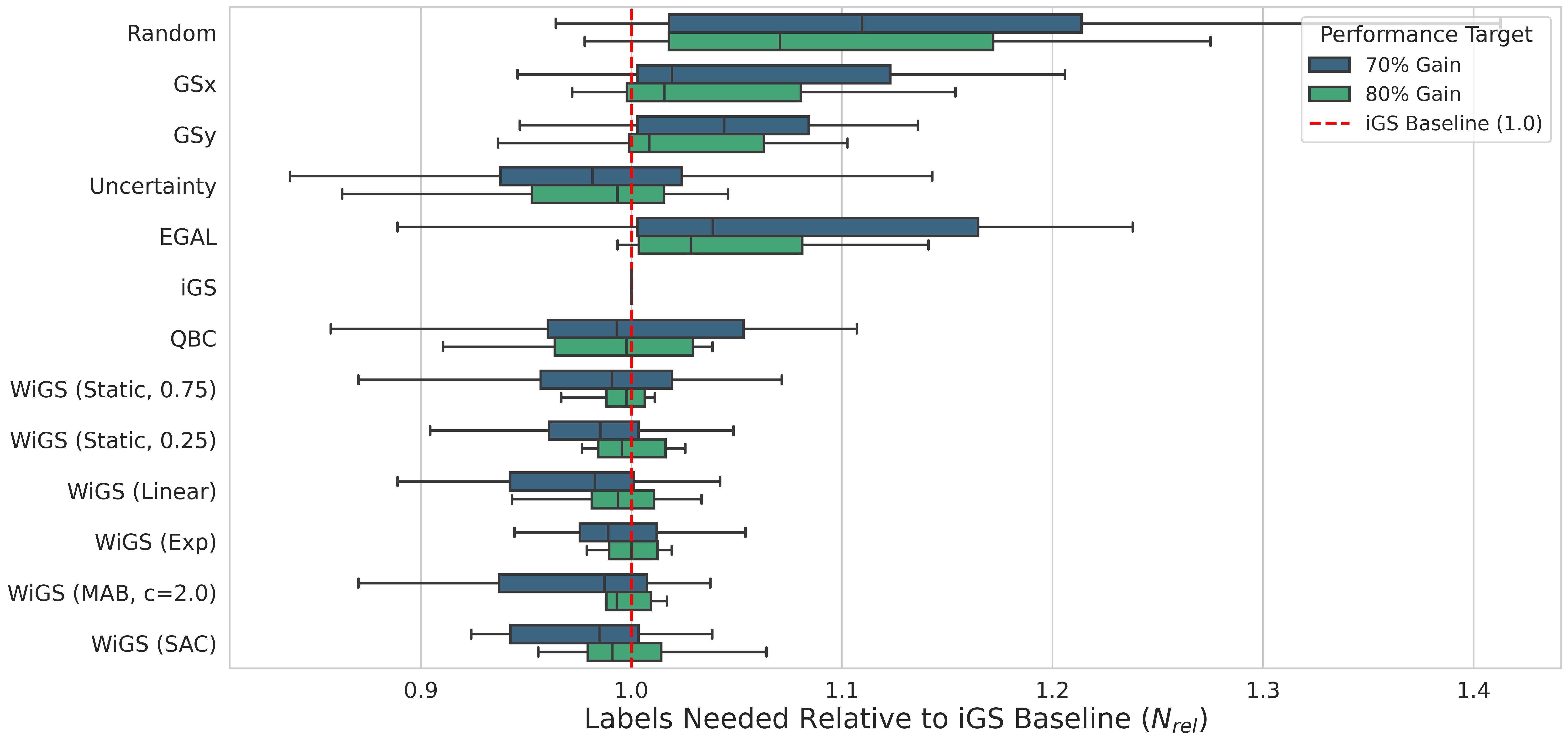} 
    \caption{\textbf{Random Forest Results (Efficiency):} Relative Label Efficiency ($N_{rel}$) aggregated across 20 datasets.}
    \label{fig:RF_Efficiency}
\end{figure}
\begin{figure}
    \centering
    \vspace*{-1cm} 
    
    \begin{subfigure}[b]{0.31\textwidth}
        \centering
        \includegraphics[width=\linewidth, keepaspectratio]{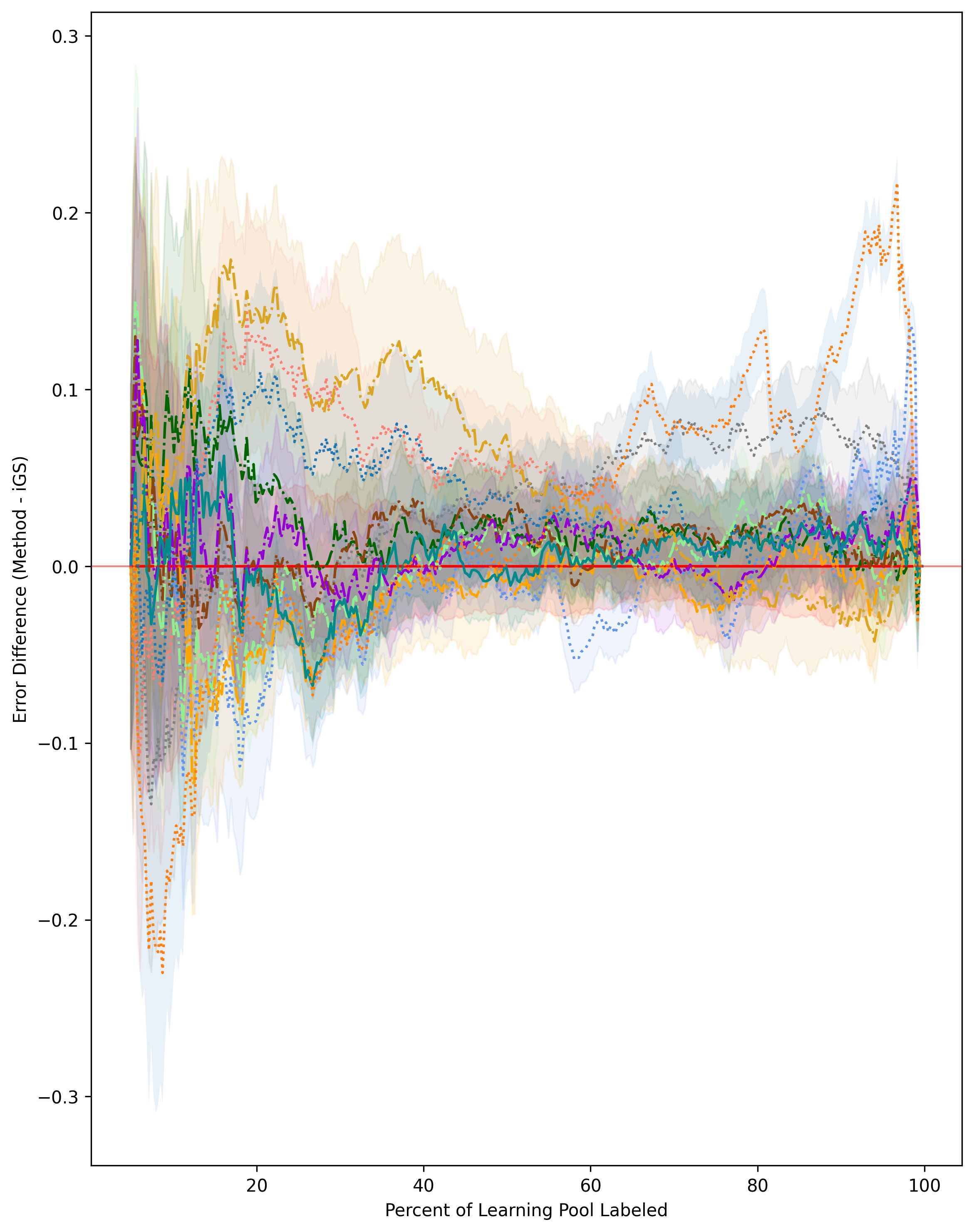}
        \caption{beer}
    \end{subfigure}
    \hfill
    \begin{subfigure}[b]{0.31\textwidth}
        \centering
        \includegraphics[width=\linewidth, keepaspectratio]{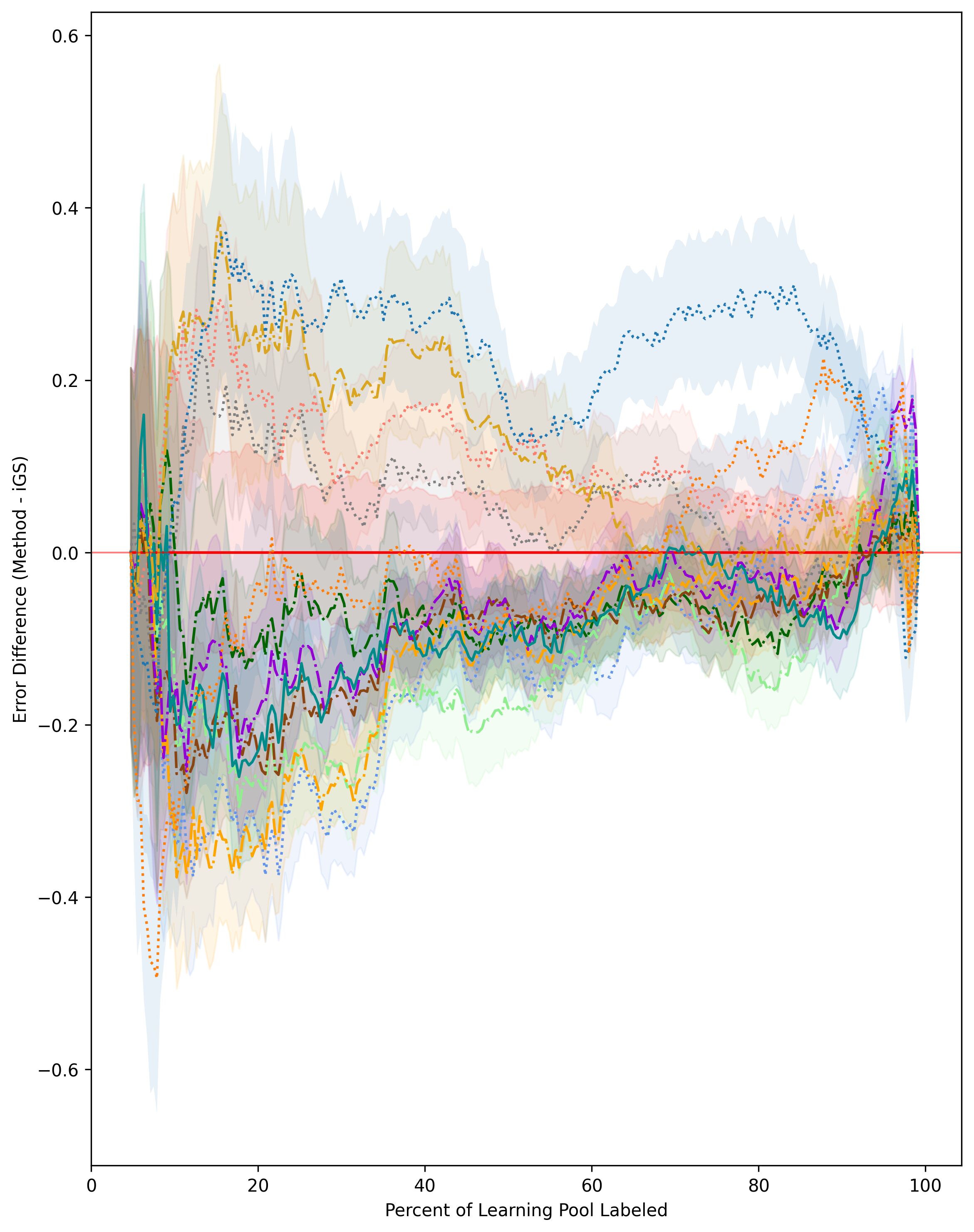}
        \caption{bodyfat}
    \end{subfigure}
    \hfill
    \begin{subfigure}[b]{0.31\textwidth}
        \centering
        \includegraphics[width=\linewidth, keepaspectratio]{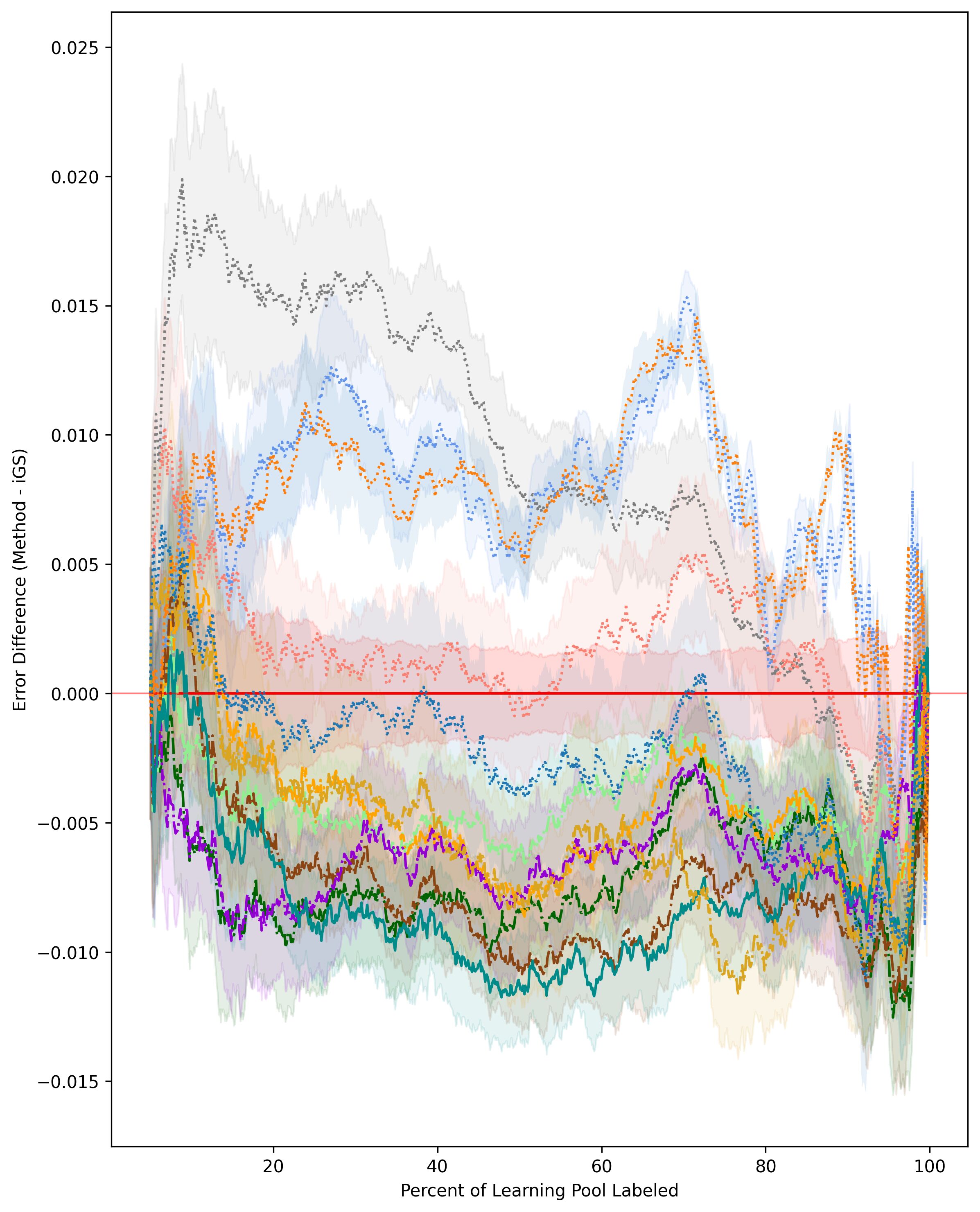}
        \caption{burbidge\_correct}
    \end{subfigure}
    
    \vspace{0.3em} 
    
    \begin{subfigure}[b]{0.31\textwidth}
        \centering
        \includegraphics[width=\linewidth, keepaspectratio]{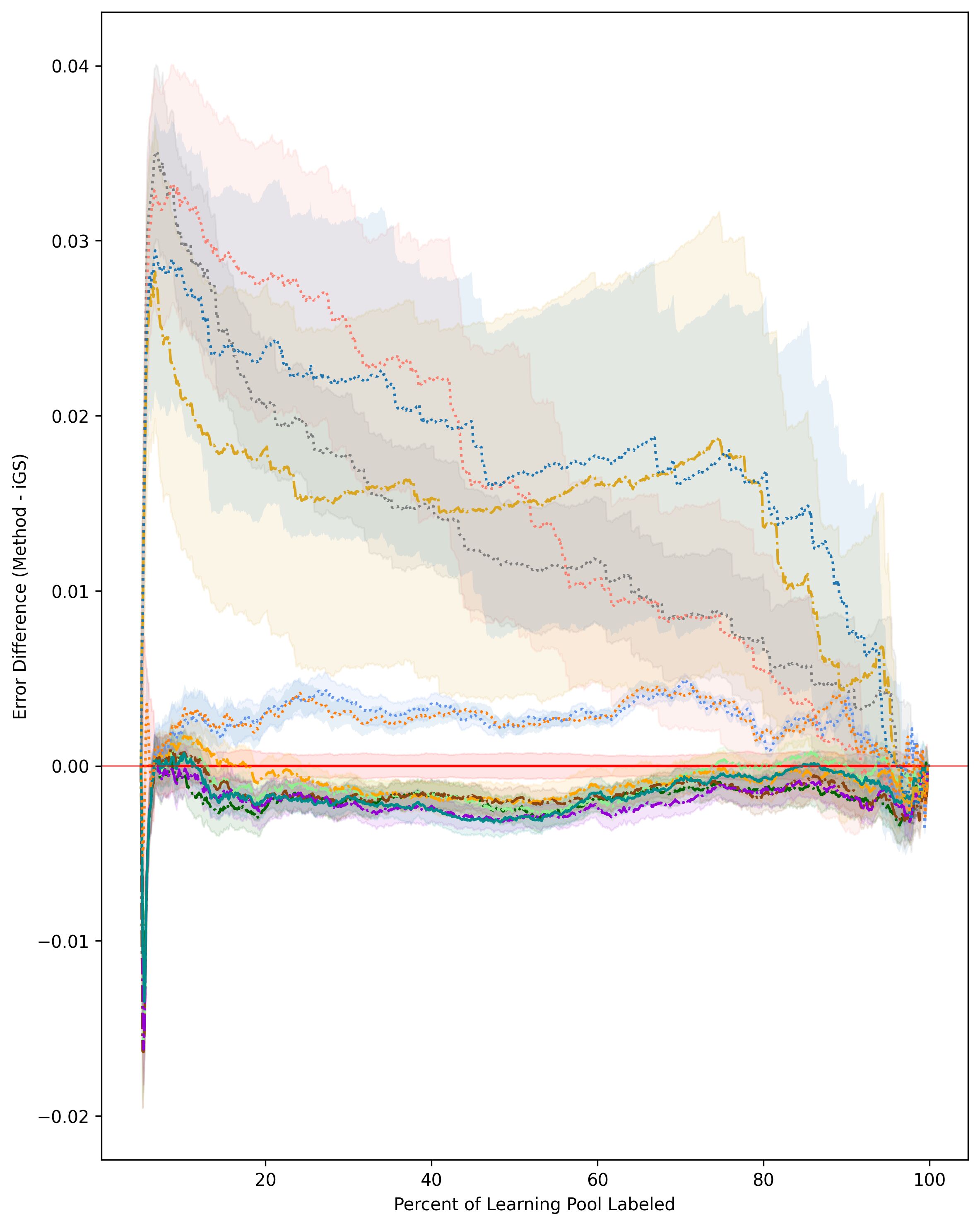}
        \caption{burbidge\_low\_noise}
    \end{subfigure}
    \hfill
    \begin{subfigure}[b]{0.31\textwidth}
        \centering
        \includegraphics[width=\linewidth, keepaspectratio]{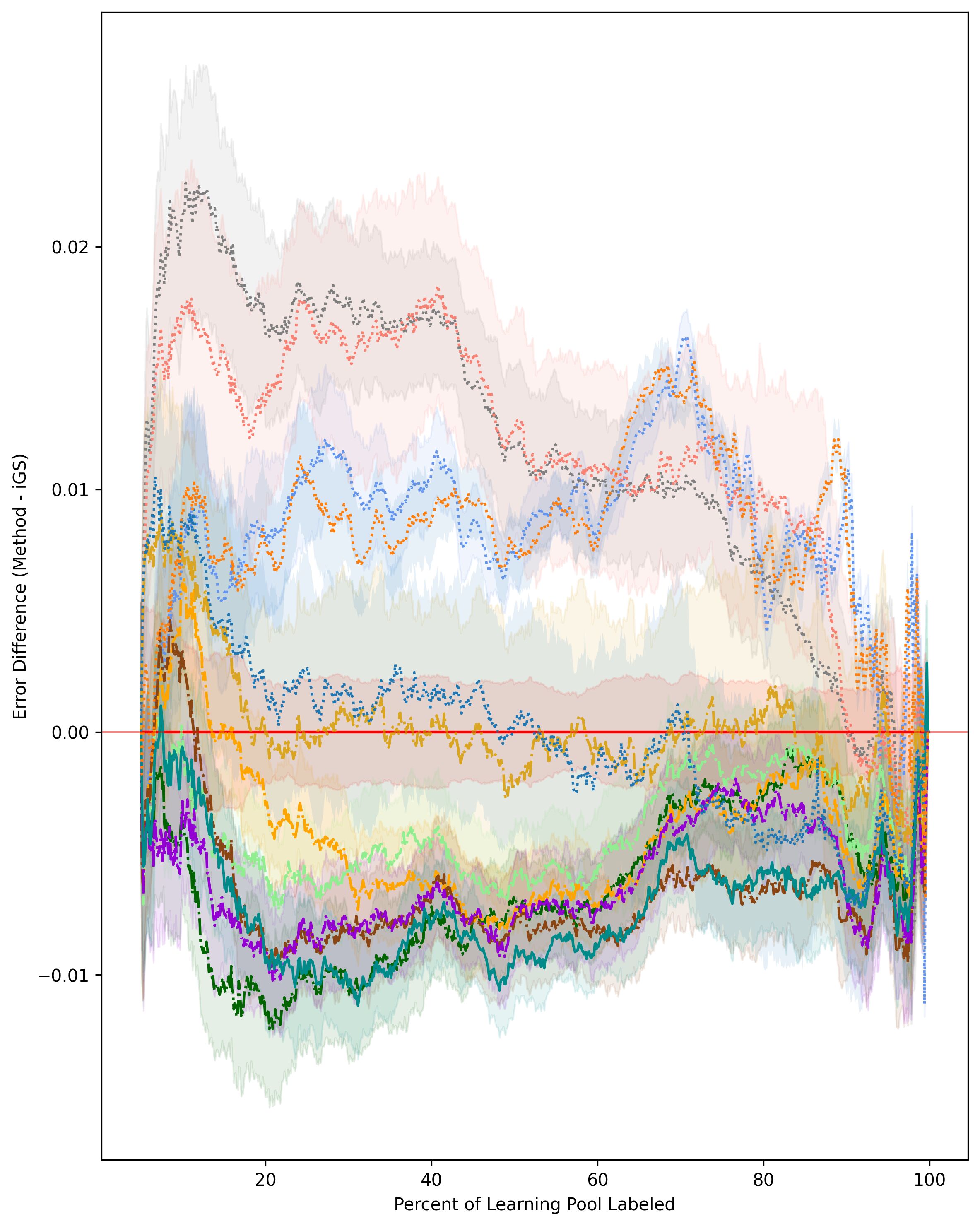}
        \caption{burbidge\_misspecified}
    \end{subfigure}
    \hfill
    \begin{subfigure}[b]{0.31\textwidth}
        \centering
        \includegraphics[width=\linewidth, keepaspectratio]{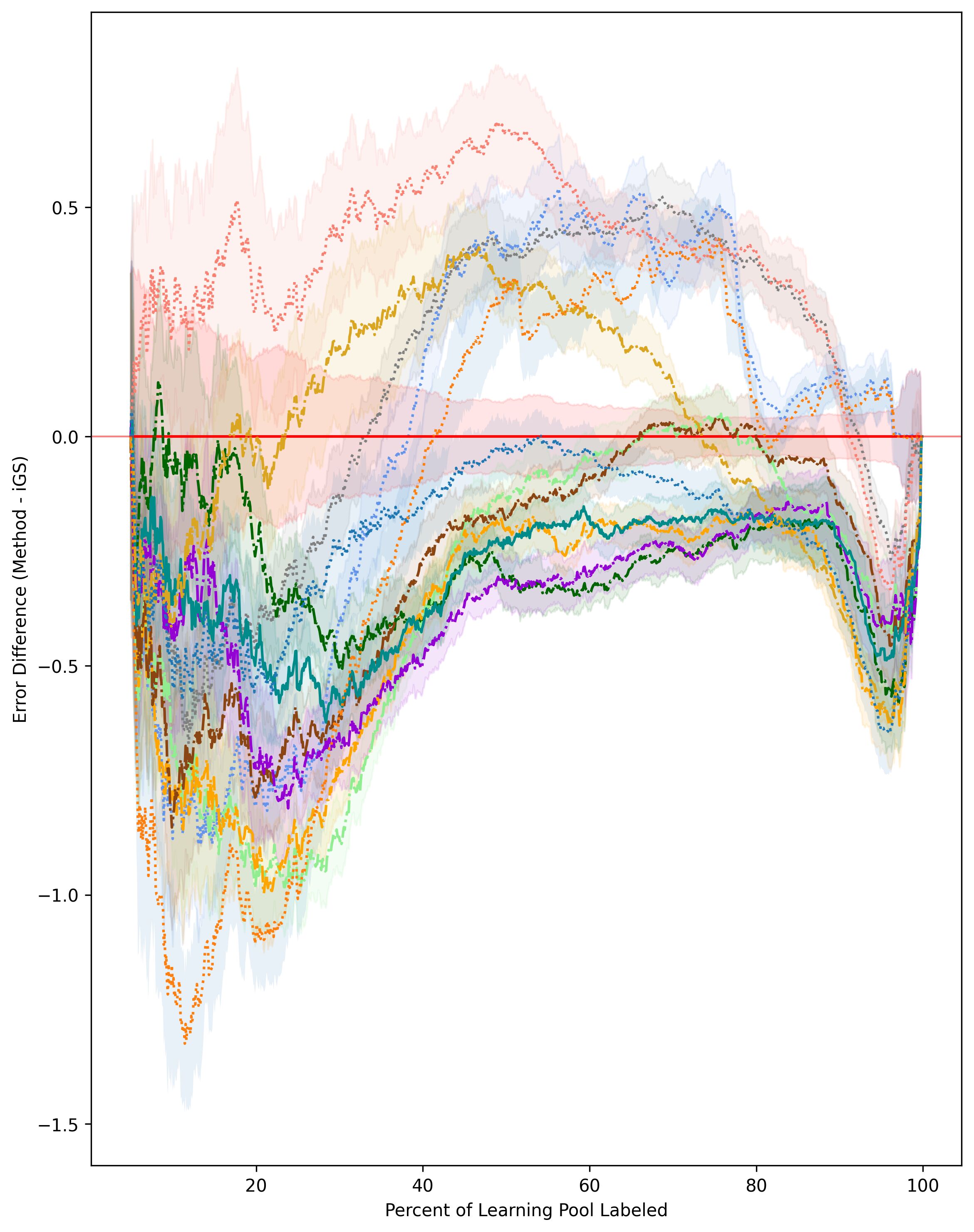}
        \caption{concrete\_4}
    \end{subfigure}
    
    \vspace{0.3em}
    
    \begin{subfigure}[b]{0.31\textwidth}
        \centering
        \includegraphics[width=\linewidth, keepaspectratio]{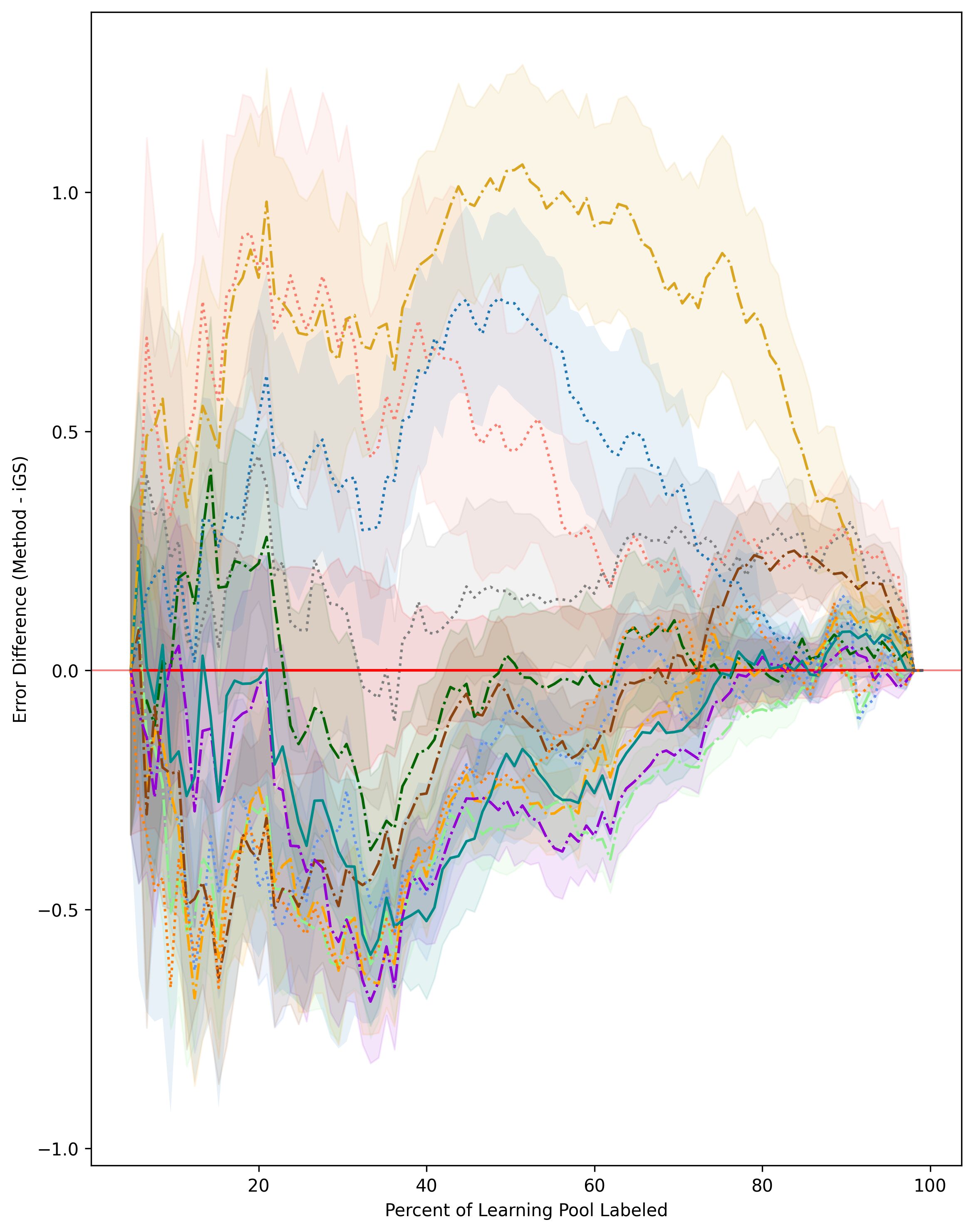}
        \caption{concrete\_cs}
    \end{subfigure}
    \hfill
    \begin{subfigure}[b]{0.31\textwidth}
        \centering
        \includegraphics[width=\linewidth, keepaspectratio]{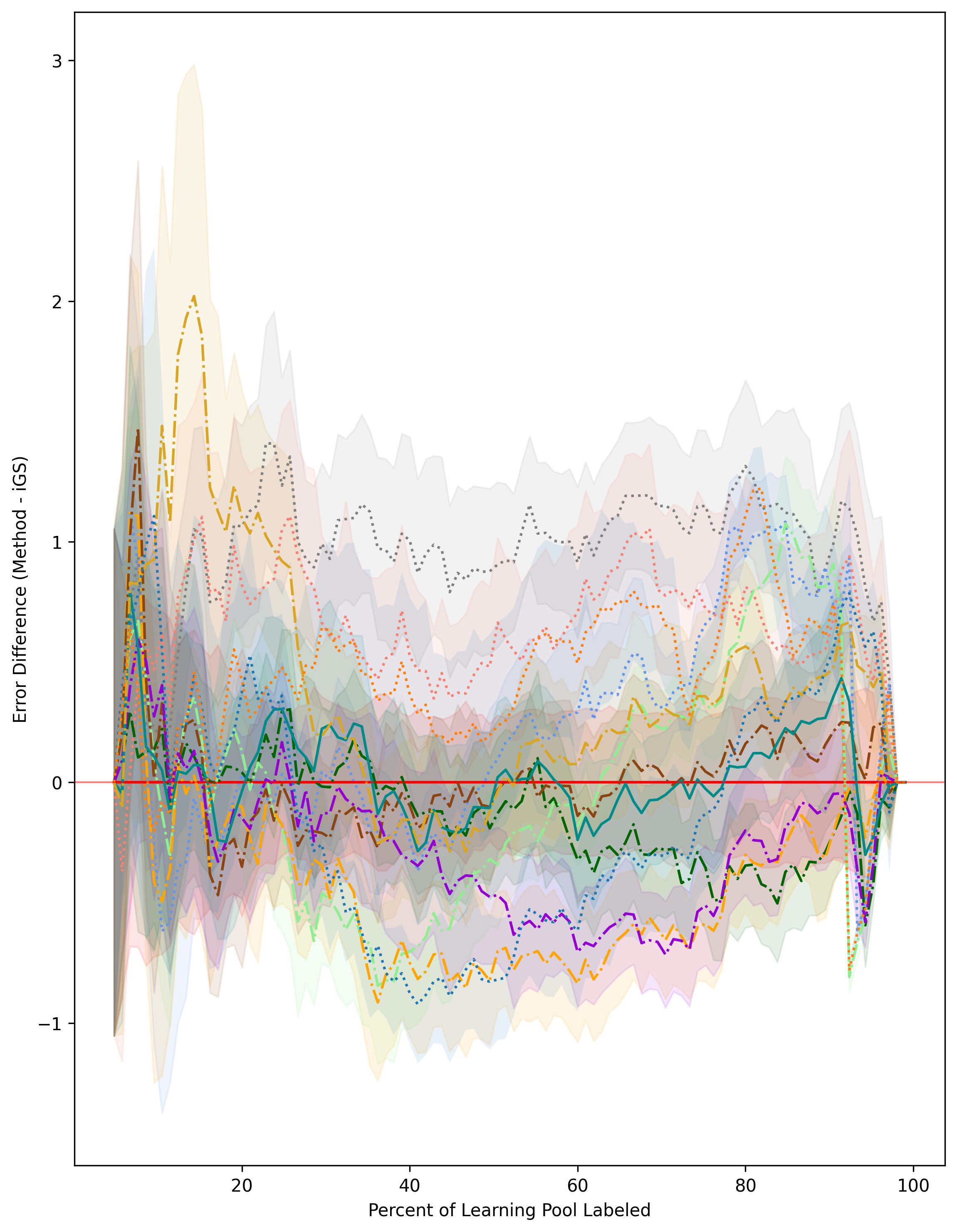}
        \caption{concrete\_flow}
    \end{subfigure}
    \hfill
    \begin{subfigure}[b]{0.31\textwidth}
        \centering
        \includegraphics[width=\linewidth, keepaspectratio]{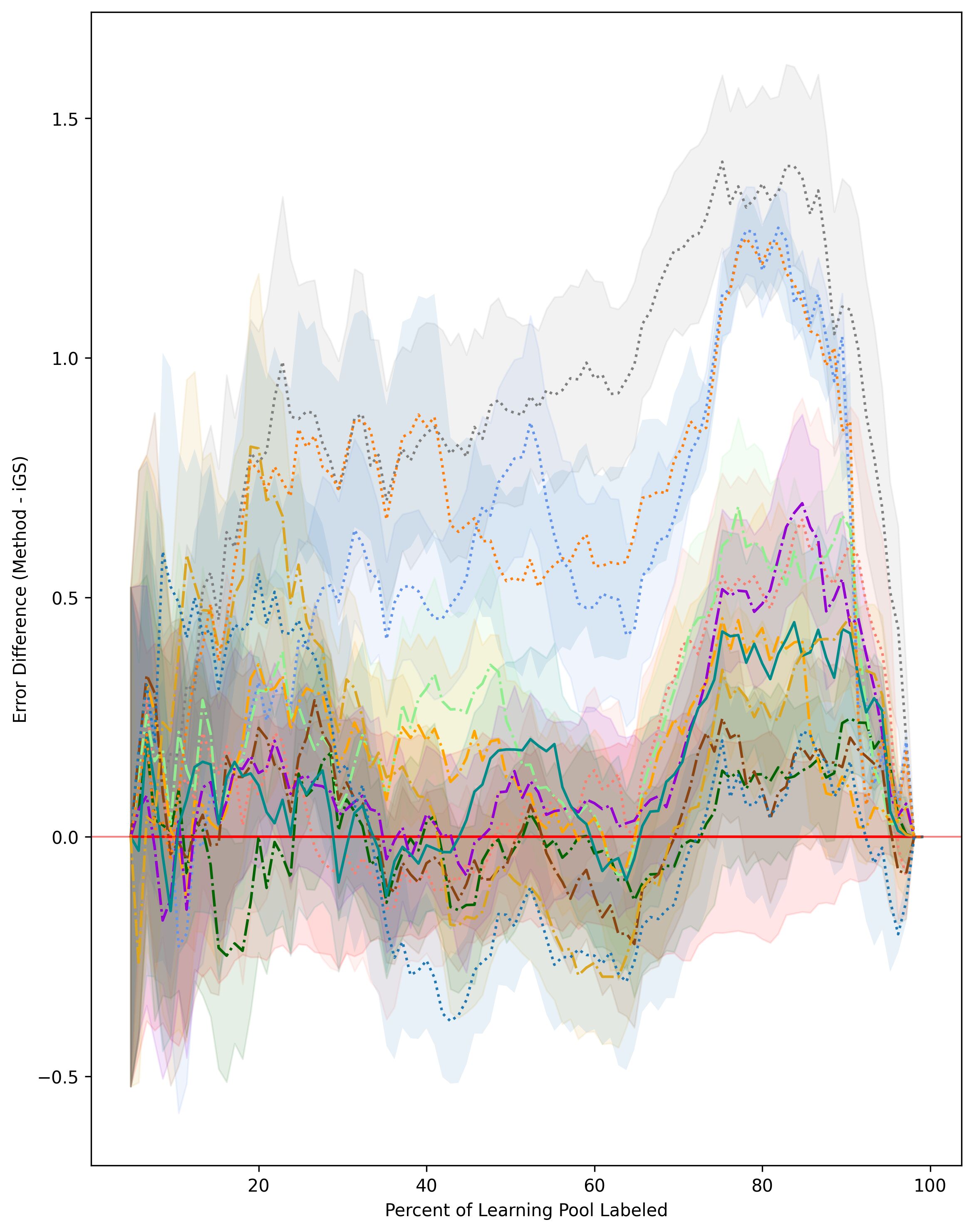}
        \caption{concrete\_slump}
    \end{subfigure}
    
    \vspace{0.5em}
    \centering
    \caption{Full-pool Random Forest RMSE trace plots for benchmark datasets (Part 1 of 2).}
    \label{fig:RFResults1}
\end{figure}

\clearpage
\begin{figure}
    \centering
    \vspace*{-1cm} 
    
    \begin{subfigure}[b]{0.31\textwidth}
        \centering
        \includegraphics[width=\linewidth, keepaspectratio]{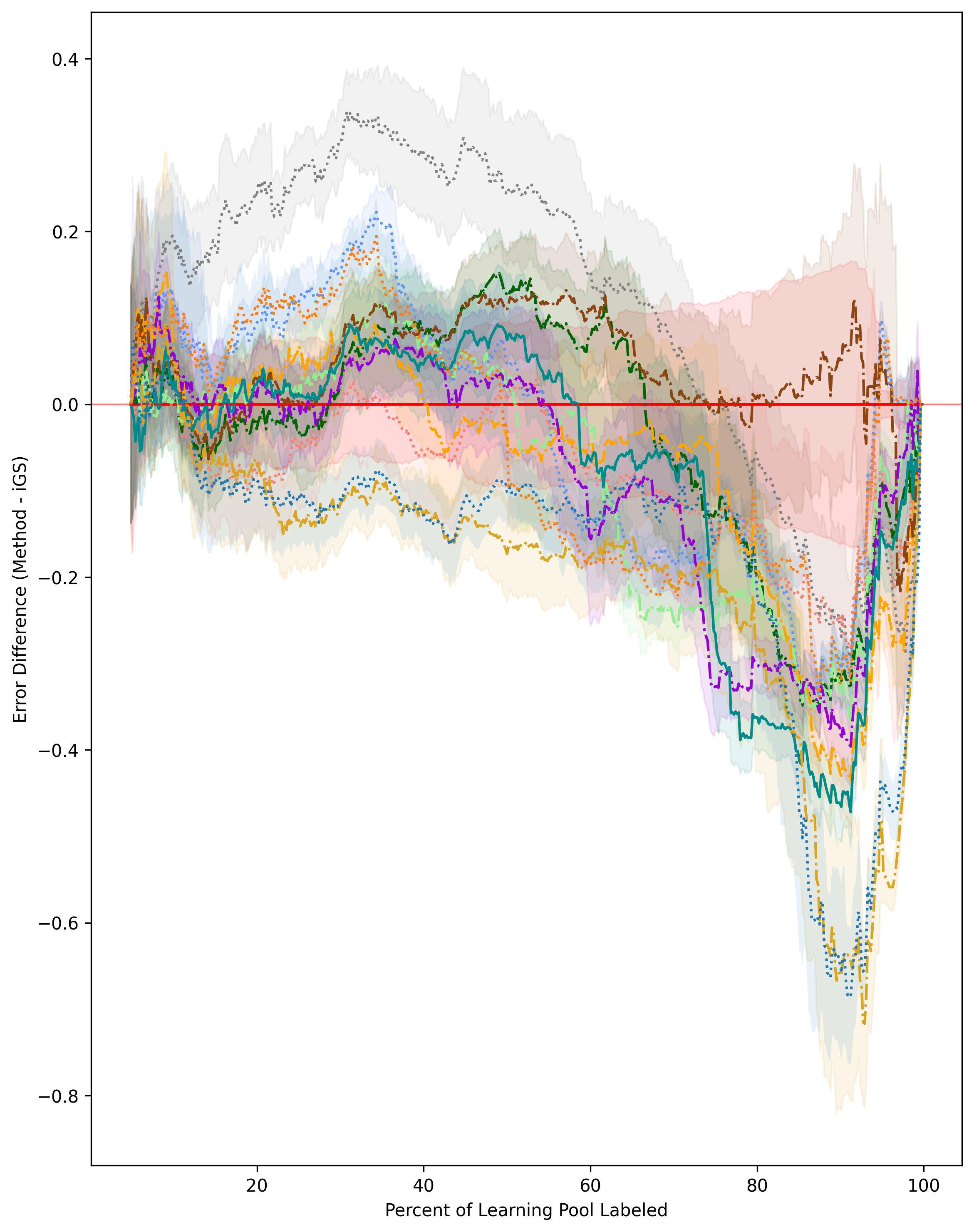}
        \caption{cps\_wage}
    \end{subfigure}
    \hfill
    \begin{subfigure}[b]{0.31\textwidth}
        \centering
        \includegraphics[width=\linewidth, keepaspectratio]{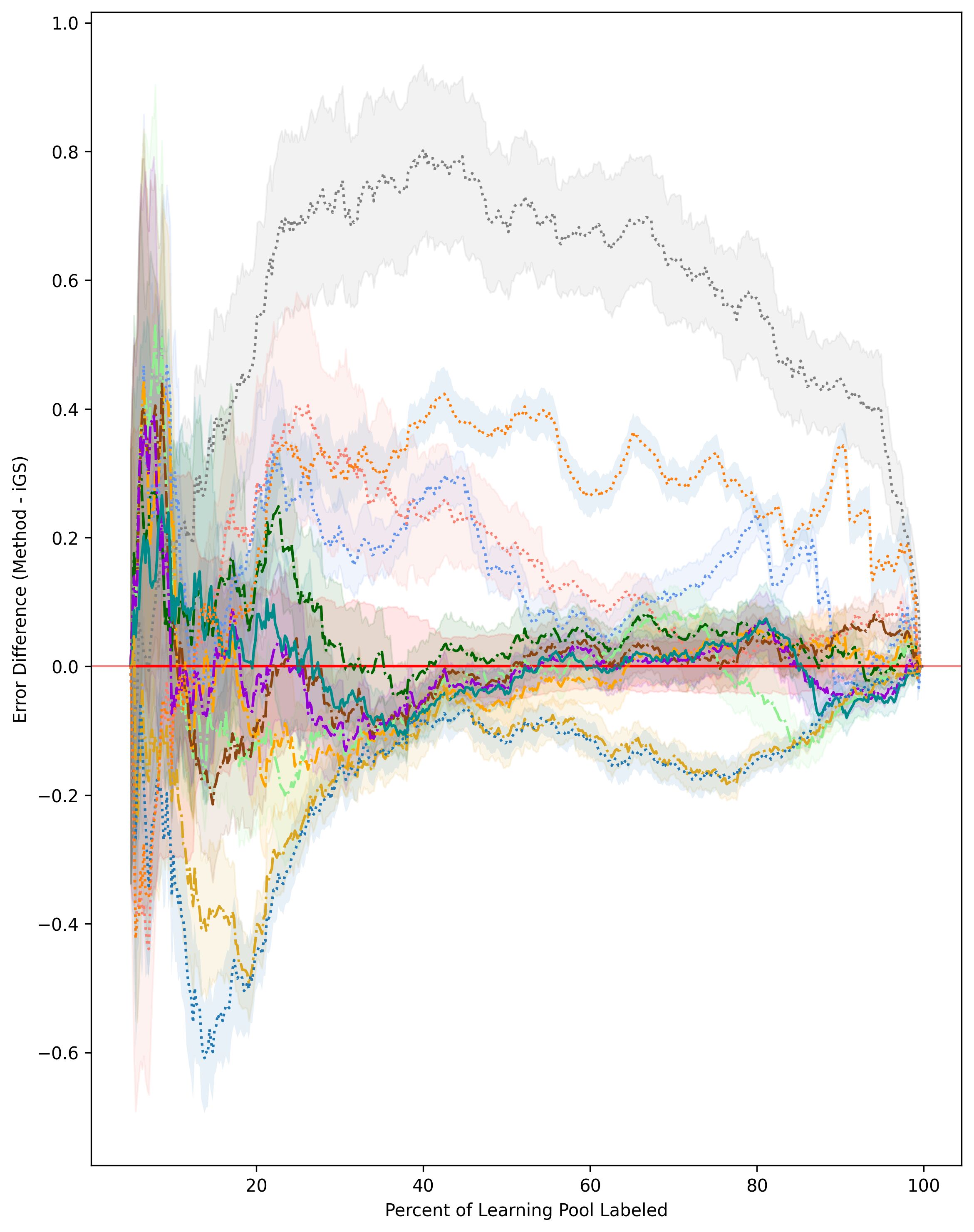}
        \caption{housing}
    \end{subfigure}
    \hfill
    \begin{subfigure}[b]{0.31\textwidth}
        \centering
        \includegraphics[width=\linewidth, keepaspectratio]{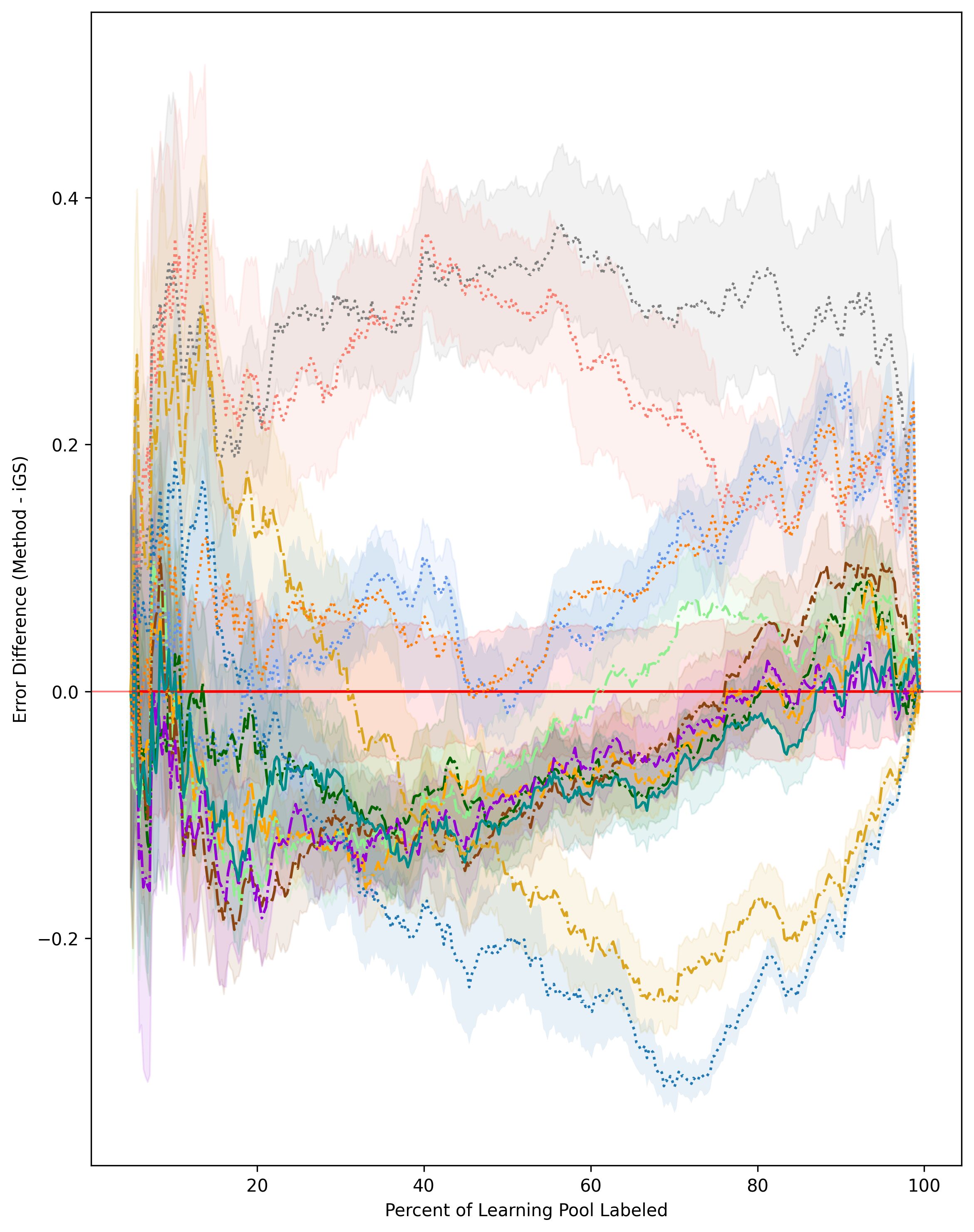}
        \caption{mpg}
    \end{subfigure}
    
    \vspace{0.3em}
    
    \begin{subfigure}[b]{0.31\textwidth}
        \centering
        \includegraphics[width=\linewidth, keepaspectratio]{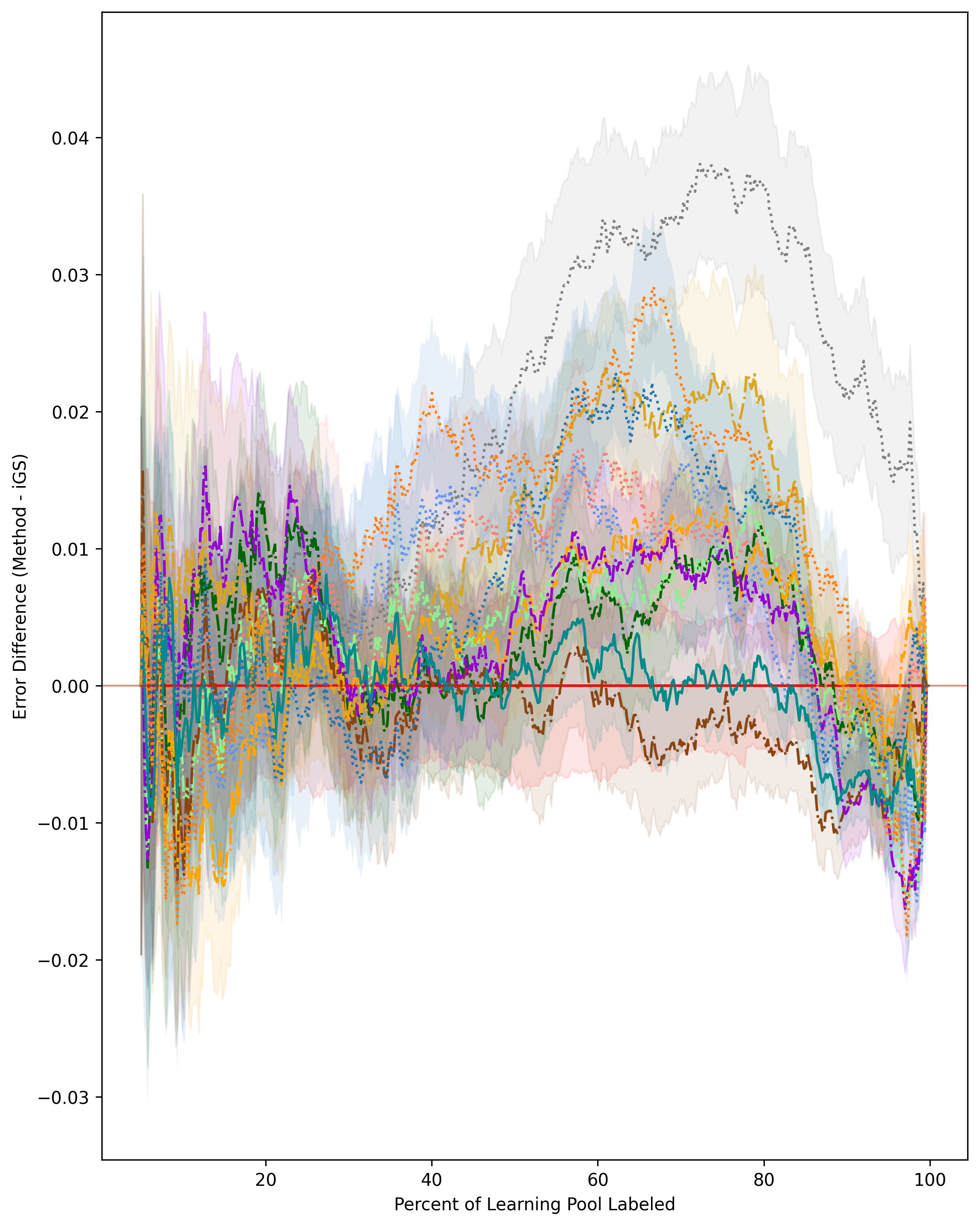}
        \caption{no2}
    \end{subfigure}
    \hfill
    \begin{subfigure}[b]{0.31\textwidth}
        \centering
        \includegraphics[width=\linewidth, keepaspectratio]{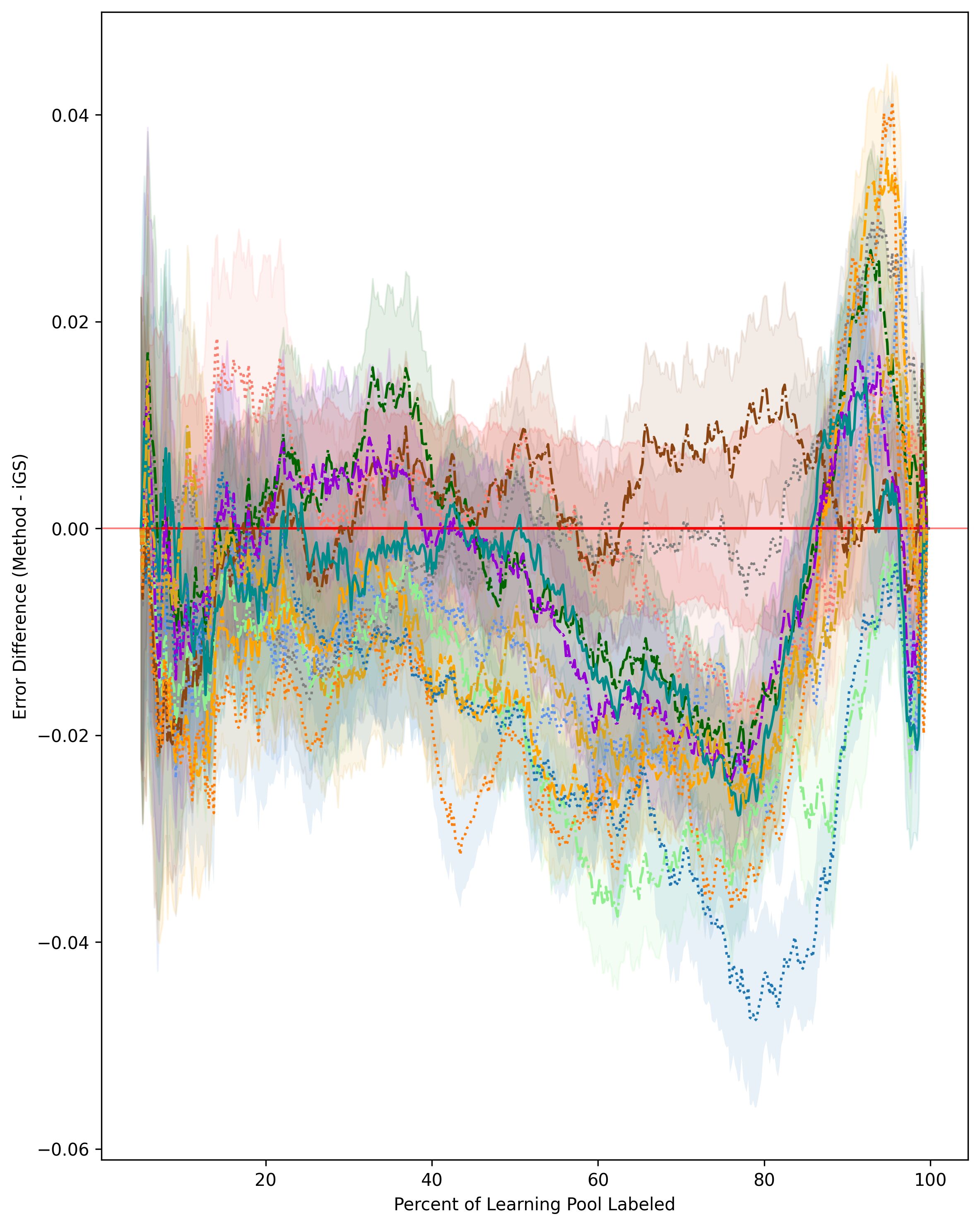}
        \caption{pm10}
    \end{subfigure}
    \hfill
    \begin{subfigure}[b]{0.31\textwidth}
        \centering
        \includegraphics[width=\linewidth, keepaspectratio]{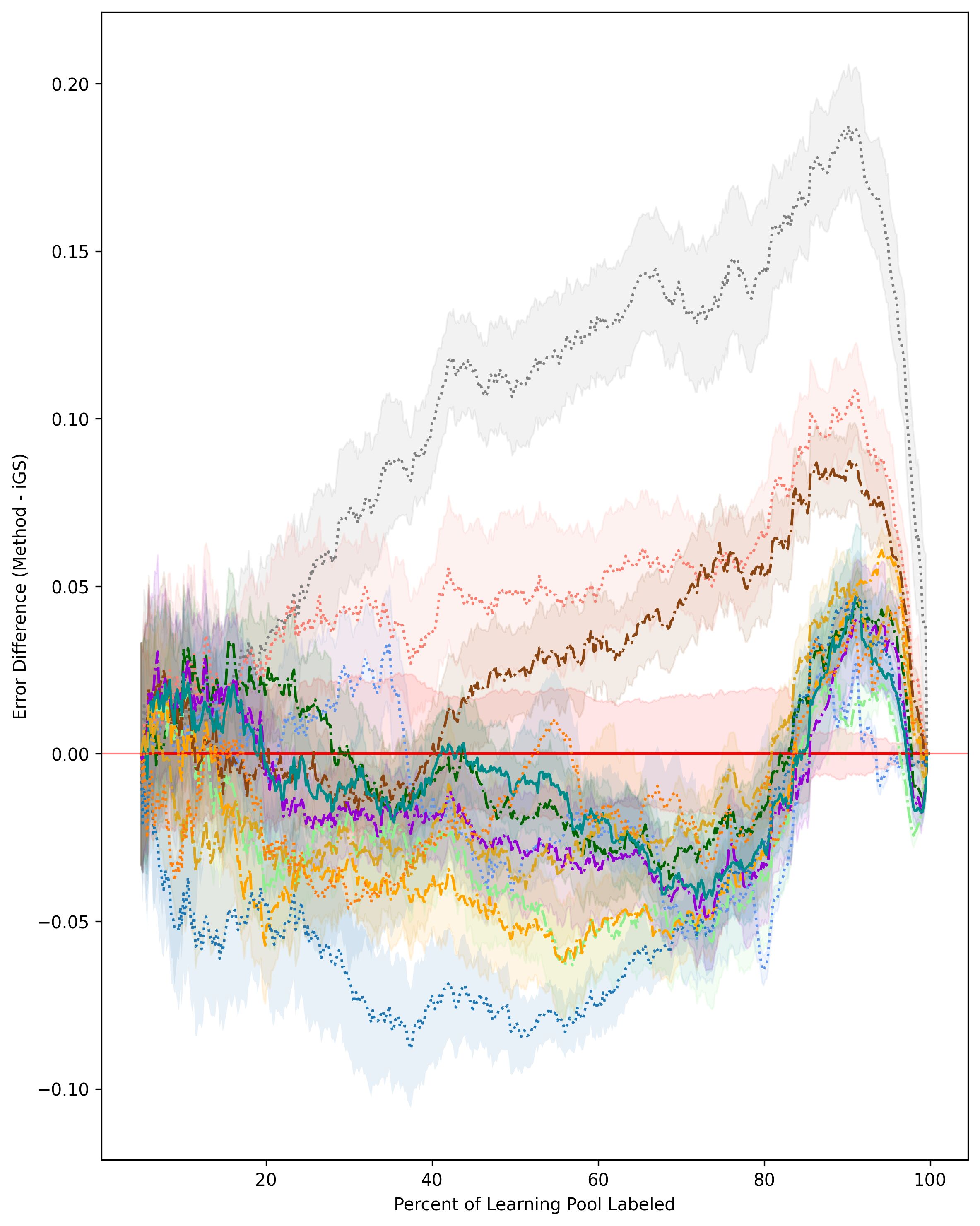}
        \caption{qsar}
    \end{subfigure}
    
    \vspace{0.3em}
    
    \begin{subfigure}[b]{0.31\textwidth}
        \centering
        \includegraphics[width=\linewidth, keepaspectratio]{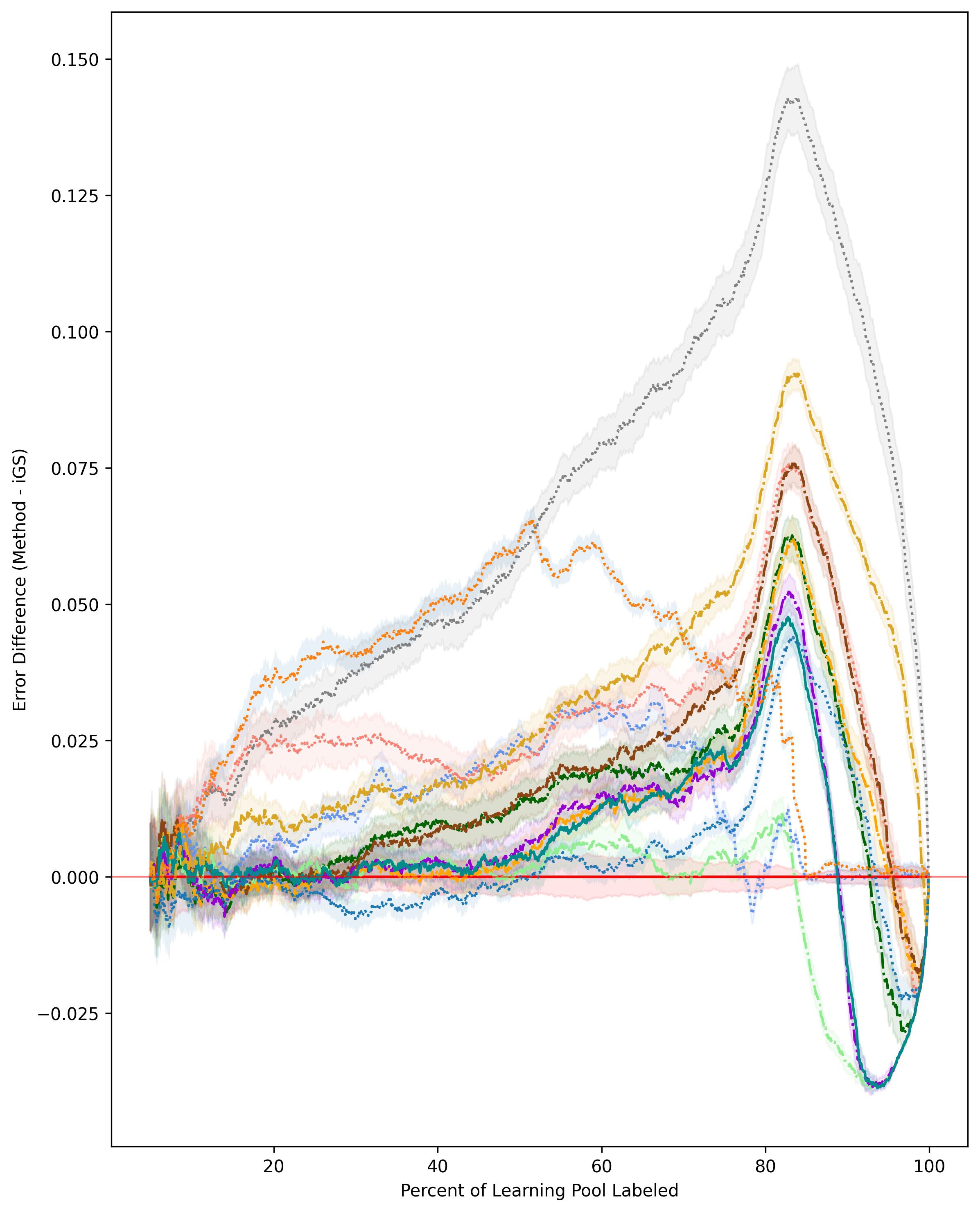}
        \caption{wine\_red}
    \end{subfigure}
    \hfill
    \begin{subfigure}[b]{0.31\textwidth}
        \centering
        \includegraphics[width=\linewidth, keepaspectratio]{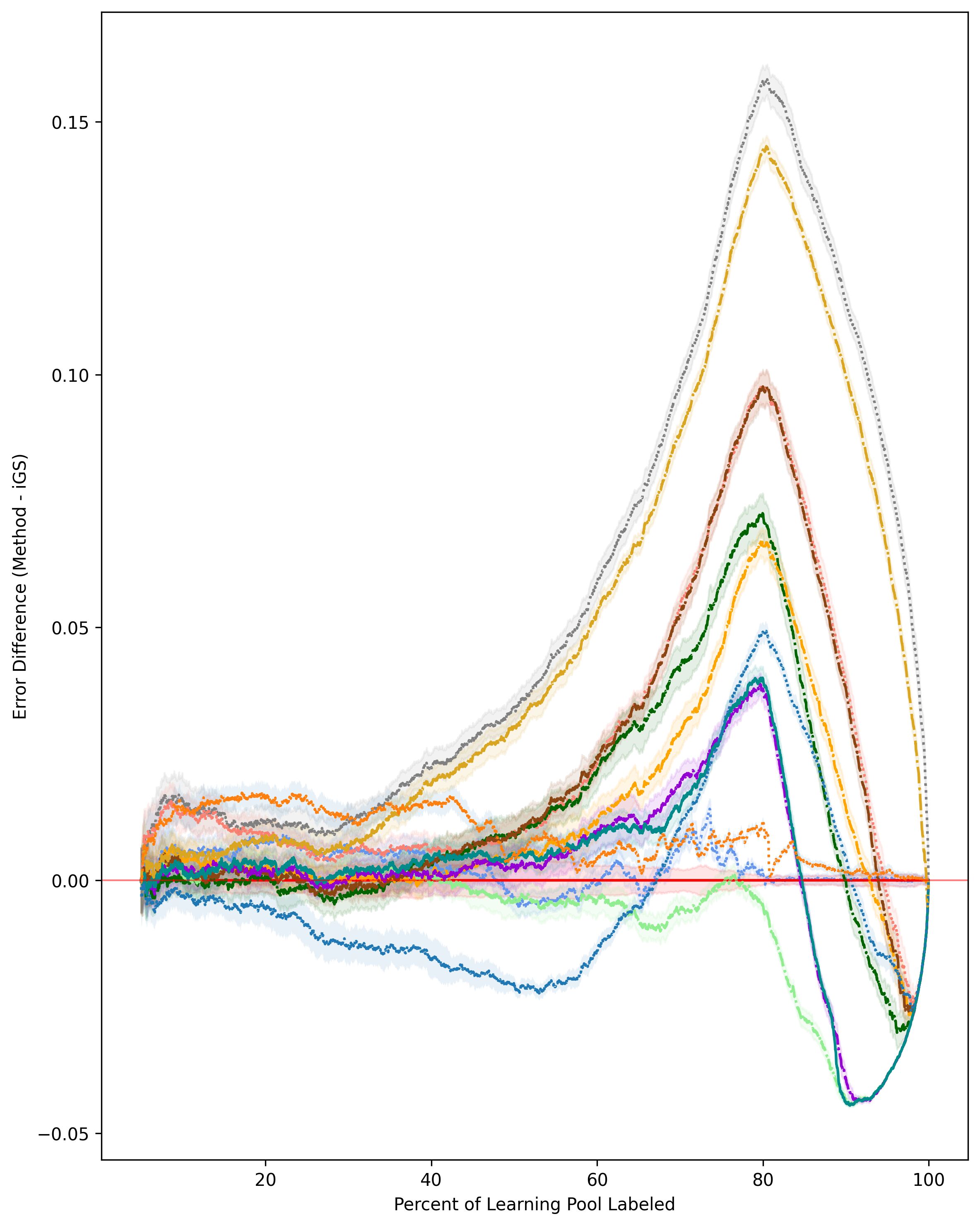}
        \caption{wine\_white}
    \end{subfigure}
    \hfill
    \begin{subfigure}[b]{0.31\textwidth}
        \centering
        \includegraphics[width=\linewidth, keepaspectratio]{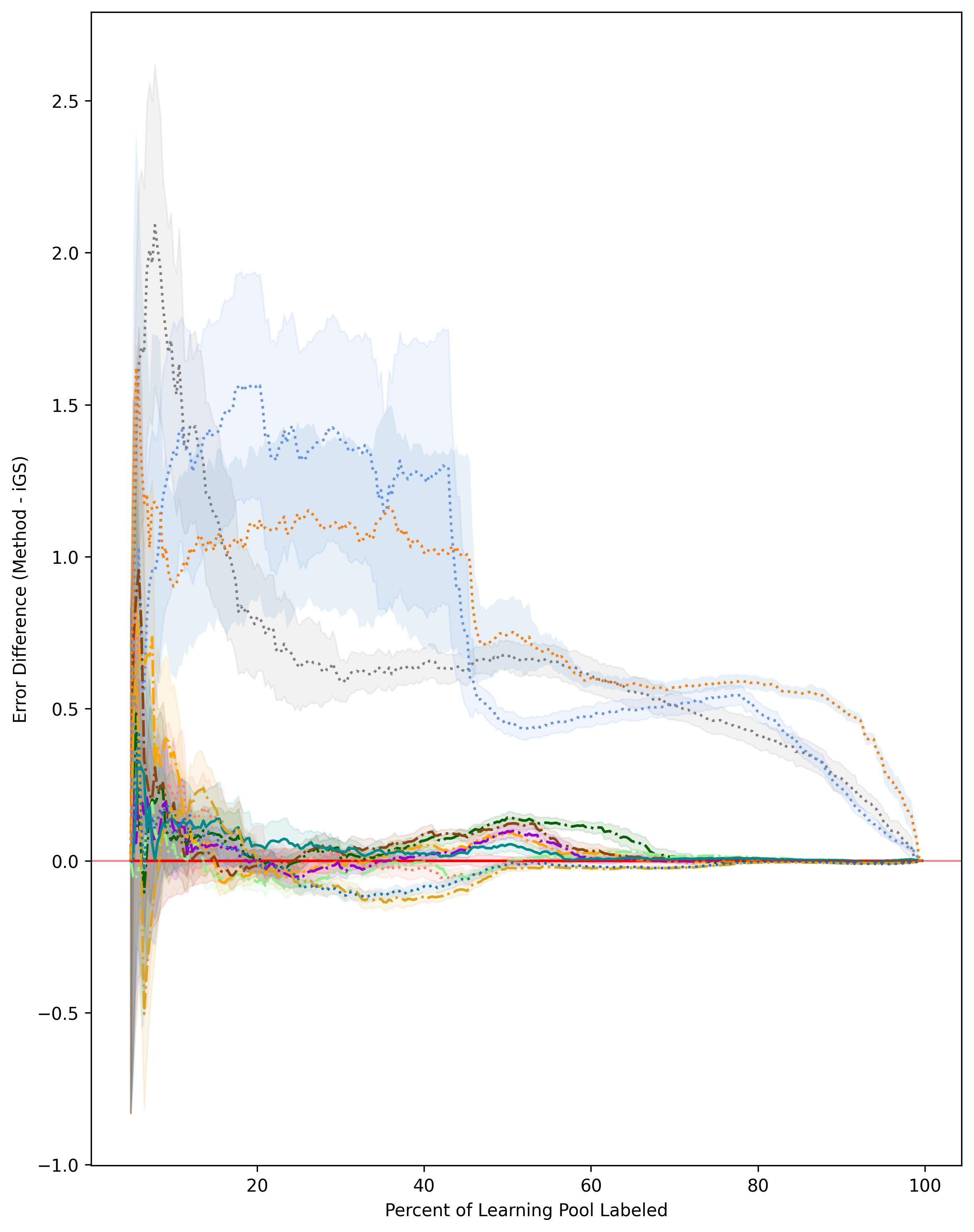}
        \caption{yacht}
    \end{subfigure}
    
    \vspace{0.5em}
    \centering
    \caption{Full-pool Random Forest RMSE trace plots for benchmark datasets (Part 2 of 2).}
    \label{fig:RFResults2}
\end{figure}
\clearpage
\begin{figure}
    \centering
    \vspace*{-1cm} 
    
    \begin{subfigure}[b]{0.31\textwidth}
        \centering
        \includegraphics[width=\linewidth, keepaspectratio]{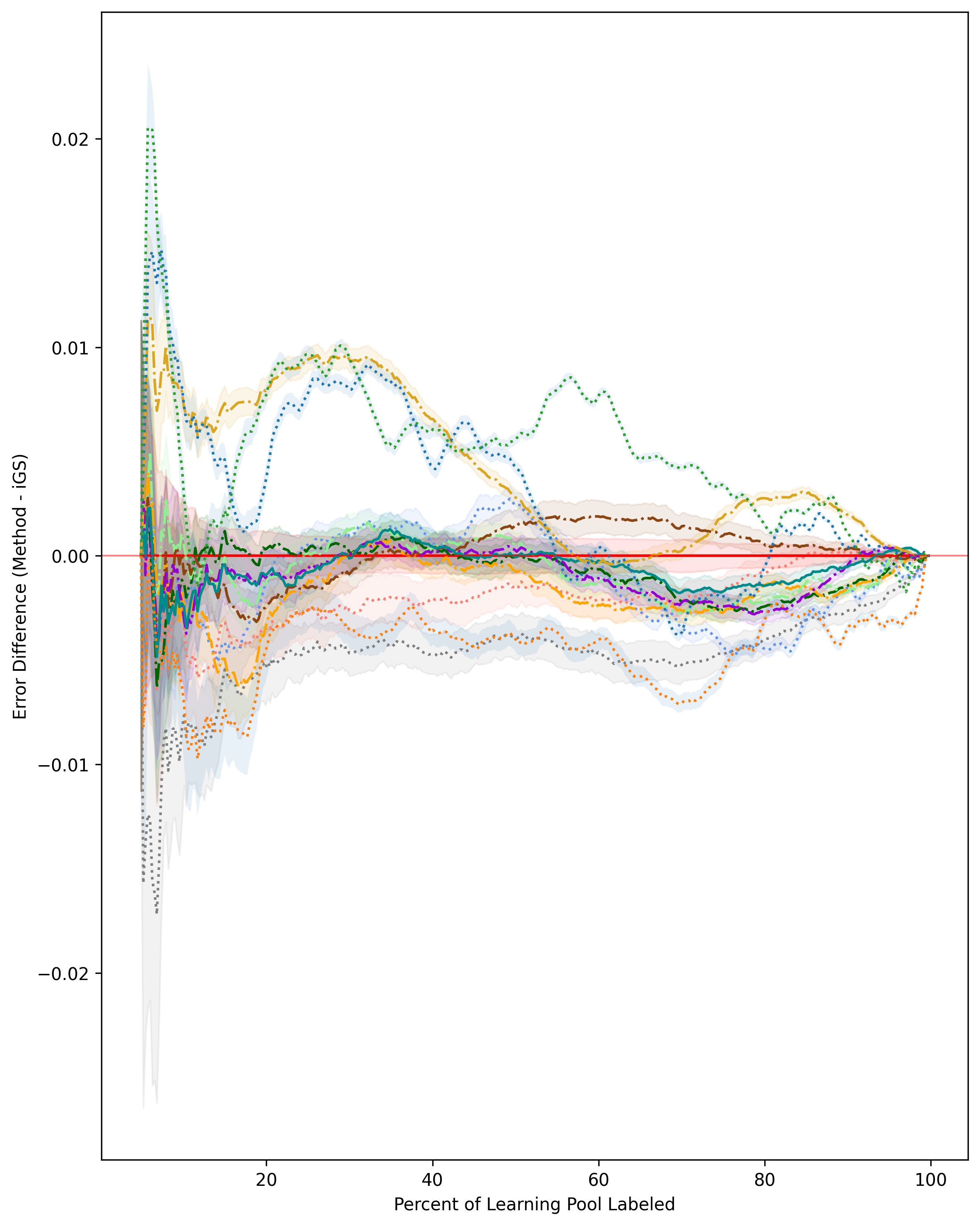}
        \caption{beer}
    \end{subfigure}
    \hfill
    \begin{subfigure}[b]{0.31\textwidth}
        \centering
        \includegraphics[width=\linewidth, keepaspectratio]{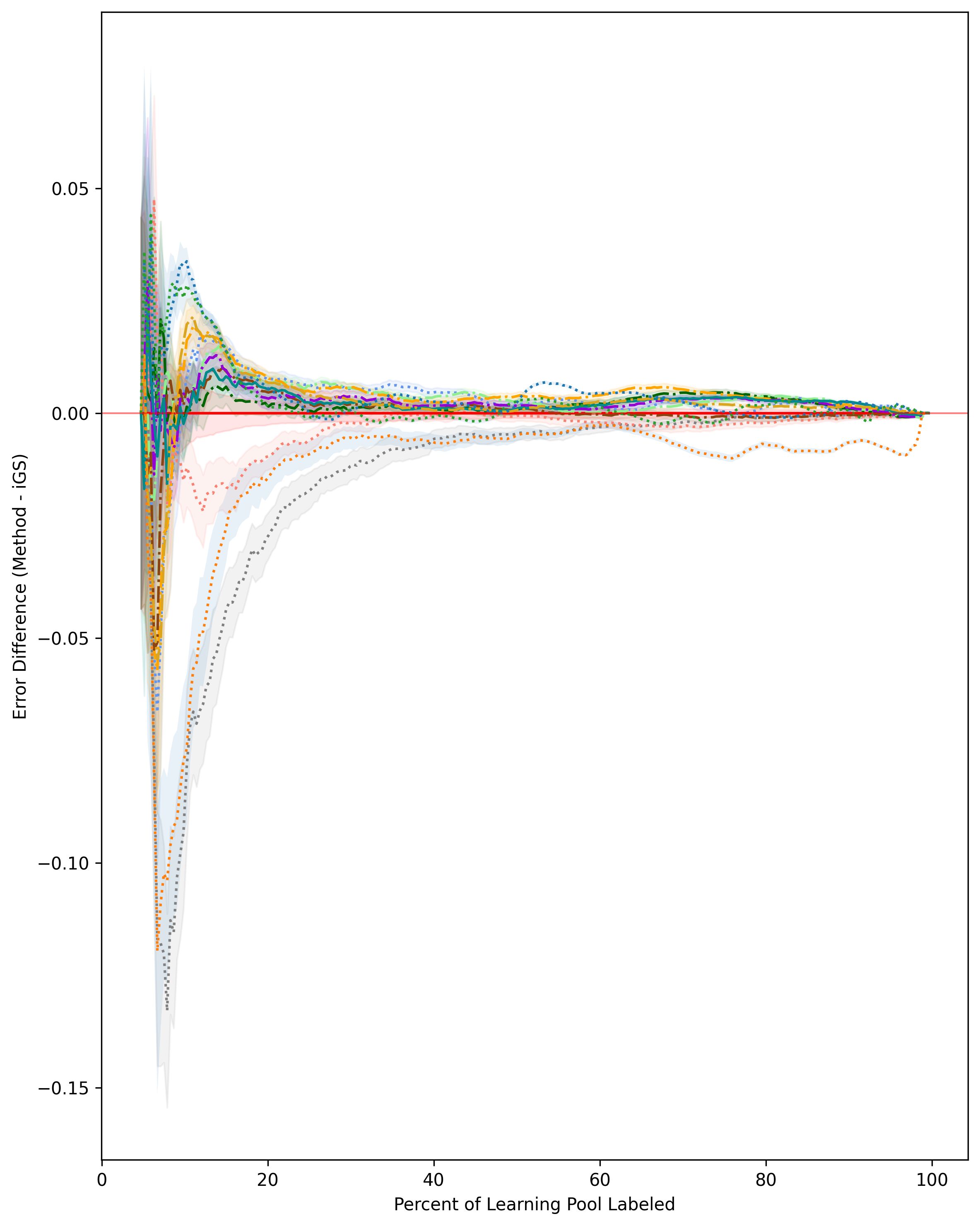}
        \caption{bodyfat}
    \end{subfigure}
    \hfill
    \begin{subfigure}[b]{0.31\textwidth}
        \centering
        \includegraphics[width=\linewidth, keepaspectratio]{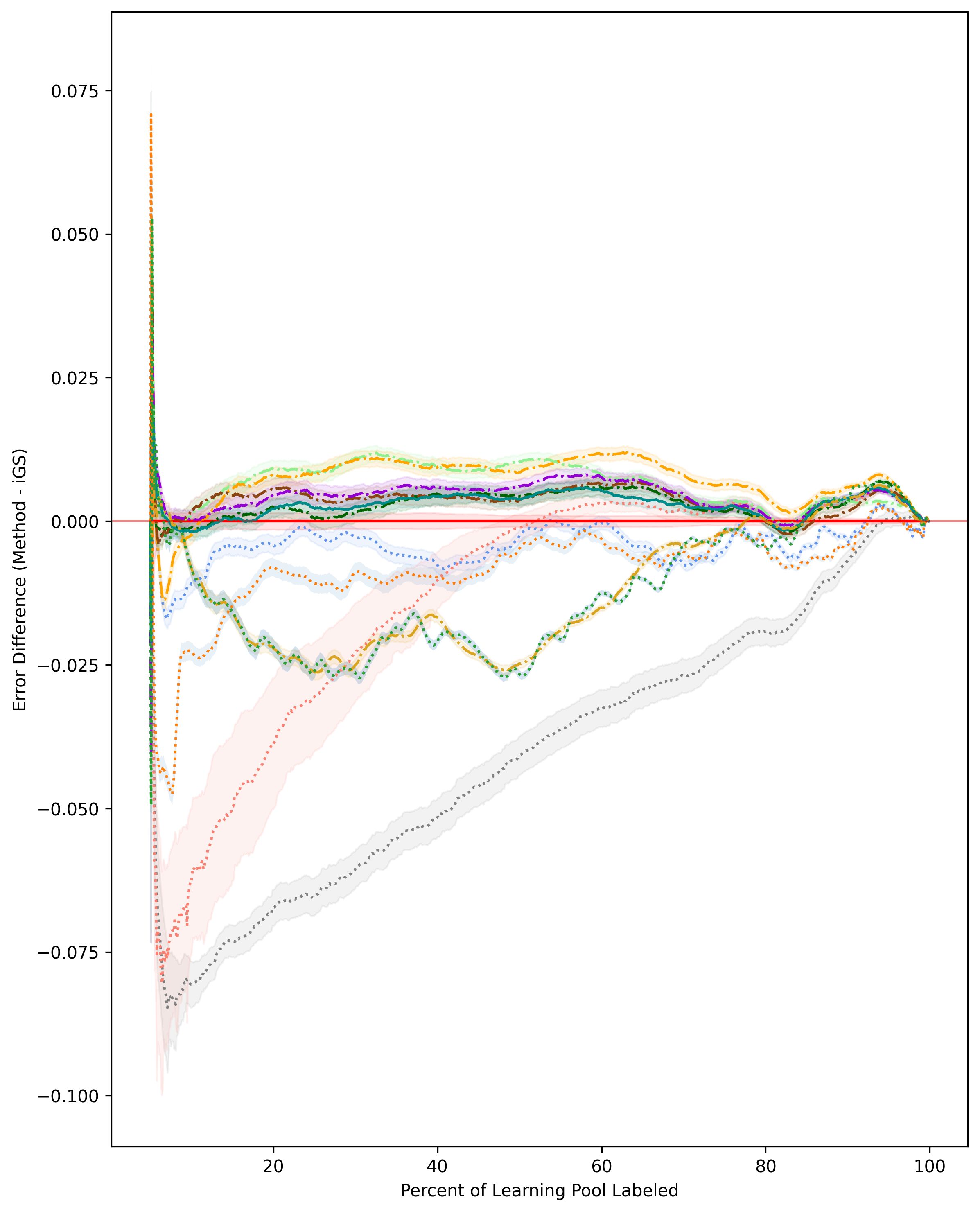}
        \caption{burbidge\_correct}
    \end{subfigure}
    
    \vspace{0.3em} 
    
    \begin{subfigure}[b]{0.31\textwidth}
        \centering
        \includegraphics[width=\linewidth, keepaspectratio]{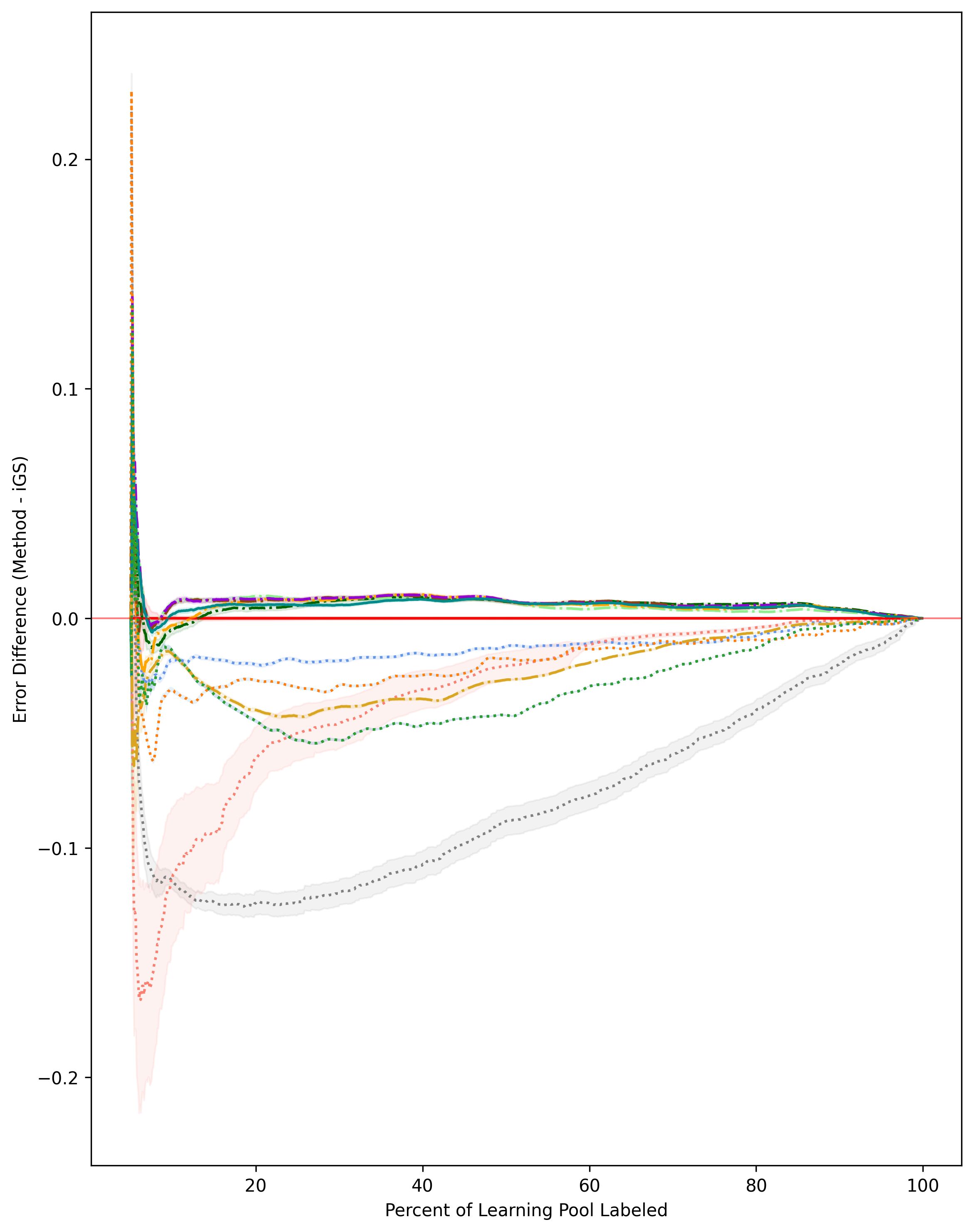}
        \caption{burbidge\_low\_noise}
    \end{subfigure}
    \hfill
    \begin{subfigure}[b]{0.31\textwidth}
        \centering
        \includegraphics[width=\linewidth, keepaspectratio]{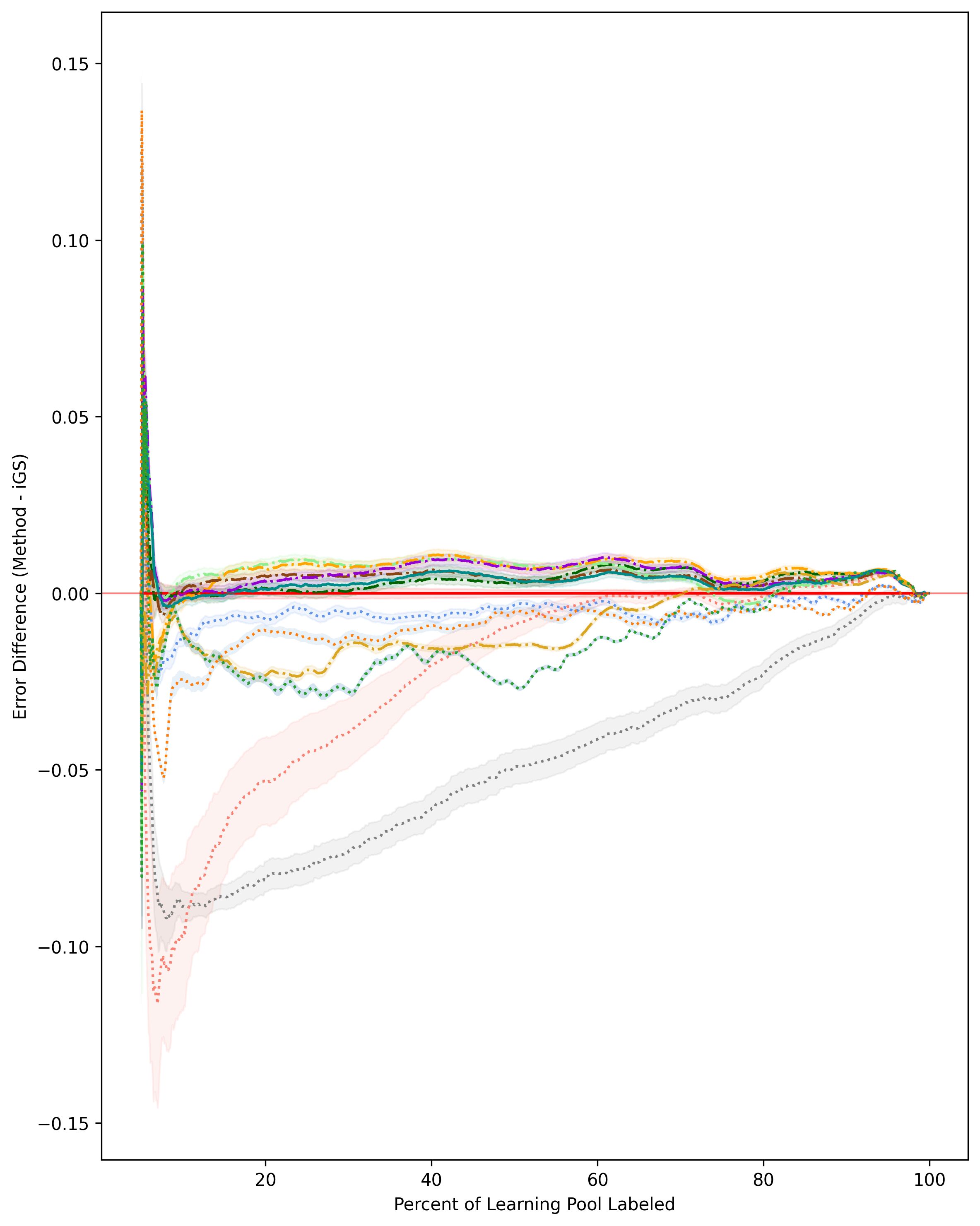}
        \caption{burbidge\_misspecified}
    \end{subfigure}
    \hfill
    \begin{subfigure}[b]{0.31\textwidth}
        \centering
        \includegraphics[width=\linewidth, keepaspectratio]{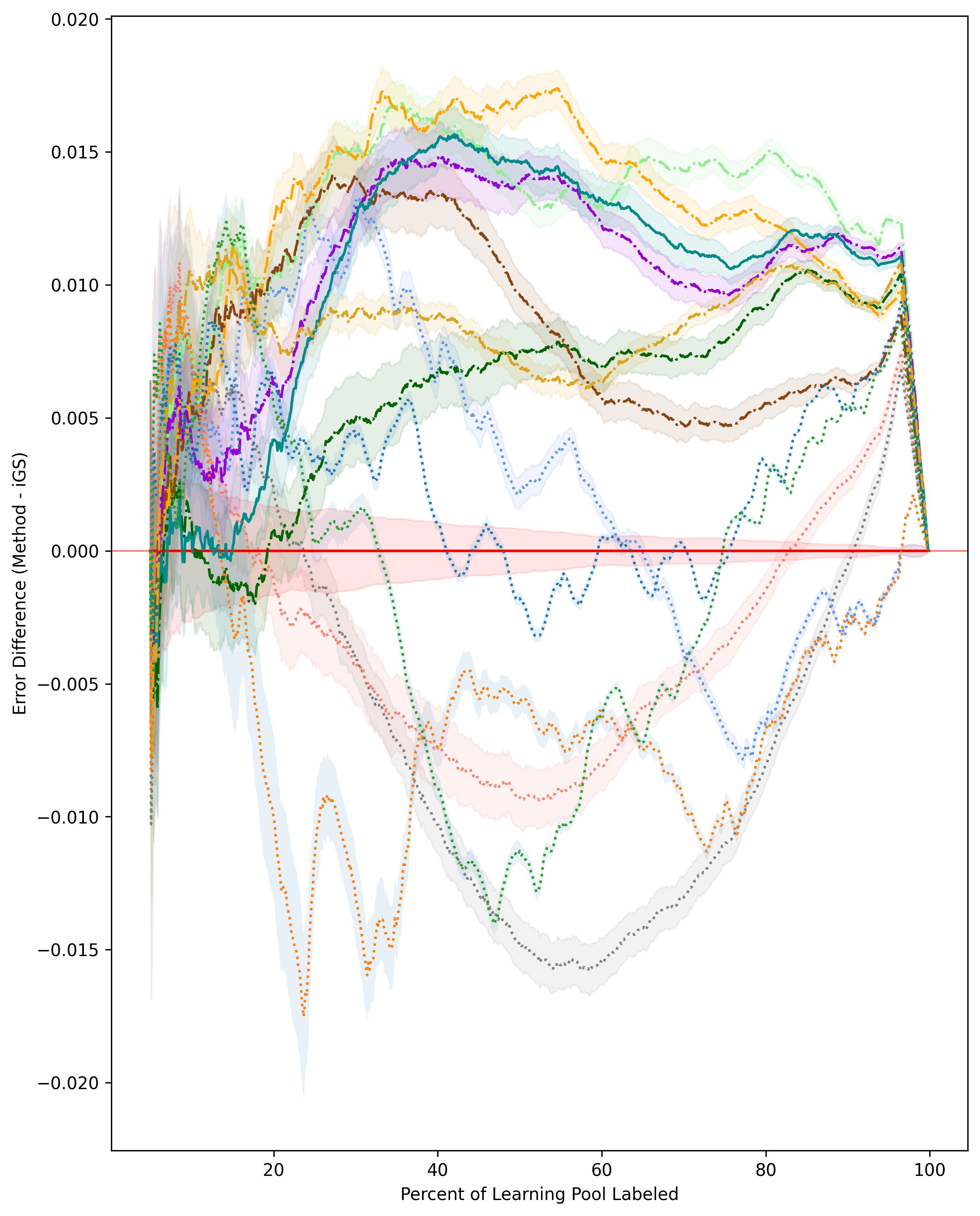}
        \caption{concrete\_4}
    \end{subfigure}
    
    \vspace{0.3em}
    
    \begin{subfigure}[b]{0.31\textwidth}
        \centering
        \includegraphics[width=\linewidth, keepaspectratio]{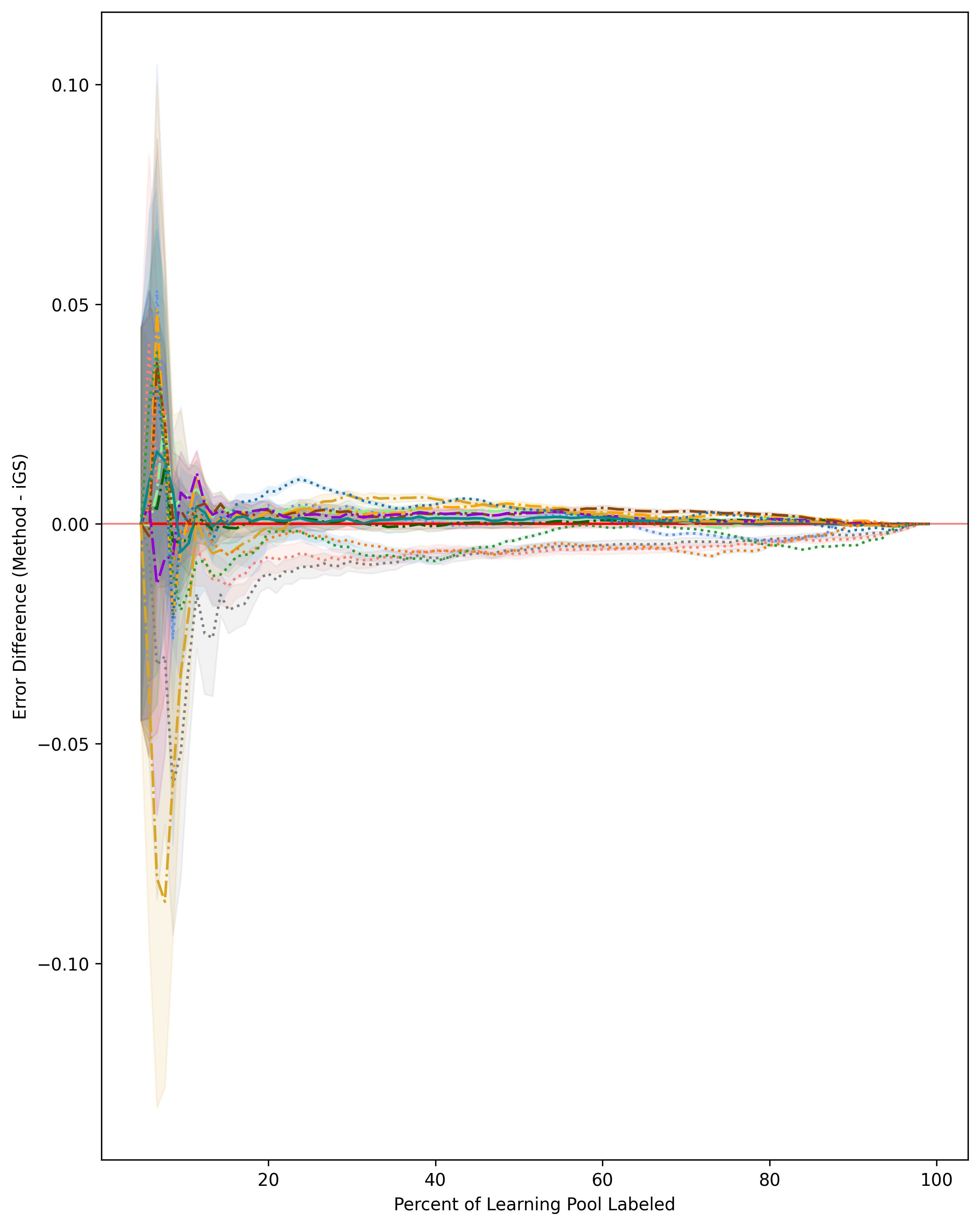}
        \caption{concrete\_cs}
    \end{subfigure}
    \hfill
    \begin{subfigure}[b]{0.31\textwidth}
        \centering
        \includegraphics[width=\linewidth, keepaspectratio]{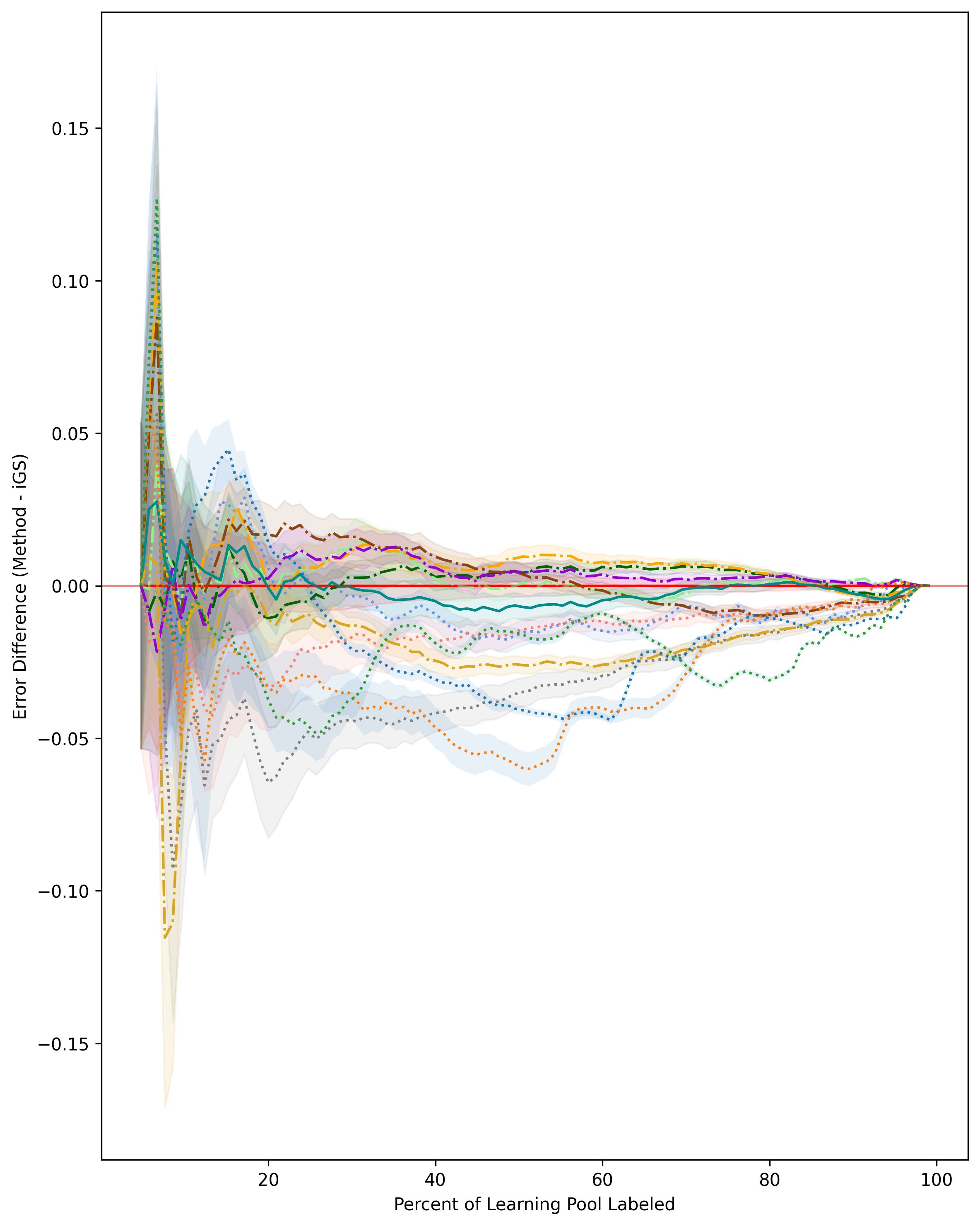}
        \caption{concrete\_flow}
    \end{subfigure}
    \hfill
    \begin{subfigure}[b]{0.31\textwidth}
        \centering
        \includegraphics[width=\linewidth, keepaspectratio]{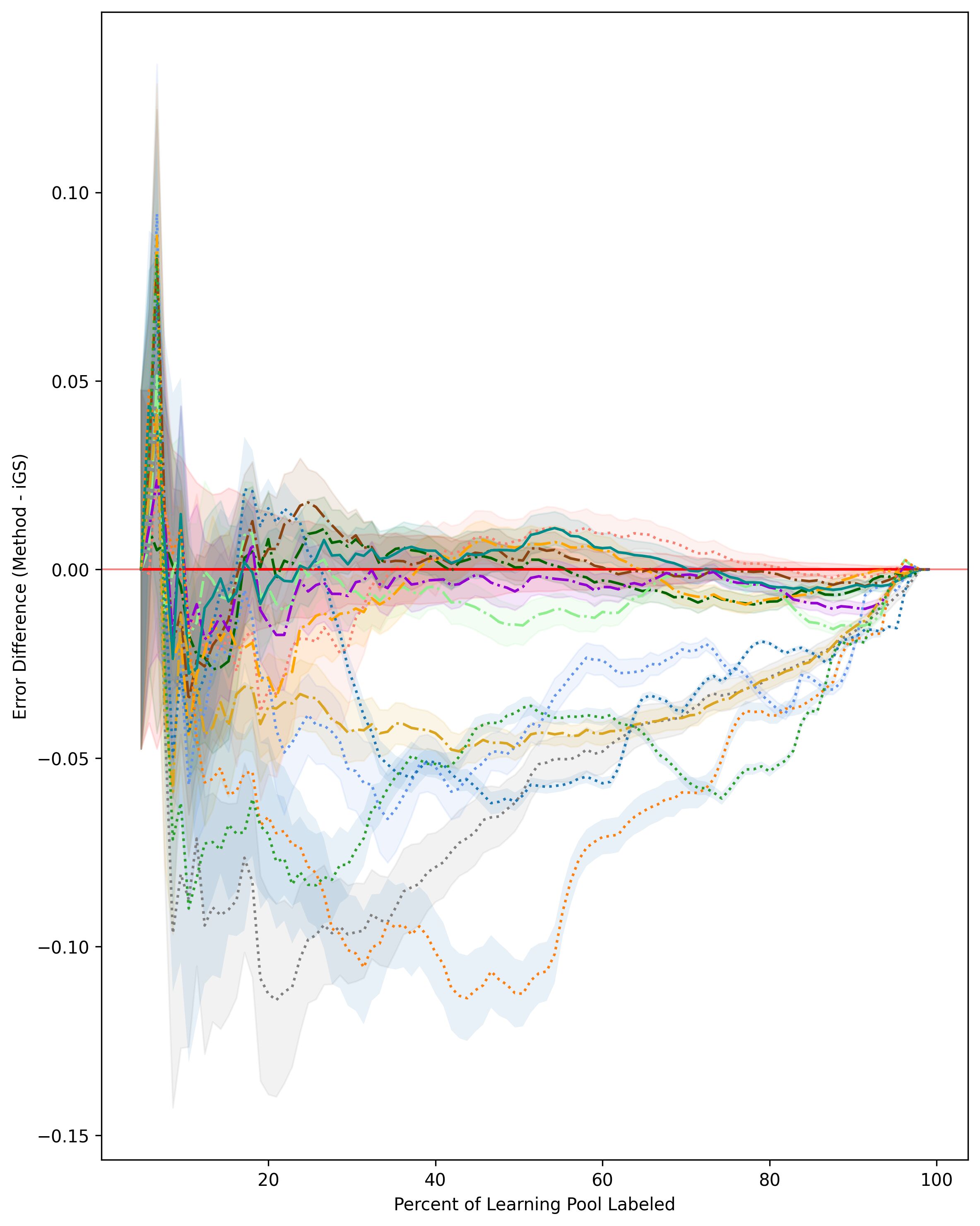}
        \caption{concrete\_slump}
    \end{subfigure}
    
    \vspace{0.5em}
    \centering
    \includegraphics[width=\linewidth]{upload_all_files/manuscript/benchmark_legend.jpg}
    \caption{Full-pool Correlation Coefficients trace plots for benchmark datasets (Part 1 of 2).}
    \label{fig:CCResults1}
\end{figure}

\clearpage
\begin{figure}
    \centering
    \vspace*{-1cm} 
    
    \begin{subfigure}[b]{0.31\textwidth}
        \centering
        \includegraphics[width=\linewidth, keepaspectratio]{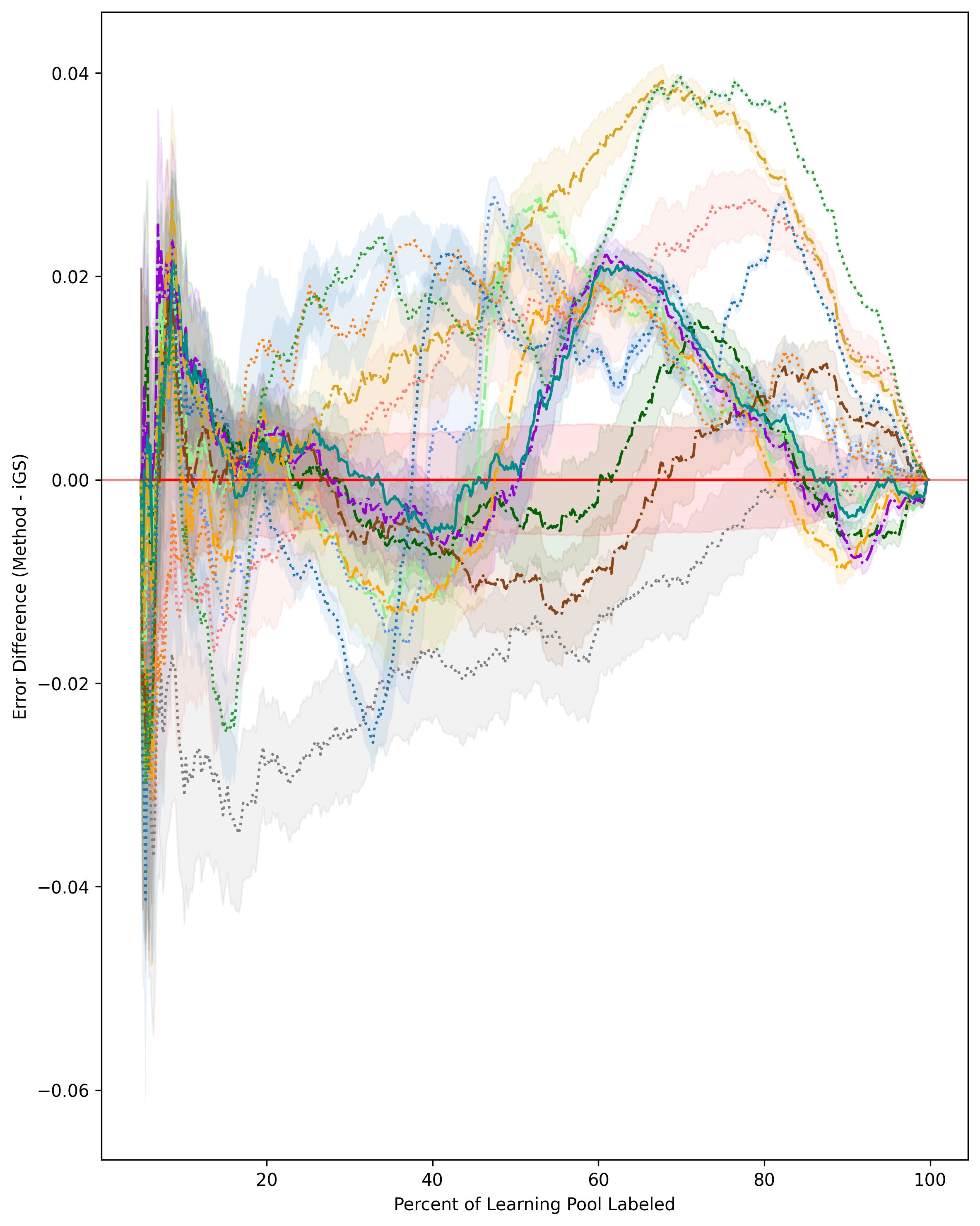}
        \caption{cps\_wage}
    \end{subfigure}
    \hfill
    \begin{subfigure}[b]{0.31\textwidth}
        \centering
        \includegraphics[width=\linewidth, keepaspectratio]{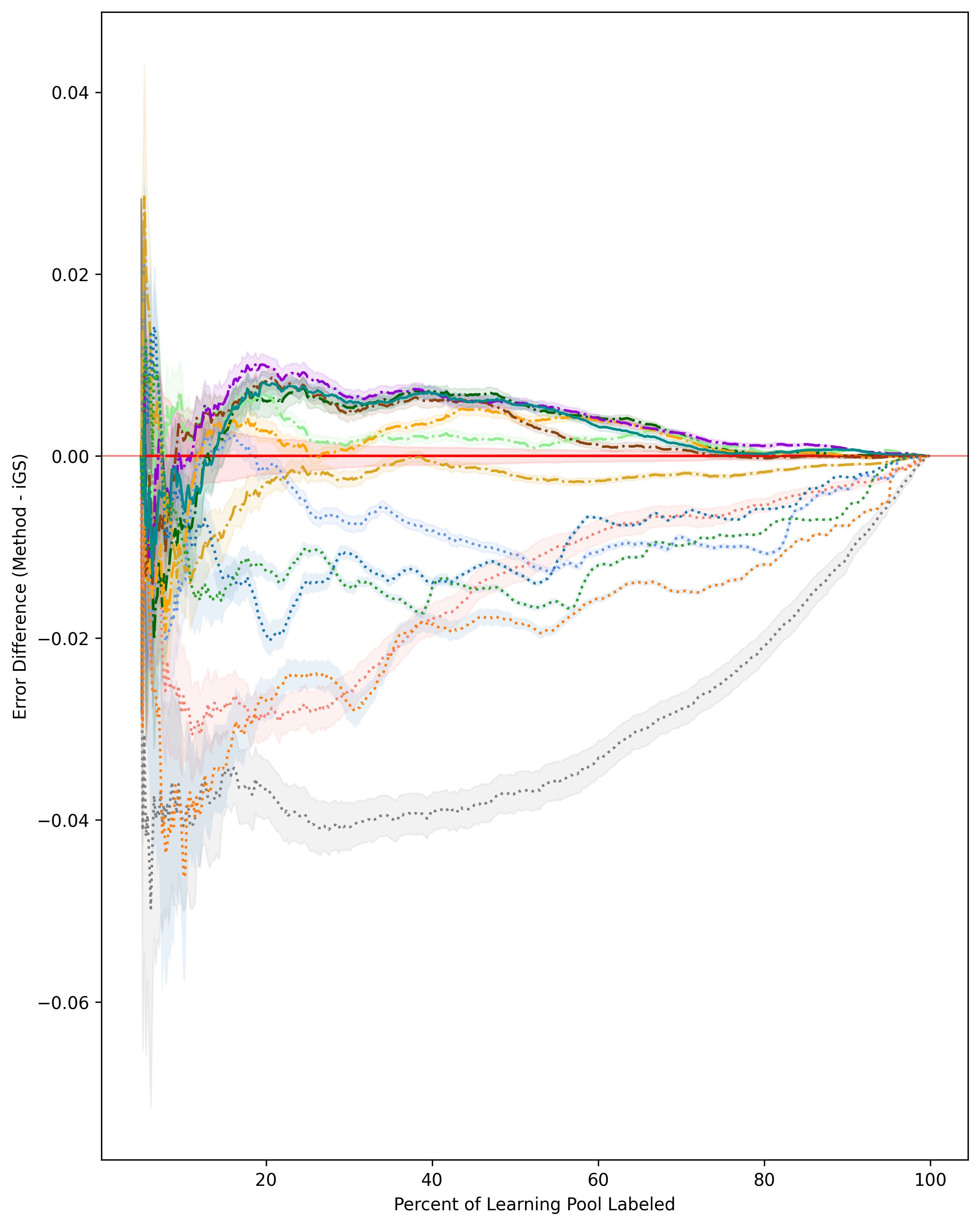}
        \caption{housing}
    \end{subfigure}
    \hfill
    \begin{subfigure}[b]{0.31\textwidth}
        \centering
        \includegraphics[width=\linewidth, keepaspectratio]{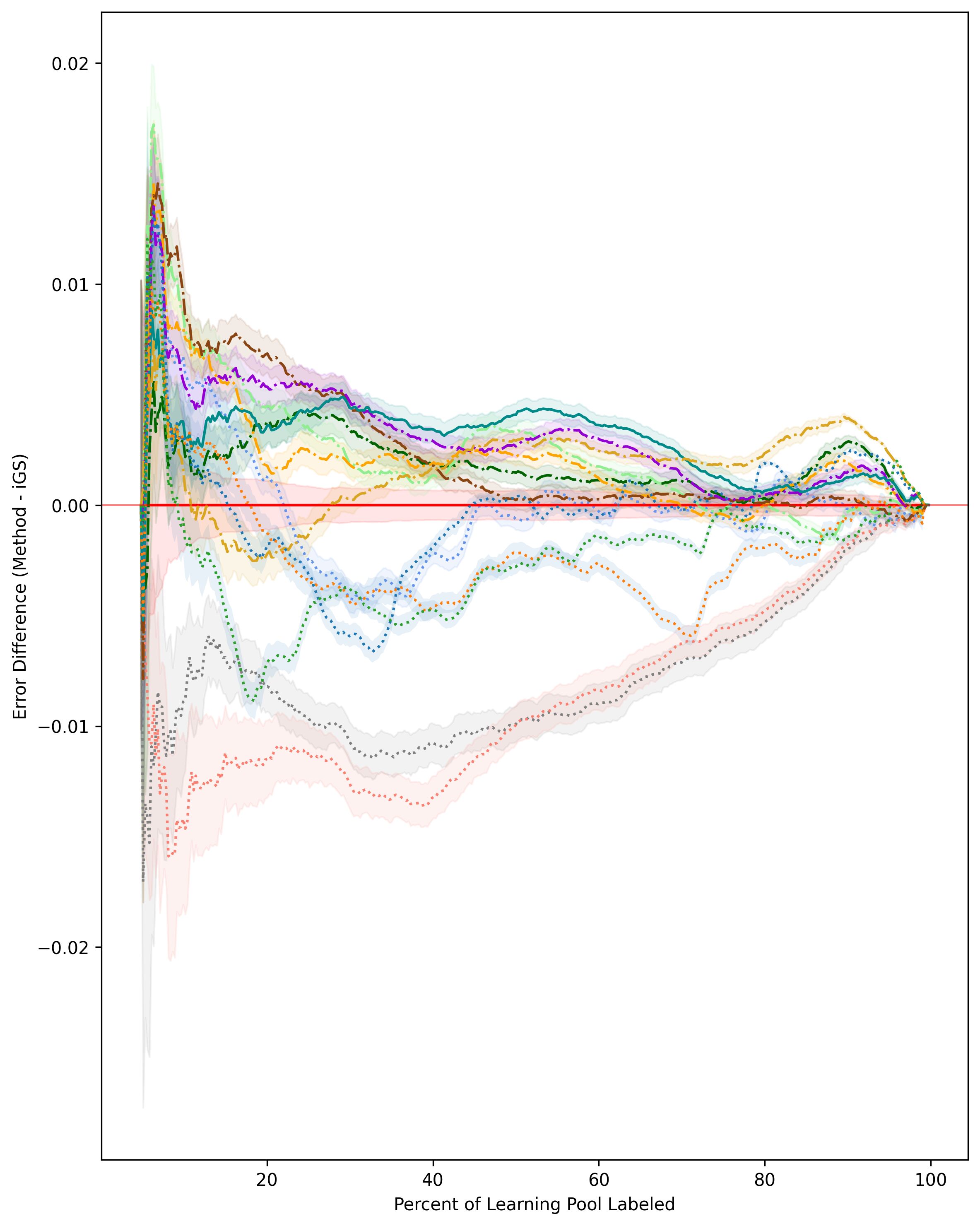}
        \caption{mpg}
    \end{subfigure}
    
    \vspace{0.3em}
    
    \begin{subfigure}[b]{0.31\textwidth}
        \centering
        \includegraphics[width=\linewidth, keepaspectratio]{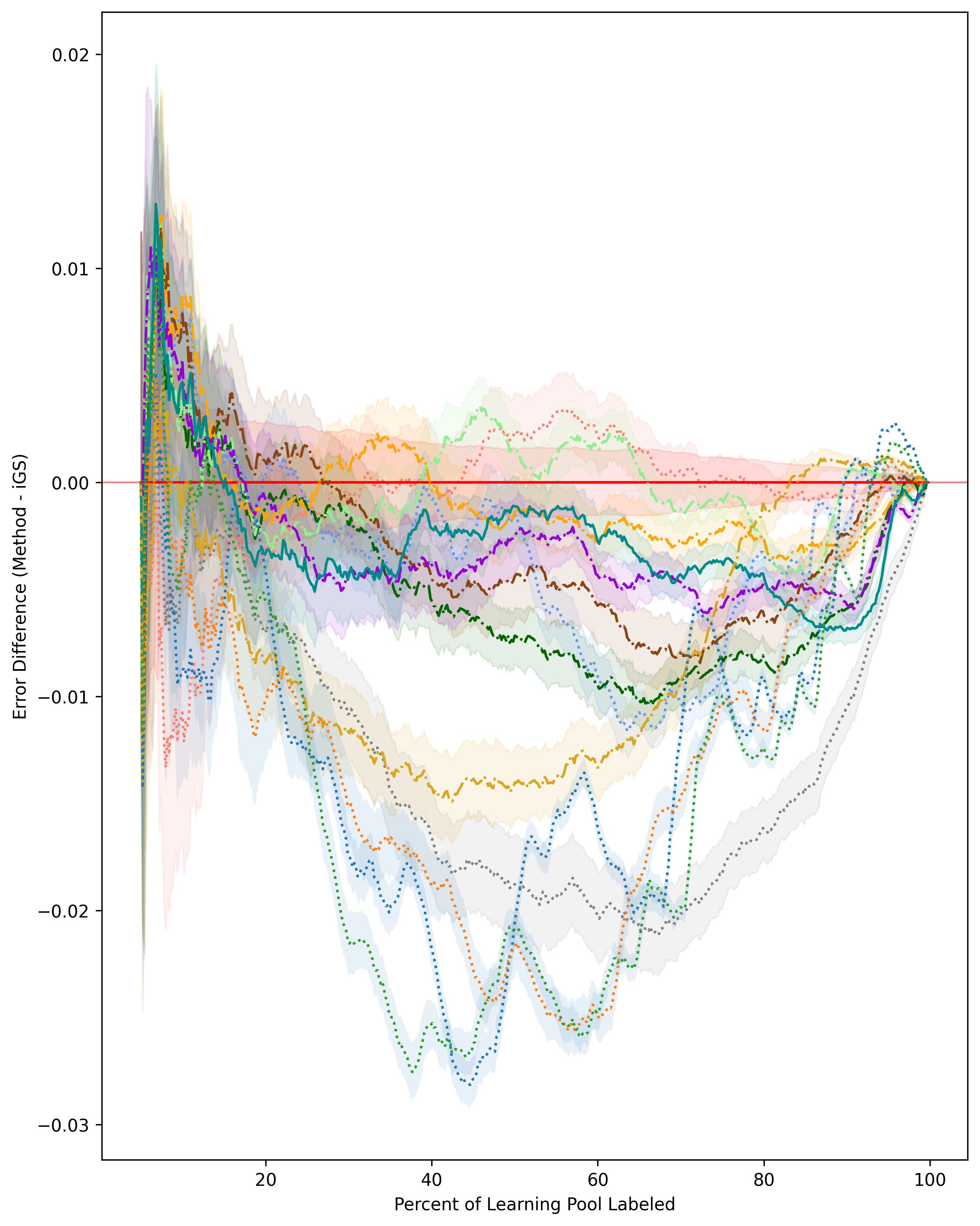}
        \caption{no2}
    \end{subfigure}
    \hfill
    \begin{subfigure}[b]{0.31\textwidth}
        \centering
        \includegraphics[width=\linewidth, keepaspectratio]{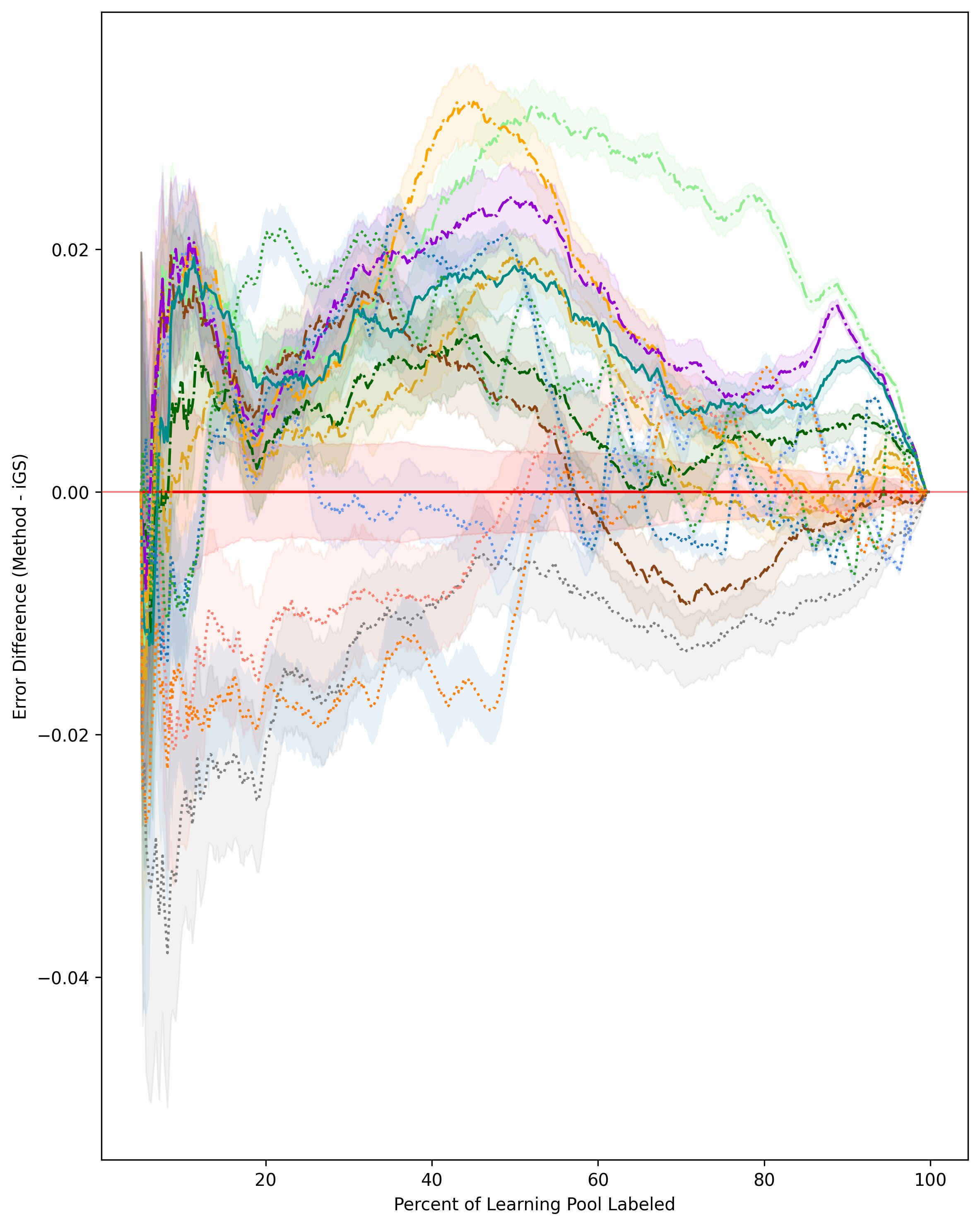}
        \caption{pm10}
    \end{subfigure}
    \hfill
    \begin{subfigure}[b]{0.31\textwidth}
        \centering
        \includegraphics[width=\linewidth, keepaspectratio]{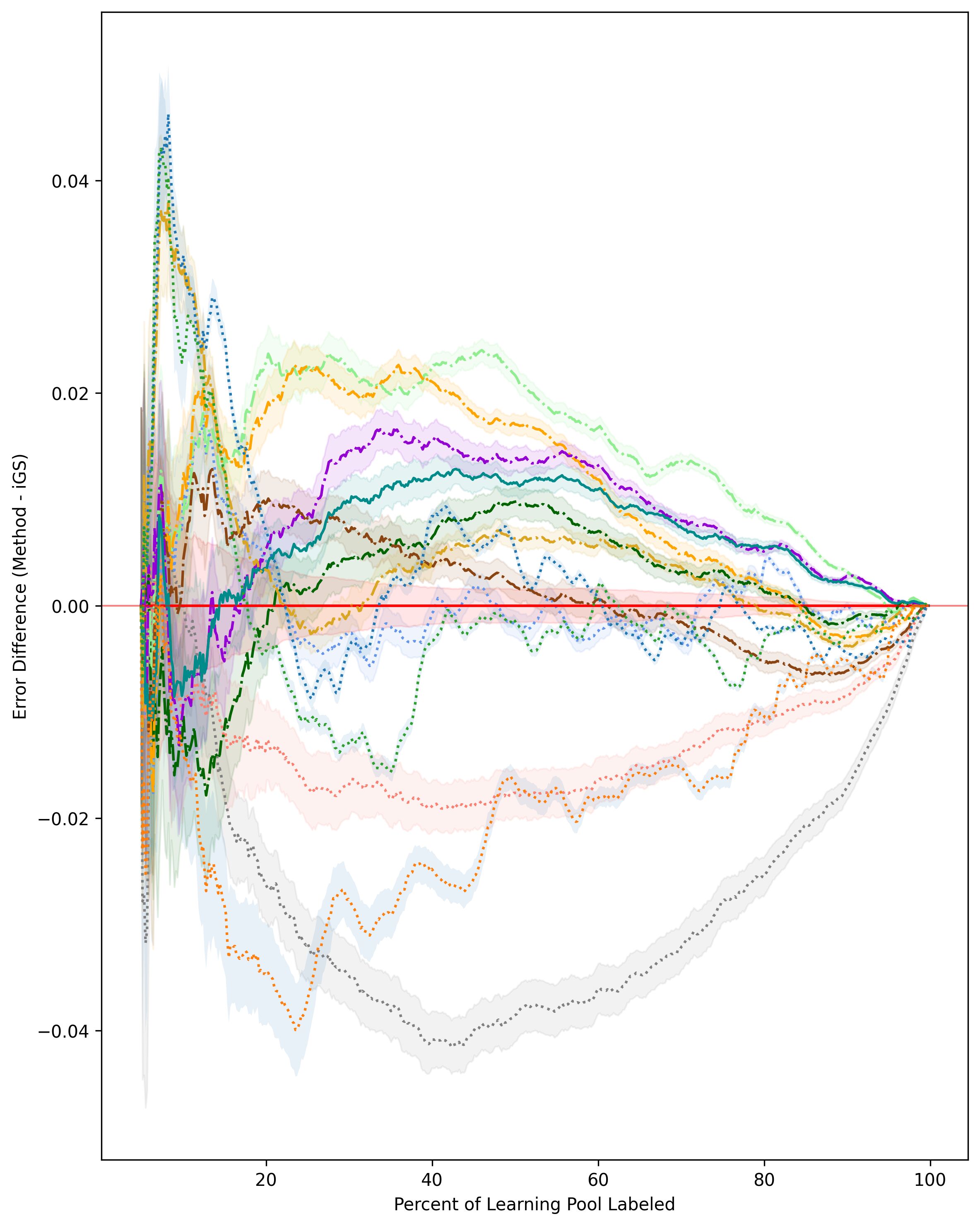}
        \caption{qsar}
    \end{subfigure}
    
    \vspace{0.3em}
    
    \begin{subfigure}[b]{0.31\textwidth}
        \centering
        \includegraphics[width=\linewidth, keepaspectratio]{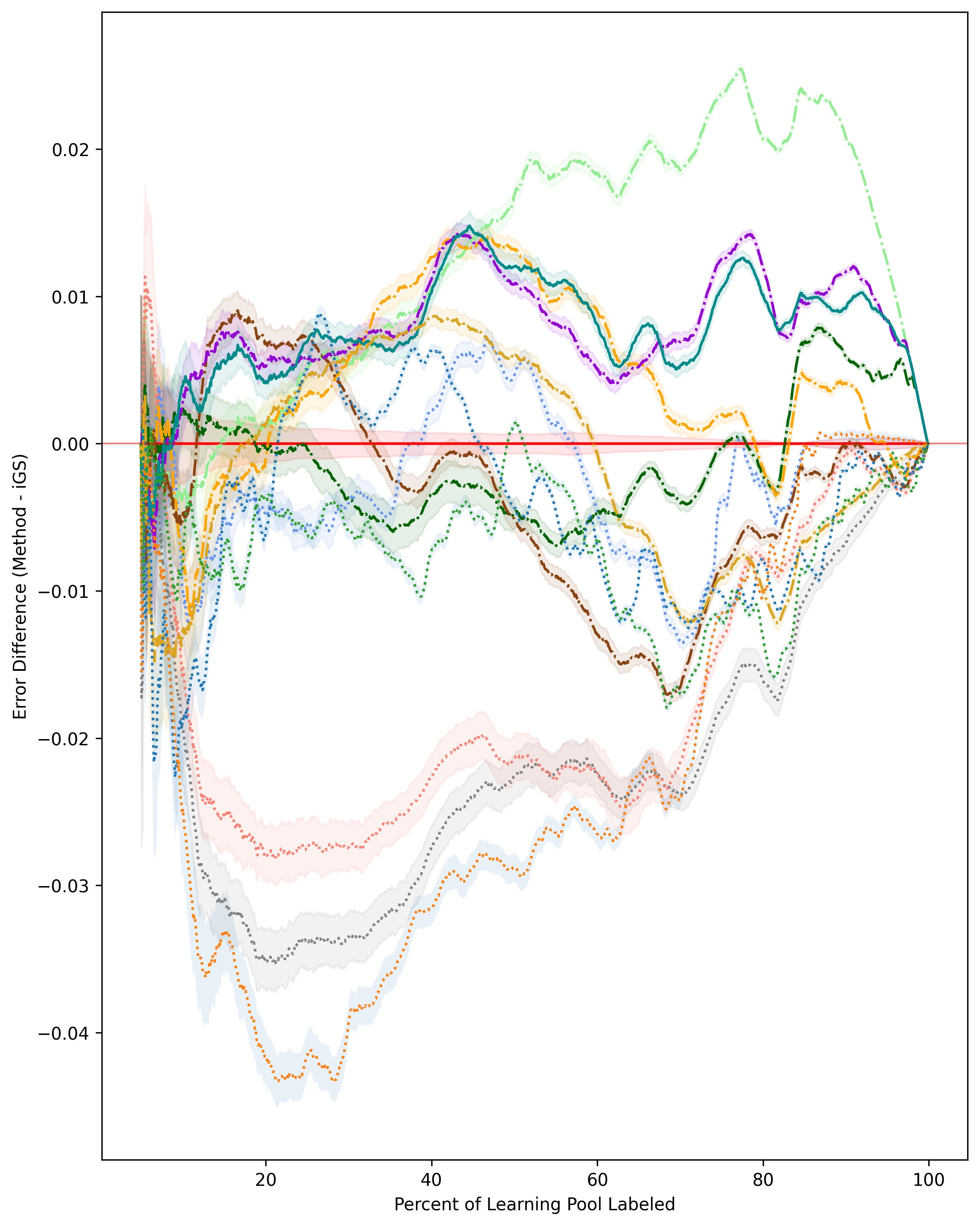}
        \caption{wine\_red}
    \end{subfigure}
    \hfill
    \begin{subfigure}[b]{0.31\textwidth}
        \centering
        \includegraphics[width=\linewidth, keepaspectratio]{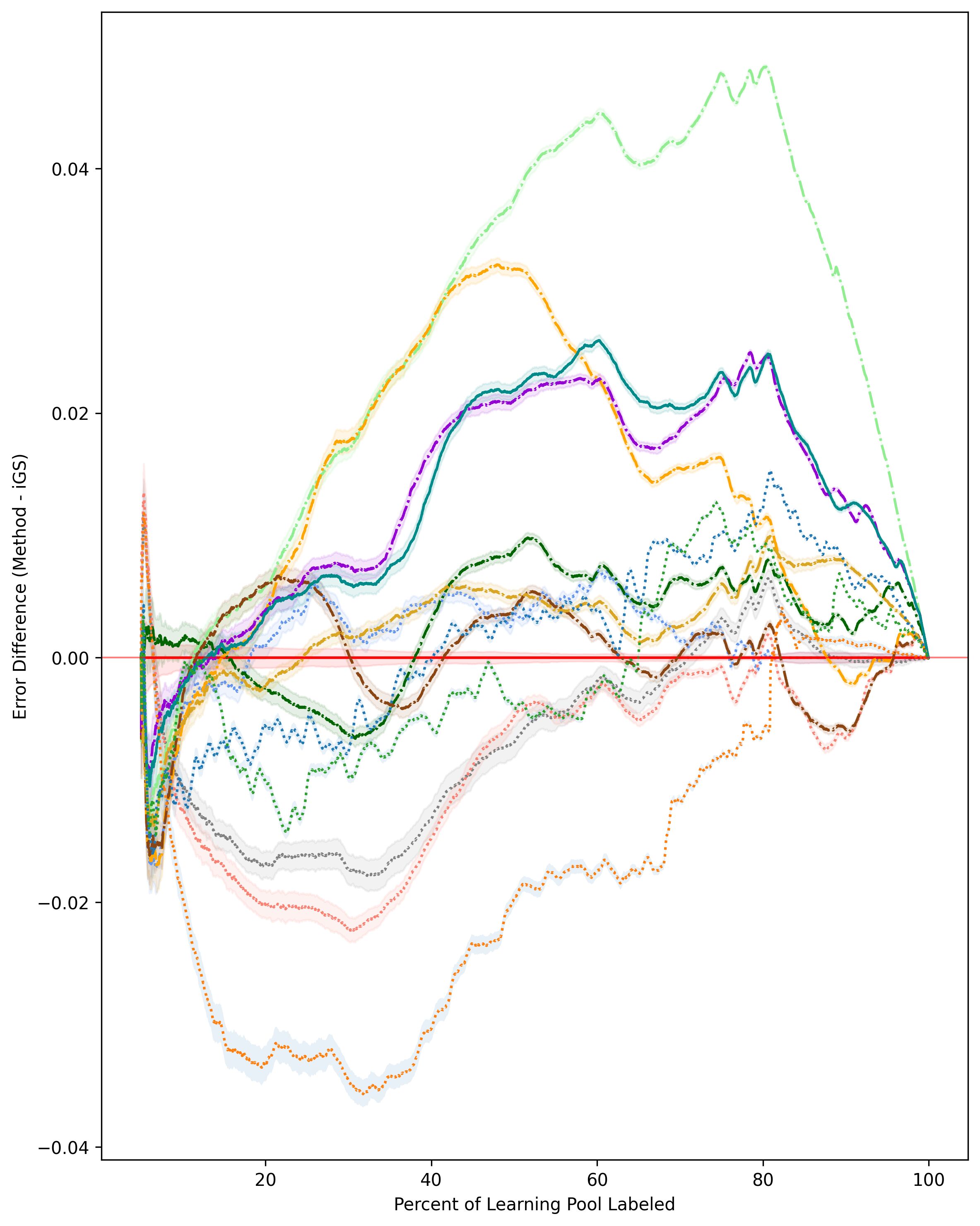}
        \caption{wine\_white}
    \end{subfigure}
    \hfill
    \begin{subfigure}[b]{0.31\textwidth}
        \centering
        \includegraphics[width=\linewidth, keepaspectratio]{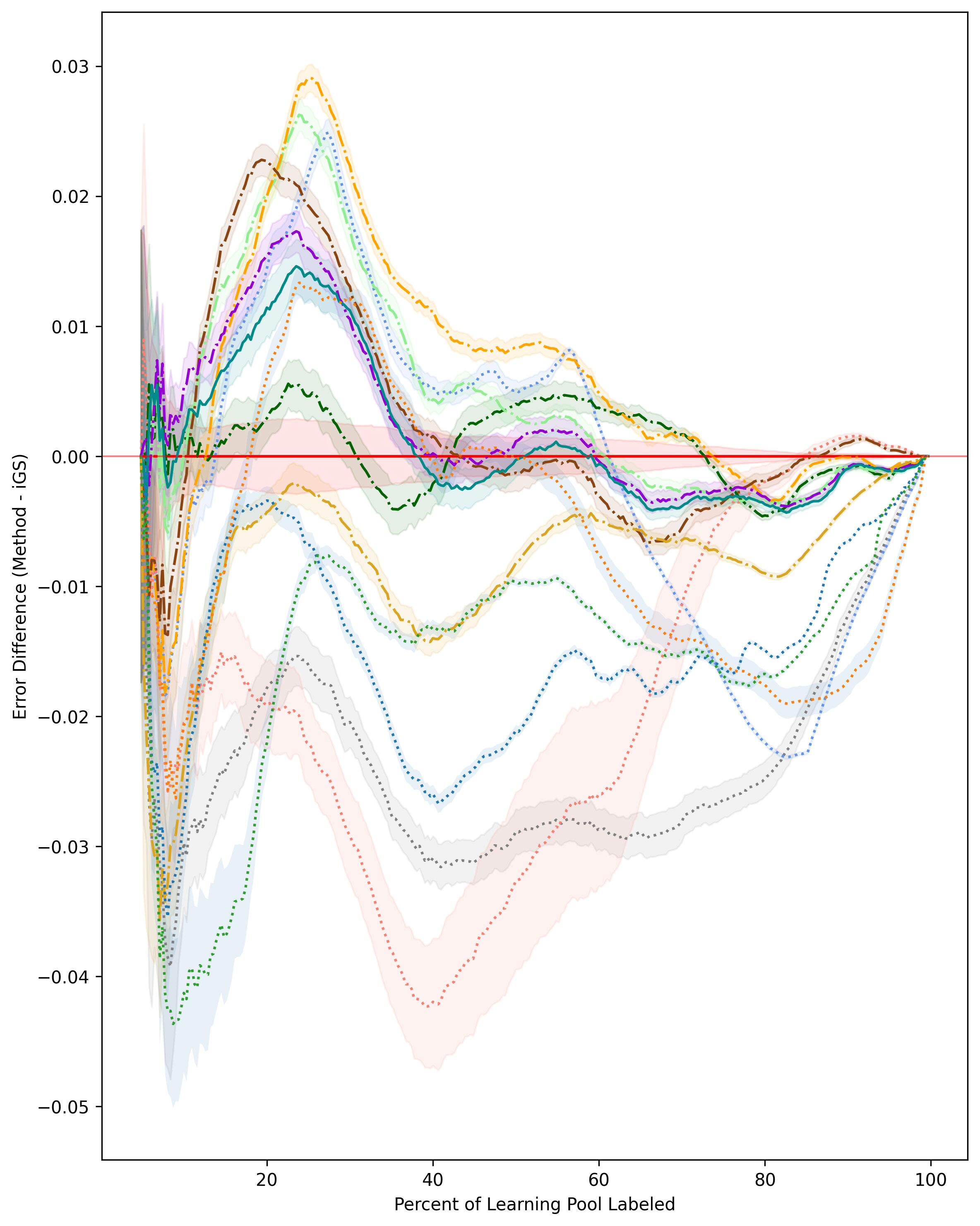}
        \caption{yacht}
    \end{subfigure}
    
    \vspace{0.5em}
    \centering
    \includegraphics[width=\linewidth]{upload_all_files/manuscript/benchmark_legend.jpg}
    \caption{Full-pool Correlation Coefficients trace plots for benchmark datasets (Part 2 of 2).}
    \label{fig:CCResults2}
\end{figure}
\clearpage


\begin{figure}
    \centering
    \vspace*{-1cm} 
    
    \begin{subfigure}[b]{0.31\textwidth}
        \centering
        \includegraphics[width=\linewidth, keepaspectratio]{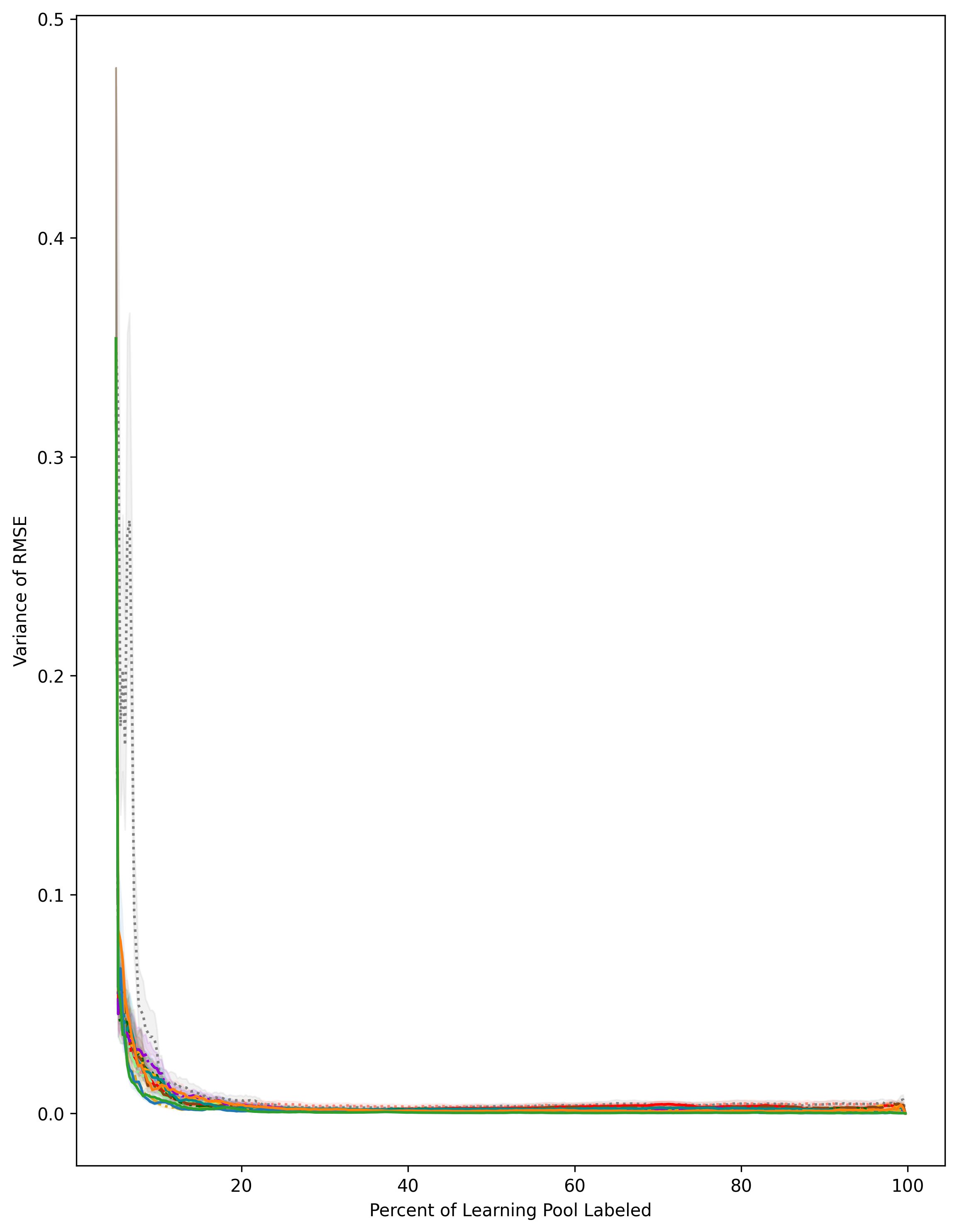}
        \caption{beer}
    \end{subfigure}
    \hfill
    \begin{subfigure}[b]{0.31\textwidth}
        \centering
        \includegraphics[width=\linewidth, keepaspectratio]{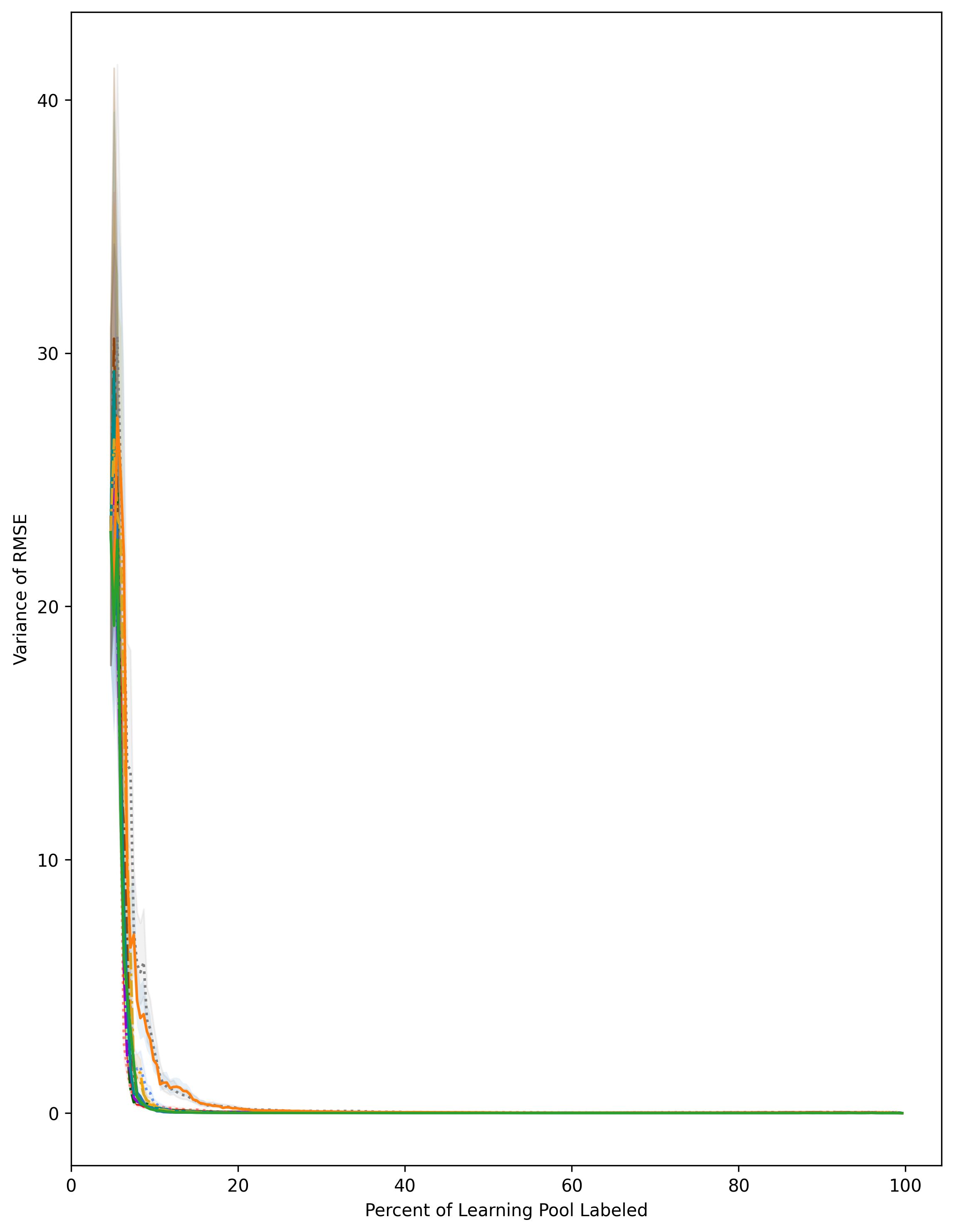}
        \caption{bodyfat}
    \end{subfigure}
    \hfill
    \begin{subfigure}[b]{0.31\textwidth}
        \centering
        \includegraphics[width=\linewidth, keepaspectratio]{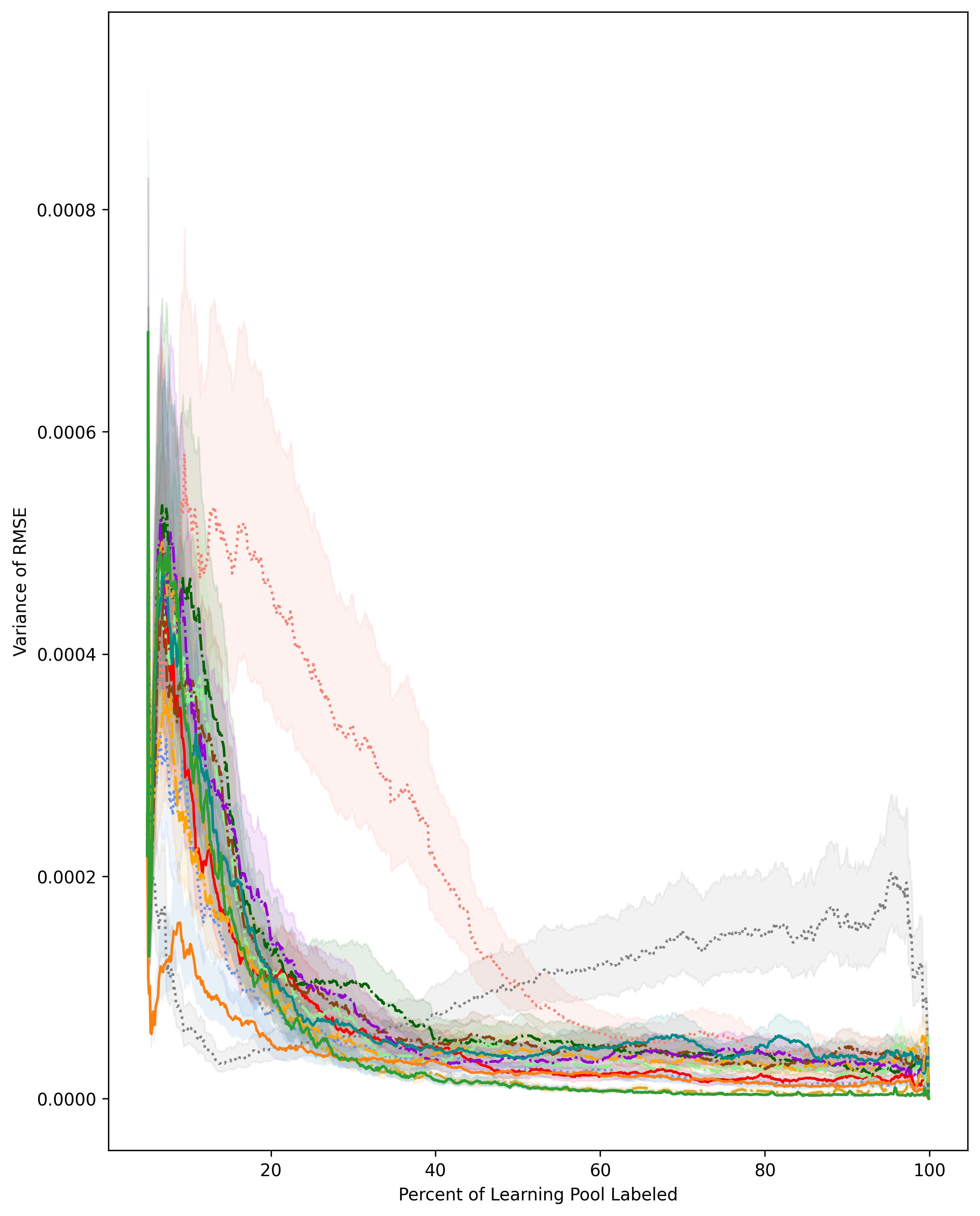}
        \caption{burbidge\_correct}
    \end{subfigure}
    
    \vspace{0.3em} 
    
    \begin{subfigure}[b]{0.31\textwidth}
        \centering
        \includegraphics[width=\linewidth, keepaspectratio]{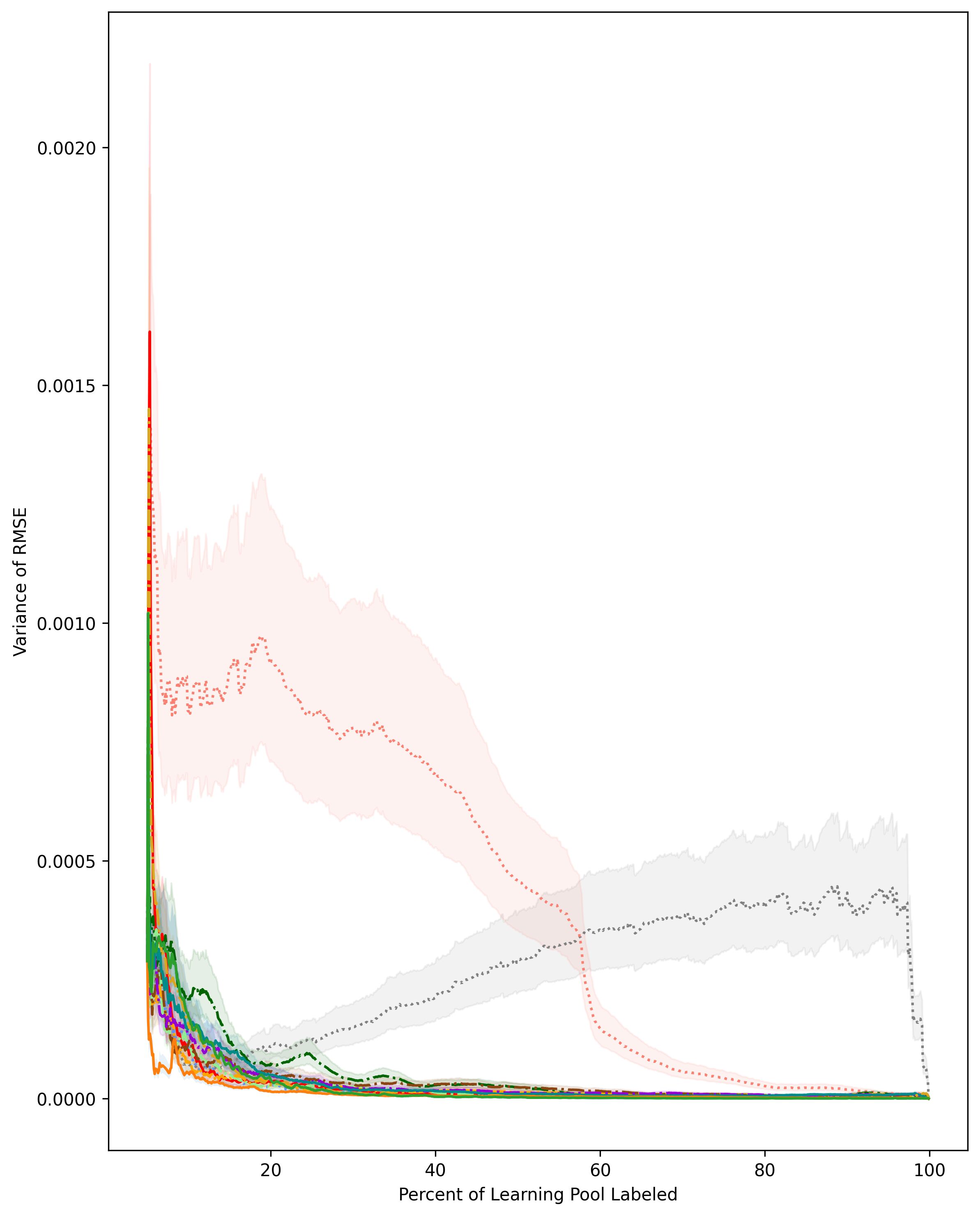}
        \caption{burbidge\_low\_noise}
    \end{subfigure}
    \hfill
    \begin{subfigure}[b]{0.31\textwidth}
        \centering
        \includegraphics[width=\linewidth, keepaspectratio]{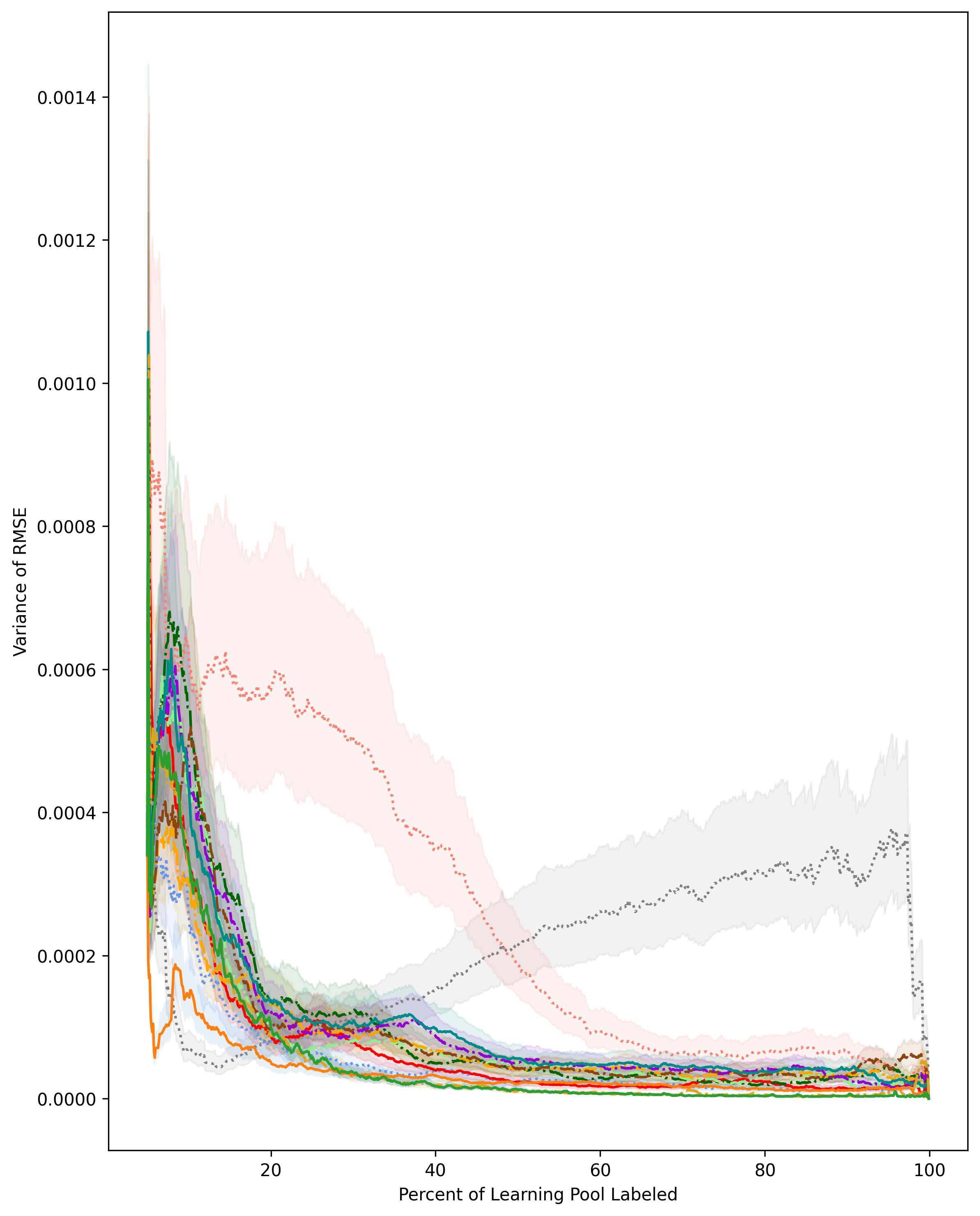}
        \caption{burbidge\_misspecified}
    \end{subfigure}
    \hfill
    \begin{subfigure}[b]{0.31\textwidth}
        \centering
        \includegraphics[width=\linewidth, keepaspectratio]{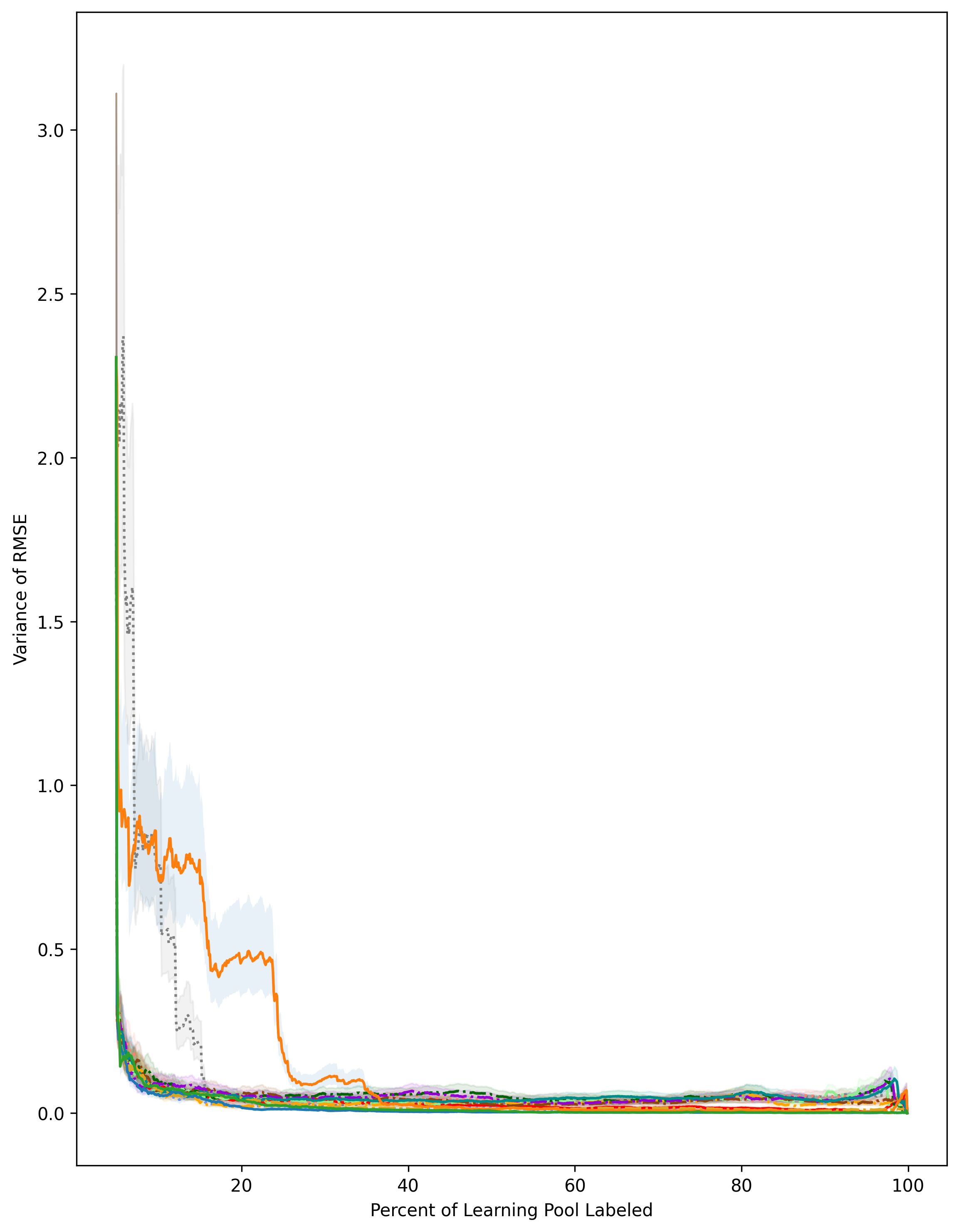}
        \caption{concrete\_4}
    \end{subfigure}
    
    \vspace{0.3em}
    
    \begin{subfigure}[b]{0.31\textwidth}
        \centering
        \includegraphics[width=\linewidth, keepaspectratio]{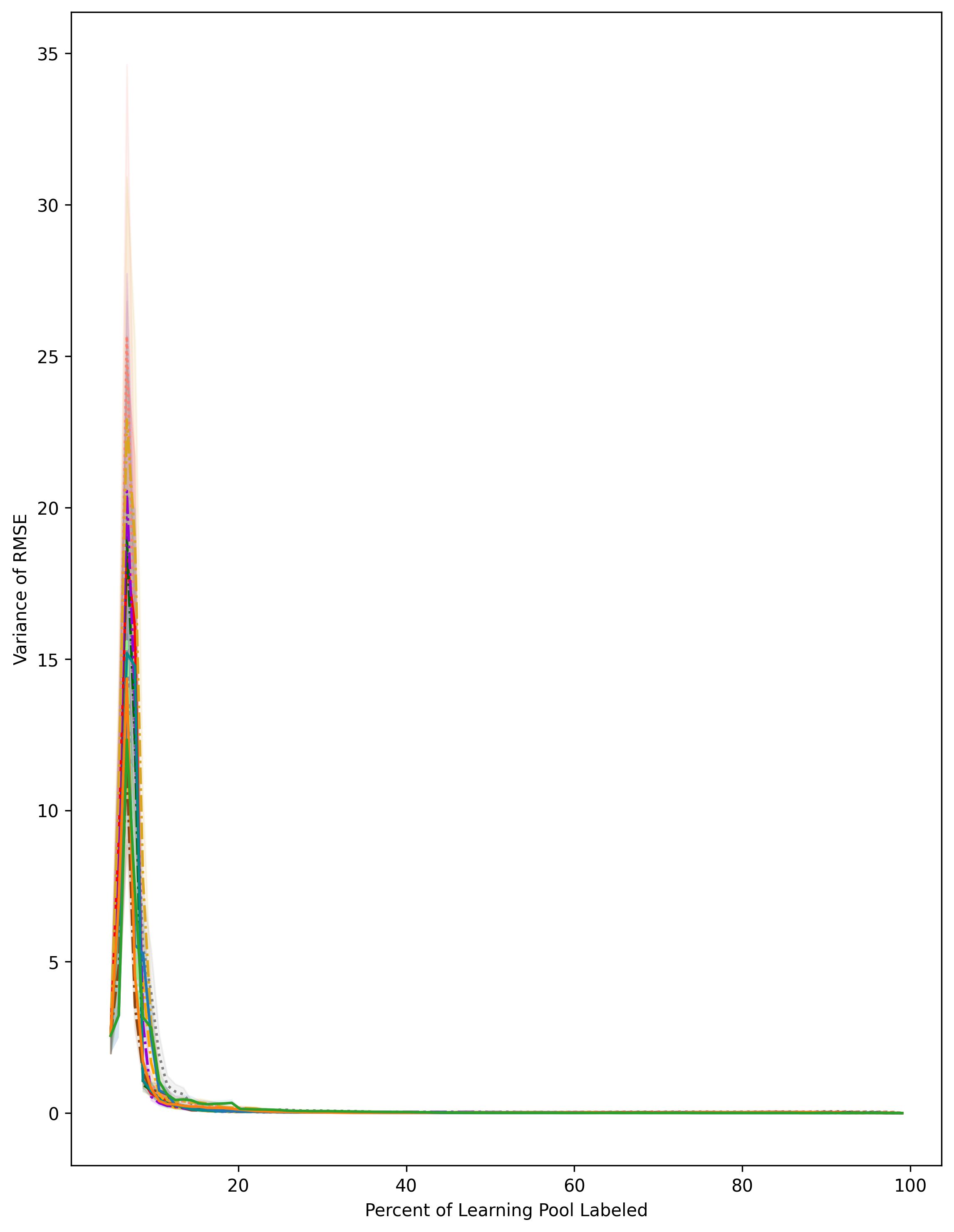}
        \caption{concrete\_cs}
    \end{subfigure}
    \hfill
    \begin{subfigure}[b]{0.31\textwidth}
        \centering
        \includegraphics[width=\linewidth, keepaspectratio]{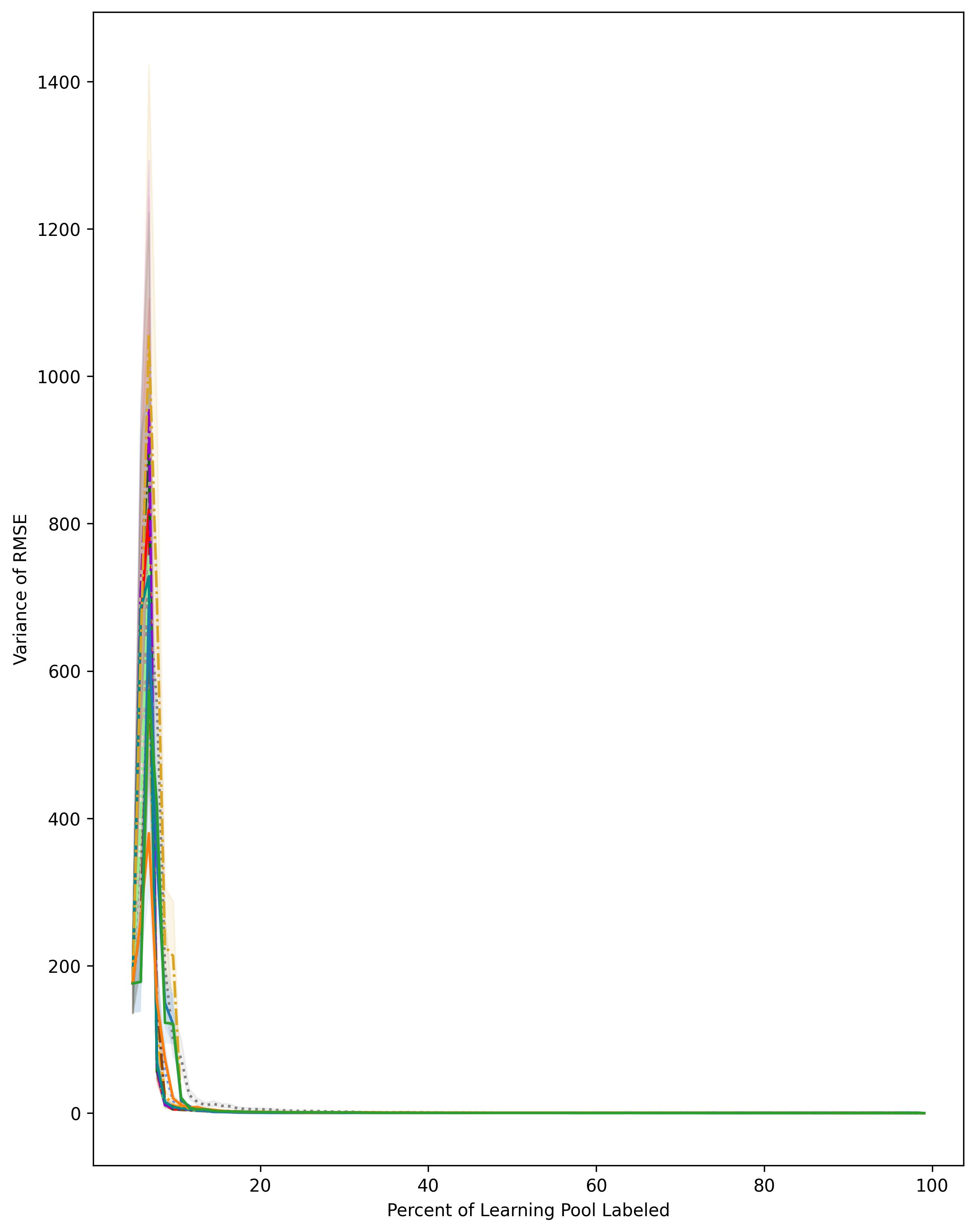}
        \caption{concrete\_flow}
    \end{subfigure}
    \hfill
    \begin{subfigure}[b]{0.31\textwidth}
        \centering
        \includegraphics[width=\linewidth, keepaspectratio]{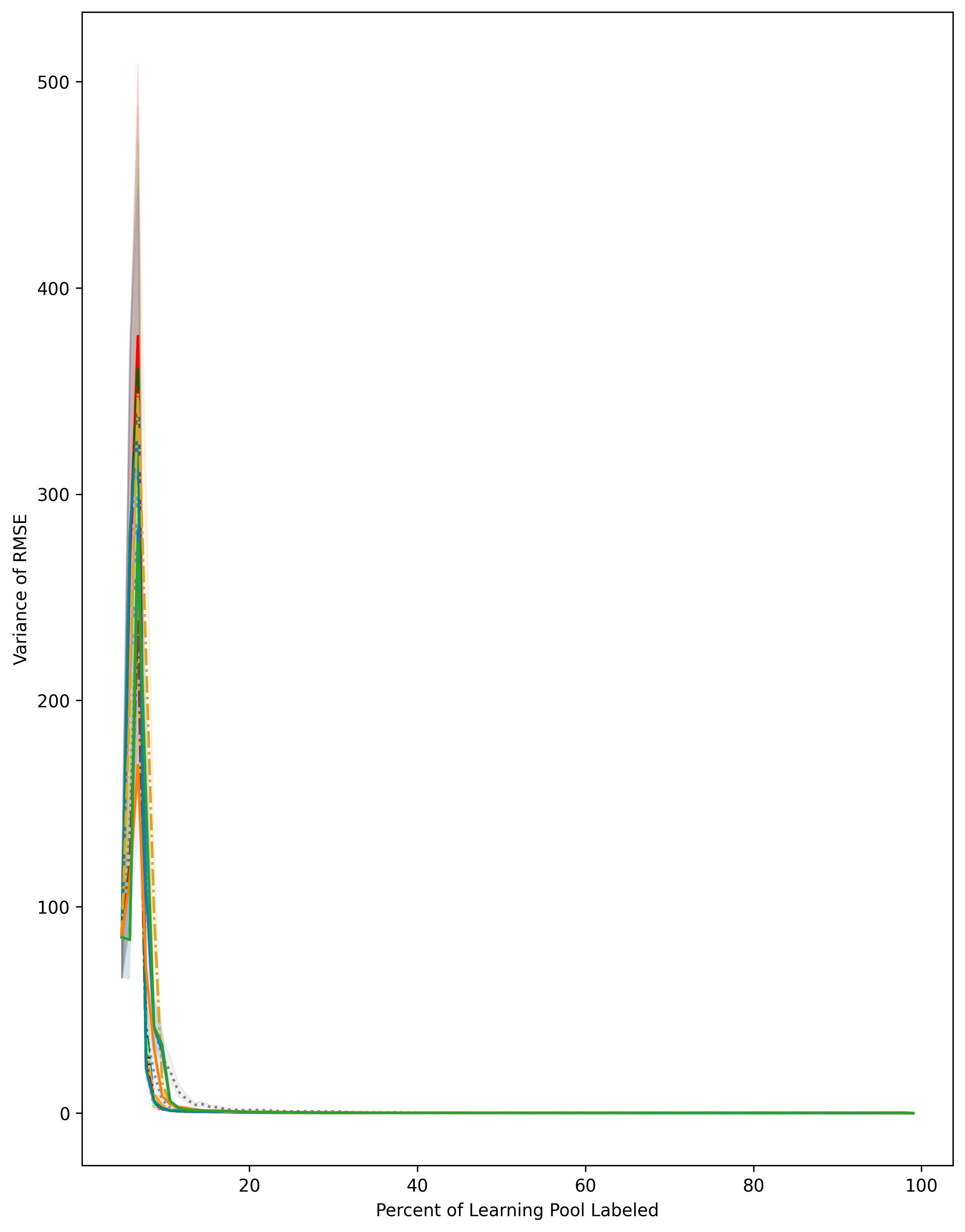}
        \caption{concrete\_slump}
    \end{subfigure}
    
    \vspace{0.5em}
    \centering
    \includegraphics[width=\linewidth]{upload_all_files/manuscript/benchmark_legend.jpg}
    
    \caption{Variance of the full-pool RMSE trace plots for benchmark datasets (Part 1 of 2).}
    \label{fig:VarianceResults1}
\end{figure}

\clearpage
\begin{figure}
    \centering
    \vspace*{-1cm} 
    
    \begin{subfigure}[b]{0.31\textwidth}
        \centering
        \includegraphics[width=\linewidth, keepaspectratio]{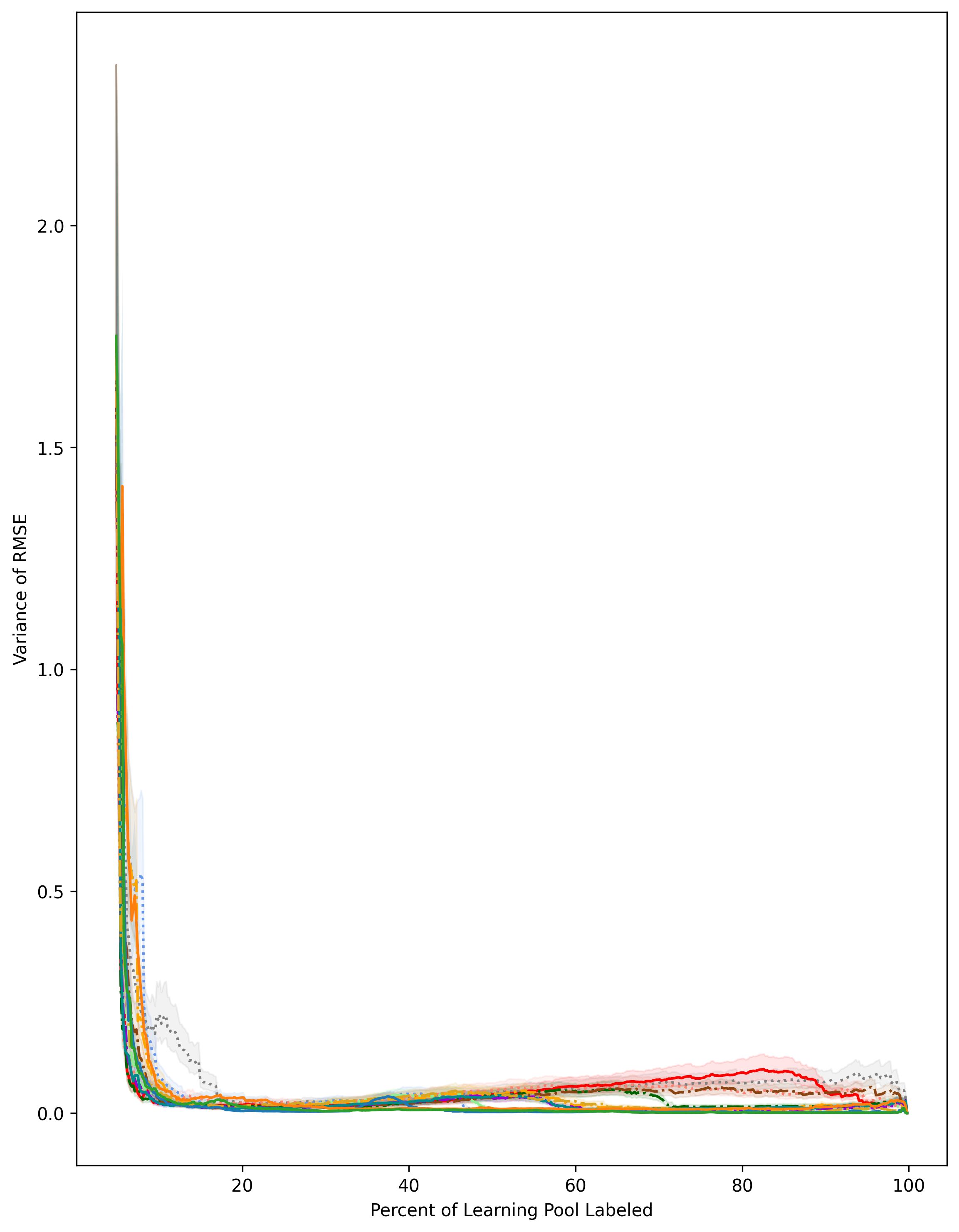}
        \caption{cps\_wage}
    \end{subfigure}
    \hfill
    \begin{subfigure}[b]{0.31\textwidth}
        \centering
        \includegraphics[width=\linewidth, keepaspectratio]{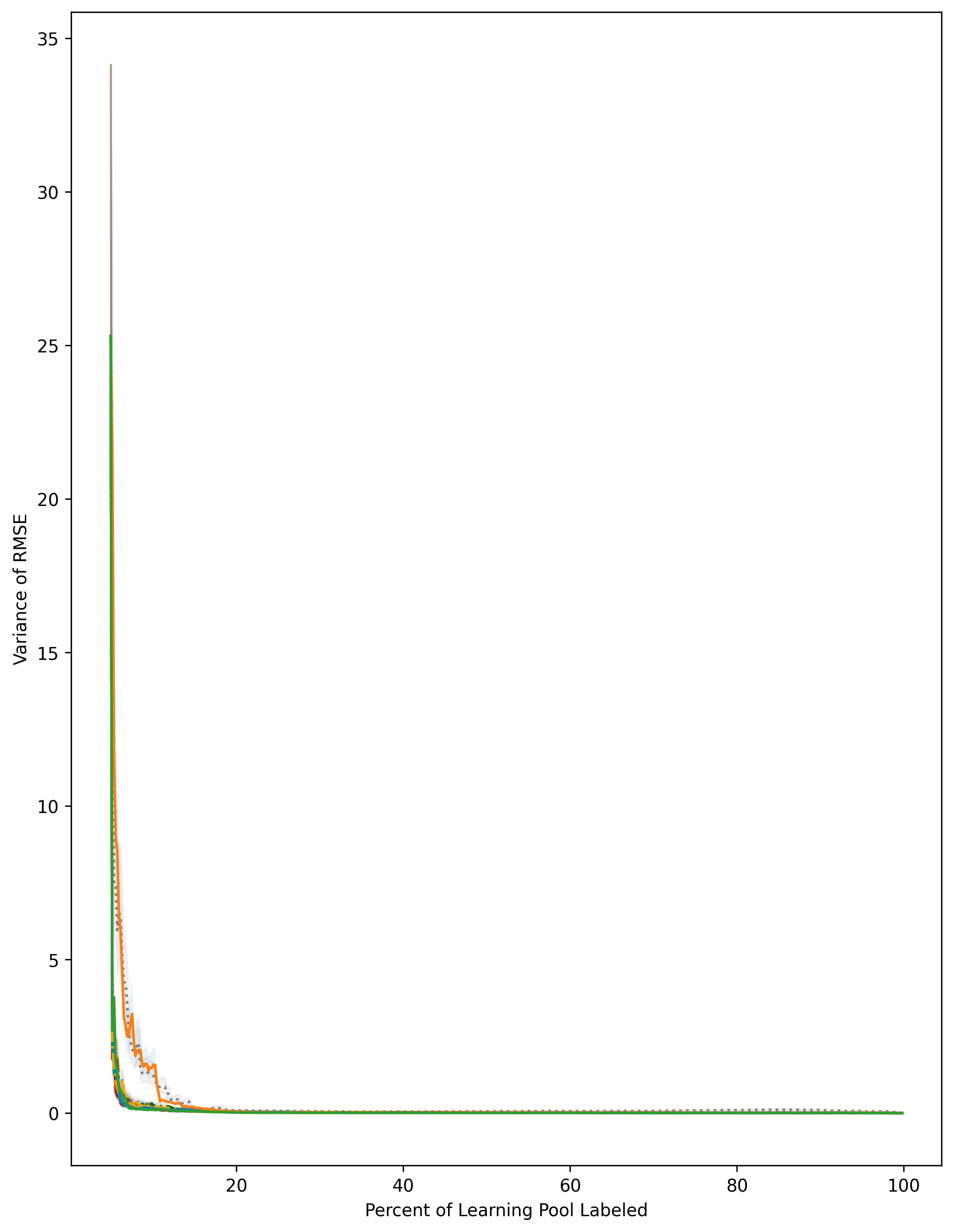}
        \caption{housing}
    \end{subfigure}
    \hfill
    \begin{subfigure}[b]{0.31\textwidth}
        \centering
        \includegraphics[width=\linewidth, keepaspectratio]{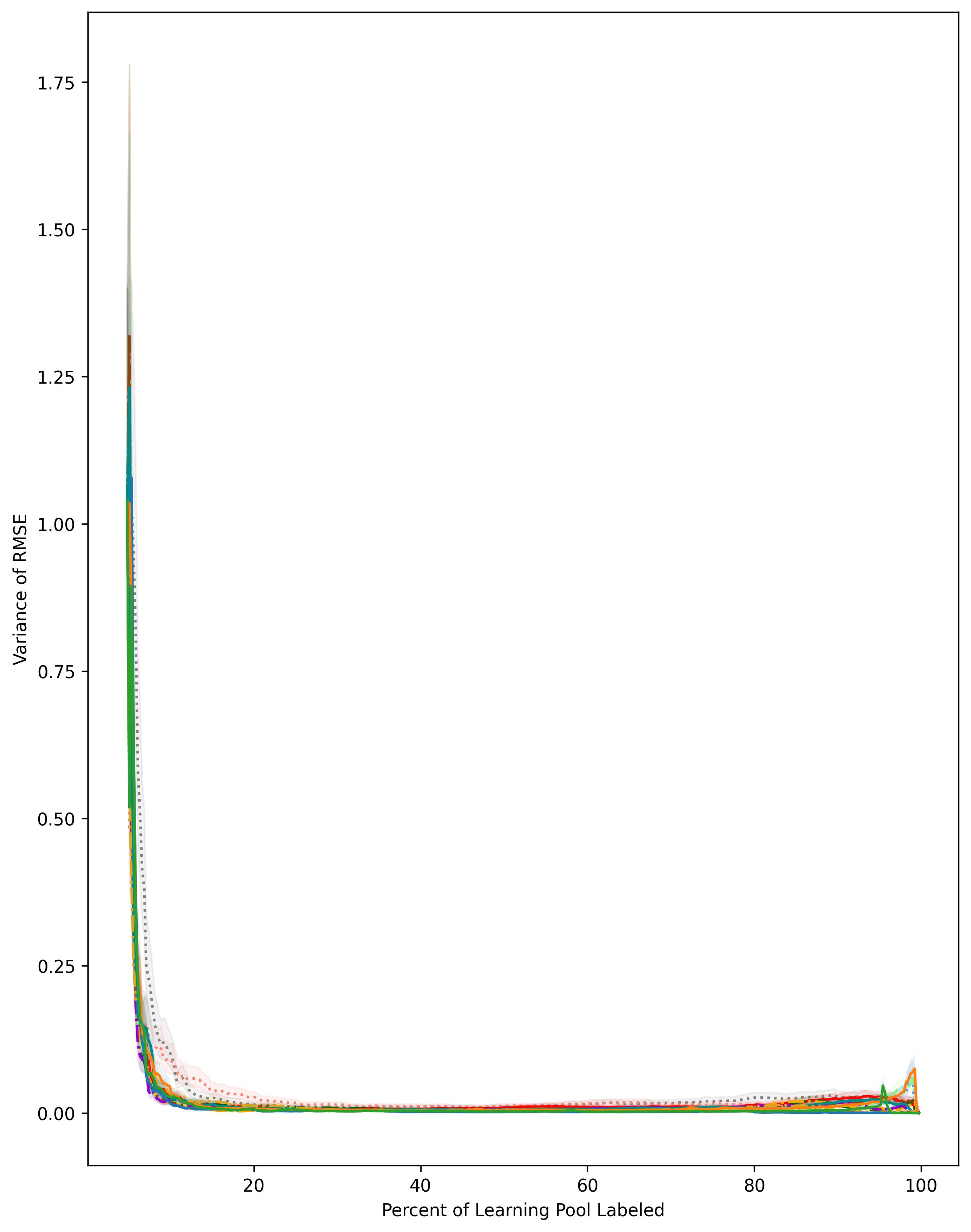}
        \caption{mpg}
    \end{subfigure}
    
    \vspace{0.3em}
    
    \begin{subfigure}[b]{0.31\textwidth}
        \centering
        \includegraphics[width=\linewidth, keepaspectratio]{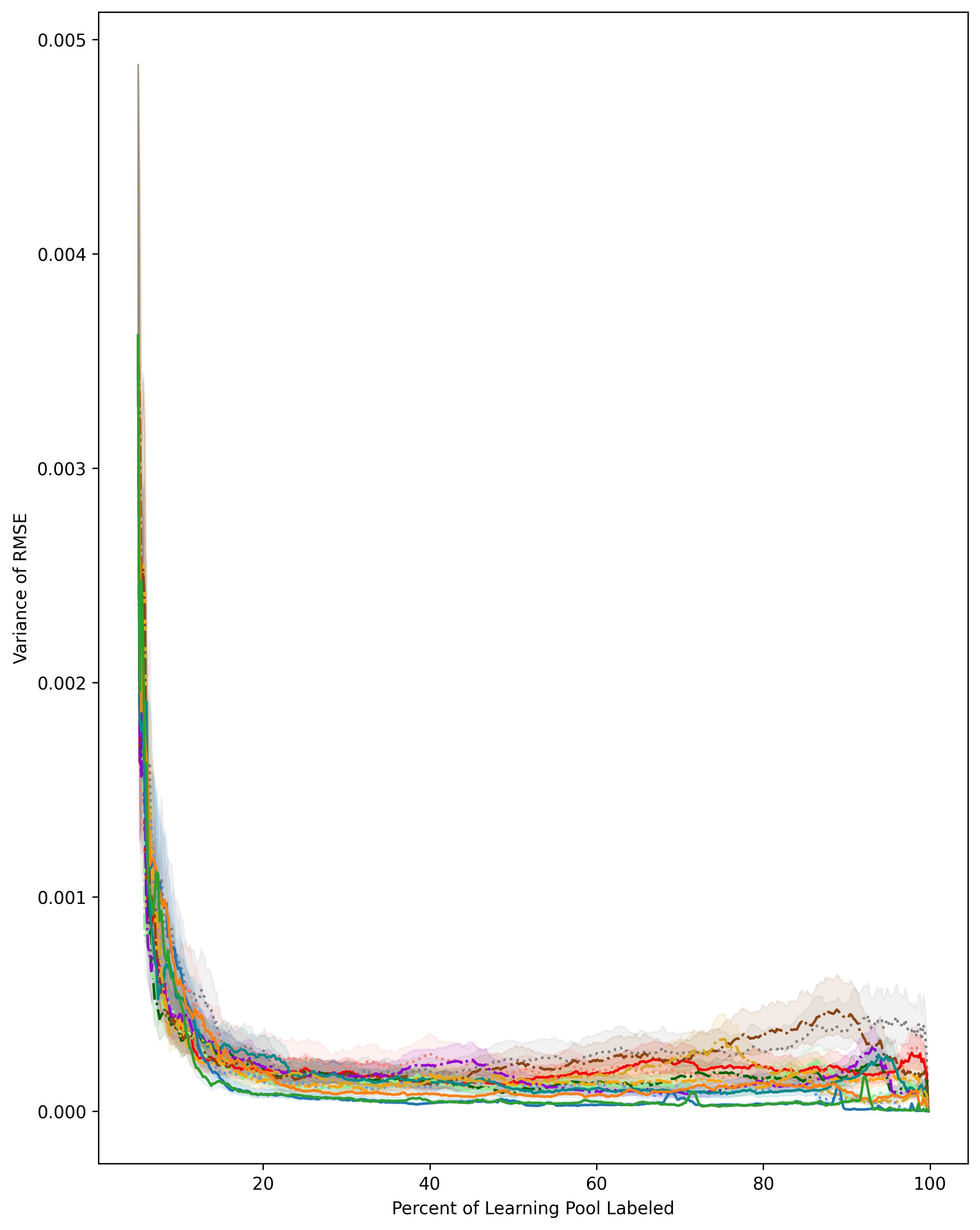}
        \caption{no2}
    \end{subfigure}
    \hfill
    \begin{subfigure}[b]{0.31\textwidth}
        \centering
        \includegraphics[width=\linewidth, keepaspectratio]{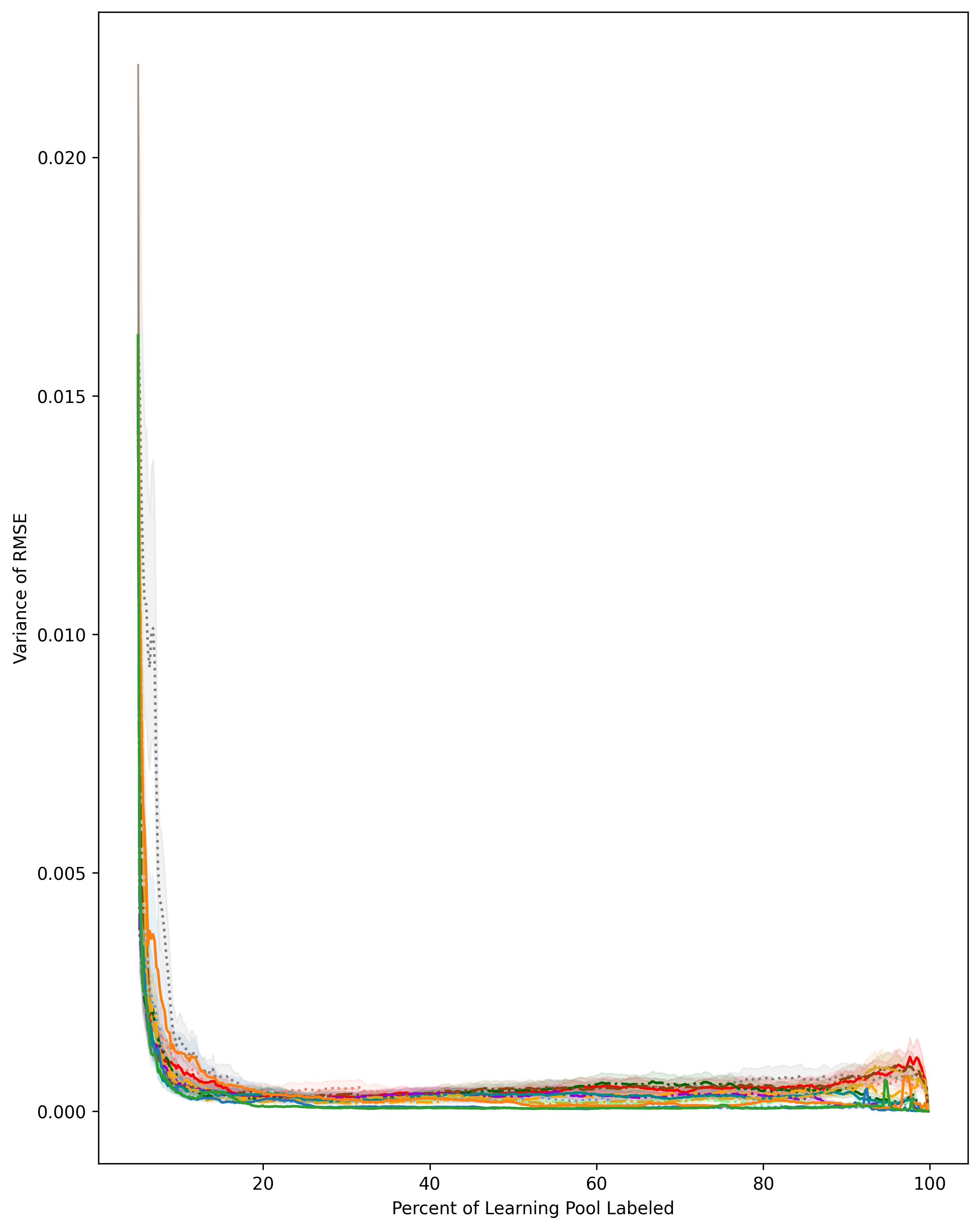}
        \caption{pm10}
    \end{subfigure}
    \hfill
    \begin{subfigure}[b]{0.31\textwidth}
        \centering
        \includegraphics[width=\linewidth, keepaspectratio]{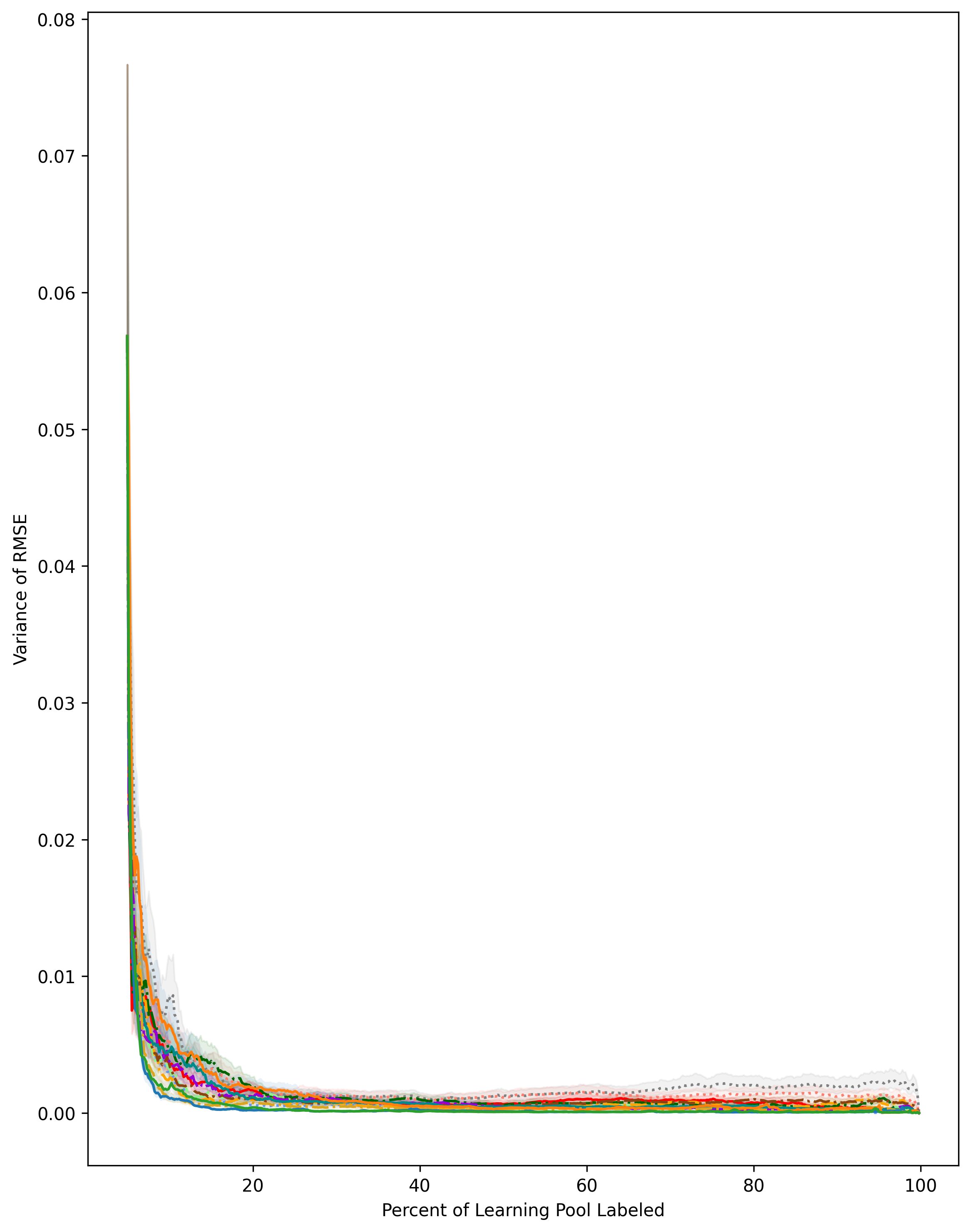}
        \caption{qsar}
    \end{subfigure}
    
    \vspace{0.3em}
    
    \begin{subfigure}[b]{0.31\textwidth}
        \centering
        \includegraphics[width=\linewidth, keepaspectratio]{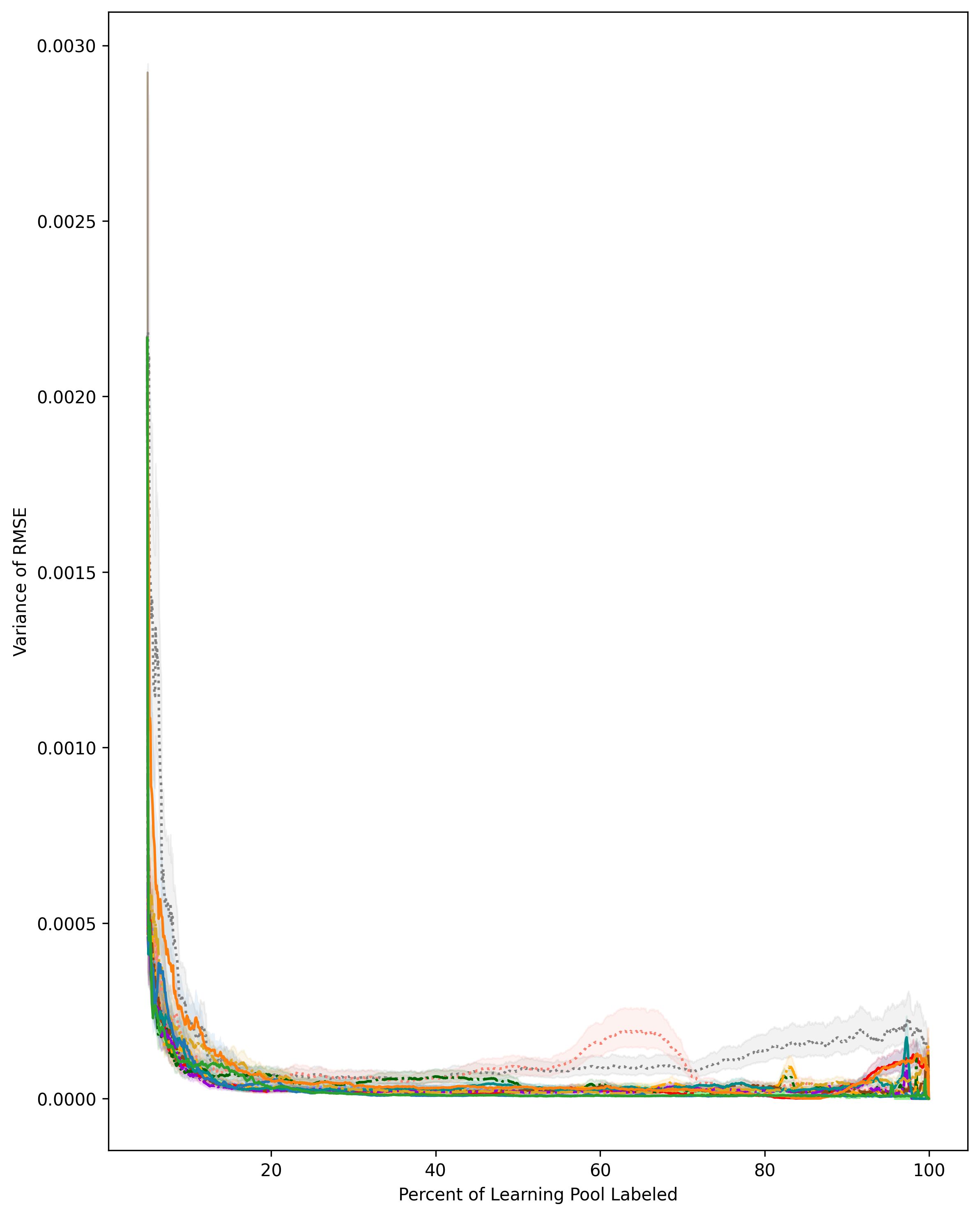}
        \caption{wine\_red}
    \end{subfigure}
    \hfill
    \begin{subfigure}[b]{0.31\textwidth}
        \centering
        \includegraphics[width=\linewidth, keepaspectratio]{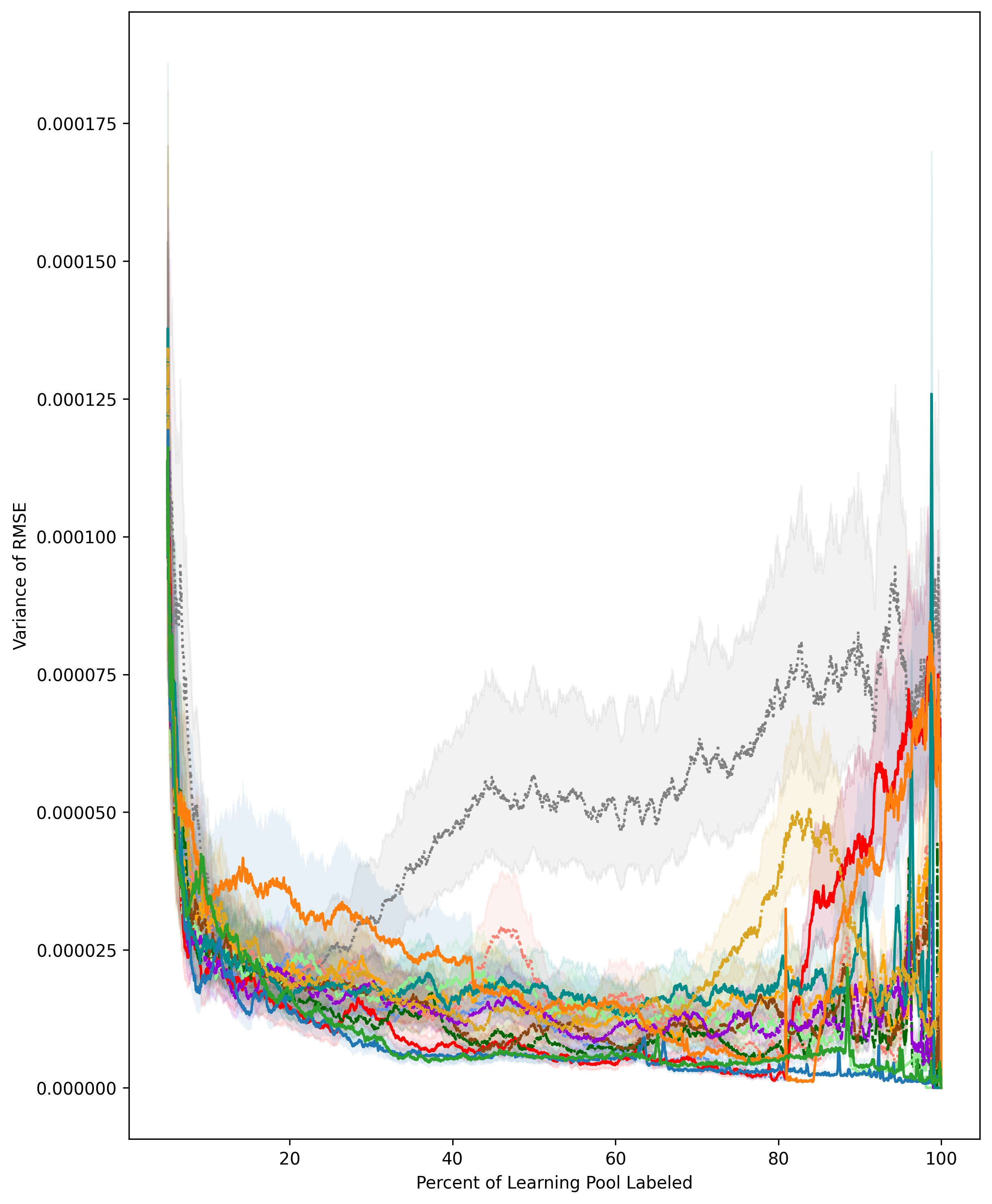}
        \caption{wine\_white}
    \end{subfigure}
    \hfill
    \begin{subfigure}[b]{0.31\textwidth}
        \centering
        \includegraphics[width=\linewidth, keepaspectratio]{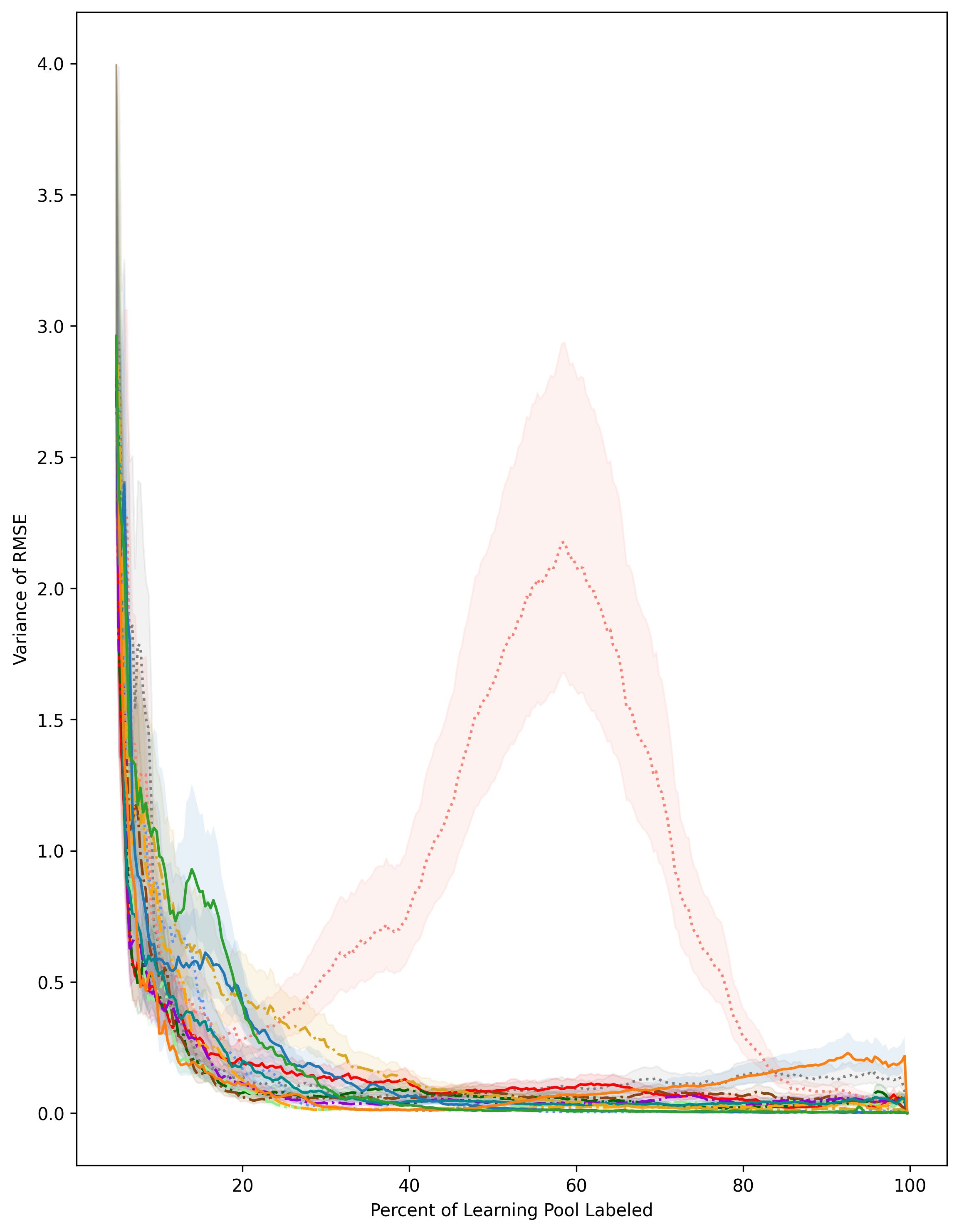}
        \caption{yacht}
    \end{subfigure}
    
    \vspace{0.5em}
    \centering
    \includegraphics[width=\linewidth]{upload_all_files/manuscript/benchmark_legend.jpg}
    
    \caption{Variance of the full-pool RMSE trace plots for benchmark datasets (Part 2 of 2).}
    \label{fig:VarianceResults2}
\end{figure}
\clearpage

\subsection{Robust Normalization Results}
\label{subsec:robust_norm}

To validate the stability of our framework, we repeated our evaluation under a different normalization scheme $\phi$. The results below show the same experiments with a \textit{Robust Normalization} scheme (scaling features based on the interquartile range) to mitigate the impact of outliers in the input space. As illustrated in the Figure \ref{fig:auc_heatmap_robust} and the trace plots below, the relative performance rankings of the strategies remain consistent with our main results. This confirms that the superiority of the adaptive WiGS-SAC agent is intrinsic to its policy learning capability and not an artifact of specific feature scaling choices.

\begin{figure*}[htbp]
    \centering
    \includegraphics[width=\textwidth]{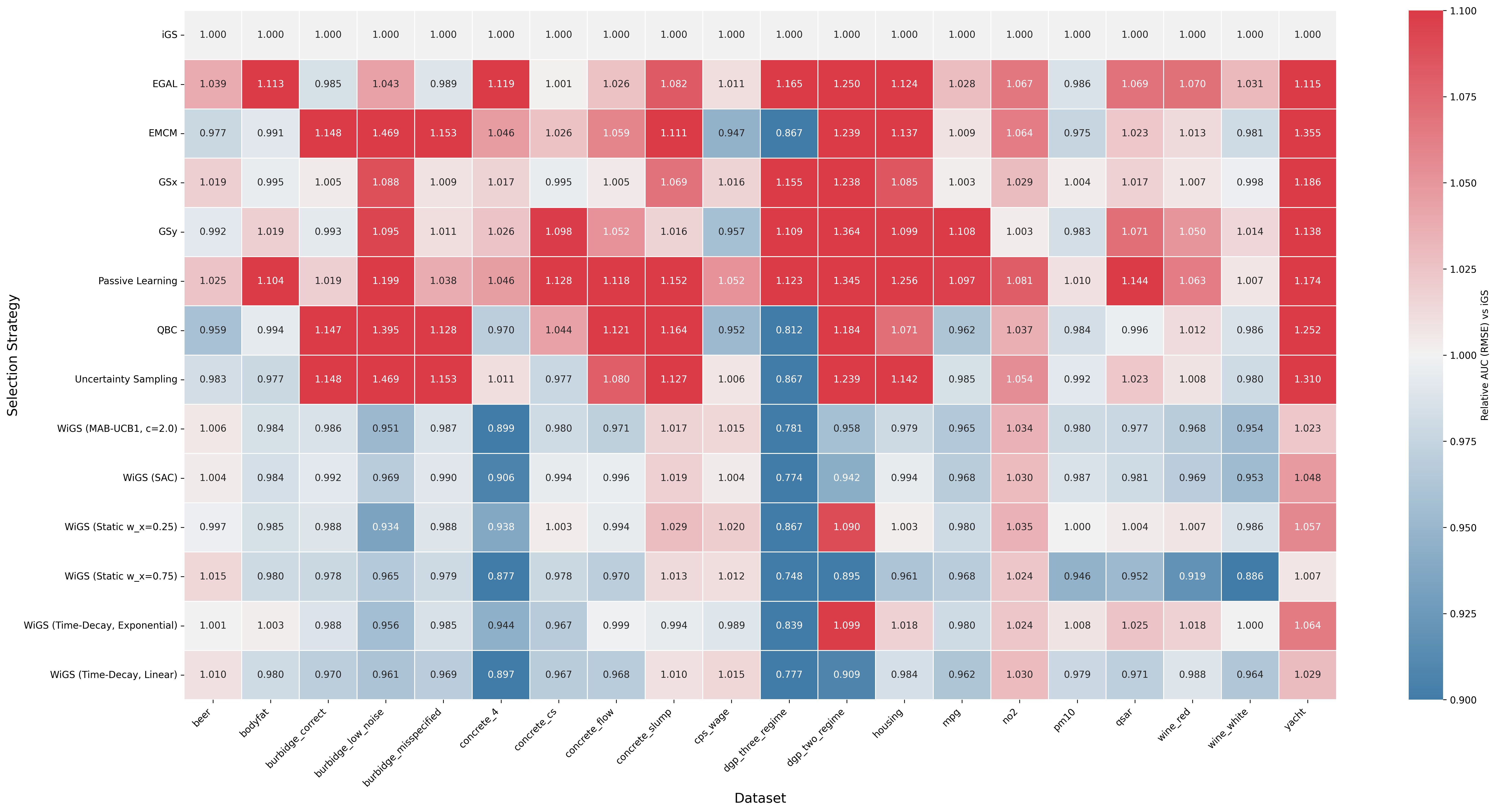}
    \caption{Robust Relative AUC values comparing active learning strategies under robust feature scaling. The results parallel those observed under standard normalization, indicating the method's robustness to preprocessing techniques.}
    \label{fig:auc_heatmap_robust}
\end{figure*}

\begin{figure*}[t]
    \centering
    \includegraphics[width=0.8\textwidth]{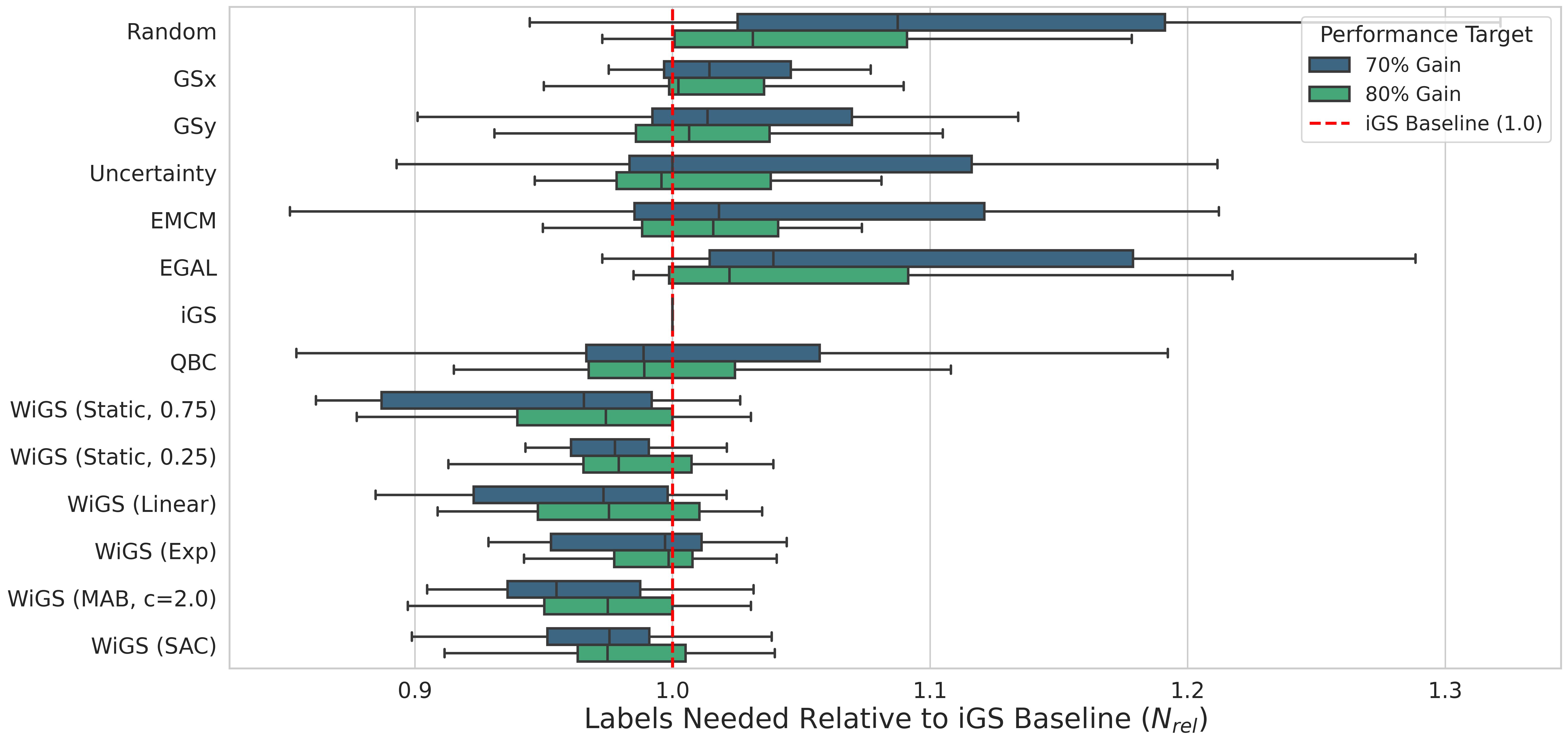}
    \caption{Relative Label Efficiency ($N_{rel}$) aggregated across 20 datasets with a robust normalization. Again, the results are similar to those under standard normalization}
    \label{fig:label_efficiency_robust}
\end{figure*}

\begin{figure}
    \centering
    \vspace*{-1cm} 
    
    \begin{subfigure}[b]{0.31\textwidth}
        \centering
        \includegraphics[width=\linewidth, keepaspectratio]{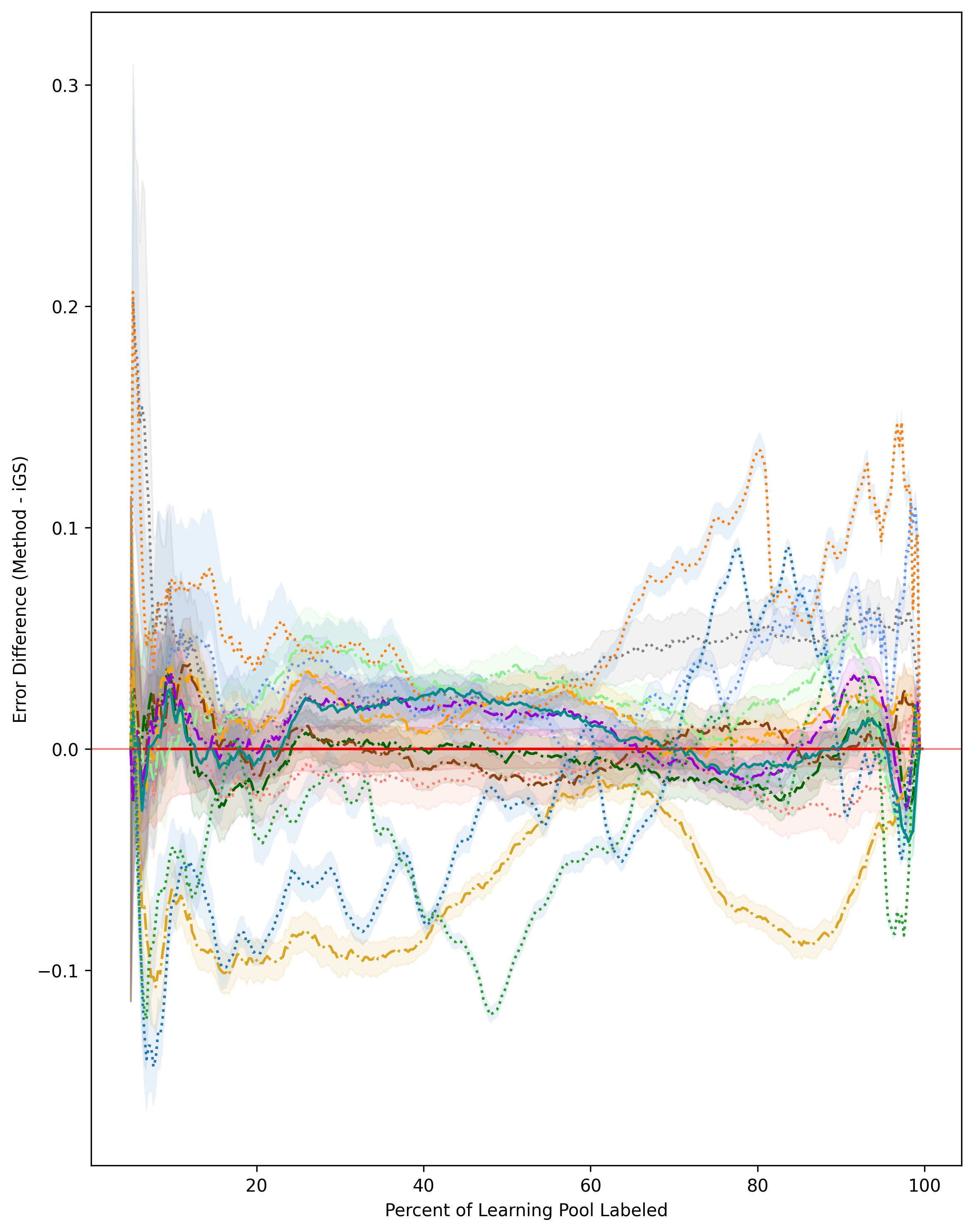}        
        \caption{beer}
    \end{subfigure}
    \hfill
    \begin{subfigure}[b]{0.31\textwidth}
        \centering
        \includegraphics[width=\linewidth, keepaspectratio]{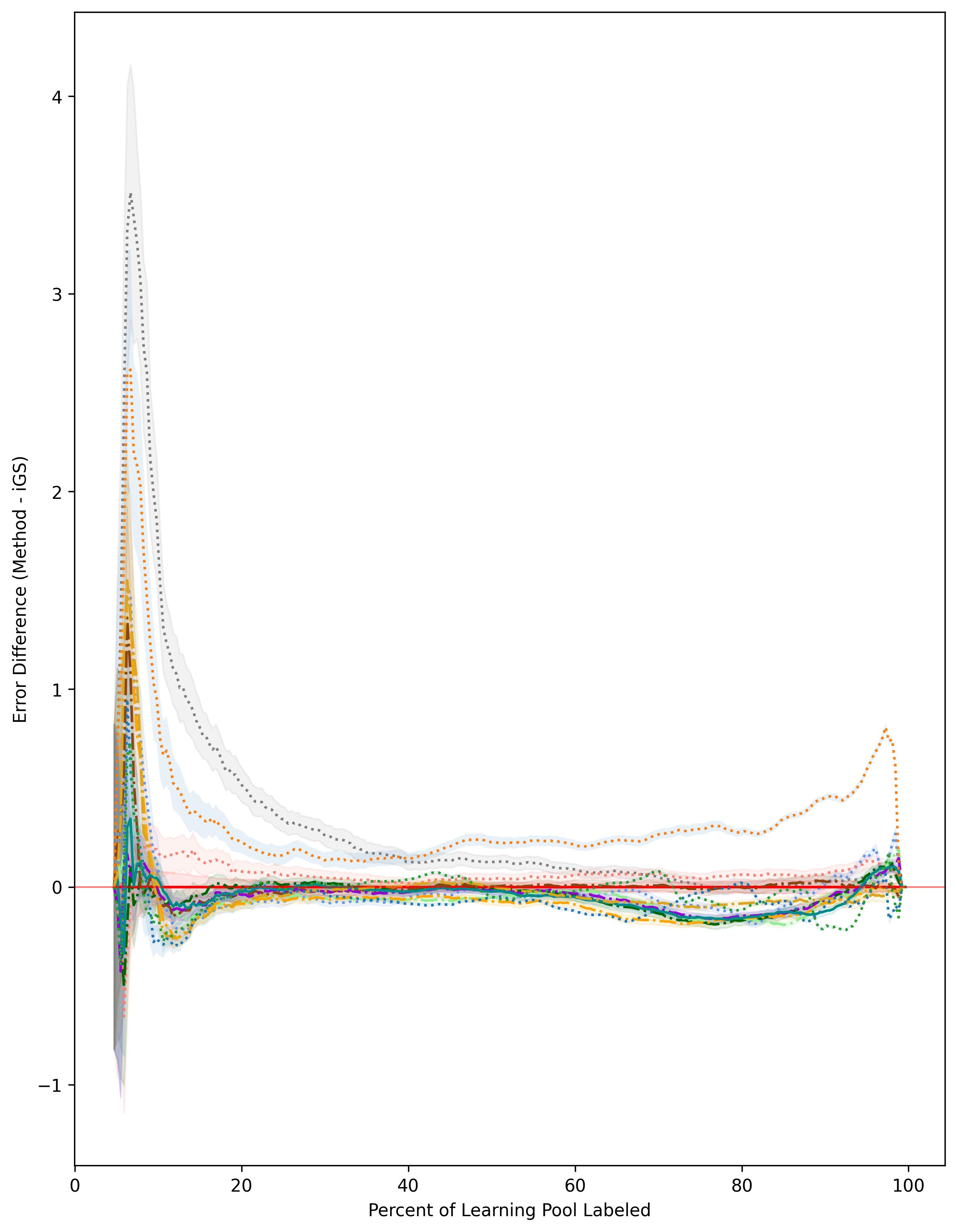}
        \caption{bodyfat}
    \end{subfigure}
    \hfill
    \begin{subfigure}[b]{0.31\textwidth}
        \centering
        \includegraphics[width=\linewidth, keepaspectratio]{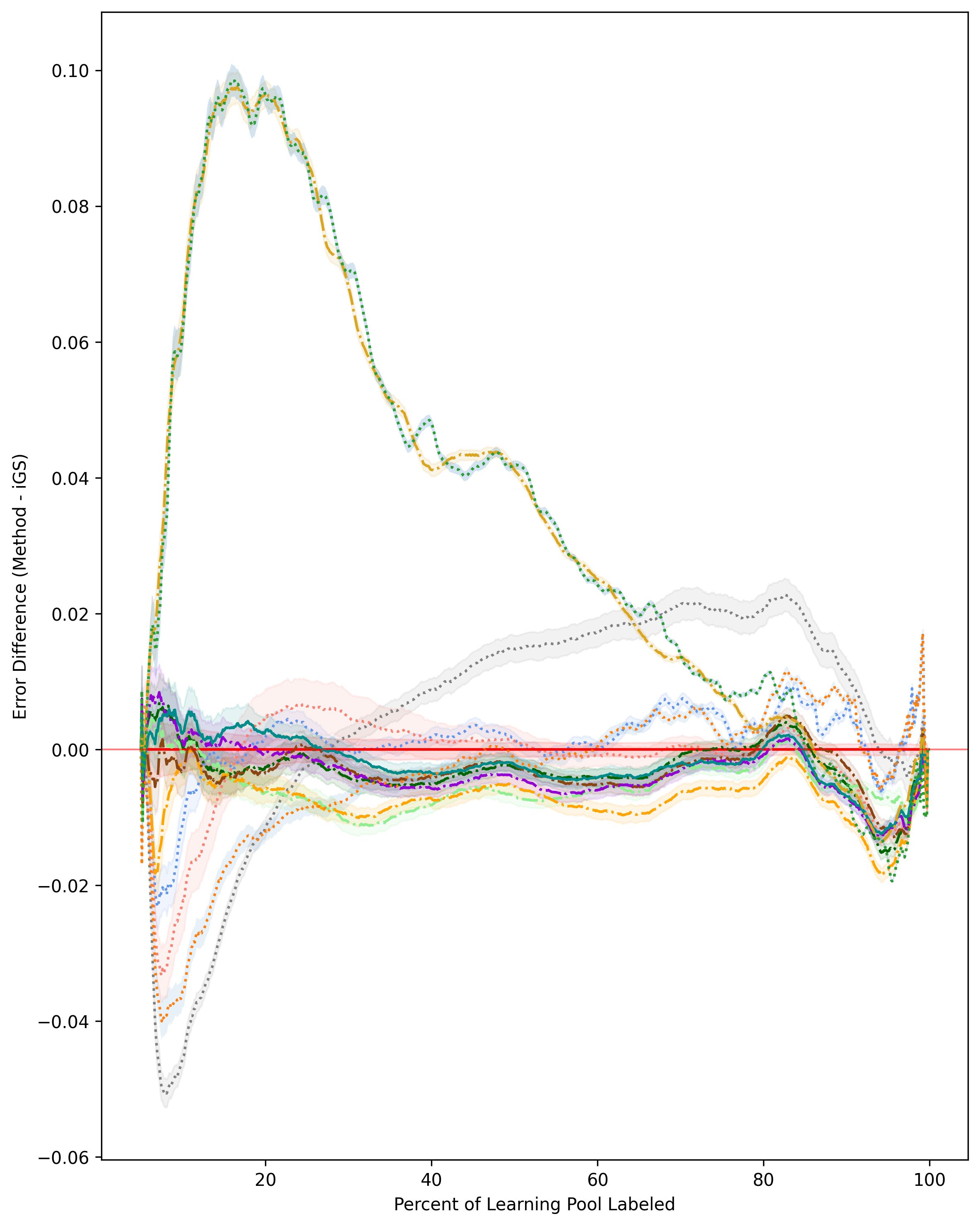}
        \caption{burbidge\_correct}
    \end{subfigure}
    
    \vspace{0.3em} 
    
    \begin{subfigure}[b]{0.31\textwidth}
        \centering
        \includegraphics[width=\linewidth, keepaspectratio]{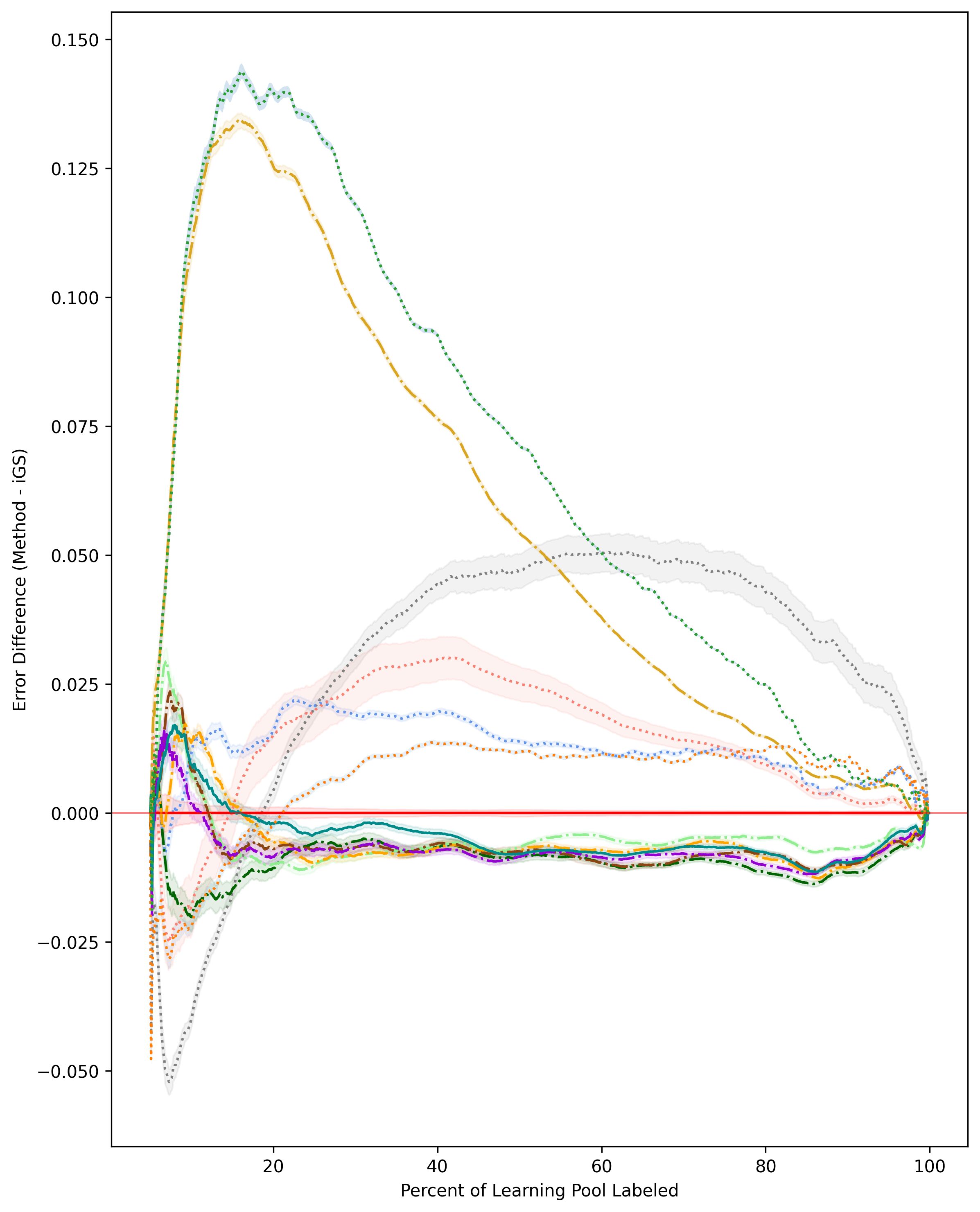}
        \caption{burbidge\_low\_noise}
    \end{subfigure}
    \hfill
    \begin{subfigure}[b]{0.31\textwidth}
        \centering
        \includegraphics[width=\linewidth, keepaspectratio]{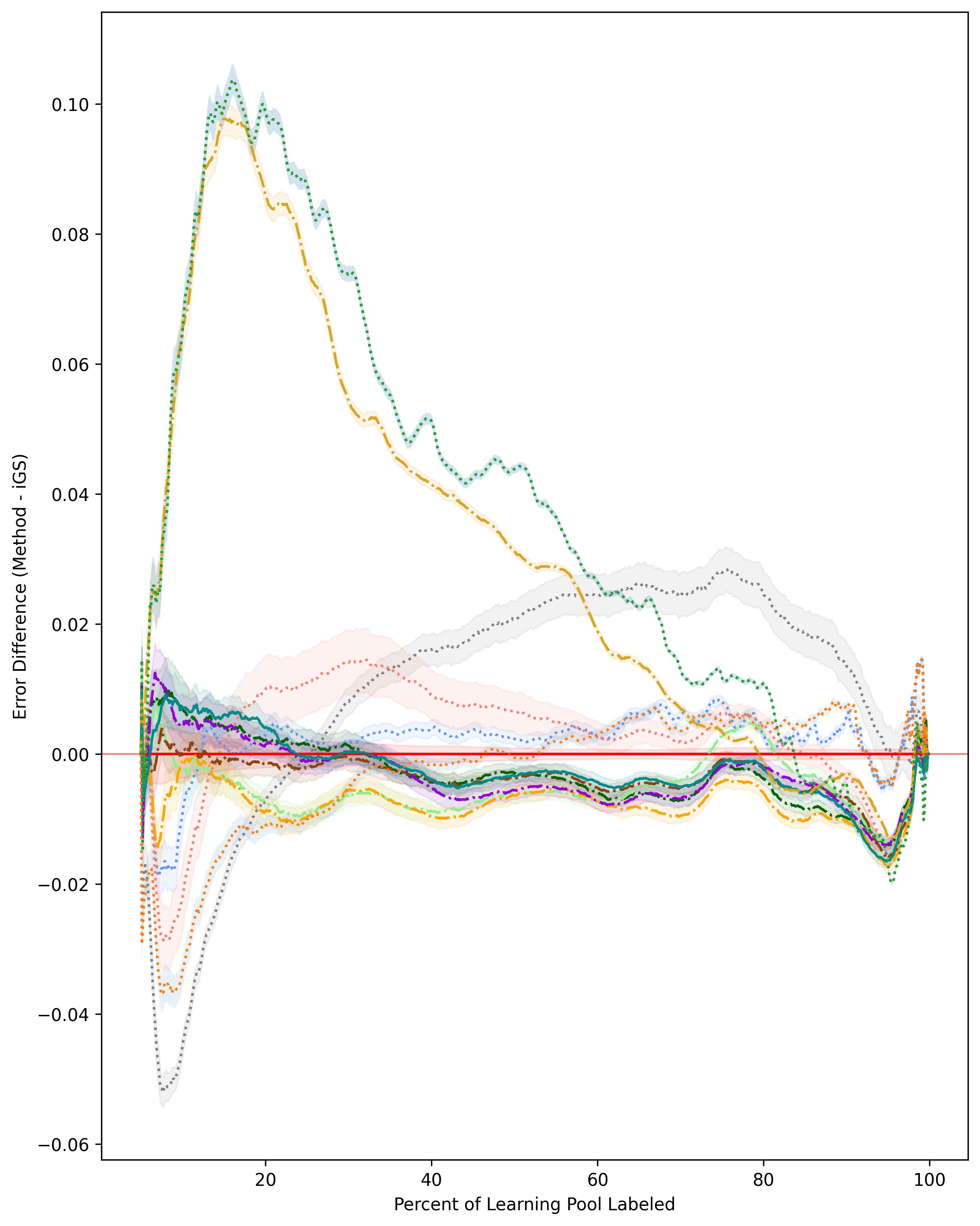}
        \caption{burbidge\_misspecified}
    \end{subfigure}
    \hfill
    \begin{subfigure}[b]{0.31\textwidth}
        \centering
        \includegraphics[width=\linewidth, keepaspectratio]{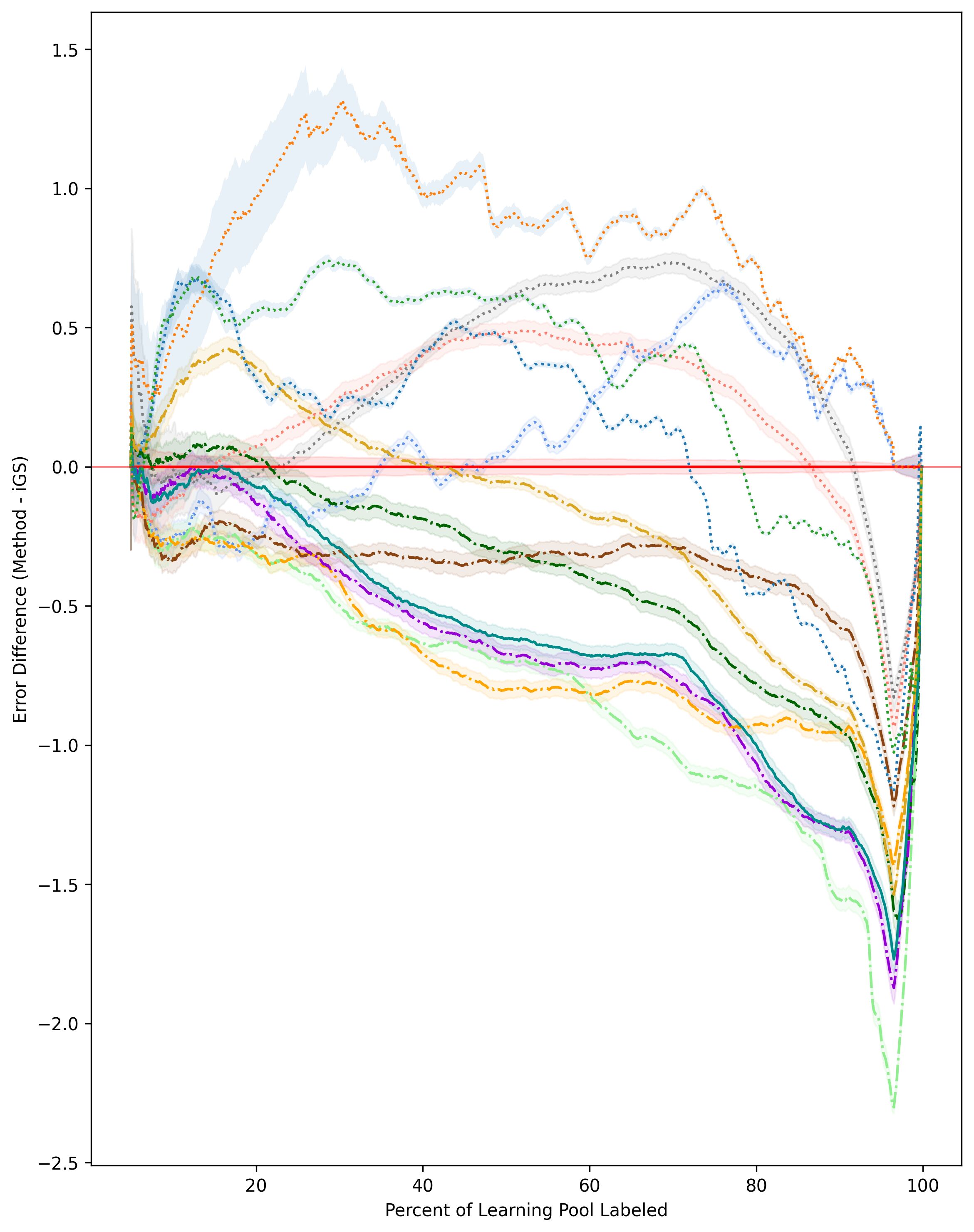}
        \caption{concrete\_4}
    \end{subfigure}
    
    \vspace{0.3em}
    
    \begin{subfigure}[b]{0.31\textwidth}
        \centering
        \includegraphics[width=\linewidth, keepaspectratio]{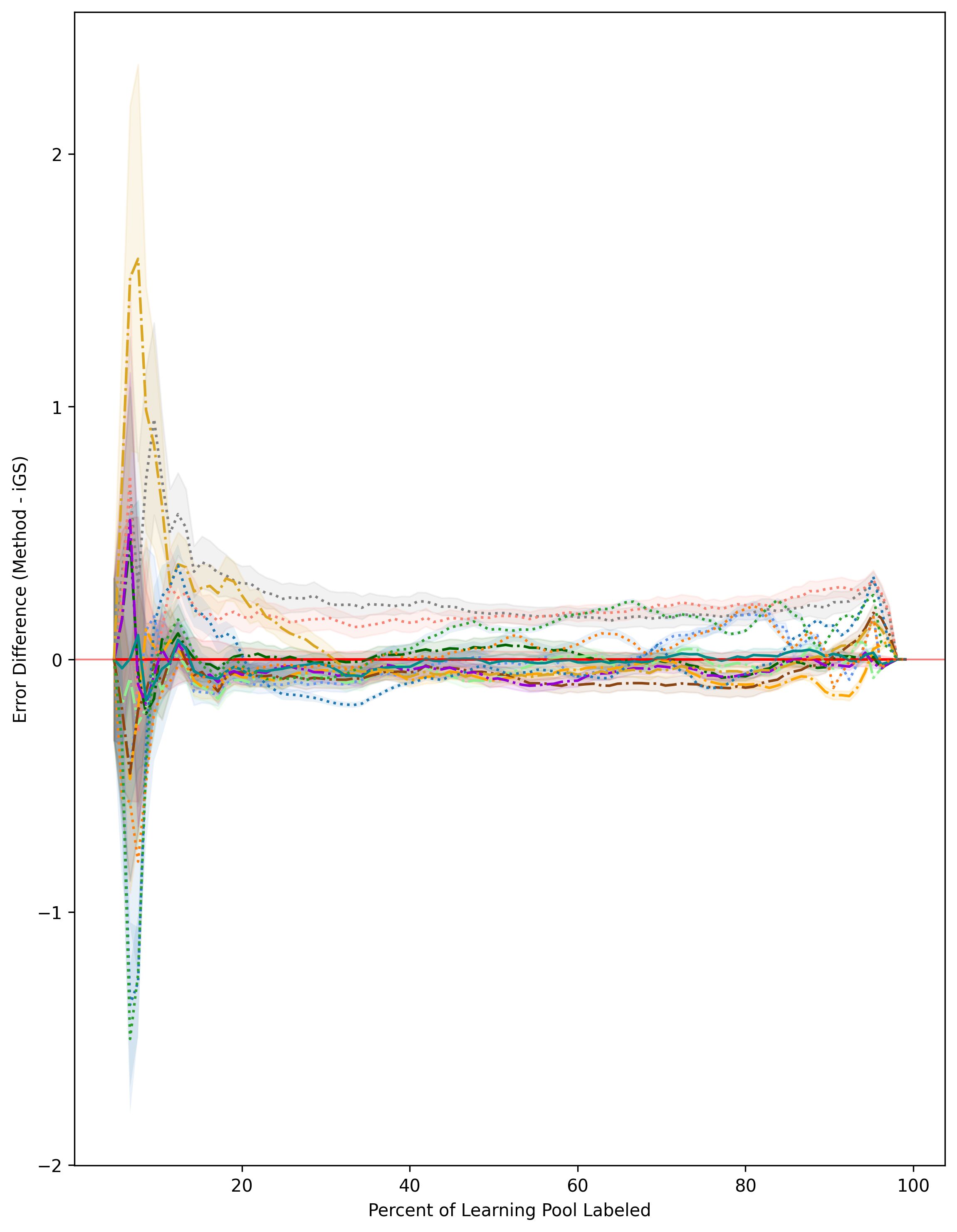}
        \caption{concrete\_cs}
    \end{subfigure}
    \hfill
    \begin{subfigure}[b]{0.31\textwidth}
        \centering
        \includegraphics[width=\linewidth, keepaspectratio]{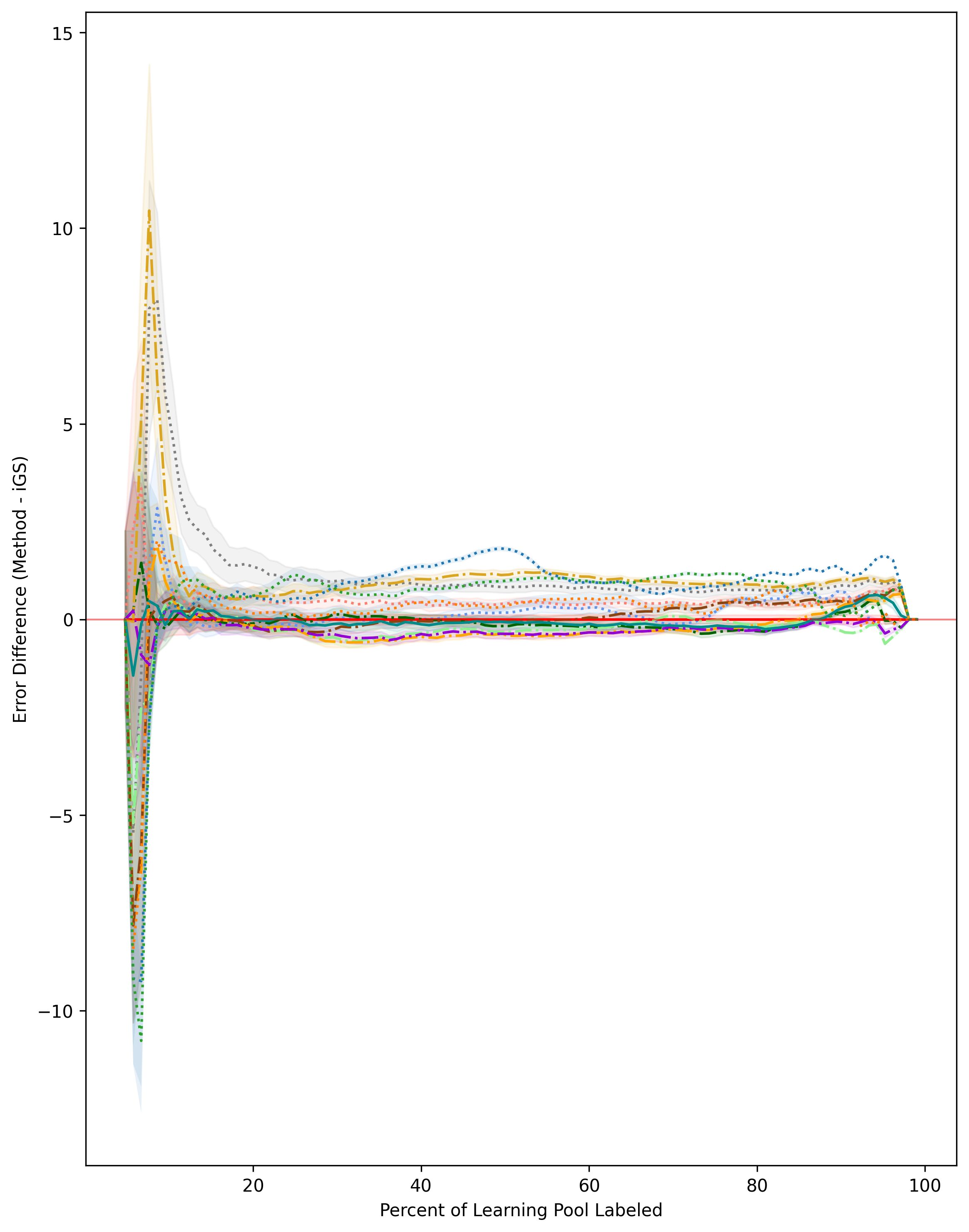}
        \caption{concrete\_flow}
    \end{subfigure}
    \hfill
    \begin{subfigure}[b]{0.31\textwidth}
        \centering
        \includegraphics[width=\linewidth, keepaspectratio]{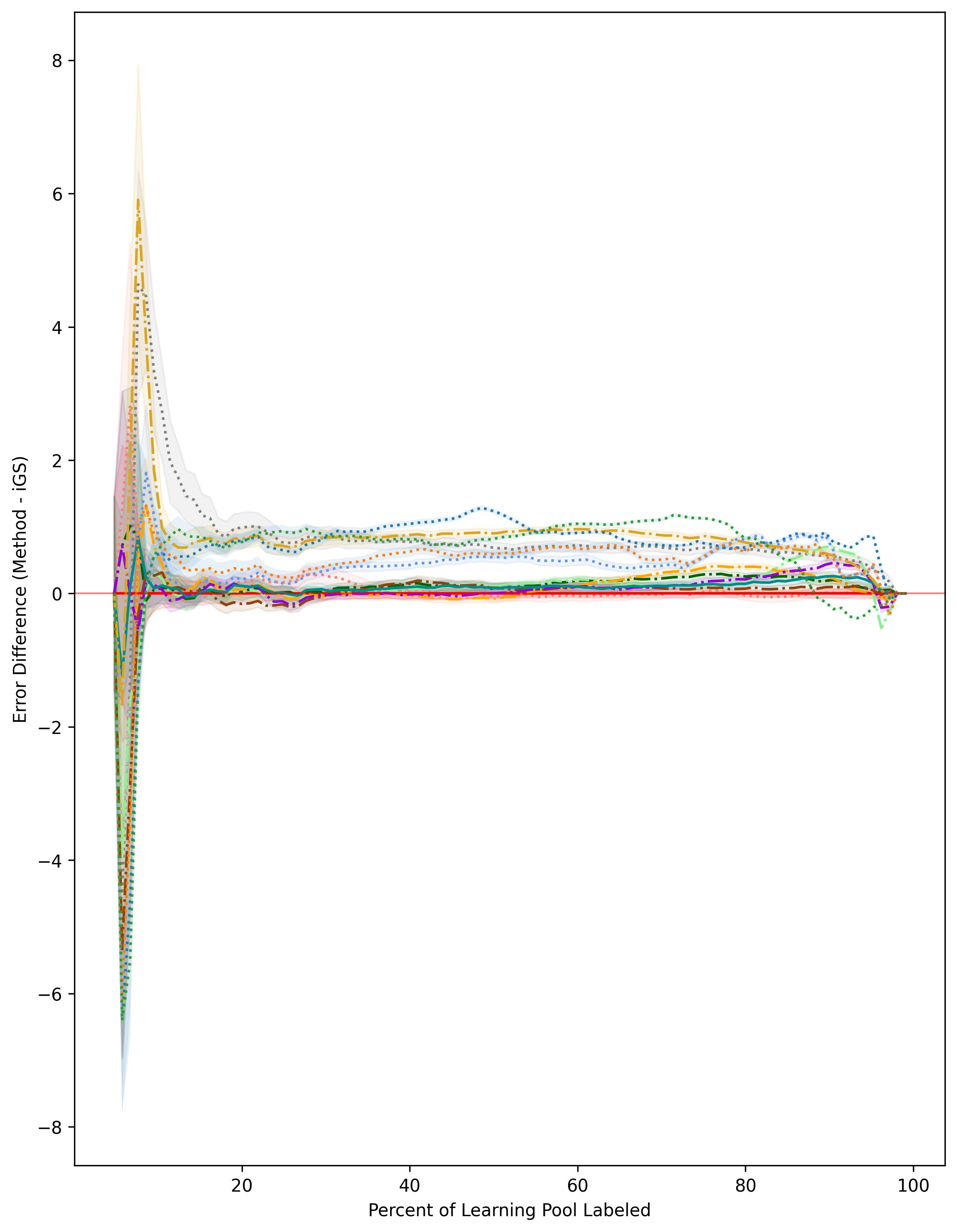}
        \caption{concrete\_slump}
    \end{subfigure}
    
    \vspace{0.5em}
    \centering
    \includegraphics[width=\linewidth]{upload_all_files/manuscript/benchmark_legend.jpg}
    
    \caption{Full-pool RMSE trace plots for benchmark datasets with robust normalization (Part 1 of 2).}
    \label{fig:MainResults_ROBUST1}
\end{figure}

\clearpage
\begin{figure}
    \centering
    \vspace*{-1cm} 
    
    \begin{subfigure}[b]{0.31\textwidth}
        \centering
        \includegraphics[width=\linewidth, keepaspectratio]{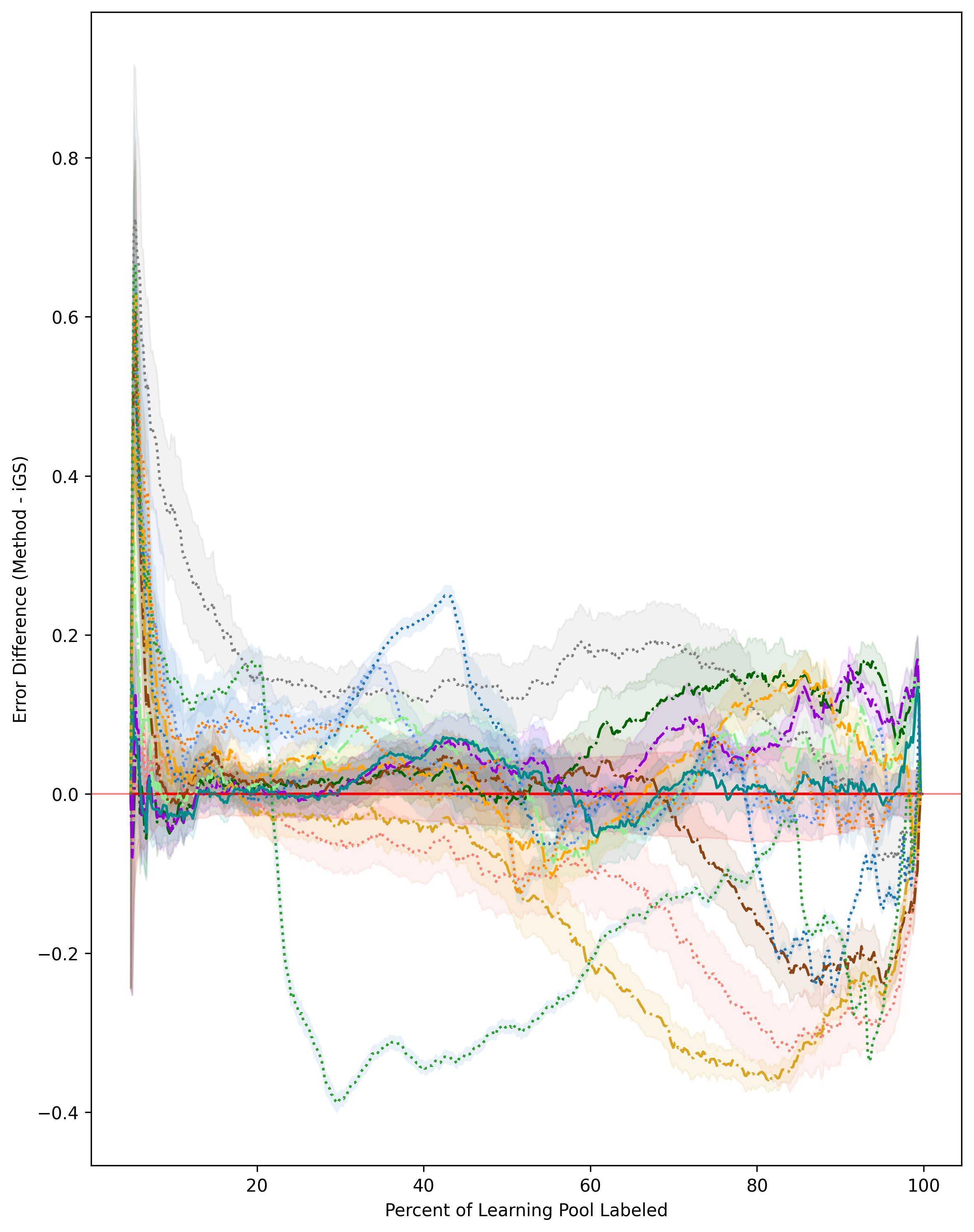}
        \caption{cps\_wage}    \end{subfigure}
    \hfill
    \begin{subfigure}[b]{0.31\textwidth}
        \centering
        \includegraphics[width=\linewidth, keepaspectratio]{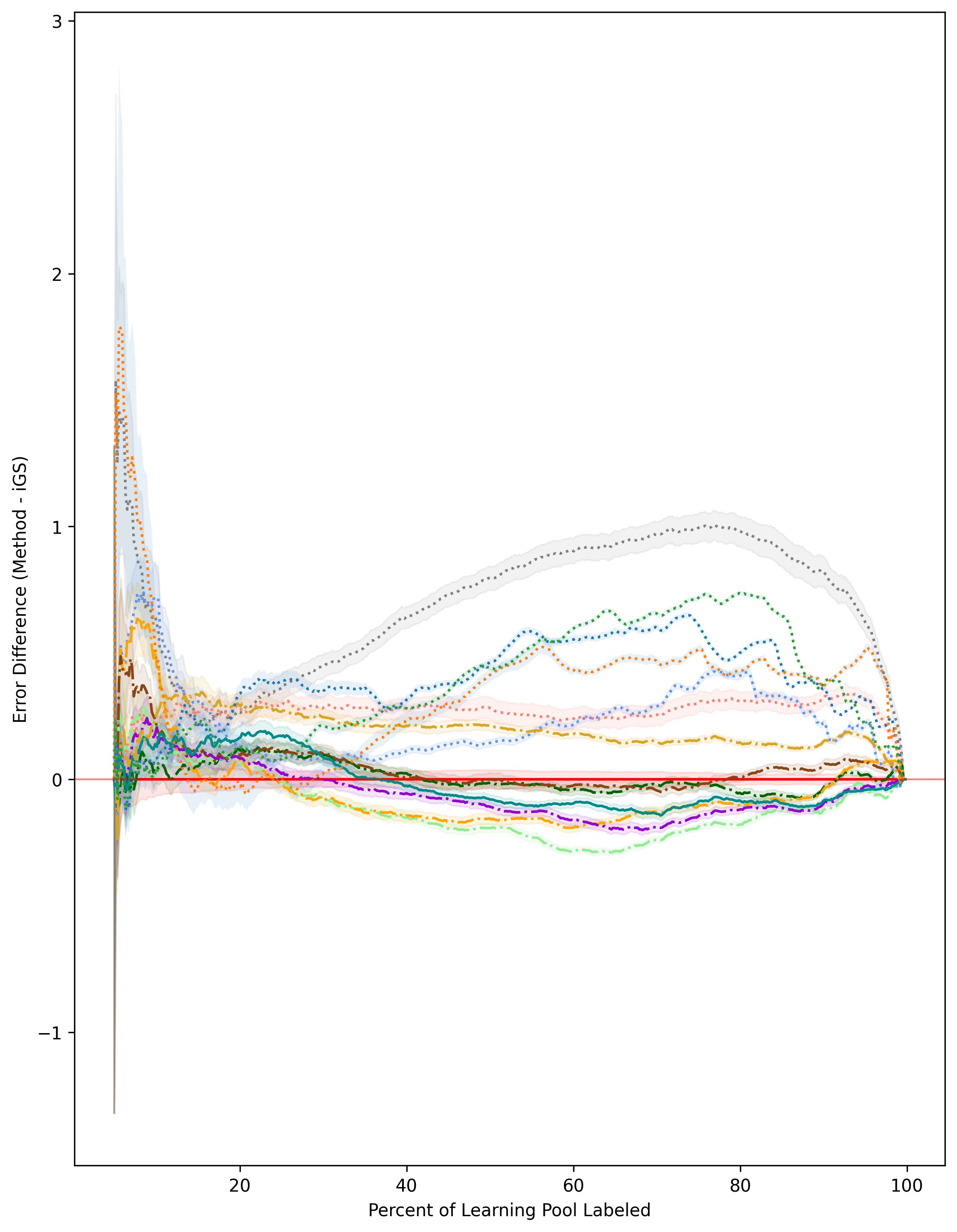}
        \caption{housing}
    \end{subfigure}
    \hfill
    \begin{subfigure}[b]{0.31\textwidth}
        \centering
        \includegraphics[width=\linewidth, keepaspectratio]{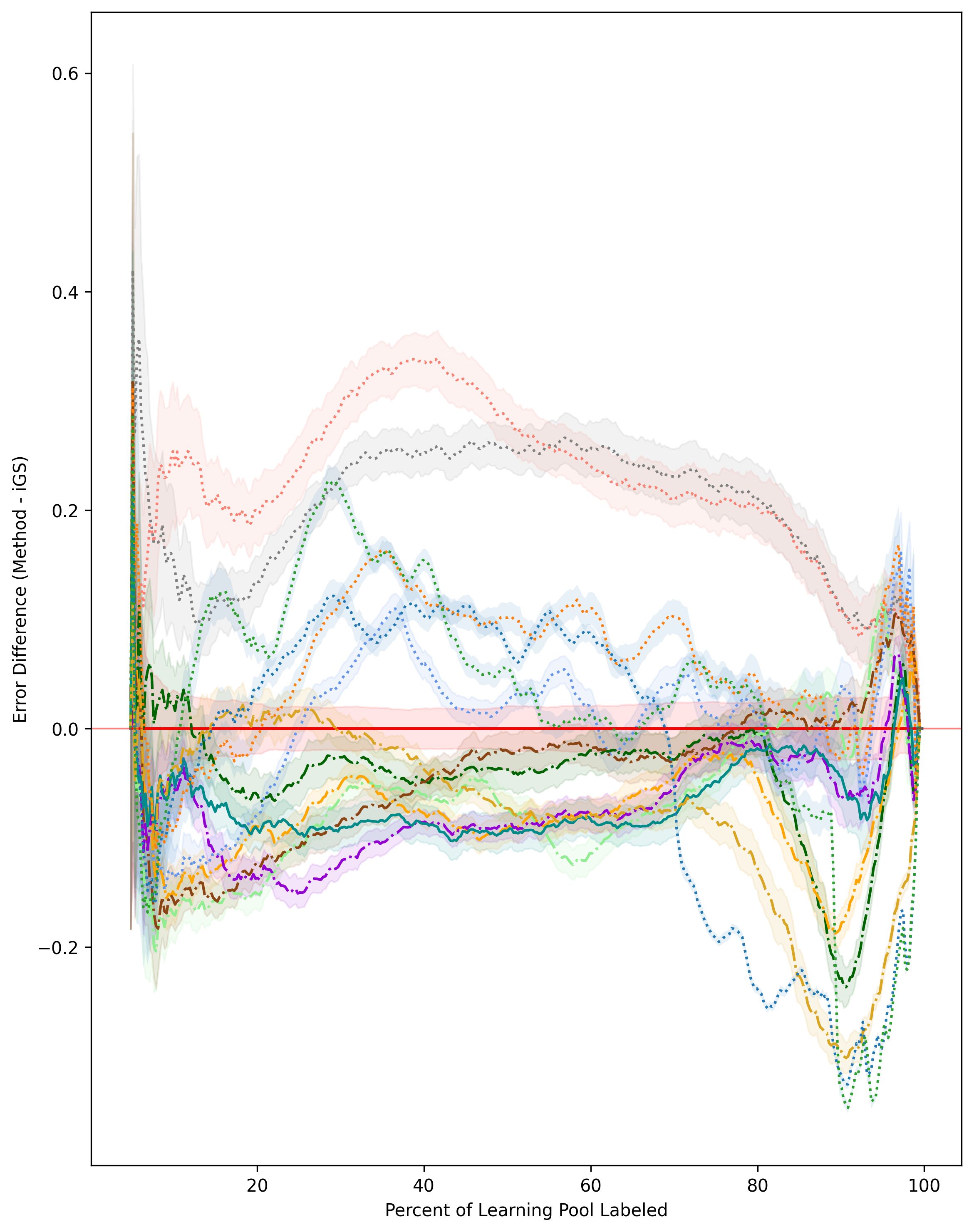}
        \caption{mpg}
    \end{subfigure}
    
    \vspace{0.3em}
    
    \begin{subfigure}[b]{0.31\textwidth}
        \centering
        \includegraphics[width=\linewidth, keepaspectratio]{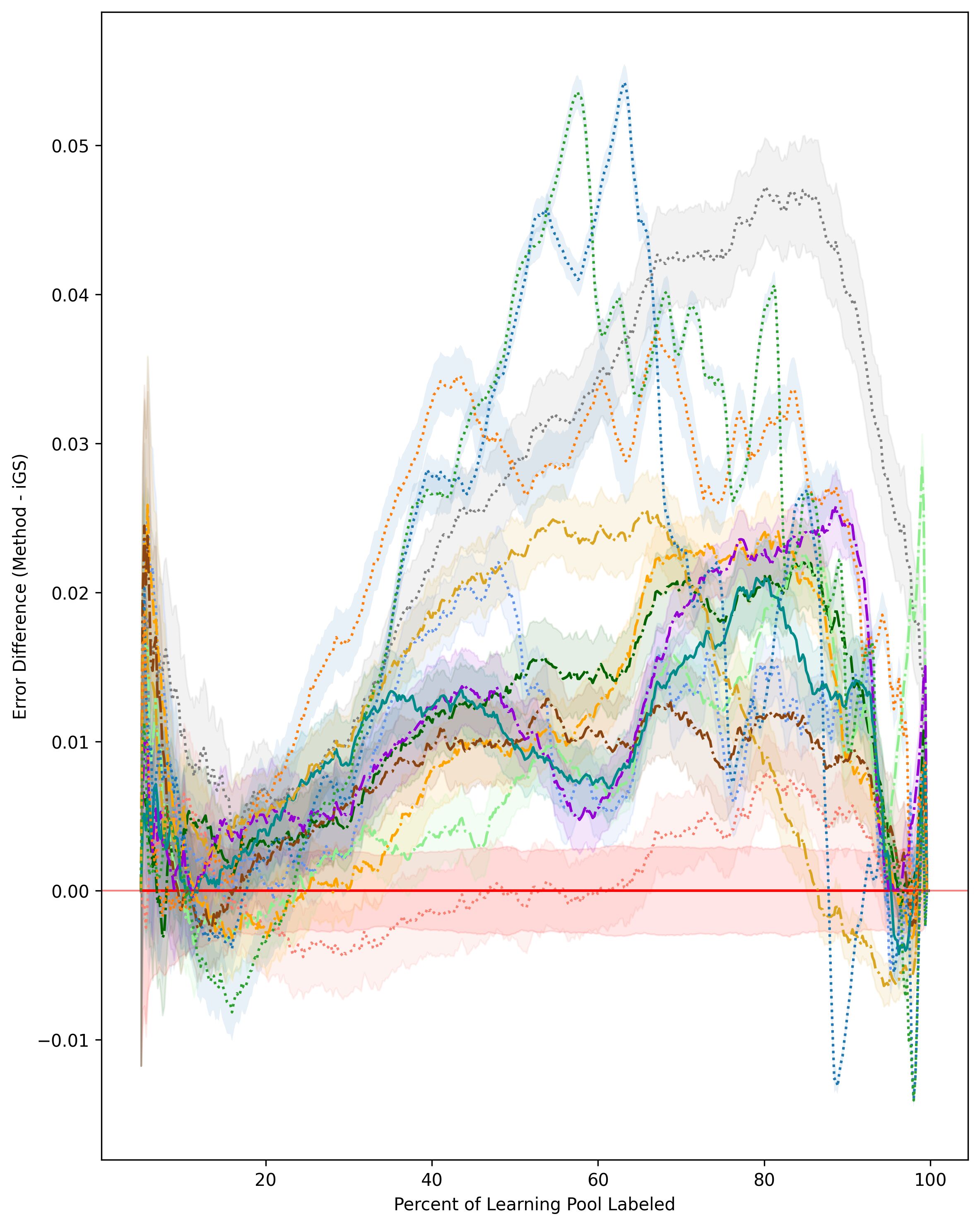}
        \caption{no2}
    \end{subfigure}
    \hfill
    \begin{subfigure}[b]{0.31\textwidth}
        \centering
        \includegraphics[width=\linewidth, keepaspectratio]{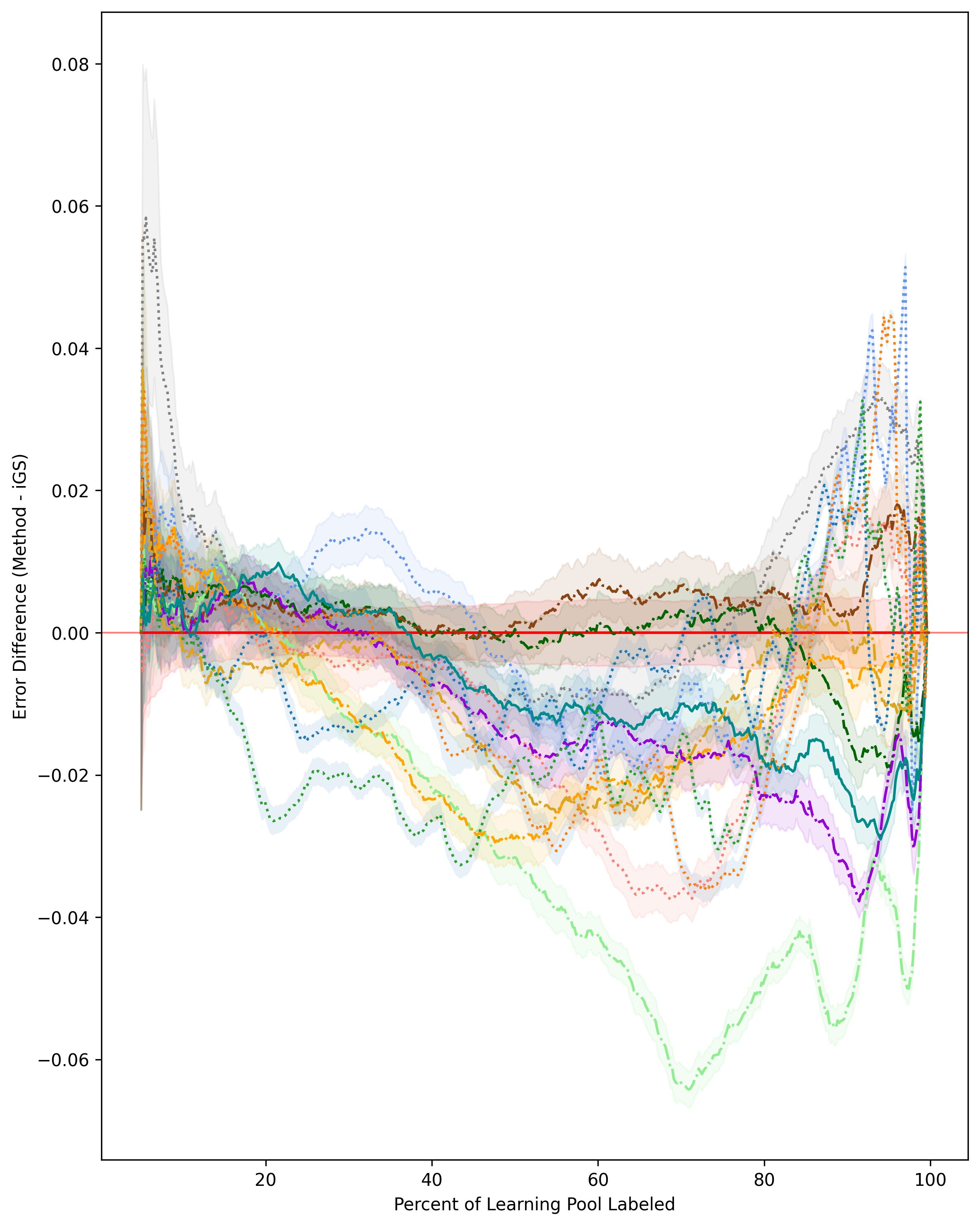}
        \caption{pm10}
    \end{subfigure}
    \hfill
    \begin{subfigure}[b]{0.31\textwidth}
        \centering
        \includegraphics[width=\linewidth, keepaspectratio]{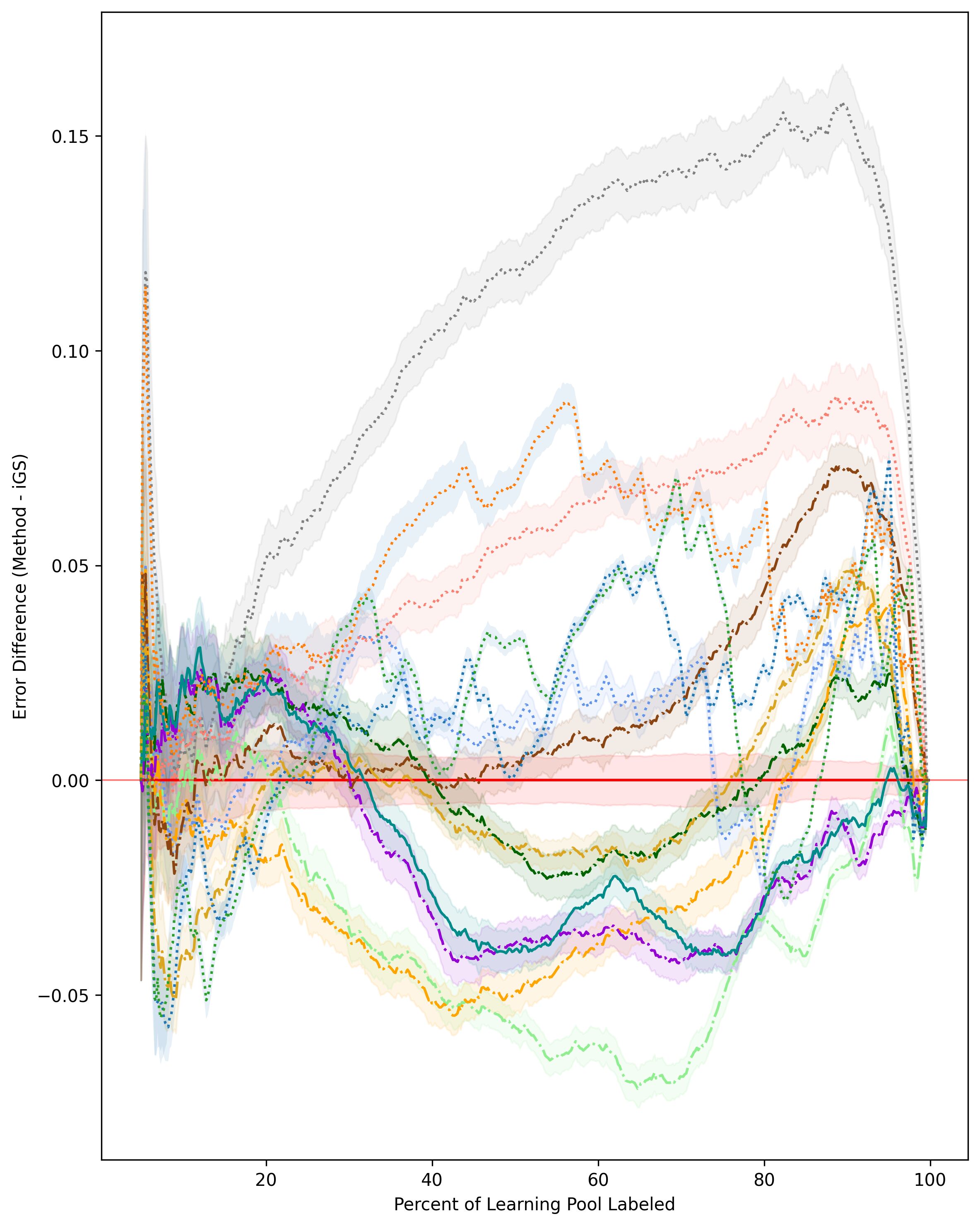}
        \caption{qsar}
    \end{subfigure}
    
    \vspace{0.3em}
    
    \begin{subfigure}[b]{0.31\textwidth}
        \centering
        \includegraphics[width=\linewidth, keepaspectratio]{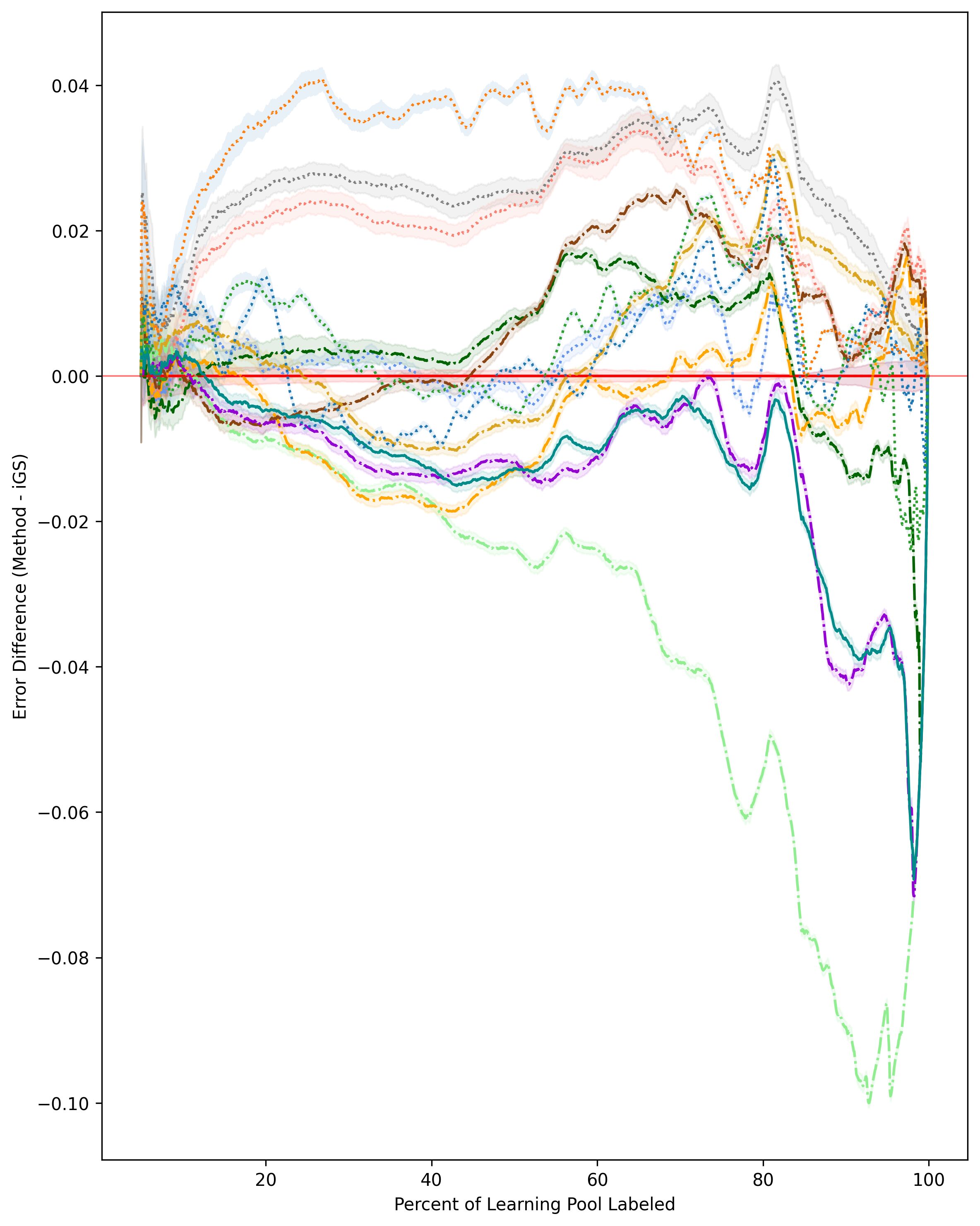}
        \caption{wine\_red}
    \end{subfigure}
    \hfill
    \begin{subfigure}[b]{0.31\textwidth}
        \centering
        \includegraphics[width=\linewidth, keepaspectratio]{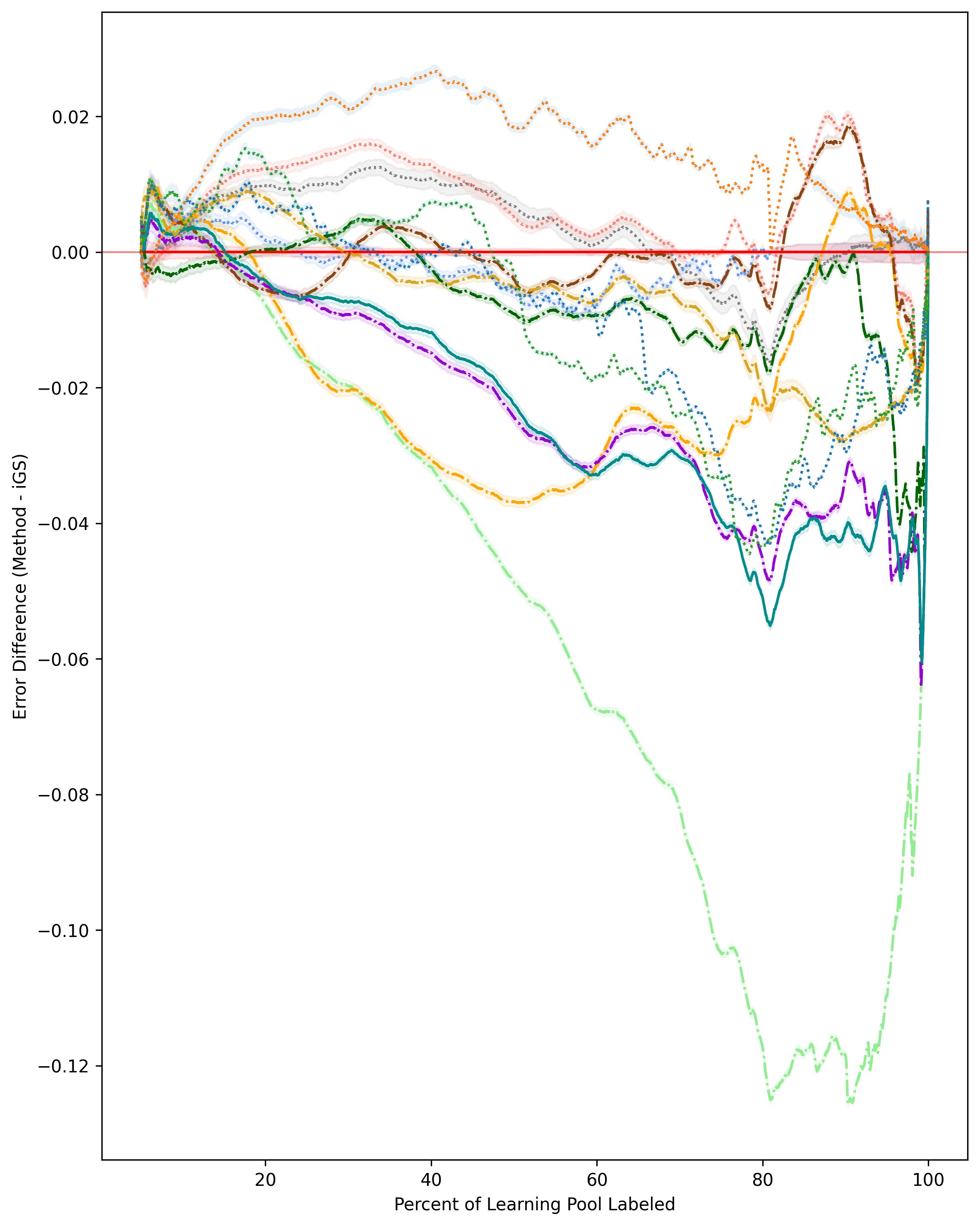}
        \caption{wine\_white}
    \end{subfigure}
    \hfill
    \begin{subfigure}[b]{0.31\textwidth}
        \centering
        \includegraphics[width=\linewidth, keepaspectratio]{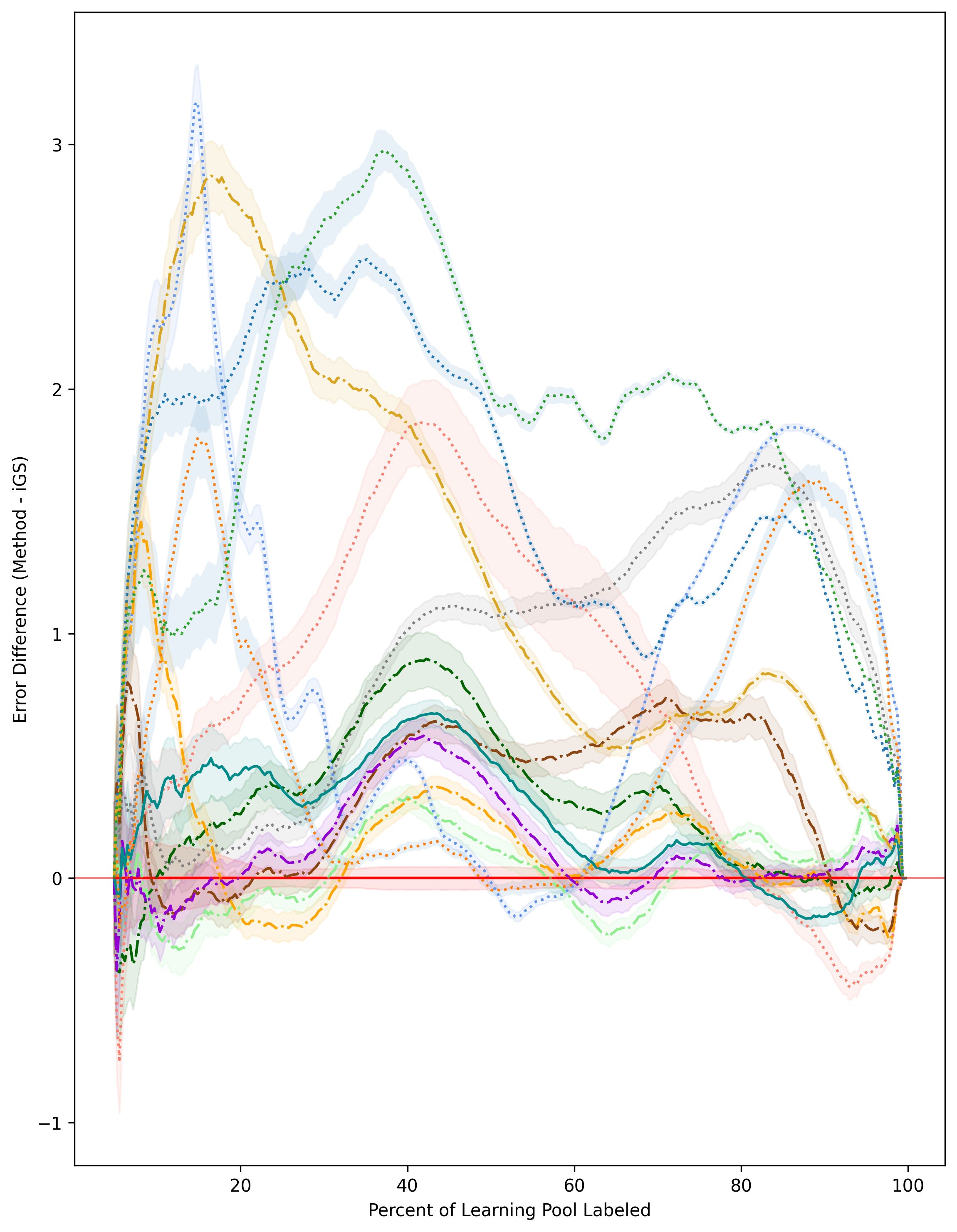}
        \caption{yacht}
    \end{subfigure}
    
    \vspace{0.5em}
    \centering
    \includegraphics[width=\linewidth]{upload_all_files/manuscript/benchmark_legend.jpg}
    
    \caption{Full-pool RMSE trace plots for benchmark datasets with robust normalization (Part 2 of 2).}
    \label{fig:MainResults_ROBUST2}
\end{figure}
\clearpage
\end{document}